\documentclass{article}
\usepackage{ijcai24}
\usepackage{times}
\usepackage{soul}
\usepackage{url}
\usepackage[hidelinks]{hyperref}
\usepackage[utf8]{inputenc}
\usepackage[small]{caption}
\usepackage{graphicx}
\usepackage{amsmath}
\usepackage{amsthm}
\usepackage{booktabs}
\usepackage{algorithm}
\usepackage{algorithmic}
\usepackage[switch]{lineno}

\usepackage{nicefrac}
\usepackage{comment}

\usepackage[all]{xy}
\usepackage{array}
\usepackage{listings}

\usepackage{multirow}
\usepackage{longtable}
\usepackage{color, colortbl}
\usepackage{lscape}
\usepackage{textgreek}

\usepackage{caption}
\usepackage{enumitem}
\usepackage{verbatim}

\usepackage{scalerel,amssymb,amsmath}
\usepackage[T1]{fontenc}
\usepackage{float}
\usepackage{systeme}
\usepackage{amssymb}
\usepackage[table,pdftex,dvipsnames]{xcolor}
\usepackage[most]{tcolorbox}
\usepackage{nicematrix}
\usepackage{accents}
\usepackage{bm}
\usepackage{mathrsfs}
\usepackage{framed}
\usepackage{setspace}
\usepackage{cleveref}
\usepackage{calc}
\usepackage{longtable}
\usepackage{gensymb}
\usepackage{breakcites}
\usepackage{subcaption}
\usepackage{hhline}
\usepackage{thmtools}
\usepackage{thm-restate}

\usepackage{epsdice}
\usepackage{mathtools}
\usepackage{amsfonts}
\usepackage{stmaryrd}
\usepackage[export]{adjustbox}

\DeclareFontFamily{U}{mathx}{}
\DeclareFontShape{U}{mathx}{m}{n}{<-> mathx10}{}
\DeclareSymbolFont{mathx}{U}{mathx}{m}{n}
\DeclareMathAccent{\widehat}{0}{mathx}{"70}
\DeclareMathAccent{\widecheck}{0}{mathx}{"71}

\usepackage{xargs}
\usepackage[pdftex,dvipsnames]{xcolor}
\usepackage[colorinlistoftodos,prependcaption,textsize=tiny]{todonotes}

\usepackage{tikz}
\usetikzlibrary{arrows,decorations.pathmorphing,positioning,fit,trees,shapes,shadows,automata,calc}
\usetikzlibrary{patterns,arrows,arrows.meta,calc,shapes,shadows,decorations.pathmorphing,decorations.pathreplacing,automata,shapes.multipart,positioning,shapes.geometric,fit,circuits,trees,shapes.gates.logic.US,fit,backgrounds}
\usepackage{pgfplots}
\usetikzlibrary{patterns,arrows,arrows.meta,calc,shapes,shadows,decorations.pathmorphing,decorations.pathreplacing,automata,shapes.multipart,positioning,shapes.geometric,fit,circuits,trees,shapes.gates.logic.US,fit,decorations.text}

\newtheorem{definition}{Definition}
\newtheorem{theorem}{Theorem}%
\newtheorem{lemma}{Lemma}
\newtheorem{proposition}{Proposition}
\newtheorem{corollary}{Corollary}%
\theoremstyle{remark}
\newtheorem{remark}{Remark}
\newtheorem{example}{Example}

\newcommandx{\unsure}[2][1=]{\todo[linecolor=red,backgroundcolor=red!25,bordercolor=red,#1]{#2}}
\newcommandx{\change}[2][1=]{\todo[linecolor=blue,backgroundcolor=blue!25,bordercolor=blue,#1]{#2}}
\newcommandx{\info}[2][1=]{\todo[linecolor=OliveGreen,backgroundcolor=OliveGreen!25,bordercolor=OliveGreen,#1]{#2}}
\newcommandx{\improvement}[2][1=]{\todo[linecolor=Plum,backgroundcolor=Plum!25,bordercolor=Plum,#1]{#2}}

\newcommand{\NN}{\ensuremath{\mathbb{N}}}

\newcommand{\RR}{\ensuremath{\mathbb{R}}}

\newcommand{\EE}{\ensuremath{\mathbb{E}}}

\newcommand{\cP}{\mathcal{P}}

\newcommand{\given}{\mid}

\newcommand{\ie}{\emph{i.e.}}
\newcommand{\eg}{\emph{e.g.}}

\newcommand{\dist}[1]{\ensuremath{\Delta(#1)}}
\newcommand{\pr}[2]{\ensuremath{\textrm{Pr}(\,#1\mid#2\,)}}

\newcommand{\agent}{{\mathfrak{a}}}
\newcommand{\nature}{{\mathfrak{n}}}
\newcommand{\publ}{\circ}
\newcommand{\priv}{\bullet}

\newcommand{\Zagent}{Z_\priv^\agent}
\newcommand{\Znature}{Z_\priv^\nature}
\newcommand{\Zpub}{Z_\publ}

\newcommand{\tinystate}[1]{\tikz\draw[black,#1] (0,0) circle (.7ex);}
\newcommand{\tsW}{\tinystate{fill=white}}
\newcommand{\tsLG}{\tinystate{fill=black!10}}
\newcommand{\tsDG}{\tinystate{fill=black!40}}
\newcommand{\tsDash}{\tinystate{fill=black!5, thin, dashed}}
\newcommand{\tsDot}{\tinystate{fill=black!20, thick, dotted}}

\newcommand{\microstate}[1]{\tikz\draw[black,#1] (0,0) circle (.4ex);}
\newcommand{\msW}{\microstate{fill=white}}
\newcommand{\msLG}{\microstate{fill=black!10}}
\newcommand{\msDG}{\microstate{fill=black!40}}

\newcommand{\msDot}{\microstate{fill=black!20, thick, dotted}}

\newcommand{\pto}{\hookrightarrow}
\newcommand{\pset}{{\scalebox{0.5}[1]{$\hookrightarrow$}}}
\newcommand{\suff}{\ltimes}
\newcommand{\aoh}{history}

\newcommand{\aohs}{histories}

\newcommand{\ocs}[1]{\ensuremath{\sigma_{\{\pi,\theta\}_{#1}}}}
\newcommand{\Ocs}[1]{\ensuremath{\text{OS}_{\{\pi,\theta\}_{#1}}}}
\newcommand{\Paths}{\text{Paths}}
\newcommand{\mypath}{\tau}
\newcommand{\agrees}{\mathcal{P}}
\newcommand{\concat}{\oplus}
\newcommand{\bigconcat}{\bigoplus}

\newcommand{\sticky}{\mathsf{stick}}
\newcommand{\upd}{\mathsf{upd}}

\newcommand{\player}{\mathsf{play}}
\newcommand{\fixed}{\mathsf{fix}}
\newcommand{\Ustick}{U^\text{stick}}
\newcommand{\last}{last}

\DeclarePairedDelimiter\tup{\langle}{\rangle}

\definecolor{myBlue}{RGB}{25, 131, 238}
\definecolor{myGreen}{RGB}{35, 158, 112}
\definecolor{myRed}{RGB}{235, 61, 42}
\newcommand{\changeB}[1]{{\color{myBlue}#1}}

\newcommand{\changeR}[1]{{\color{myRed}#1}}

\newcommand{\detr}{{det}}
\newcommand{\mix}{{mix}}
\newcommand{\rel}{\mathsf{rel}}
\newcommand{\hsim}{{\sim_{\rel^\nature(\theta)}}}

\newenvironment{reusefigure}[2][htbp]
    {\addtocounter{figure}{-1}%
    \renewcommand{\thefigure}{\ref{#2}}%
    \renewcommand{\addcontentsline}[3]{}%
    \begin{figure}[#1]}
    {\end{figure}}

\title{Imprecise Probabilities Meet Partial Observability: \\
Game Semantics for Robust POMDPs}

\author{
Eline M. Bovy$^1$
\and
Marnix Suilen$^1$\and
Sebastian Junges$^1$\And
Nils Jansen$^{1,2}$
\affiliations
$^1$Radboud University, The Netherlands\\
$^2$Ruhr-University Bochum, Germany
\emails
\{eline.bovy, marnix.suilen, sebastian.junges\}@ru.nl,
n.jansen@rub.de
}

\begin{document}

\maketitle

\begin{abstract}
    Partially observable Markov decision processes (POMDPs)
    rely on the key assumption that probability distributions are precisely known.
    Robust POMDPs (RPOMDPs) alleviate this concern by defining imprecise probabilities, referred to as uncertainty sets.
    While robust MDPs have been studied extensively, work on RPOMDPs is limited and primarily focuses on algorithmic solution methods. 
    We expand the theoretical understanding of RPOMDPs by showing that 1)~different assumptions on the uncertainty sets affect optimal policies and values; 
    2)~RPOMDPs have a  partially observable stochastic game (POSG) semantic; 
    and 3)~the same RPOMDP with different assumptions leads to semantically different POSGs and, thus, different policies and values.
    These novel semantics for RPOMDPs give access to results for POSGs, studied in game theory; concretely, we show the existence of a Nash equilibrium. 
    Finally, we classify the existing RPOMDP literature using our semantics, clarifying under which uncertainty assumptions these existing works operate.
\end{abstract}

\section{Introduction}

Partially observable Markov decision processes (POMDPs) are the standard model for decision-making under stochastic uncertainty and incomplete state information~\cite{DBLP:journals/ai/KaelblingLC98}.
A common objective in a POMDP is for an agent to compute a policy that maximizes the expected discounted reward.
While POMDPs have been studied extensively, a key assumption planning methods for POMDPs rely on is that the model dynamics, \ie, the transition and observation probabilities, are precisely known.
Under that assumption, it is known that an optimal policy of a POMDP is the solution to a fully observable infinite-state \emph{belief} MDP~\cite{DBLP:journals/ai/KaelblingLC98}.

In the fully observable setting, Markov decision processes (MDPs)~\cite{DBLP:books/wi/Puterman94} have been extended to \emph{robust} MDPs (RMDPs) to account for an additional layer of uncertainty around the probabilities that govern the model dynamics known as the uncertainty set.
These RMDPs have been studied extensively, in terms of their semantics~\cite{DBLP:journals/mor/Iyengar05,DBLP:journals/ior/NilimG05,DBLP:journals/mor/WiesemannKR13}, efficient algorithms to solve specific classes of RMDPs~\cite{DBLP:conf/nips/BehzadianPH21,DBLP:journals/jmlr/HoPW21,DBLP:conf/icml/WangHP23}, and their application in reinforcement learning~\cite{DBLP:journals/jmlr/JakschOA10,DBLP:conf/nips/PetrikS14,DBLP:conf/nips/SuilenS0022,DBLP:journals/make/MoosHASCP22}.

Robust MDPs can be seen as games between the \emph{agent}, who aims to maximize their reward by choosing an action at each state, and \emph{nature}, who aims to minimize the agent's reward by selecting adversarial probability distributions from the uncertainty set.
As a consequence, RMDPs and zero-sum stochastic games (SG)~\cite{shapley1953stochastic,gillette1957stochastic} are closely related, see in particular~\cite[Section~5]{DBLP:journals/mor/Iyengar05} for a reduction from (finite horizon) RMDP to SG.

For RMDPs, two semantics exist for nature's behavior when encountering the same state and action twice.
\emph{Static} uncertainty semantics require nature to always select the same probability distribution, while \emph{dynamic} uncertainty semantics allow nature to make a new choice every time a state-action pair is encountered.
\cite[Lemma 3.3]{DBLP:journals/mor/Iyengar05} established that for finite horizon and discounted infinite horizon reward maximization in certain RMDPs, static and dynamic uncertainty semantics coincide, meaning that for a given agent's policy, both semantics result in precisely the same value.

Extensions to robust POMDPs (RPOMDPs) exist~\cite{DBLP:conf/icml/Osogami15, DBLP:conf/cdc/ChamieM18, DBLP:journals/jet/Saghafian18,DBLP:conf/ijcai/Suilen0CT20, DBLP:journals/siamjo/NakaoJS21,DBLP:conf/aaai/Cubuktepe0JMST21,bovy2023thesis}, but primarily focus on algorithmic approaches to compute optimal policies.
Notably, these algorithms compute optimal policies under different implicit assumptions on the semantics of RPOMDPs, particularly concerning static and dynamic uncertainty.

\paragraph{Contributions.}
This paper sets out to clarify and expand the theoretical understanding of RPOMDPs.
Specifically, we define semantics with associated value functions and policies for RPOMDPs under various assumptions on the uncertainty.
We explicitly define the semantics of RPOMDPs via zero-sum two-sided partially observable stochastic games (POSGs)~\cite{Springer:HSVI}.
Our key contributions are: 
\begin{enumerate}
    \item \textbf{Uncertainty assumptions matter.}
    We introduce a continuum of uncertainty assumptions for RPOMDPs called \emph{stickiness}.
    Stickiness determines when nature's choices for resolving the uncertainty become fixed.
    The two extremes, immediately and never, coincide with the static and dynamic uncertainty semantics of RMDPs.
    We show in \Cref{thm:uncertainty:matters} that, in contrast to RMDPs, these two extremes no longer coincide for RPOMDPs. 
    Specifically, they may lead to different optimal values.
    Moreover, the \emph{order of play} (whether the agent or nature makes the first move) matters. 
    We show that the differences in these assumptions can lead to significant differences in optimal values.
    We account for these results by providing a new RPOMDP definition that explicitly accounts for these uncertainty assumptions in \Cref{def:rpomdp}. 
    \item \textbf{Robust POMDPs are POSGs.} 
    We provide a formal POSG semantic for RPOMDPs with explicit stickiness and order of play.
    We establish a direct correspondence between policies of POSGs and RPOMDPs that ensure equal values for both models (\Cref{thm:equivalent:values}).
    Moreover, different uncertainty assumptions in the RPOMDP lead to semantically different POSGs and hence explain the result listed in Contribution 1.
    Finally, we use the POSG semantics to prove the existence of Nash equilibria, which we use in turn to prove the existence of optimal values in finite horizon RPOMDPs (\Cref{thm:exist:nash}).
    \item \textbf{Classification of existing RPOMDP works.}
    We provide a classification of existing RPOMDP literature into our semantic framework (\Cref{sec:related:work}).
\end{enumerate}
The extended version of this paper, with all the appendices, can be found at \cite{bovy2024imprecise}.

\section{Preliminaries}\label{sec:preliminaries}
A discrete probability distribution over a finite set $X$ is a function $\mu \colon X \to [0,1]$ such that $\sum_{x \in X} \mu(x) = 1$.
For infinite sets, we only consider finite probability distributions.
That is, for an infinite set $X$, a finite probability distribution over $X$ is a function $\lambda \colon X \to [0,1]$ with finitely many $x\in X.\, \lambda(x) \neq 0$ and $\sum_{x \in X} \lambda(x) = 1$.
The set of all probability distributions over $X$ is denoted as $\dist{X}$, and $\mathcal{P}(X)$ is the powerset of $X$.
By $(X \to Y)$, we denote the set of all functions $f \colon X \to Y$, and $f \colon X \pto Y$ for a partial function.
The symbol $\perp$ is used for undefined.
Finally, we use Currying to describe functions that map to functions, \eg, $f \colon X \to (Y \to Z)$ represents a function that maps each $x \in X$ to a function $g_x \colon Y \to Z$.

\subsection{Markov Models}\label{subsec:markov:models}

\begin{definition}[POMDP]
    A partially observable Markov decision process (POMDP) is a tuple $\tup{S,A,T,R,Z,O}$ where $S$, $A$, $Z$ are finite sets of states, actions, and observations, respectively.
    $T \colon S  \times A \to \dist{S}$, $R \colon S  \times A \to \RR$, and $O \colon S\to Z$ are the transition, reward, and observation functions, respectively.
\end{definition}
\noindent
This definition uses POMDPs with \emph{deterministic observations}, in contrast to the more standard stochastic observation functions~\cite{DBLP:journals/ai/KaelblingLC98}.
However, every POMDP with stochastic observations can be transformed into such a POMDP~\cite{DBLP:journals/ai/ChatterjeeCGK16}. 
For convenience, we sometimes write $T(s,a,s')$ for $T(s,a)(s')$.

A \emph{Markov decision process} (MDP) is a POMDP where all states are fully observable. 
We simplify the tuple definition to $\tup{S,A,T,R}$ in the MDP case.

\paragraph{Paths and histories. }
A \emph{path} in a (PO)MDP $M$ is a sequence of successive states and actions: $\mypath = \tup{s_0,a_0,\dots,s_n} \in (S \times A)^* \times S$ such that $T(s_i,a_i,s_{i+1}) > 0$ for all $i \geq 0$.
We denote the set of paths in $M$ by $\Paths^M$.
The concatenation of two paths is written as $\mypath \concat \mypath'$.
A history in a POMDP is a sequence of observations and actions observed from a path $\tup{s_0,a_0,\dots}$: $h \in (Z \times A)^* \times Z$ such that $h = \tup{O(s_0),a_0,O(s_1),a_1\dots}$.

\paragraph{Policies. }
A history-based \emph{stochastic} policy\footnote{Also known as a behavioral strategy.} is a function that maps histories to distributions over actions, that is, $\pi \colon (Z \times A)^*\times Z \to \dist{A}$.
The policy $\pi$ is \emph{deterministic}, or pure, if it only maps to single actions, and stationary if its domain is $Z$, \ie, it only maps the current observation.
The set of all history-based stochastic policies is denoted by $\Pi$
and the set of all history-based deterministic policies by $\Pi^\detr$.
A history-based \emph{mixed} policy is a probability distribution over the set of history-based deterministic policies, that is, $\pi^{\mix} \in \dist{\Pi^\detr}$
The set of all history-based mixed policies is denoted by $\Pi^\mix$.
Throughout the rest of the text, unless otherwise mentioned, all policies are history-based
, and unless indicated by either $\detr$ or $\mix$, the (sets of) policies are stochastic.

\paragraph{Values.}
We maximize the expected reward, either with a finite horizon $K \in \NN$ (denoted fh) or in the infinite horizon with a discount factor $\gamma \in (0,1)$ (denoted dih).
We denote these \emph{objectives} by $\phi \in \{\text{fh},\text{dih}\}$.
The value of a policy $\pi \in \Pi$ in a (PO)MDP for the objective $\phi$ is given by the value function $V_\phi^\pi$, and the optimal value is $V_\phi^*$.
The value of a policy for either objective is~\cite{DBLP:books/sp/12/Spaan12}:%
\begin{align*}
    V_\text{fh}^\pi = \EE \left[ \sum_{t=0}^{K-1} r_t \given \pi \right], \quad
    V_\text{dih}^\pi = \EE \left[ \sum_{t=0}^\infty \gamma^t r_t \given \pi \right],
\end{align*}
where $r_t$ is the reward collected at time $t$ under policy $\pi$.
The optimal value $V_\phi^*$ is defined as $\sup_{\pi \in \Pi} V_\phi^\pi$.

\subsection{Robust MDPs}\label{subsec:RMDP}
Robust MDPs extend standard MDPs by defining an uncertainty set of probability distributions that a state-action pair can map to instead of a single fixed and known distribution. 
Let $U$ be a finite set of variables, and define $\bm{U} \subseteq (U \to \RR)$ as the uncertainty set.
Let $\bm{U}$ be non-empty, a robust MDP is then defined as follows.

\begin{definition}[RMDP]\label{def:RMDP}
    A robust MDP (RMDP) is a tuple $\tup{S,A,\bm{T},R}$ where $S,A$, and $R$ are again states, actions, and the reward function.
    $\bm{T} \colon \bm{U} \to (S \times A \to \dist{S})$ is the uncertain transition function, consisting of a possibly infinite set of transition functions $T \colon S \times A \to \dist{S}$, where every $T \in \bm{T}$ is determined by a variable assignment $(U \to \RR) \in \bm{U}$.
\end{definition}

\begin{remark}
    The variable assignment $\bm{U}$ maps the variables to $\RR$ and not to $[0,1]$ as mapping to the reals gives more freedom in defining the uncertainty set, allowing for more complicated dependencies between transitions.
    The uncertain transition function $\bm{T}$ ensures that all state-action pairs are mapped to probability distributions.
\end{remark}

\paragraph{Game interpretation.}
As already mentioned in the introduction, we interpret RMDPs as games between the agent, who selects actions through a policy $\pi: (S\times A)^* \times S \to \dist{A}$, and nature, who uses its policy $\theta: (S \times A \times \bm{U})^* \times S \to \dist{\bm{U}}$ to select variable assignments $u \in \bm{U}$ from the uncertainty set to determine the probability distributions, such that $\bm{T}$ is non-empty.
That is, any variable selection $u$ must yield a valid probability distribution for all state-action pairs:
\[
\forall s \in S, a \in A. \, \bm{T}(u)(s,a) \in \dist{S}.
\]
The sets of the agent's and nature's policies are again $\Pi$ and $\Theta$, respectively.
The sets of deterministic and mixed policies are constructed analogously as for POMDPs. 

The maximal value that a policy can achieve over all possible ways to resolve the uncertainty is defined for both objectives, respectively, as
\begin{equation*}\label{eq:RMDP:supinf}
V_\text{fh}^* = \sup_{\pi \in \Pi} \inf_{\theta \in \Theta} \EE \left[ \sum_{t=0}^{K-1} r_t \right], \,\, V_\text{dih}^* = \sup_{\pi \in \Pi} \inf_{\theta \in \Theta} \EE \left[ \sum_{t=0}^\infty \gamma^t r_t \right].
\end{equation*}
It is often assumed that nature plays stationary and deterministic in RMDPs.
Under certain conditions on the uncertainty set, this assumption is non-restrictive as nature's best policy falls within this class~\cite{DBLP:journals/mor/Iyengar05,DBLP:journals/mor/WiesemannKR13,DBLP:journals/corr/abs-2312-03618}.

\begin{remark}
Our definition of RMDPs is more general than common definitions: 
Most RMDP definitions assume a form of independence in the uncertainty set between different states (or actions),  known as $s$- (or $(s,a)$-) \emph{rectangularity}.
Our definition subsumes these rectangular RMDPs. 
While rectangular RMDPs satisfy a saddle point condition, meaning the $\sup\inf$ may be reversed in the definition of $V_\phi^*$, this has not been shown for RMDPs in general.
Our result in \Cref{thm:exist:nash} shows that the saddle point condition holds for RPOMDPs in general for finite horizon.
This extends to RMDPs using a fully observable observation function.
We refer to~\cite{DBLP:journals/mor/WiesemannKR13} for a more standard definition of rectangularity and an overview of the computational properties of rectangular RMDPs, and~\cite{DBLP:conf/birthday/0001JK22} for an overview on non-rectangular RMDPs.
\end{remark}

\begin{figure}[t]
    \centering
    \resizebox{\columnwidth}{!}{
    \begin{tikzpicture}[state/.append style={shape = ellipse}, >=stealth,
    bobbel/.style={minimum size=1mm,inner sep=0pt,fill=black,circle},
    mynode/.style={rectangle,fill=white,anchor=center}]]
    \node[state] (s1) at (1,0) {$s_1$};
    \node[state] (s2) at (5,0) {$s_2$};
    \node[bobbel] (s1b) at (3,0.5) {};
    \node[bobbel] (s2b) at (3,-0.5) {};
    \draw[<-] (s1.west) -- +(-0.6,0);
    \draw (s1) -- (s1b);
    \draw (s2) -- (s2b);
    \draw (s1b) edge[->, bend left = 10] node[above right]{$p$} (s2);
    \draw (s1b) edge[->, bend right = 30] node[above left]{$1-p$} (s1.north);
    \draw (s2b) edge[->, bend left = 10] node[below right]{$q$} (s1);
    \draw (s2b) edge[->, bend right = 30] node[below right]{$1-q$} (s2.south);
    \node[anchor = west, text width=5cm] at (5.8,0.25) {\small $\bm{U}^1 = \left\{{p} \in [0.1,0.9], {q} \in [0.1,0.9]\right\}$};
    \node[anchor = west, text width=4cm] at (5.8,-0.25) {\small $\bm{U}^2 = \left\{{p} \in [0.1,0.4], {q = 2p}\right\}$};
\end{tikzpicture}
    }
    \caption{An example RMDP with two uncertainty sets.}
    \label{fig:small_example}
\end{figure}
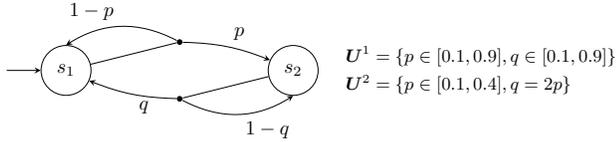

\begin{example}\label{exam:RMDP:uncertainty:set}
    \Cref{fig:small_example} depicts a small RMDP together with two possible uncertainty sets $\bm{U}^1$ and $\bm{U}^2$.
    In this RMDP, the agent only has singleton choices, while nature chooses variable assignments for $p$ and $q$.
    Given an uncertainty set and a variable assignment in that uncertainty set, for example, $u = \{p \mapsto 0.3, q \mapsto 0.5\} \in \bm{U}^1$, we get a fully defined transition function.
    $\bm{U}^1$ is an $(s,a)$-rectangular uncertainty set since each variable influences the transition probabilities in only one state-action pair.
    $\bm{U}^1$ could hence be split into two independent uncertainty sets: $\bm{U}^1 = \{p \in [0.1,0.9]\} \times \{q\in[0.1,0.9]\}$.
    In contrast, $\bm{U}^2$ is not $(s,a)$-rectangular, since the value of $q$ depends on $p$, so $p$ influences transitions from state $s_1$ as well as from state $s_2$.
\end{example}

\paragraph{Static and dynamic uncertainty.} 
A prominent semantic concern on RMDPs is whether nature must play consistently when a state is repeatedly visited.
\emph{Static uncertainty} semantics require nature to choose a single variable assignment $u\in \bm{U}$ once-and-for-all, fixing all probability distributions form the start.
On the other hand, \emph{dynamic uncertainty} semantics allow nature to choose a new variable assignment independently each time a state is visited.
In~\cite[Lemma 3.3]{DBLP:journals/mor/Iyengar05}, it is shown that on $(s,a)$-rectangular RMDPs with a finite horizon or discounted infinite horizon objective, these semantics, and thus the values, coincide.

\begin{remark}
    Although our use of variables in the transition function is similar to, \eg,~\cite{DBLP:journals/mor/WiesemannKR13}, it is not standard.
    Often, the transition function directly maps to uncertainty sets, \eg,~\cite{DBLP:journals/mor/Iyengar05, DBLP:journals/ior/NilimG05, DBLP:conf/icml/HoPW18}.
    The use of variables has the following benefits over directly mapping to uncertainty sets: (1)~support for various semantics, such as different forms of rectangularity, without changing the signature of the uncertain transition function~$\bm{T}$;
    (2)~it allows us to keep track of partial restriction on nature's choice, which is needed when moving to the partially observable setting (\Cref{subsec:stickiness}).
\end{remark}

\section{RPOMDPs and Uncertainty Assumptions}\label{sec:RPOMDP}
In this section, we define a game-based framework for robust POMDP semantics that can be instantiated by making different \emph{uncertainty assumptions}.
Specifically, we incorporate two key assumptions into our RPOMDP definition: \emph{stickiness} and \emph{order of play}.
Stickiness concerns the moment at which nature must choose the values of the variables $U$ and extends static and dynamic uncertainty from RMDPs to the partially observable setting.
The order of play specifies whether the agent or nature moves first. 
It determines the moment nature observes the most recent agent action.

This section is structured as follows.
We briefly discuss our assumptions about partial observability to introduce notation needed and then
formally define RPOMDPs.
Next, we clarify how notions such as paths and histories carry over from POMDPs and RMDPs to RPOMDPs.
We briefly describe the order-of-play assumption and provide a more elaborate discussion of stickiness in \Cref{subsec:stickiness}.
Finally, in \Cref{subsec:valueOfRPOMDP}, we discuss the optimal value of RPOMDPs under different uncertainty assumptions and demonstrate that these assumptions matter, \ie, yield different optimal values (\Cref{thm:uncertainty:matters}).

\paragraph{RPOMDPs.}
Analogous to RMDPs, we interpret RPOMDPs as a game between the agent and nature. 
To make our RPOMDP definition as general as possible, we assume partial observability for both the agent and nature.
We factorize the observations into three parts: \emph{private} observations of agent and nature, respectively, and \emph{public} observations that both players observe.
Hence, each player obtains two observations in each state.
For the remainder of the paper, we use $\agent$ and $\nature$ to denote whether a set or function belongs to the agent or to nature, respectively.
Likewise, we use $\priv$ and $\publ$ to denote whether a set or function relates to private or public observations.

\begin{definition}[RPOMDP]\label{def:rpomdp}
A robust POMDP (RPOMDP) is a tuple 
$\tup{S, A, \bm{T}, R, \Zagent, \Znature, \Zpub, O^\agent_\priv, O^\nature_\priv, O_\publ,\sticky,\player}$, 
where $S, A, \bm{T}$, and $R$ are sets of states and actions, the uncertain transition function, and the reward function, as in RMDPs. 
The sets $\Zagent, \Znature$, and $\Zpub$ are the private observations for the agent, for nature, and the public observations, respectively.
$O^\agent_\priv \colon S \to \Zagent, O^\nature_\priv \colon S \to \Znature$, and $O_\publ \colon S \to \Zpub$ are the observation functions belonging to the agent, nature, and public observations.
$\sticky\colon U \times \Znature \times \Zpub \times A \to \{0,1\}$ is the \emph{stickiness function}, and $\player \in \{\agent, \nature\}$ the \emph{order of play}, \ie, which player moves first.
\end{definition}
\noindent
As for POMDPs, we consider deterministic observations.
We show in
\Cref{app:det.obs}
that RPOMDPs with stochastic or uncertain observations can be rewritten in RPOMDPs with deterministic observations.

\paragraph{Paths and histories.}
A path through an RPOMDP $M$ is a sequence $\mypath=\tup{s_0,a_0,u_0,s_1, \dots, s_n} \in (S \times A \times \bm{U})^* \times S$ that consists of environment states, agent actions, and nature's variable assignments $u \in \bm{U}$, such that for all $i > 0$:
\begin{align*}
\bm{T}(u_{i-1})(s_{i-1},a_{i-1},s_i) >0.
\end{align*}
As before, we denote the set of paths in $M$ by $\Paths^M$.
A history is the observable fragment of a path for either the agent or nature.
The agent's histories are sequences in $H^{\agent,M} \subseteq ( \Zagent \times \Zpub \times A)^* \times \Zagent \times \Zpub$, observing the agent's private and public observations of the states and its own actions.
Nature's histories are sequences in $H^{\nature,M} \subseteq (\Znature \times \Zpub \times A \times \bm{U} )^* \times \Znature \times \Zpub$, observing its private and public observations of the states, the agent's actions, and variable assignments $u \in \bm{U}$ that resolve the uncertainty. 
The histories for the agent and nature are obtained from a path by applying the relevant observation functions, respectively, similar to POMDPs. 
We give an explicit mapping in 
\Cref{app:appendix_prelim}.

\paragraph{Order of play.}
For any given path, both the agent and nature must make a move. 
We consider turn-based games and must, therefore, select who picks their move first\footnote{In our setting, the case that both players pick their actions simultaneously is equivalent to letting nature move first, as we assume the agent never directly observes the selection of nature.
See \cite{DBLP:conf/mfcs/KwiatkowskaNPSY22} for more information about simultaneous stochastic games.
}. 
We encode this information directly in the signature of the nature policy below.  
We remark that after both players have made their move, the resulting state is equivalent as we assume that nature always observes the actions picked previously. %

\paragraph{Policies.}
As with RMDPs, we denote the agent's policies by $\pi \in \Pi$ and nature's by $\theta \in \Theta$.
Specifically, the agent's policies are defined as maps from the agent's histories to distributions over actions $\pi \colon H^{\agent,M} \to \dist{A}$.
Nature's policies are maps from nature's histories and the last agent action to \emph{finite} distributions over variable assignments $\theta \colon H^{\nature,M} \times A \to \dist{\bm{U}}$.
When nature moves first, the last agent action is not available and therefore not part of nature's policy: $\theta \colon H^{\nature,M} \to \dist{\bm{U}}$.
The sets of deterministic and mixed policies are constructed analogously as for POMDPs in \Cref{sec:preliminaries}.

\subsection{Stickiness: Restricting Nature's Choices}\label{subsec:stickiness}
Stickiness describes whether nature's choice at one point should remain fixed (`stick') in the future\footnote{The name follows from the idea that nature always chooses values for all variables, but some values stick for the rest of time. 
Whether a variable sticks is determined by the stickiness.}.
The simplest instances of stickiness are when nature's choices never stick or when they all stick from the start.
If nature's choices never stick, so values never stick to variables, we say the RPOMDP has \emph{zero} stickiness.
If nature's choices stick from the start, so values directly stick to all variables, we say the RPOMDP has \emph{full} stickiness.
Zero and full stickiness correspond to dynamic and static uncertainty in RMDPs, respectively.

Zero and full stickiness are only the two extremes of a spectrum of different stickiness types.
In addition, RPOMDPs admit partial types of stickiness, where nature may have to fix variable values but can delay some choices depending on the specific stickiness function.
We now give an intuitive example on stickiness before moving to the formal definition.
For explicit examples of stickiness, including so-called \emph{observation-based} stickiness, see 
\Cref{app:observation:stickiness}.

\begin{example}[Stickiness]\label{exam:stickiness_concrete}
    Consider the following drone delivery problem, naturally modeled as (R)POMDP.
    The agent controls a drone that has to deliver packages.
    States encode the drone's location, actions are direction and speed adjustments, and observations are location estimations.
    The transition probabilities represent the chance of reaching adjacent locations.
    Different types of stickiness can model different sources of uncertainty on those probabilities:

    \textbf{Full stickiness. }
    The drone experiences an unknown drift probability caused by, \eg, a dented blade.
    The agent must account for this unknown but \emph{fixed probability}.

    \textbf{Zero stickiness. }
    Wind influences the probability of reaching adjacent states.
    While predictable to a certain degree, a margin of uncertainty will remain.
    As the wind changes over time, the agent has to account for \emph{changing probabilities}.
    
    \textbf{Partial stickiness. }
    We need partial stickiness when nature eventually has to commit to a probability, but not from the start.
    Suppose we extend our problem %
    with a municipality that has created no-fly zones and will install monitors in these zones to detect violations.
    We encode the no-fly zones in the state space to reason about the probability of the agent being detected.
    Initially, the municipality will try out possible placements for their monitors.
    The probability of being detected, hence, lies in an uncertainty set formed by the different placements of monitors.
    Once the placement of the monitors is final, the probability of getting caught in a no-fly zone becomes fixed.
    A partial stickiness function that returns $1$ when observing a drone in a no-fly zone, fixing the number of monitors at that point, captures such scenarios.
\end{example}

We allow partial stickiness to depend on what nature observes, \ie, its private observations $\Znature$, public observations $\Zpub$, and the agent's actions $A$.

\begin{definition}
    The \emph{stickiness} of an RPOMDP is a Boolean function indicating whether nature's choice of a value for variable $v \in U$ should remain fixed:
    \[
    \sticky\colon U \times \Znature \times \Zpub \times A \to \{0,1\}.
    \]
\end{definition}
Below, we describe how we use the $\sticky$ function to compute restrictions on nature's choices and, with that, define valid nature policies.

\paragraph{Fixed variables and agreeing assignments.}
Depending on the stickiness of the RPOMDP, past choices of nature may restrict its future choices.
Let $\bm{U}^{\pset}$ denote the set of partial variable assignments $U \pto \RR$.
Let $u^\bot \in \bm{U}^\pset$ be the totally undefined variable assignment: $\forall v\in U. u^\bot(v) = \bot$.
We define a function $\fixed \colon \Paths^M \to \bm{U}^\pset$ such that $\fixed(\mypath)$ defines the partial variable assignment that remains fixed based on the stickiness function. 
This function is inductively defined as $\fixed(s_I) = \emptyset = u^\bot$ for the initial path $s_I$, and 
\begin{align*}
&\fixed(\mypath \concat \tup{a, u, s'})(v)  = \\
& \,
\begin{cases} 
 u(v) & \text{ if } \fixed(\mypath)(v) \text{ undefined, } v \in \Ustick(\last(\mypath), a), \\ 
\fixed(\mypath)(v) & \text{ otherwise,}
\end{cases}\\
& \text{using}\quad\Ustick(s,a) = \{ v \mid \sticky(v, O^\nature_\priv(s), O_\publ(s), a) = 1 \} 
\end{align*} 
to denote the variables that stick. 
We can straightforwardly lift the definition of $\fixed$ to nature's histories using 
\[ 
\Ustick_h(z^\nature_\priv, z_\publ,a) = \{ v \mid \sticky(v, z^\nature_\priv, z_\publ, a) = 1 \}. 
\] 
Two partial functions agree if they assign equal values to all defined inputs.
We use $\bm{U}^\agrees(u)$ for the variable assignments that agree with partial variable assignment $u$.

\paragraph{Valid paths, histories, and policies.}
Let $\mypath = \tup{s_0,a_0,u_0,s_1, \dots, s_n}\in\Paths^M$.
For $k< n$, we denote the prefix $\mypath_{0:k} = \tup{s_0,a_0,u_0,s_1, \dots, s_k}$.
A path is valid, if for every $k < n$, $u_{k} \in \bm{U}^\agrees(\fixed(\mypath_{0:k}))$. 
A history is valid if it corresponds to some valid path. 
A nature policy is valid if all variable assignments that nature randomizes over given a history and action are in the set of variable assignments that agree with the variable restrictions generated by the history. That is, $\forall h^\nature\in H^\nature, \forall a\in A, \forall u \in \bm{U}.\,$
\[\theta(h^\nature,a)(u) > 0 \implies u \in \bm{U}^\agrees(\fixed(h^\nature)).\]
From here on, all paths, histories, and policies are assumed to be valid.

\subsection{The Value of an RPOMDP}\label{subsec:valueOfRPOMDP}

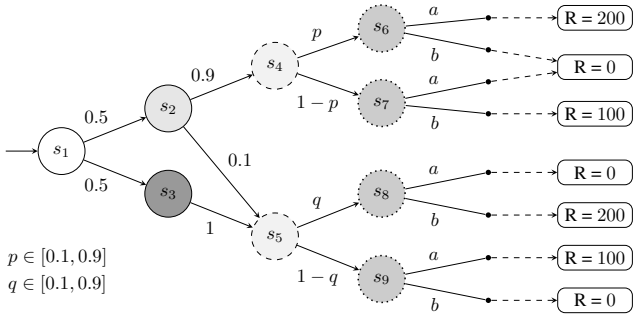
\begin{figure}[t]
    \centering
    \resizebox{\columnwidth}{!}{
    \begin{tikzpicture}[state/.append style={shape = circle}, >=stealth,
    bobbel/.style={minimum size=1mm,inner sep=0pt,fill=black,circle}]
    \node[state] (s1) at (1,0) {$s_1$};
    \node[state, fill=black!10] (s2) at ($(s1) + (2,0.8)$) {$s_2$};
    \node[state, fill=black!40] (s3) at ($(s1) + (2,-0.8)$) {$s_3$};
    \node[state, fill=black!5, thin, dashed] (s4) at ($(s2) + (2,0.8)$) {$s_4$};
    \node[state, fill=black!5, thin, dashed] (s5) at ($(s3) + (2,-0.8)$) {$s_5$};
    \node[state, fill=black!20, thick, dotted] (s6) at ($(s4) + (2,0.7)$) {$s_6$};
    \node[state, fill=black!20, thick, dotted] (s7) at ($(s4) + (2,-0.7)$) {$s_7$};
    \node[state, fill=black!20, thick, dotted] (s8) at ($(s5) + (2,0.8)$) {$s_8$};
    \node[state, fill=black!20, thick, dotted] (s9) at ($(s5) + (2,-0.8)$) {$s_9$};
    \node[bobbel] (s6ba) at ($(s6) + (2,0.2)$) {};
    \node[bobbel] (s6bb) at ($(s6) + (2,-0.4)$) {};
    \node[bobbel] (s7ba) at ($(s7) + (2,0.4)$) {};
    \node[bobbel] (s7bb) at ($(s7) + (2,-0.2)$) {};
    \node[bobbel] (s8ba) at ($(s8) + (2,0.4)$) {};
    \node[bobbel] (s8bb) at ($(s8) + (2,-0.4)$) {};
    \node[bobbel] (s9ba) at ($(s9) + (2,0.4)$) {};
    \node[bobbel] (s9bb) at ($(s9) + (2,-0.4)$) {};
    \node[rectangle, draw, minimum width=14mm, rounded corners] (rPos1) at ($(s6ba) + (2,0)$) {R = $200$};
    \node[rectangle, draw, minimum width=14mm, rounded corners] (rNeg1) at ($(s6bb) + (2,-0.325)$)  {R = $0$};
    \node[rectangle, draw, minimum width=14mm, rounded corners] (rNeu1) at ($(s7bb) + (2,0)$) {R = $100$};
    \node[rectangle, draw, minimum width=14mm, rounded corners] (rNeg2a) at ($(s8ba) + (2,0)$) {R = $0$};
    \node[rectangle, draw, minimum width=14mm, rounded corners] (rPos2) at ($(s8bb) + (2,0)$) {R = $200$};
    \node[rectangle, draw, minimum width=14mm, rounded corners] (rNeu2) at ($(s9ba) + (2,0)$) {R = $100$};
    \node[rectangle, draw, minimum width=14mm, rounded corners] (rNeg2b) at ($(s9bb) + (2,0)$) {R = $0$};
    \draw[<-] (s1.west) -- +(-0.6,0);
    \draw (s1) edge[->] node[above left]{$0.5$} (s2);
    \draw (s1) edge[->] node[below left]{$0.5$} (s3);
    \draw (s2) edge[->] node[above left]{$0.9$} (s4);
    \draw (s2) edge[->] node[above right]{$0.1$} (s5);
    \draw (s3) edge[->] node[below left]{$1$} (s5);
    \draw (s4) edge[->] node[above left]{$p$} (s6);
    \draw (s4) edge[->] node[below, xshift=-2mm, yshift=-1mm]{$1-p$} (s7);
    \draw (s5) edge[->] node[above left]{$q$} (s8);
    \draw (s5) edge[->] node[below, xshift=-2mm, yshift=-1mm]{$1-q$} (s9);
    \draw (s6) edge node[above left]{$a$} (s6ba);
    \draw (s6) edge node[below left]{$b$} (s6bb);
    \draw (s7) edge node[above left]{$a$} (s7ba);
    \draw (s7) edge node[below left]{$b$} (s7bb);
    \draw (s8) edge node[above left]{$a$} (s8ba);
    \draw (s8) edge node[below left]{$b$} (s8bb);
    \draw (s9) edge node[above left]{$a$} (s9ba);
    \draw (s9) edge node[below left]{$b$} (s9bb);
    \draw [dashed] (s6ba) edge[->] (rPos1);
    \draw [dashed] (s6bb) edge[->] (rNeg1);
    \draw [dashed] (s7ba) edge[->] (rNeg1);
    \draw [dashed] (s7bb) edge[->] (rNeu1);
    \draw [dashed] (s8ba) edge[->] (rNeg2a);
    \draw [dashed] (s8bb) edge[->] (rPos2);
    \draw [dashed] (s9ba) edge[->] (rNeu2);
    \draw [dashed] (s9bb) edge[->] (rNeg2b);
    \node[text width=4cm] at (2,-2) {$p \in [0.1,0.9]$};
    \node[text width=4cm] at (2,-2.5) {$q \in [0.1,0.9]$};
\end{tikzpicture}
    }
    \caption{An RPOMDP where full and zero stickiness do not coincide in their optimal value.}
    \label{fig:full_vs_zero_sticky_rPOMDP}
\end{figure}

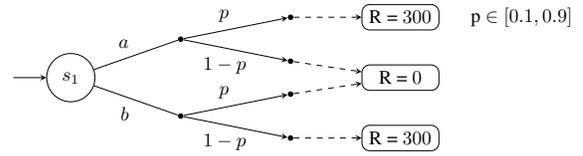
\begin{figure}[t]
    \centering
    \resizebox{0.9\columnwidth}{!}{
    \begin{tikzpicture}[state/.append style={shape = circle}, >=stealth,
    bobbel/.style={minimum size=1mm,inner sep=0pt,fill=black,circle}]
    \node[state] (s1) at (1,0) {$s_1$};
    \draw[<-] (s1.west) -- +(-0.6,0);
    \node[bobbel] (s1ab) at ($(s1) + (2,0.7)$) {};
    \node[bobbel] (s1bb) at ($(s1) + (2,-0.7)$) {};
    \node[bobbel] (s1apr) at ($(s1ab) + (2,0.4)$) {};
    \node[bobbel] (s1a-pr) at ($(s1ab) + (2,-0.4)$) {};
    \node[bobbel] (s1bpr) at ($(s1bb) + (2,0.4)$) {};
    \node[bobbel] (s1b-pr) at ($(s1bb) + (2,-0.4)$) {};
    \node[rectangle, draw, minimum width=14mm, rounded corners] (rPos1) at ($(s1apr) + (2,0)$) {R = $300$};
    \node[rectangle, draw, minimum width=14mm, rounded corners] (rNeg) at ($(s1a-pr) + (2,-0.3)$) {R = $0$};
    \node[rectangle, draw, minimum width=14mm, rounded corners] (rPos2) at ($(s1b-pr) + (2,0)$) {R = $300$};
    \draw (s1) edge[-] node[above left]{$a$} (s1ab);
    \draw (s1) edge[-] node[below left]{$b$} (s1bb);
    \draw (s1ab) edge[->] node[above left]{$p$} (s1apr);
    \draw (s1ab) edge[->] node[below, xshift=-2mm]{$1-p$} (s1a-pr);
    \draw (s1bb) edge[->] node[above left]{$p$} (s1bpr);
    \draw (s1bb) edge[->] node[below, xshift=-2mm]{$1-p$} (s1b-pr);
    \draw [dashed] (s1apr) edge[->] (rPos1);
    \draw [dashed] (s1a-pr) edge[->] (rNeg);
    \draw [dashed] (s1bpr) edge[->] (rNeg);
    \draw [dashed] (s1b-pr) edge[->] (rPos2);
    \node[] at ($(rPos1.east) + (1.5,0)$) {p $\in [0.1,0.9]$};
\end{tikzpicture}
    }
    \caption{An RPOMDP where nature first and agent first semantics do not coincide in their optimal value.}
    \label{fig:agent_vs_nature_first_RPOMDP_small}
\end{figure}

\paragraph{Values.}
The values of an RPOMDP given agent policy $\pi \in~\Pi$ and nature policy $\theta \in \Theta$ for both the finite horizon and discounted infinite horizon objective are
\[
V_\text{fh}^{\pi,\theta} = \EE \left[ \sum_{t=0}^{K-1} r_t \mid \pi, \theta \right], \quad V_\text{dih}^{\pi,\theta} = \EE \left[ \sum_{t=0}^\infty \gamma^t r_t \mid \pi, \theta \right].
\]
Optimal values are defined as $V_\phi^* = \sup_{\pi \in \Pi} \inf_{\theta \in \Theta} V_\phi^{\pi,\theta}$.
To the best of our knowledge, it is as of yet unknown whether such optimal values and their policies exist for every RPOMDP.
Various RPOMDP papers claim the existence of an optimal value for their specific RPOMDP, but these results do not extend to the general RPOMDPs we consider in this paper~\cite{DBLP:conf/icml/Osogami15,DBLP:journals/siamjo/NakaoJS21}.
We prove that the optimal value for finite horizon exists for general RPOMDPs in \Cref{thm:exist:nash}.

By changing the stickiness or order of play of an RPOMDP, the optimal value may change:
\begin{restatable}[Uncertainty assumptions matter]{theorem}{theoremI}\label{thm:uncertainty:matters}
For an RPOMDP $M$, let $V_\text{fh}^{*,M}$ denote its optimal value for the finite horizon.
In general, RPOMDPs with different stickiness functions, including static and dynamic uncertainty, may lead to different optimal values.
Furthermore, a different order of play may also lead to different optimal values.
Formally:
\begin{enumerate}
    \item There exist RPOMDPs $M_1, M_2$ that only differ in their stickiness functions, such that $V_\text{fh}^{*,M_1} \neq V_\text{fh}^{*,M_2}$,
    \item There exist RPOMDPs $M_1, M_2$ that only differ in their order of play, such that $V_\text{fh}^{*,M_1} \neq V_\text{fh}^{*,M_2}$.
\end{enumerate}
\end{restatable}
\noindent We sketch the proof here, for details see
\Cref{app:assumptions_matter}.

\begin{proof}[Proof sketch]
We construct explicit RPOMDPs and show that the optimal values do not coincide.
For the first point regarding stickiness, consider the finite horizon RPOMDP in \Cref{fig:full_vs_zero_sticky_rPOMDP}.
For zero stickiness, 
the value is $65\frac{1}{2}$, with agent policy 
$\pi = \{\tsW\tsLG\tsDash\tsDot \mapsto \{a \mapsto \frac{1}{3}, b \mapsto \frac{2}{3}\}, 
\tsW\tsDG\tsDash\tsDot \mapsto \{a \mapsto \frac{2}{3}, b \mapsto \frac{1}{3}\}\}$ 
and nature policy 
$\theta = \{\tsW\tsLG\tsDash \mapsto \{p \mapsto \frac{83}{270}, q\mapsto \frac{1}{10}\},
\tsW\tsDG\tsDash \mapsto \{p \mapsto \_, q\mapsto \frac{1}{3}\}\}$.
For full stickiness, 
the value is $66\frac{2}{3}$, with agent policy 
$\pi = \{\tsW\tsLG\tsDash\tsDot \mapsto \{a \mapsto \frac{1}{3}, b\mapsto \frac{2}{3}\},
\tsW\tsDG\tsDash\tsDot \mapsto \{a \mapsto \frac{7}{10}, b\mapsto \frac{3}{10}\}$
and nature policy $\theta = \{\tsW \mapsto \{p \mapsto \frac{1}{3}, q\mapsto \frac{1}{3}\}\}$.

For the order of play, consider the finite horizon RPOMDP in \Cref{fig:agent_vs_nature_first_RPOMDP_small}.
For the agent first, the value is $30$, with nature policy 
$\theta = \{\tup{\tsW,a} \mapsto \{p \mapsto 0.1\}, \tup{\tsW,b} \mapsto \{p \mapsto 0.9\}\}$
and any agent policy.
For nature first, the value is $150$, with nature policy 
$\theta = \{\tsW \mapsto 
\{p \mapsto 0.5\}\}$
and agent policy
$\pi = \{\tsW \mapsto \{a \mapsto 0.5, b\mapsto 0.5\}\}$.
\end{proof}

Note that the optimal nature policies in these two RPOMDPs are deterministic.
We show in
\cref{app:assumptions_matter} 
that deterministic policies suffice in these specific RPOMDPs due to the linearity of the value function in the nature policies.
Furthermore, the RPOMDP we use to show that the order of play matters is fully observable and non-rectangular.
In 
\Cref{app:assumptions_matter},
we show that the order of play still matters under some form of rectangularity.

\begin{remark}
    For $(s, a)$-rectangular RMDPs,~\cite[Theorem 2.2]{DBLP:journals/mor/Iyengar05} shows that static and dynamic semantics in RMDPs lead to the same optimal value.
    Iyengar establishes that in $(s, a)$-rectangular RMDPs, memoryless policies are sufficient for the agent.
    In response, nature may also play memoryless, as there is no incentive for nature to change its choice after its initial choice. 
    As a consequence, zero and full stickiness coincide.
    This statement does not apply to RPOMDPs, where agents generally use memory.
    As shown in the $(s, a)$-rectangular RPOMDP in \Cref{fig:full_vs_zero_sticky_rPOMDP} and \Cref{thm:uncertainty:matters}, the optimal nature policy in this model's zero stickiness case uses information from previous observations, resulting in a smaller reward.
\end{remark}

\section{POSG Semantics for RPOMDPs}\label{sec:POSG}

We formalize the underlying game of an RPOMDP as a zero-sum two-sided partially observable stochastic game (POSG)~\cite{Springer:HSVI}, which is more widely studied than RPOMDPs.
Our transformation allows us to carry over results from POSGs to RPOMDPs.
In particular, we prove that our POSGs always have a Nash equilibrium for the finite horizon objective, which shows that optimal values and agent policies always exist in our finite horizon RPOMDPs.

\paragraph{Tracking fixed assignments.}
In our game, we explicitly keep track of the fixed variable assignments $u^\pset$.
The update function $\upd\colon \bm{U}^\pset \times \bm{U} \times \Znature \times \Zpub \times A \to \bm{U}^\pset$ updates the restricted variables after each valid nature choice following the stickiness of the RPOMDP $M$.
\begin{align*}
& \upd(u^\pset, u,z^\nature_\priv,z_\publ,a)(v) = \begin{cases}
    u(v) &  \text{ if } v \in \Ustick_h(z^\nature_\priv,z_\publ,a),\\
    u^\pset(v) &  \text{otherwise.}
\end{cases}
\end{align*}
By construction, recursively applying the update function on a path $\mypath$ yields $\fixed(\mypath)$.

\begin{definition}%
\label{def:equivalent:agent:zsposg}
Given an (agent first) RPOMDP $\tup{S, A, \bm{T}, R, \Zagent, \Znature, \Zpub, O^\agent_\priv, O^\nature_\priv, O_\publ,\sticky,\agent}$, we define the POSG  $\tup{\mathcal{S^\agent, S^\nature, A^\agent, A^\nature, T, R, Z^\agent, Z^\nature, O^\agent, O^\nature}}$, with a set $\mathcal{S}^\agent = S \times \bm{U}^\pset$ of \emph{agent states},  a set $\mathcal{S}^\nature = S \times \bm{U}^\pset \times A$ of \emph{nature states}, a finite set $\mathcal{A}^\agent = A$ of \emph{agent actions}, and a set $\mathcal{A}^\nature = \bm{U}$ of \emph{nature actions}. The observations are defined as follows:
$\mathcal{Z}^\agent = \Zagent \times \Zpub$ is the finite set of the agent's observations, and $\mathcal{Z}^\nature = \Znature \times \Zpub \times (A \cup \bot)$ the finite set of nature's observations.
The transition, reward, and observation functions are then defined as:
\begin{itemize}
    \item $\mathcal{T} = \mathcal{T}^\agent \cup \mathcal{T}^\nature$, the transition function, where
         $\mathcal{T}^\agent \colon \mathcal{S}^\agent \times \mathcal{A}^\agent \to \mathcal{S}^\nature$ is the agent's transition function, defined by \\
         $\mathcal{T^\agent}(\tup{s,u^\pset},a) = \tup{s,u^\pset,a} \in S^\nature$
     and
         $\mathcal{T}^\nature \colon \mathcal{S}^\nature \times \mathcal{A}^\nature \to \dist{\mathcal{S}^\agent}$ is nature's transition function, such that \\
        $\mathcal{T^\nature}(\tup{s,u^\pset,a},u,\tup{s',\upd(u^\pset,u,O^\nature_\priv(s),O_\publ(s),a)}) =$\\
    $\strut\quad \begin{cases}
       \bm{T}(u)(s,a,s') & \quad \text{if $u \in \bm{U}^\agrees(u^\pset)$,}\\
       0 & \quad \text{otherwise.}
    \end{cases}$
    \item $\mathcal{R}\colon \mathcal{S^\agent} \times \mathcal{A^\agent} \to \RR$ the reward function, given by
    $\mathcal{R}(\tup{s,u^\pset},a) = R(s,a)$. 
    State-action pairs $S^\nature \times A^\nature$ have zero reward.
    \item $\mathcal{O}^\agent\colon (\mathcal{S}^\agent\cup \mathcal{S}^\nature) \to \mathcal{Z}^\agent$ the deterministic observations function of the agent defined as:\\
    $\mathcal{O^\agent}(s) = \begin{cases}
        \tup{O^\agent_\priv(s'), O_\publ(s')} & \text{if $s = \tup{s',u^\pset}\in \mathcal{S^\agent}$,}\\
        \tup{O^\agent_\priv(s'), O_\publ(s')} & \text{if $s = \tup{s',u^\pset,a} \in \mathcal{S^\nature}$.}
    \end{cases}$
    \item $\mathcal{O}^\nature\colon (\mathcal{S}^\agent\cup \mathcal{S}^\nature) \to \mathcal{Z}^\nature$ the deterministic observations function of nature defined as:\\
    $\mathcal{O^\nature}(s) = \begin{cases}
        \tup{O^\nature_\priv(s'), O_\publ(s'),\bot} & \text{if $s = \tup{s',u^\pset} \in \mathcal{S^\agent}$,}\\
        \tup{O^\nature_\priv(s'), O_\publ(s'),a} &  \text{if $s = \tup{s',u^\pset,a} \in \mathcal{S^\nature}$.}
    \end{cases}$
\end{itemize}
\end{definition}

\paragraph{Game behavior.}
This game starts in an $\mathcal{S}^\agent$ state consisting of the initial state $s_I \in S$ of the RPOMDP and the totally undefined variable assignment $u^\bot \in \bm{U}^\pset$.
At any agent state $\tup{s, u^\pset}$, both players observe their private and public observations of state $s$.
After the agent chooses their action $a$, the game transitions deterministically to a nature state $\tup{s, u^\pset, a}$.
Again, both players observe their private and public observations of state $s$, with which nature observes the agent's last action.
Next, nature selects a variable assignment $u \in \bm{U}^\agrees(u^\pset)$ from the set of variable assignments that agree with nature's past choices and hence account for the stickiness of the RPOMDP.
Then the uncertain transition function $\bm{T}$ is resolved with $u$ after which the game stochastically moves to the next agent state $\tup{s',\upd(u^\pset,u, O^\nature_\priv(s), O^\publ(s), a)}$, where $s'$ can be reached from $s$ given action $a$ and the resolved transition function.

\paragraph{Nature chooses first.}
The POSG above is defined for RPOMDPs where the agent plays first, where $\player = \agent$.
As nature can observe the agent's action choice, it may use this information to choose a transition function from the uncertainty set.
If we assume that nature plays first, this information is not available yet; hence, the structure of the POSG needs to be changed to reflect this. 
For the remainder of the main paper, we focus on the case where the agent moves first, \ie, RPOMDPs with $\player = \agent$. 
Our results carry over to RPOMDPs, where nature moves first. See 
\Cref{app:nature_first}.

\paragraph{Paths and histories. }
A path in a POSG is a sequence of successive states and actions that alternate between agent and nature: $\tup{s^\agent_0,a^\agent_0,s^\nature_0,a^\nature_0,s^\agent_1,a^\agent_1,\dots} \in (\mathcal{S}^\agent \times \mathcal{A}^\agent \times \mathcal{S}^\nature \times \mathcal{A}^\nature)^* \times \mathcal{S}^\agent$,
A path is valid if $\forall s^\agent_i, a^\agent_i, s^\nature_i. \, \mathcal{T}^\agent(s^{\agent}_{i}, a^{\agent}_{i}, s^{\nature}_{i}) > 0$, and $\forall s^\nature_i, a^\nature_i, s^\agent_{i+1}. \, \mathcal{T}^\nature(s^\nature_i, a^\nature_i, s^\agent_{i+1}) > 0$.
The set of paths in $G$ is $\Paths^G$.
In the POSGs we consider, players only observe their own actions.
A history for the agent or nature is a path mapped to their respective observations: the agent only observes agent actions, and their histories are sequences of the form $\tup{\mathcal{O}^\agent(s^\agent_0),a^\agent_0,\mathcal{O}^\agent(s^\nature_0),\mathcal{O}^\agent(s^\agent_1),a^\agent_1,\mathcal{O}^\agent(s^\nature_1),\dots} \in (\mathcal{Z}^\agent \times \mathcal{A}^\agent \times \mathcal{Z}^\agent)^* \times \mathcal{Z}^\agent$, while the histories of nature are sequences in $(\mathcal{Z}^\nature \times \mathcal{Z}^\nature \times \mathcal{A}^\nature)^* \times \mathcal{Z}^\nature$.
The sets of agent and nature histories in $G$ are $H^{\agent,G}$ and $H^{\nature, G}$, respectively.

\paragraph{Policies and values.}
A policy for the agent in POSG $G$ is a function $\pi \colon  H^{\agent, G} \to \dist{A^\agent}$, and a policy for nature is a function $\theta \colon H^{\nature,G} \times \mathcal{Z^\nature} \to \dist{A^\nature}$.
The sets of all agent and nature policies in $G$ are denoted by $\Pi^G$ and $\Theta^G$, respectively.
The sets of deterministic and mixed policies are constructed analogously as for POMDPs in \Cref{sec:preliminaries}.
The value of a POSG is the expected reward collected under both players' policies $\pi, \theta$:
\[
V_\text{fh}^{\pi,\theta} = \EE \left[ \sum_{t=0}^{K-1} r_t \mid \pi, \theta \right], \quad V_\text{dih}^{\pi,\theta} = \EE \left[ \sum_{t=0}^\infty \gamma^t r_t \mid \pi, \theta \right].
\]

\subsection{Correctness of the Transformation}
In the following, we show the correctness of our transformation from RPOMDP to POSG.
We do this by (1) constructing bijections between paths and histories of an RPOMDP and its POSG, (2) using these bijections to derive bijections between the agent and nature policies for both RPOMDP and POSG, and (3) concluding with an equivalence between the values for both models.
All proofs, including the explicit construction of all bijections, can be found in 
\Cref{app:value:function:proofs}.

\begin{proposition}[Bijection between paths and histories]
Let $M$ be an RPOMDP, and $G$ the POSG of $M$.
There exists a bijection $f \colon \Paths^M \to \Paths^G$ and bijections between individual players' histories:
    \begin{itemize}
        \item Let $H^{\agent,M}$ and $H^{\agent,G}$ be the set of all agent histories in $M$ and $G$, respectively.
        There exists a bijection $f^{\agent,h}\colon H^{\agent,M} \to H^{\agent,G}$.
        \item Let $H^{\nature,M}$ and $H^{\nature,G}$ be the set of all nature histories in $M$ and $G$, respectively.
        There exists a bijection $f^{\nature,h}\colon H^{\nature,M} \to H^{\nature,G}$.
    \end{itemize}
\end{proposition}
\noindent
Using the bijection between histories, we relate agent policies $\Pi^M$ with $\Pi^G$ and nature policies $\Theta^M$ with $\Theta^G$.
\begin{proposition}[Bijection between policies]\label{cor:bijection:policies}
    Let $M$ be an RPOMDP, and $G$ the POSG of $M$.
    There exist bijections $f^{\pi}\colon \Pi^M \to \Pi^G$ and $f^{\theta}\colon \Theta^M \to \Theta^G$ between the agent's and nature's policies in $M$ and $G$, respectively.
\end{proposition}

An agent RPOMDP policy $\pi^M \in \Pi^M$ and an agent POSG policy $\pi^G \in \Pi^G$ are \emph{corresponding} if $\pi^M$ maps to $\pi^G$ via the bijection~$f^{\pi}$, \ie, $\pi^G = f^{\pi}(\pi^M)$.
Similarly, a nature RPOMDP policy $\theta^M$ and a nature POSG policy $\theta^G$ are \emph{corresponding} if $\theta^G = f^\theta(\theta^M)$. 
From \Cref{cor:bijection:policies} it then follows that for two corresponding agent policies and two corresponding nature policies, the values of the RPOMDP and the POSG coincide.

\begin{restatable}[Equivalent values]{theorem}{theoremII}\label{thm:equivalent:values}
Let $M$ be an RPOMDP, and $G$ the POSG of $M$.
Let $\pi^M \in \Pi^{M}, \pi^G = f^\pi(\pi^M) \in \Pi^{G}$ be corresponding agent policies, and $\theta^M \in \Theta^{M}, \theta^G = f^\theta(\theta^M) \in \Theta^{G}$ be corresponding nature policies.
Then, their values for the RPOMDP and POSG coincide:
\[
V_{\phi}^{\pi^M,\theta^M} = V_{\phi}^{\pi^G,\theta^G}.
\]
\end{restatable}

By showing that there is a bijection between RPOMDP and POSG policies and that the values coincide, we have established that these POSGs form an operational model for RPOMDP semantics.

\subsection{Existence of Nash Equilibria}\label{sec:nash_equilibrium}
Using the RPOMDP to POSG transformation, we prove the existence of optimal values and policies for the agent in an RPOMDP for the finite horizon objective.
That is, the existence of maximal values agent policies can achieve against all nature policies, such that
$V_\text{fh}^* = \sup_{\pi \in \Pi} \inf_{\theta \in \Theta} \EE [ \sum_{t=0}^{K-1} r_t \mid \pi, \theta]$.
From \Cref{thm:equivalent:values}, it follows that if the values $V_\phi^*$ exist in the POSG $G$ of an RPOMDP $M$, they also exist in $M$.

The value $V_\phi^{\pi, \theta}$ of a POSG $G$ is a \emph{Nash equilibrium}, and both players' policies are Nash optimal, denoted $\pi^*, \theta^*$, if there is no incentive for either player to change their policy.
That is, for either objective $\phi \in \{\text{fh},\text{dih}\}$ we have:
\begin{align*}
\forall \pi \in \Pi^G.  V_\phi^{\pi^*,\theta^*} \geq  V_\phi^{\pi,\theta^*} \,\, \wedge \,\,  
\forall \theta \in \Theta^G. V_\phi^{\pi^*,\theta^*} \leq  V_\phi^{\pi^*,\theta}.
\end{align*}

Since the uncertainty set is infinite, our POSGs do not meet the standard requirements for a Nash equilibrium to exist~\cite{Book_game_theory_multi-leveled, DBLP:journals/corr/abs-2305-10546}.
Yet, our POSGs exhibit enough structure to show that a Nash equilibrium always exists for the finite horizon objective.

\begin{restatable}[Existence of finite horizon Nash equilibrium]{theorem}{theoremIII}\label{thm:exist:nash}
Let $M$ be an RPOMDP 
and $G$ the POSG of $M$. 
For the finite horizon objective $V_\emph{fh}^{\pi,\theta} = \sum_{t=0}^{k-1} [r_t \mid \pi, \theta]$ we have the following saddle point condition in $G$:
\begin{equation}\label{eq:nash:equilibrium}
    \sup_{\pi\in \Pi^G}\inf_{\theta\in \Theta^G}V_\emph{fh}^{\pi,\theta} = \inf_{\theta\in \Theta^G}\sup_{\pi\in \Pi^G}V_\emph{fh}^{\pi,\theta}.
\end{equation}
From \Cref{eq:nash:equilibrium}, the existence of a Nash equilibrium in $G$ follows immediately~\cite{Book_game_theory_multi-leveled}.
\end{restatable}
\noindent We sketch the proof here; for details see
\Cref{app:nashEquilibrium}.

\begin{proof}[Proof sketch]
    We show the existence of the Nash equilibrium for our RPOMDPs by first defining a sufficient statistic 
    (\Cref{app:suff.stat}).
    This statistic tracks histories and nature's policy and is an adaptation of the definition of \cite{Springer:HSVI}.
    We use the sufficient statistic to construct the state space of a non-observable occupancy game 
    (\Cref{app:occupacy_game}) 
    between agent and nature.
    Additionally, we show that we can reason about the optimal value and policies of the occupancy game, and thus those of the POSG, with the sets of mixed agent and nature policies instead of the sets of stochastic policies 
    (\Cref{app:mixed_policies}).
    Using the sets of mixed policies, we show that the constructed occupancy game is a semi-infinite convex game, as defined by~\cite{Convex_semi-infinite_games}
    (\Cref{app:convex_semi-infinite_game}).
    Finally, we show that our occupancy game meets the conditions given by~\cite{Convex_semi-infinite_games} for the existence of a saddle point.
    From the existence of the saddle point, the existence of the Nash equilibrium and an optimal agent policy immediately follows~\cite{Book_game_theory_multi-leveled}.  
\end{proof}

Whether a Nash equilibrium exists in the POSG $G$ for discounted infinite horizon objective $V_\emph{dih}^{\pi,\theta} = \sum_{t=0}^{\infty} [\gamma^t r_t \mid \pi, \theta]$ or a saddle point condition that would imply this Nash equilibrium remains an open problem.

\paragraph{Other semantic implications for RPOMDPs.}
To shed light on the reason why two RPOMDPs that only differ in either their stickiness or order of play can lead to different optimal values, we look at the structure of the POSGs of these RPOMDPs.
Specifically, the RPOMDP from \Cref{fig:full_vs_zero_sticky_rPOMDP} with either zero or full stickiness leads to the two POSGs depicted in \Cref{fig:stickiness:POSGs}.
The key difference between these POSGs is that in the zero stickiness case, every variable assignment by nature leads to the same two successor states, while in the full stickiness case, any variable assignment by nature leads to two \emph{unique} successor states and thus an infinitely branching POSG.
A similar structural difference can be seen in the two POSGs depicted in \Cref{fig:order:POSGs}, which show the difference in the order of play for the RPOMDP in \Cref{fig:agent_vs_nature_first_RPOMDP_small}.

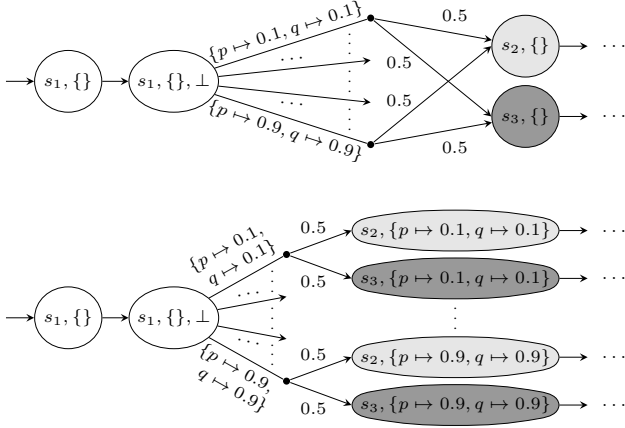
\begin{figure}[t]
    \centering
    \begin{subfigure}[b]{\columnwidth}
        \resizebox{\columnwidth}{!}{
        \begin{tikzpicture}[state/.append style={shape = ellipse}, >=stealth,
    bobbel/.style={minimum size=1mm,inner sep=0pt,fill=black,circle}]]
    \node[state, inner sep=-0.2pt] (s1) at (1,0) {\scriptsize $s_1,\{\}$};
    \node[state, inner sep=-2pt] (s1_bot) at (2.5,0) {\scriptsize $s_1,\{\},\bot$};
    \node[bobbel] (s1b_u1) at (5.3,0.9) {};
    \node[bobbel] (s1b_u2) at (5.3,-0.9) {};
    \node[state, fill=black!10, inner sep=-0.2pt] (s2) at (7.5,0.5) {\scriptsize $s_2,\{\}$};
    \node[state, fill=black!40, inner sep=-0.2pt] (s3) at (7.5,-0.5) {\scriptsize $s_3,\{\}$};
    \draw[<-] (s1.west) -- +(-0.4,0);
    \draw[->] (s2.east) -- +(0.4,0);
    \draw[->] (s3.east) -- +(0.4,0);
    \node[] at ($(s2.east) + (0.8,0)$) {\scriptsize $\dots$};
    \node[] at ($(s3.east) + (0.8,0)$) {\scriptsize $\dots$};
    \draw (s1) edge[->] (s1_bot);
    \draw (s1b_u1) edge[->] node[above right]{\scriptsize $0.5$} (s2);
    \draw (s1b_u1) edge[->] node[below left, pos = 0.4, yshift=1.5mm]{\scriptsize $0.5$} (s3.west);
    \draw (s1b_u2) edge[->] node[above left, pos = 0.4, yshift=-1.5mm]{\scriptsize $0.5$} (s2.west);
    \draw (s1b_u2) edge[->] node[below right]{\scriptsize $0.5$} (s3);
    \draw (s1_bot) edge[->] node[sloped, anchor=center, above]{\scriptsize $\dots$} (5.3,0.3);
    \draw (s1_bot) edge[->] node[sloped, anchor=center, below]{\scriptsize $\dots$} (5.3,-0.3);
    \node[] at (5,0.1) {\tiny $\vdots$};
    \node[] at (5,0.64) {\tiny $\vdots$};
    \node[] at (5,-0.44) {\tiny $\vdots$};
    \draw (s1_bot) edge node[sloped, anchor=center, above, xshift=-0.2mm, yshift=-0.6mm]{\scriptsize $\{p\mapsto 0.1,q\mapsto 0.1\}$} (s1b_u1);
    \draw (s1_bot) edge node[sloped, anchor=center, below, xshift=-0.2mm, yshift=0.6mm]{\scriptsize $\{p\mapsto 0.9,q\mapsto 0.9\}$} (s1b_u2);
\end{tikzpicture}%
        }
    \end{subfigure}\\[3mm]
    \begin{subfigure}[b]{\columnwidth}
        \resizebox{\columnwidth}{!}{
        \begin{tikzpicture}[state/.append style={shape = ellipse}, >=stealth,
    bobbel/.style={minimum size=1mm,inner sep=0pt,fill=black,circle}]]
    \node[state, inner sep=-0.2pt] (s1) at (1,0) {\scriptsize $s_1,\{\}$};
    \node[state, inner sep=-2pt] (s1_bot) at (2.5,0) {\scriptsize $s_1,\{\},\bot$};
    \node[bobbel] (s1b_u1) at (4.1,0.9) {};
    \node[bobbel] (s1b_u2) at (4.1,-0.9) {};
    \coordinate (c_s2_u1) at ($(s1b_u1) + (2.4,0.34)$);
    \coordinate (c_s3_u1) at ($(s1b_u1) + (2.4,-0.34)$);
    \coordinate (c_s2_u2) at ($(s1b_u2) + (2.4,0.34)$);
    \coordinate (c_s3_u2) at ($(s1b_u2) + (2.4,-0.34)$);
    \draw[fill=black!10] ($(c_s2_u1) + (-1.47,0)$) 
                        to [out = 85, in = 190] ($(c_s2_u1) + (-1.2,0.19)$) 
                        to [out = 10, in = 180] ($(c_s2_u1) + (0,0.29)$)
                        to [out = 0, in = 170] ($(c_s2_u1) + (1.2,0.19)$)
                        to [out = -10, in = 95] ($(c_s2_u1) + (1.47,0)$)
                        to [out = -95, in = 10] ($(c_s2_u1) + (1.2,-0.19)$)
                        to [out = 190, in = 0] ($(c_s2_u1) + (0,-0.29)$)
                        to [out = 180, in = -10] ($(c_s2_u1) + (-1.2,-0.19)$)
                        to [out = 170, in = -85] ($(c_s2_u1) + (-1.47,0)$);
    \draw[fill=black!40] ($(c_s3_u1) + (-1.47,0)$) 
                        to [out = 85, in = 190] ($(c_s3_u1) + (-1.2,0.19)$) 
                        to [out = 10, in = 180] ($(c_s3_u1) + (0,0.29)$)
                        to [out = 0, in = 170] ($(c_s3_u1) + (1.2,0.19)$)
                        to [out = -10, in = 95] ($(c_s3_u1) + (1.47,0)$)
                        to [out = -95, in = 10] ($(c_s3_u1) + (1.2,-0.19)$)
                        to [out = 190, in = 0] ($(c_s3_u1) + (0,-0.29)$)
                        to [out = 180, in = -10] ($(c_s3_u1) + (-1.2,-0.19)$)
                        to [out = 170, in = -85] ($(c_s3_u1) + (-1.47,0)$);
    \draw[fill=black!10] ($(c_s2_u2) + (-1.47,0)$) 
                        to [out = 85, in = 190] ($(c_s2_u2) + (-1.2,0.19)$) 
                        to [out = 10, in = 180] ($(c_s2_u2) + (0,0.29)$)
                        to [out = 0, in = 170] ($(c_s2_u2) + (1.2,0.19)$)
                        to [out = -10, in = 95] ($(c_s2_u2) + (1.47,0)$)
                        to [out = -95, in = 10] ($(c_s2_u2) + (1.2,-0.19)$)
                        to [out = 190, in = 0] ($(c_s2_u2) + (0,-0.29)$)
                        to [out = 180, in = -10] ($(c_s2_u2) + (-1.2,-0.19)$)
                        to [out = 170, in = -85] ($(c_s2_u2) + (-1.47,0)$);
    \draw[fill=black!40] ($(c_s3_u2) + (-1.47,0)$) 
                        to [out = 85, in = 190] ($(c_s3_u2) + (-1.2,0.19)$) 
                        to [out = 10, in = 180] ($(c_s3_u2) + (0,0.29)$)
                        to [out = 0, in = 170] ($(c_s3_u2) + (1.2,0.19)$)
                        to [out = -10, in = 95] ($(c_s3_u2) + (1.47,0)$)
                        to [out = -95, in = 10] ($(c_s3_u2) + (1.2,-0.19)$)
                        to [out = 190, in = 0] ($(c_s3_u2) + (0,-0.29)$)
                        to [out = 180, in = -10] ($(c_s3_u2) + (-1.2,-0.19)$)
                        to [out = 170, in = -85] ($(c_s3_u2) + (-1.47,0)$);
    \node[state, draw=none, inner sep=-9pt, minimum width = 29mm] (s2_u1) at (c_s2_u1) {\scriptsize $s_2,\{p\mapsto 0.1,q\mapsto 0.1\}$};
    \node[state, draw=none, inner sep=-9pt, minimum width = 29mm] (s3_u1) at (c_s3_u1) {\scriptsize $s_3,\{p\mapsto 0.1,q\mapsto 0.1\}$};
    \node[] at (6.5,0.1) {\tiny $\vdots$};
    \node[state, draw=none, inner sep=-9pt, minimum width = 29mm] (s2_u2) at (c_s2_u2) {\scriptsize $s_2,\{p\mapsto 0.9,q\mapsto 0.9\}$};
    \node[state, draw=none, inner sep=-9pt, minimum width = 29mm] (s3_u2) at (c_s3_u2) {\scriptsize $s_3,\{p\mapsto 0.9,q\mapsto 0.9\}$};  
    \draw[<-] (s1.west) -- +(-0.4,0);
    \draw[->] (s2_u1.east) -- +(0.4,0);
    \draw[->] (s3_u1.east) -- +(0.4,0);
    \draw[->] (s2_u2.east) -- +(0.4,0);
    \draw[->] (s3_u2.east) -- +(0.4,0);
    \node[] at ($(s2_u1.east) + (0.8,0)$) {\scriptsize $\dots$};
    \node[] at ($(s3_u1.east) + (0.8,0)$) {\scriptsize $\dots$};
    \node[] at ($(s2_u2.east) + (0.8,0)$) {\scriptsize $\dots$};
    \node[] at ($(s3_u2.east) + (0.8,0)$) {\scriptsize $\dots$};
    \draw (s1) edge[->] (s1_bot);
    \draw (s1b_u1) edge[->] node[above left, xshift=2mm]{\scriptsize $0.5$} (s2_u1.west);
    \draw (s1b_u1) edge[->] node[below left, xshift=2mm]{\scriptsize $0.5$} (s3_u1.west);
    \draw (s1b_u2) edge[->] node[above left, xshift=2mm]{\scriptsize $0.5$} (s2_u2.west);
    \draw (s1b_u2) edge[->] node[below left, xshift=2mm]{\scriptsize $0.5$} (s3_u2.west);
    \draw (s1_bot) edge[->] node[sloped, anchor=center, above]{\scriptsize $\dots$} (4.1,0.3);
    \draw (s1_bot) edge[->] node[sloped, anchor=center, below]{\scriptsize $\dots$} (4.1,-0.3);
    \node[] at (3.9,0.1) {\tiny $\vdots$};
    \node[] at (3.9,0.62) {\tiny $\vdots$};
    \node[] at (3.9,-0.42) {\tiny $\vdots$};
    \draw (s1_bot) edge node[sloped, anchor=center, above, xshift=-0.2mm, yshift=-0.6mm] {\scriptsize$\begin{aligned}
        \{ &p\mapsto 0.1,\\[-1mm]
        &q\mapsto 0.1\}
        \end{aligned}$} (s1b_u1);
    \draw (s1_bot) edge node[sloped, anchor=center, below, xshift=-0.2mm, yshift=0.6mm]{\scriptsize $\begin{aligned}
        \{ &p\mapsto 0.9,\\[-1mm]
        &q\mapsto 0.9\}
        \end{aligned}$} (s1b_u2);
\end{tikzpicture}%
        }
    \end{subfigure}
    \caption{First states of zero stickiness (top) and full stickiness (bottom) POSGs of the RPOMDP in  \Cref{fig:full_vs_zero_sticky_rPOMDP}.}
    \label{fig:stickiness:POSGs}
\end{figure}

\begin{figure}[t]
    \centering
    \begin{subfigure}[b]{0.9\columnwidth}
        \resizebox{\columnwidth}{!}{
        \begin{tikzpicture}[state/.append style={shape = ellipse}, >=stealth,
    bobbel/.style={minimum size=1mm,inner sep=0pt,fill=black,circle},
    mynode/.style={rectangle,fill=white,anchor=center}]]
    \node[state, inner sep=-0.2pt] (s1) at (1,0) {\scriptsize $s_1,\{\}$};
    \node[state, inner sep=-2pt] (s1_a) at (3.2,1.15) {\scriptsize $s_1,\{\},a$};
    \node[state, inner sep=-2pt] (s1_b) at (3.2,-1.15) {\scriptsize $s_1,\{\},b$};
    \node[bobbel] (s1ba_u1) at ($(s1_a) + (2.5,0.85)$) {};
    \node[bobbel] (s1ba_u2) at ($(s1_a) + (2.5,-0.85)$) {};
    \node[bobbel] (s1bb_u1) at ($(s1_b) + (2.5,0.85)$) {};
    \node[bobbel] (s1bb_u2) at ($(s1_b) + (2.5,-0.85)$) {};
    \node[rectangle, draw, minimum width=11mm, rounded corners] (rPosa) at ($(s1ba_u1) + (1.5,-0.25)$) {\scriptsize R = $300$};
    \node[rectangle, draw, minimum width=11mm, rounded corners] (rNega) at ($(s1ba_u2) + (1.5,0.25)$) {\scriptsize R = $0$};
    \node[rectangle, draw, minimum width=11mm, rounded corners] (rNegb) at ($(s1bb_u1) + (1.5,-0.25)$) {\scriptsize R = $0$};
    \node[rectangle, draw, minimum width=11mm, rounded corners] (rPosb) at ($(s1bb_u2) + (1.5,0.25)$) {\scriptsize R = $300$};
    \draw[<-] (s1.west) -- +(-0.4,0);
    \draw (s1) edge[->] node[above left, yshift=-0.5mm, xshift=1mm]{\scriptsize $a$} (s1_a);
    \draw (s1) edge[->] node[below left, yshift=0.5mm, xshift=1mm]{\scriptsize $b$} (s1_b);
    \draw (s1ba_u1) edge[->] node[above, yshift=-0.3mm, xshift=1mm]{\scriptsize $0.1$} (rPosa);
    \draw (s1ba_u1) edge[->] node[below left, pos=0.4, yshift=1mm, xshift=0.8mm]{\scriptsize $0.9$} (rNega);
    \draw (s1ba_u2) edge[->] node[above left, pos=0.4, yshift=-1mm, xshift=0.8mm]{\scriptsize $0.9$} (rPosa);
    \draw (s1ba_u2) edge[->] node[below, yshift=0.3mm, xshift=1mm]{\scriptsize $0.1$} (rNega);
    \draw (s1_a) edge[->] node[sloped, anchor=center, above]{\scriptsize $\dots$} ($(s1ba_u1) + (0,-0.55)$);
    \draw (s1_a) edge[->] node[sloped, anchor=center, below]{\scriptsize $\dots$} ($(s1ba_u2) + (0,0.55)$);
    \node[] at ($(s1_a -| s1ba_u1) + (-0.25,0.6)$) {\tiny $\vdots$};
    \node[] at ($(s1_a -| s1ba_u1) + (-0.25,0.1)$) {\tiny $\vdots$};
    \node[] at ($(s1_a -| s1ba_u1) + (-0.25,-0.4)$) {\tiny $\vdots$};
    \draw (s1_a) edge node[sloped, anchor=center, above, xshift=-0.2mm, yshift=-0.6mm]{\scriptsize $\{p\mapsto 0.1\}$} (s1ba_u1);
    \draw (s1_a) edge node[sloped, anchor=center, below, xshift=-0.2mm, yshift=0.6mm]{\scriptsize $\{p\mapsto 0.9\}$} (s1ba_u2);
    \draw (s1bb_u1) edge[->] node[above, yshift=-0.3mm, xshift=1mm]{\scriptsize $0.1$} (rNegb);
    \draw (s1bb_u1) edge[->] node[below left, pos=0.4, yshift=1mm, xshift=0.8mm]{\scriptsize $0.9$} (rPosb);
    \draw (s1bb_u2) edge[->] node[above left, pos=0.4, yshift=-1mm, xshift=0.8mm]{\scriptsize $0.9$} (rNegb);
    \draw (s1bb_u2) edge[->] node[below, yshift=0.3mm, xshift=1mm]{\scriptsize $0.1$} (rPosb);
    \draw (s1_b) edge[->] node[sloped, anchor=center, above]{\scriptsize $\dots$} ($(s1bb_u1) + (0,-0.55)$);
    \draw (s1_b) edge[->] node[sloped, anchor=center, below]{\scriptsize $\dots$} ($(s1bb_u2) + (0,0.55)$);
    \node[] at ($(s1_b -| s1bb_u1) + (-0.25,0.6)$) {\tiny $\vdots$};
    \node[] at ($(s1_b -| s1bb_u1) + (-0.25,0.1)$) {\tiny $\vdots$};
    \node[] at ($(s1_b -| s1bb_u1) + (-0.25,-0.4)$) {\tiny $\vdots$};
    \draw (s1_b) edge node[sloped, anchor=center, above, xshift=-0.2mm, yshift=-0.6mm]{\scriptsize $\{p\mapsto 0.1\}$} (s1bb_u1);
    \draw (s1_b) edge node[sloped, anchor=center, below, xshift=-0.2mm, yshift=0.6mm]{\scriptsize $\{p\mapsto 0.9\}$} (s1bb_u2);
\end{tikzpicture}%
        }
    \end{subfigure}\\[3mm]
    \begin{subfigure}[b]{0.9\columnwidth}
        \resizebox{\columnwidth}{!}{
        \begin{tikzpicture}[state/.append style={shape = ellipse}, >=stealth,
    bobbel/.style={minimum size=1mm,inner sep=0pt,fill=black,circle}]
    \node[state, inner sep=-0.2pt] (s1) at (1,0) {\scriptsize $s_1,\{\}$};
    \node[state, inner sep=-4pt] (s1_u1) at (4.2,0.85) {\scriptsize $s_1,\{\},\{p\mapsto 0.1\}$};
    \node[state, inner sep=-4pt] (s1_u2) at (4.2,-0.85) {\scriptsize $s_1,\{\},\{p\mapsto 0.9\}$};
    \node[bobbel] (s1_u1_ab) at ($(s1_u1) + (1.5,0.6)$) {};
    \node[bobbel] (s1_u1_bb) at ($(s1_u1) + (1.5,-0.6)$) {};
    \node[bobbel] (s1_u2_ab) at ($(s1_u2) + (1.5,0.6)$) {};
    \node[bobbel] (s1_u2_bb) at ($(s1_u2) + (1.5,-0.6)$) {};
    \node[rectangle, draw, minimum width=11mm, rounded corners] (rPosa) at ($(s1_u1_ab) + (1.5,-0.25)$) {\scriptsize R = $300$};
    \node[rectangle, draw, minimum width=11mm, rounded corners] (rNega) at ($(s1_u1_bb) + (1.5,0.25)$) {\scriptsize R = $0$};
    \node[rectangle, draw, minimum width=11mm, rounded corners] (rPosb) at ($(s1_u2_ab) + (1.5,-0.25)$) {\scriptsize R = $300$};
    \node[rectangle, draw, minimum width=11mm, rounded corners] (rNegb) at ($(s1_u2_bb) + (1.5,0.25)$) {\scriptsize R = $0$};
    \draw[<-] (s1.west) -- +(-0.4,0);
    \draw (s1) edge[->] node[sloped, anchor=center, above, xshift=-0.2mm, yshift=-0.6mm]{\scriptsize $\{p\mapsto 0.1\}$} (s1_u1.west);
    \draw (s1) edge[->] node[sloped, anchor=center, below, xshift=-0.2mm, yshift=0.6mm]{\scriptsize $\{p\mapsto 0.9\}$} (s1_u2.west);
    \draw (s1) edge[->] node[sloped, anchor=center, above]{\scriptsize $\dots$} (s1_u1.west |- 0,0.3);
    \draw (s1) edge[->] node[sloped, anchor=center, below]{\scriptsize $\dots$} (s1_u1.west |- 0,-0.3);
    \node[] at ($(s1_u1.west |- 1,0.6) + (-0.25,0)$) {\tiny $\vdots$};
    \node[] at ($(s1_u1.west |- 1,0.1) + (-0.25,0)$) {\tiny $\vdots$};
    \node[] at ($(s1_u1.west |- 1,-0.38) + (-0.25,0)$) {\tiny $\vdots$};
    \node[] at (s1_u1 |- 1,0.1) {\tiny $\vdots$};
    \draw (s1_u1) edge node[above left, yshift=-0.5mm, xshift=1mm]{\scriptsize $a$} (s1_u1_ab);
    \draw (s1_u1) edge node[below left, yshift=0.5mm, xshift=1mm]{\scriptsize $b$} (s1_u1_bb);
    \draw (s1_u1_ab) edge[->] node[above, yshift=-0.3mm, xshift=1mm]{\scriptsize $0.1$} (rPosa);
    \draw (s1_u1_ab) edge[->] node[below left, pos=0.4, yshift=1.5mm]{\scriptsize $0.9$} (rNega);
    \draw (s1_u1_bb) edge[->] node[above left, pos=0.4, yshift=-1.5mm]{\scriptsize $0.9$} (rPosa);
    \draw (s1_u1_bb) edge[->] node[below, yshift=0.3mm, xshift=1mm]{\scriptsize $0.1$} (rNega);
    \draw (s1_u2) edge node[above left, yshift=-0.5mm, xshift=1mm]{\scriptsize $a$} (s1_u2_ab);
    \draw (s1_u2) edge node[below left, yshift=0.5mm, xshift=1mm]{\scriptsize $b$} (s1_u2_bb);
    \draw (s1_u2_ab) edge[->] node[above, yshift=-0.3mm, xshift=1mm]{\scriptsize $0.9$} (rPosb);
    \draw (s1_u2_ab) edge[->] node[below left, pos=0.4, yshift=1.5mm]{\scriptsize $0.1$} (rNegb);
    \draw (s1_u2_bb) edge[->] node[above left, pos=0.4, yshift=-1.5mm]{\scriptsize $0.1$} (rPosb);
    \draw (s1_u2_bb) edge[->] node[below, yshift=0.3mm, xshift=1mm]{\scriptsize $0.9$} (rNegb);
\end{tikzpicture}%
        }
    \end{subfigure}
    \caption{Agent first (top) and nature first (bottom) POSGs of the RPOMDP in \Cref{fig:agent_vs_nature_first_RPOMDP_small}.}
    \label{fig:order:POSGs}
\end{figure}
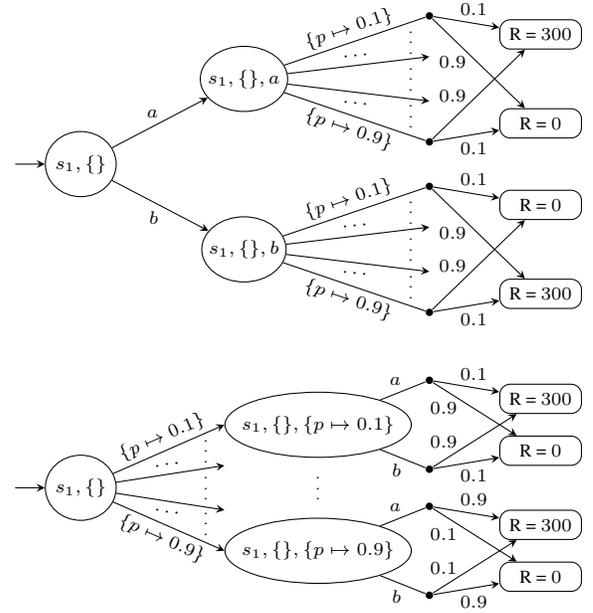

\section{Related Work}\label{sec:related:work}
In this section, we first classify the existing RPOMDP literature into the different assumptions discussed in this paper, and then we provide a general overview of the related work.

\subsection{Classification of RPOMDP Methods}
\Cref{tab:classification} provides an overview of RPOMDP solution methods within our game semantics, specifically classifying the type of stickiness and the order of play for these methods.

Note that in the table, full stickiness and nature first order of play are always combined, as are zero stickiness and agent first order of play. 
This can be explained by those combinations being the most intuitive extensions of static and dynamic uncertainty to the partially observable setting.
We also remark that~\cite{DBLP:journals/jet/Saghafian18} defines their problem by one fixed but unknown model that is chosen non-deterministically from the start, implying full stickiness and nature first, but their algorithmic solution method operates with zero stickiness and agent first semantics.

\begin{table}[t]
    \centering
    \resizebox{\columnwidth}{!}{
    \begin{tabular}{lll}
    \toprule
        Reference & Stickiness & Order of play\\
        \midrule
        \cite{DBLP:conf/icml/Osogami15} & Zero  & Agent first\\
        \cite{DBLP:conf/cdc/ChamieM18} & Zero  & Agent first\\
        \cite{DBLP:journals/jet/Saghafian18} & Zero  & Agent first\\
        \cite{DBLP:journals/siamjo/NakaoJS21} & Zero  & Agent first\\
        \cite{DBLP:conf/ijcai/Suilen0CT20} & Full  & Nature first\\
        \cite{DBLP:conf/aaai/Cubuktepe0JMST21} & Full  & Nature first\\
        \bottomrule
    \end{tabular}
    }
    \caption{Classification of existing RPOMDP literature.}
    \label{tab:classification}
\end{table}

\subsection{Further Related Work}

The connection between rectangular RMDPs and stochastic games is well-established, see for instance~\cite{DBLP:journals/mor/Iyengar05,DBLP:conf/nips/XuM10,DBLP:journals/mor/WiesemannKR13}.
Yet, a key difference is that in RMDPs, nature is typically assumed to play stationary, as already mentioned in \Cref{subsec:RMDP}.
This assumption is common because it is either sufficient for nature to play stationary or there are computational reasons.
In SGs, on the other hand, history-based policies, as we also use, are common for both agent and nature.
For a more elaborate discussion, see~\cite[Section 2.2]{DBLP:journals/corr/abs-2312-03618}. 
Recent work explores the connection between RMDPs and SGs in more depth~\cite{chatterjee2023solving}.

For RPOMDPs, the connection with POSGs has also been alluded to before.
\cite{DBLP:conf/icml/Osogami15} briefly mention zero-sum games in their proof of convexity of the value function.
\cite{DBLP:journals/jet/Saghafian18} draws a link to nonzero-sum games, as they assume non-adversarial behavior for nature.
\cite{Rasouli2018RobustPO} states a correspondence between the perfect Bayesian equilibrium in a zero-sum and the optimal value and policies in their RPOMDPs.
Finally, \cite{DBLP:journals/siamjo/NakaoJS21} reasons about their RPOMDPs via games as well, but they assume the agent can observe nature's earlier choices.

\section{Conclusion}

This paper provides a semantic study of RPOMDPs, \ie, the extension of RMDPs to the partially observable setting. 
We demonstrate that semantic choices that are irrelevant on RMDPs are important in RPOMDPs. We concretely provide semantics expressed as partially observable stochastic games and use this to derive novel results about the existence of Nash equilibria. 
Finally, we categorize algorithms from the literature based on our semantic framework. 
For future work, we aim to adapt solution methods for POSGs, like~\cite{Springer:HSVI}, to solve RPOMDPs.
We also plan to investigate the existence of a Nash equilibrium in the infinite horizon case.

\clearpage
\newpage
\section*{Acknowledgements}
We would like to thank the anonymous reviewers for their valuable feedback.
This research has been partially funded by the NWO grant OCENW.KLEIN.187, the NWO Veni grant 222.147 (ProMiSe), and the ERC Starting Grant 101077178 (DEUCE).

\bibliographystyle{named}
\bibliography{biblio}

\begin{thebibliography}{}

\bibitem[\protect\citeauthoryear{Behzadian \bgroup \em et al.\egroup
  }{2021}]{DBLP:conf/nips/BehzadianPH21}
Bahram Behzadian, Marek Petrik, and Chin~Pang Ho.
\newblock Fast algorithms for {$L_\infty$}-constrained s-rectangular robust
  {MDPs}.
\newblock In {\em NeurIPS}, pages 25982--25992, 2021.

\bibitem[\protect\citeauthoryear{Bovy \bgroup \em et al.\egroup
  }{2024}]{bovy2024imprecise}
Eline~M. Bovy, Marnix Suilen, Sebastian Junges, and Nils Jansen.
\newblock Imprecise probabilities meet partial observability: Game semantics
  for robust {POMDP}s.
\newblock {\em CoRR}, abs/2405.04941, 2024.

\bibitem[\protect\citeauthoryear{Bovy}{2023}]{bovy2023thesis}
Eline~M. Bovy.
\newblock {\em The Underlying Belief Model of Uncertain Partially Observable
  {M}arkov Decision Processes}.
\newblock Master thesis, Radboud {U}niversity, 2023.

\bibitem[\protect\citeauthoryear{Chamie and
  Mostafa}{2018}]{DBLP:conf/cdc/ChamieM18}
Mahmoud~El Chamie and Hala Mostafa.
\newblock Robust action selection in partially observable {M}arkov decision
  processes with model uncertainty.
\newblock In {\em {CDC}}, pages 5586--5591. {IEEE}, 2018.

\bibitem[\protect\citeauthoryear{Chatterjee \bgroup \em et al.\egroup
  }{2016}]{DBLP:journals/ai/ChatterjeeCGK16}
Krishnendu Chatterjee, Martin Chmelik, Raghav Gupta, and Ayush Kanodia.
\newblock Optimal cost almost-sure reachability in {POMDPs}.
\newblock {\em Artif. Intell.}, 234:26--48, 2016.

\bibitem[\protect\citeauthoryear{Chatterjee \bgroup \em et al.\egroup
  }{2023}]{chatterjee2023solving}
Krishnendu Chatterjee, Ehsan~Kafshdar Goharshady, Mehrdad Karrabi, Petr
  Novotn{\`y}, and {\DJ}or{\dj}e {\v{Z}}ikeli{\'c}.
\newblock Solving long-run average reward robust {MDPs} via stochastic games.
\newblock {\em CoRR}, abs/2312.13912, 2023.

\bibitem[\protect\citeauthoryear{Cubuktepe \bgroup \em et al.\egroup
  }{2021}]{DBLP:conf/aaai/Cubuktepe0JMST21}
Murat Cubuktepe, Nils Jansen, Sebastian Junges, Ahmadreza Marandi, Marnix
  Suilen, and Ufuk Topcu.
\newblock Robust finite-state controllers for uncertain {POMDPs}.
\newblock In {\em {AAAI}}, pages 11792--11800. {AAAI} Press, 2021.

\bibitem[\protect\citeauthoryear{Delage \bgroup \em et al.\egroup
  }{2023}]{Springer:HSVI}
Aurélien Delage, Olivier Buffet, Jilles~S. Dibangoye, and Abdallah Saffidine.
\newblock {HSVI} can solve zero-sum partially observable stochastic games.
\newblock {\em Dynamic Games and Applications}, 2023.

\bibitem[\protect\citeauthoryear{Fijalkow \bgroup \em et al.\egroup
  }{2023}]{DBLP:journals/corr/abs-2305-10546}
Nathana{\"{e}}l Fijalkow, Nathalie Bertrand, Patricia Bouyer{-}Decitre, Romain
  Brenguier, Arnaud Carayol, John Fearnley, Hugo Gimbert, Florian Horn, Rasmus
  Ibsen{-}Jensen, Nicolas Markey, Benjamin Monmege, Petr Novotn{\'{y}}, Mickael
  Randour, Ocan Sankur, Sylvain Schmitz, Olivier Serre, and Mateusz Skomra.
\newblock Games on graphs.
\newblock {\em CoRR}, abs/2305.10546, 2023.

\bibitem[\protect\citeauthoryear{{GeoGebra GmbH}}{2024}]{geogebra}
{GeoGebra GmbH}.
\newblock Geogebra (online), 2024.
\newblock Available at
  \href{https://www.geogebra.org}{https://www.geogebra.org}.

\bibitem[\protect\citeauthoryear{Gillette}{1957}]{gillette1957stochastic}
Dean Gillette.
\newblock Stochastic games with zero stop probabilities.
\newblock {\em Contributions to the Theory of Games}, 3:179--187, 1957.

\bibitem[\protect\citeauthoryear{Grand{-}Cl{\'{e}}ment \bgroup \em et
  al.\egroup }{2023}]{DBLP:journals/corr/abs-2312-03618}
Julien Grand{-}Cl{\'{e}}ment, Marek Petrik, and Nicolas Vieille.
\newblock Beyond discounted returns: Robust {M}arkov decision processes with
  average and blackwell optimality.
\newblock {\em CoRR}, abs/2312.03618, 2023.

\bibitem[\protect\citeauthoryear{Ho \bgroup \em et al.\egroup
  }{2018}]{DBLP:conf/icml/HoPW18}
Chin~Pang Ho, Marek Petrik, and Wolfram Wiesemann.
\newblock Fast bellman updates for robust {MDPs}.
\newblock In {\em {ICML}}, volume~80 of {\em Proceedings of Machine Learning
  Research}, pages 1984--1993. {PMLR}, 2018.

\bibitem[\protect\citeauthoryear{Ho \bgroup \em et al.\egroup
  }{2021}]{DBLP:journals/jmlr/HoPW21}
Chin~Pang Ho, Marek Petrik, and Wolfram Wiesemann.
\newblock Partial policy iteration for {$L_1$}-robust {M}arkov decision
  processes.
\newblock {\em J. Mach. Learn. Res.}, 22:275:1--275:46, 2021.

\bibitem[\protect\citeauthoryear{Iyengar}{2005}]{DBLP:journals/mor/Iyengar05}
Garud~N. Iyengar.
\newblock Robust dynamic programming.
\newblock {\em Math. Oper. Res.}, 30(2):257--280, 2005.

\bibitem[\protect\citeauthoryear{Jaksch \bgroup \em et al.\egroup
  }{2010}]{DBLP:journals/jmlr/JakschOA10}
Thomas Jaksch, Ronald Ortner, and Peter Auer.
\newblock Near-optimal regret bounds for reinforcement learning.
\newblock {\em J. Mach. Learn. Res.}, 11:1563--1600, 2010.

\bibitem[\protect\citeauthoryear{Jansen \bgroup \em et al.\egroup
  }{2022}]{DBLP:conf/birthday/0001JK22}
Nils Jansen, Sebastian Junges, and Joost{-}Pieter Katoen.
\newblock Parameter synthesis in {M}arkov models: {A} gentle survey.
\newblock In {\em Principles of Systems Design}, volume 13660 of {\em LNCS},
  pages 407--437. Springer, 2022.

\bibitem[\protect\citeauthoryear{Kaelbling \bgroup \em et al.\egroup
  }{1998}]{DBLP:journals/ai/KaelblingLC98}
Leslie~Pack Kaelbling, Michael~L. Littman, and Anthony~R. Cassandra.
\newblock Planning and acting in partially observable stochastic domains.
\newblock {\em Artif. Intell.}, 101(1-2):99--134, 1998.

\bibitem[\protect\citeauthoryear{Kuhn}{1953}]{kuhn1953extensive}
Harold~W Kuhn.
\newblock Extensive games and the problem of information.
\newblock {\em Contributions to the Theory of Games}, 2(28):193--216, 1953.

\bibitem[\protect\citeauthoryear{Kwiatkowska \bgroup \em et al.\egroup
  }{2022}]{DBLP:conf/mfcs/KwiatkowskaNPSY22}
Marta Kwiatkowska, Gethin Norman, David Parker, Gabriel Santos, and Rui Yan.
\newblock Probabilistic model checking for strategic equilibria-based decision
  making: Advances and challenges (invited talk).
\newblock In {\em {MFCS}}, volume 241 of {\em LIPIcs}, pages 4:1--4:22. Schloss
  Dagstuhl - Leibniz-Zentrum f{\"{u}}r Informatik, 2022.

\bibitem[\protect\citeauthoryear{Lopez and
  Vercher}{1986}]{Convex_semi-infinite_games}
M.A. Lopez and D.E. Vercher.
\newblock Convex semi-infinite games.
\newblock {\em Journal of optimization theory and applications},
  50(2):289--312, 1986.

\bibitem[\protect\citeauthoryear{Moos \bgroup \em et al.\egroup
  }{2022}]{DBLP:journals/make/MoosHASCP22}
Janosch Moos, Kay Hansel, Hany Abdulsamad, Svenja Stark, Debora Clever, and Jan
  Peters.
\newblock Robust reinforcement learning: {A} review of foundations and recent
  advances.
\newblock {\em Mach. Learn. Knowl. Extr.}, 4(1):276--315, 2022.

\bibitem[\protect\citeauthoryear{Nakao \bgroup \em et al.\egroup
  }{2021}]{DBLP:journals/siamjo/NakaoJS21}
Hideaki Nakao, Ruiwei Jiang, and Siqian Shen.
\newblock Distributionally robust partially observable {M}arkov decision
  process with moment-based ambiguity.
\newblock {\em {SIAM} J. Optim.}, 31(1):461--488, 2021.

\bibitem[\protect\citeauthoryear{Nilim and
  Ghaoui}{2005}]{DBLP:journals/ior/NilimG05}
Arnab Nilim and Laurent~El Ghaoui.
\newblock Robust control of {M}arkov decision processes with uncertain
  transition matrices.
\newblock {\em Oper. Res.}, 53(5):780--798, 2005.

\bibitem[\protect\citeauthoryear{Osogami}{2015}]{DBLP:conf/icml/Osogami15}
Takayuki Osogami.
\newblock Robust partially observable {M}arkov decision process.
\newblock In {\em {ICML}}, volume~37 of {\em {JMLR} Workshop and Conference
  Proceedings}, pages 106--115. JMLR.org, 2015.

\bibitem[\protect\citeauthoryear{Peters}{2015}]{Book_game_theory_multi-leveled}
Hans Peters.
\newblock {\em Game Theory: A Multi-Leveled Approach}.
\newblock Springer Texts in Business and Economics. Springer, second edition,
  2015.

\bibitem[\protect\citeauthoryear{Petrik and
  Subramanian}{2014}]{DBLP:conf/nips/PetrikS14}
Marek Petrik and Dharmashankar Subramanian.
\newblock {RAAM:} the benefits of robustness in approximating aggregated {MDPs}
  in reinforcement learning.
\newblock In {\em {NIPS}}, pages 1979--1987, 2014.

\bibitem[\protect\citeauthoryear{Puterman}{1994}]{DBLP:books/wi/Puterman94}
Martin~L. Puterman.
\newblock {\em {M}arkov Decision Processes: Discrete Stochastic Dynamic
  Programming}.
\newblock Wiley Series in Probability and Statistics. Wiley, 1994.

\bibitem[\protect\citeauthoryear{Rasouli and
  Saghafian}{2018}]{Rasouli2018RobustPO}
Mohammad Rasouli and Soroush Saghafian.
\newblock Robust partially observable {M}arkov decision processes.
\newblock {\em HKS Working Paper}, RWP18-027, 2018.

\bibitem[\protect\citeauthoryear{Saghafian}{2018}]{DBLP:journals/jet/Saghafian18}
Soroush Saghafian.
\newblock Ambiguous partially observable {M}arkov decision processes:
  Structural results and applications.
\newblock {\em J. Econ. Theory}, 178:1--35, 2018.

\bibitem[\protect\citeauthoryear{Shapley}{1953}]{shapley1953stochastic}
Lloyd~S Shapley.
\newblock Stochastic games.
\newblock {\em Proceedings of the national academy of sciences},
  39(10):1095--1100, 1953.

\bibitem[\protect\citeauthoryear{Spaan}{2012}]{DBLP:books/sp/12/Spaan12}
Matthijs T.~J. Spaan.
\newblock Partially observable {M}arkov decision processes.
\newblock In {\em Reinforcement Learning}, volume~12 of {\em Adaptation,
  Learning, and Optimization}, pages 387--414. Springer, 2012.

\bibitem[\protect\citeauthoryear{Suilen \bgroup \em et al.\egroup
  }{2020}]{DBLP:conf/ijcai/Suilen0CT20}
Marnix Suilen, Nils Jansen, Murat Cubuktepe, and Ufuk Topcu.
\newblock Robust policy synthesis for uncertain {{POMDPs}} via convex
  optimization.
\newblock In {\em {IJCAI}}, pages 4113--4120. ijcai.org, 2020.

\bibitem[\protect\citeauthoryear{Suilen \bgroup \em et al.\egroup
  }{2022}]{DBLP:conf/nips/SuilenS0022}
Marnix Suilen, Thiago~D. Sim{\~{a}}o, David Parker, and Nils Jansen.
\newblock Robust anytime learning of {M}arkov decision processes.
\newblock In {\em NeurIPS}, 2022.

\bibitem[\protect\citeauthoryear{Wang \bgroup \em et al.\egroup
  }{2023}]{DBLP:conf/icml/WangHP23}
Qiuhao Wang, Chin~Pang Ho, and Marek Petrik.
\newblock Policy gradient in robust {MDPs} with global convergence guarantee.
\newblock In {\em {ICML}}, volume 202 of {\em Proceedings of Machine Learning
  Research}, pages 35763--35797. {PMLR}, 2023.

\bibitem[\protect\citeauthoryear{Wiesemann \bgroup \em et al.\egroup
  }{2013}]{DBLP:journals/mor/WiesemannKR13}
Wolfram Wiesemann, Daniel Kuhn, and Ber{\c{c}} Rustem.
\newblock Robust {M}arkov decision processes.
\newblock {\em Math. Oper. Res.}, 38(1):153--183, 2013.

\bibitem[\protect\citeauthoryear{Xu and Mannor}{2010}]{DBLP:conf/nips/XuM10}
Huan Xu and Shie Mannor.
\newblock Distributionally robust {M}arkov decision processes.
\newblock In {\em {NIPS}}, pages 2505--2513. Curran Associates, Inc., 2010.

\bibitem[\protect\citeauthoryear{Zalinescu}{2002}]{Book_convex_analysis_general_vector_spaces}
Constantin Zalinescu.
\newblock {\em Convex analysis in general vector spaces}.
\newblock World scientific, 2002.

\end{thebibliography}

\clearpage
\newpage
\onecolumn
\appendix
{
\parindent0pt
\setlength{\parskip}{3pt}
\section{Appendix Overview \& Additional Preliminaries}\label{app:notation_and_prelim}
These appendices contain additional results and proofs for the claims made in the main text.
In particular, the appendices are structured as follows:
\begin{itemize}
    \item \Cref{app:notation_and_prelim} contains an overview of key notation used and additional preliminaries (mostly around paths and histories) used in subsequent appendices.
    \item  \Cref{app:det.obs} show how an arbitrary RPOMDP, with uncertainty in both transitions and observations, can be transformed into an RPOMDP with deterministic state-based observations as we use throughout the main text inline with \Cref{def:rpomdp}.
    \item \Cref{app:observation:stickiness} defines observation-based stickiness and shows its workings in an example.
    \item \Cref{app:assumptions_matter} details the result from \Cref{thm:uncertainty:matters}. 
    That is, it shows that different forms of stickiness or a different order of play can lead to different optimal values in RPOMDPs.
    \item  \Cref{app:nature_first} discussed the required changes to the constructions from the main paper when considering RPOMDPs with nature first semantics.
    \item \Cref{app:value:function:proofs} contains the proof of \Cref{thm:equivalent:values} and the propositions it builds on.
    \item \Cref{app:nashEquilibrium} contains the proof of \Cref{thm:exist:nash}.
\end{itemize}

\subsection{Glossary of Key Notation}

\begin{table*}[ht!]
    \centering
    \begin{tabular}{ll}
    \toprule
        Notation & Description  \\
        \midrule
        $K \in \NN$ & Finite horizon bound  \\
        $\gamma$ & Discount factor \\
        $\phi \in \{\text{fh,dih}\}$ & Objective, either finite-horizon (fh) or discounted infinite horizon (dih) 
        \\
        $v \in U$ & Variable $v$ in the set of variables $U$ \\
        $\bm{U} \subseteq (U \to \RR)$ & Uncertainty set, defined as a set of admissible variable assignments $u \in \bm{U}$ \\
        $\bm{T}(u) \colon S \times A \to \dist{S}$ & The uncertain transition function $\bm{T}$ instantiated by variable assignment $u$\\
        $\bm{U}^{\pset}$ & Set of partial variable assignments $U \pto \RR$ \\
        $u^\perp$ & The completely undefined variable assignment \\
        $\bm{U}^\agrees(u^\pset)$ & Set of variable assignments that agree with partial assignment $u^\pset$ \\  
        $|\mypath|$ & Horizon length of a path $\mypath$ \\
        $|h|$ & Horizon length of a history $h$ \\
        $\concat$ & Concatenation of two tuples or paths \\
        $\mypath_{0:k}$ & Prefix of horizon length $k$ of path $\mypath$\\
        $H_t$ & The subset of histories of length $t$ \\
        $H_{0:t}$ & The subset of histories of length $0$ to $t$ \\
        $\pi \in \Pi$ & Stochastic agent policies \\
        $\theta \in \Theta$ & Stochastic nature policies \\
        $\pi^\detr \in \Pi^\detr$ & Deterministic agent policies \\
        $\theta^\detr \in \Theta^\detr$ & Deterministic nature policies \\
        $\pi^\mix \in \Pi^\mix$ & Mixed agent policies \\
        $\theta^\mix \in \Theta^\mix$ & Mixed nature policies \\
        $\pi_t, \theta_t$ & A policy defined on histories of length $t$ (can be part of a larger policy $\pi/\theta$) \\
        $\pi_{0:t}, \theta_{0:t}$ & A policy defined on histories of length $0$ to $t$ (can be part of a larger policy $\pi/\theta$)\\
        $\Pi_t, \Theta_t$ & The set of policies defined on histories of length $t$\\
        $\Pi_{0:t}, \Theta_{0:t}$ & The set of policies defined on histories of length $0$ to $t$\\
        $\agent$, $\nature$ & Agent and nature \\
        $z_\priv$, $z_\publ$ & Private and public observations \\ 
        $\vec{x} = \tup{x_0,x_1,\dots,x_{n-1}} \in \RR^n$ & Vector \\
        \bottomrule
    \end{tabular}
    \caption{Glossary of key notation.}
    \label{tab:notation}
\end{table*}

\clearpage
\newpage
\subsection{Additional Preliminaries}\label{app:appendix_prelim}
This appendix contains additional definitions and concepts used throughout the rest of the appendices.

\paragraph{Rectangularity.}
Rectangularity concerns dependencies between transitions in the model.
If all transitions originating from different states are independent, we call the model and the uncertainty set $s$-rectangular.
An $s$-rectangular uncertainty set $\bm{U}$ can be rewritten as the Cartesian product of smaller, $s$ related uncertainty sets:
\[
\bm{U} = \bigtimes_{s\in S} \bm{U}^s.
\]
Similarly, if all transitions following different actions are independent, we call the model and the uncertainty set $a$-rectangular~\cite{DBLP:journals/mor/WiesemannKR13}.
An $a$-rectangular uncertainty set $\bm{U}$ can be rewritten as the Cartesian product of smaller, $a$ related uncertainty sets:
\[\bm{U} = \bigtimes_{a\in A} \bm{U}^a.\]
Combining these two, so if all transitions originating from different states and following from different actions are independent, we call the model and the uncertainty set $s,a$-rectangular~\cite{DBLP:journals/mor/Iyengar05}.
An $s,a$-rectangular uncertainty set $\bm{U}$ can be rewritten as the Cartesian product of smaller, $s$ and $a$ related uncertainty sets:
\[\bm{U} = \bigtimes_{\substack{s\in S,\\a\in A}} \bm{U}^{s,a}.\]
If we do not have any known independencies, we call the model and the uncertainty set non-rectangular.

\paragraph{Belief.}
A belief is a distribution over the set of states $\dist{S}$ representing the probability of being in a state given the \aoh{}.
We compute a belief given an \aoh{} by repeatedly applying the belief update \cite{DBLP:journals/ai/KaelblingLC98}, starting from the initial belief based on the initial state.
Note that we adjusted the belief update to work with three observation functions.

\begin{definition}[Belief update]
    Given a belief $b$, an agent action $a$, a nature action $u$, and observations $z^\agent_\priv,z^\nature_\priv,z_\publ$, we compute the successor belief $b'$ for each state $s' \in S$ as follows:
    \begin{align*}
    b'(s') &= \pr{s'}{b,a,u,z^\agent_\priv,z^\nature_\priv,z_\publ}\\
    &= \frac{O^\agent_\priv(s',z^\agent_\priv)O^\nature_\priv(s',z^\nature_\priv)O_\publ(s',z_\publ)\sum_{s\in S}b(s)\cdot \bm{T}(u)(s,a,s')}{\sum_{s'\in S}O^\agent_\priv(s',z^\agent_\priv)O^\nature_\priv(s',z^\nature_\priv)O_\publ(s',z_\publ)\sum_{s\in S}b(s)\cdot \bm{T}(u)(s,a,s')}.\\
    \end{align*}
\end{definition}
We use the belief update in our proof of the existence of a Nash equilibrium in \Cref{app:nashEquilibrium}.

\paragraph{Dirac distribution.}
A distribution is called \emph{Dirac} if it assigns probability one to precisely one element. %
We use Dirac distributions in our proof of the existence of a Nash equilibrium in \Cref{app:nashEquilibrium}.

\paragraph{Graph-preserving}
Given an RPOMDP $M = \tup{S, A, \bm{T}, R, \Zagent, \Znature, \Zpub, O^\agent_\priv, O^\nature_\priv, O_\publ,\sticky,\player}$ with uncertainty set $\bm{U}$, the uncertainty set is called graph-preserving if all variable assignments preserve the underlying structure of the RPOMDP.
In other words, if a transition is possible given one variable assignment, it is possible given all variable assignment:
\[
    \forall s,s' \in S, \forall a \in A. \left( \exists u\in \bm{U}.\; \bm{T}(u)(s,a,s') = 0 \implies \forall u\in \bm{U}.\; \bm{T}(u)(s,a,s') = 0\right).
\]

\paragraph{Convex set}
A subset $X$ of a Euclidean space is convex if given any two elements in the set the line drawn between those two elements is entirely contained in the set. Given 2 elements $x_0, x_1 \in X$ and scalar $\alpha \in [0,1]$ we have that:
\[
    \alpha x_0 + (1-\alpha) x_1 \in X.
\]

As a result, a convex set has the property that any convex combination of its elements is again contained in that set. Given $k$ elements $x_0, x_1, \dots, x_{k-1} \in X$ and $k$ non-negative scalars $\lambda_0, \lambda_1, \dots, \lambda_{k-1} \in [0,1]$ such that $\sum_{i=0}^{k-1} \lambda_i = 1$, we have that:
\[
    \sum_{i=0}^k \lambda_i x_i \in X.
\]

\paragraph{Joint histories.}
The \emph{joint history} combines the agent and nature histories in a single sequence:
\[H^M\subseteq (\Zagent \times \Znature \times \Zpub \times A \times \bm{U} )^* \times \Zagent \times \Znature \times \Zpub,\]
\[H^G\subseteq (\mathcal{Z}^\agent \times \mathcal{Z}^\nature \times \mathcal{A}^\agent \times \mathcal{Z}^\agent \times \mathcal{Z}^\nature \times \mathcal{A}^\nature)^* \times \mathcal{Z}^\agent \times \mathcal{Z}^\nature.\]
Neither player can observe the joint history.
Given a joint history $h \in H^M$ of $h\in H^G$, we use the superscripts $\agent$ and $\nature$ to indicate the agent and nature observable parts of the history $h$.
So we get $h^\agent \in H^{\agent,M}$ (or $\in H^{\agent,G}$) and $h^\nature \in H^{\nature,M}$ (or $\in H^{\nature,G}$).
We use joint histories in \Cref{app:value:function:proofs,app:nashEquilibrium}.

\paragraph{Paths to histories.}
The following six functions map paths to histories in the RPOMDP and POSGs.
The sets of histories in the RPOMDP and POSGs are constructed by applying these mappings to the sets of paths.
$\Paths^\suff$ indicates the set of all path segments, see \Cref{app:value:function:proofs}.
\begin{definition}[Paths to joint histories in RPOMDPs]\label{app:def:paths_to_joint_histories_RPOMDP}\strut\\
    Let $O^{M,\suff}: \Paths^{M,\suff} \to H^{M,\suff}$ defined by:
    \begin{align*}
    O^{M,\suff}(\tup{s}) &= \tup{O^\agent_\priv(s),O^\nature_\priv(s),O_\publ(s)}.\\
    O^{M,\suff}(\tup{s,a,u}) &= \tup{O^\agent_\priv(s),O^\nature_\priv(s),O_\publ(s),a,u}.\\
    O^{M,\suff}(\tup{s,a,u}\concat{\mypath^M}') &= O^{M,\suff}(\tup{s,a,u})\concat O^{M,\suff}({\mypath^M}').
    \intertext{Let $O^{M}: \Paths^{M} \to H^{M}$ defined by:}
    O^M(\tup{s}) &= \tup{O^\agent_\priv(s),O^\nature_\priv(s),O_\publ(s)}.\\
    O^M(\tup{s,a,u}) &= \tup{O^\agent_\priv(s),O^\nature_\priv(s),O_\publ(s),a,u}.\\
    O^M(\tup{s,a,u}\concat{\mypath^M}') &= O^M(\tup{s,a,u})\concat O^{M,\suff}({\mypath^M}').
    \end{align*}
\end{definition}

\begin{definition}[Paths to agent histories in RPOMDPs]\label{app:def:paths_to_agent_histories_RPOMDP}\strut\\
    Let $O^{\agent, M,\suff}: \Paths^{M,\suff} \to H^{\agent, M,\suff}$ defined by:
    \begin{align*} 
    O^{\agent, M,\suff}(\tup{s}) &= \tup{O^\agent_\priv(s),O_\publ(s)}.\\
    O^{\agent, M,\suff}(\tup{s,a,u}) &= \tup{O^\agent_\priv(s),O_\publ(s),a}.\\
    O^{\agent, M,\suff}(\tup{s,a,u}\concat{\mypath^M}') &= O^{\agent, M,\suff}(\tup{s,a,u})\concat O^{\agent, M,\suff}({\mypath^M}').
    \intertext{Let $O^{\agent, M}: \Paths^{M} \to H^{\agent, M}$ defined by:}
    O^{\agent, M}(\tup{s}) &= \tup{O^\agent_\priv(s),O^\nature_\priv(s),O_\publ(s)}.\\
    O^{\agent, M}(\tup{s,a,u}) &= \tup{O^\agent_\priv(s),O_\publ(s),a}.\\
    O^{\agent, M}(\tup{s,a,u}\concat{\mypath^M}') &= O^{\agent, M}(\tup{s,a,u})\concat O^{\agent, M,\suff}({\mypath^M}').
    \end{align*}
\end{definition}

\begin{definition}[Paths to nature histories in RPOMDPs]\label{app:def:paths_to_nature_histories_RPOMDP}\strut\\
    Let $O^{\nature, M,\suff}: \Paths^{M,\suff} \to H^{\nature, M,\suff}$ defined by:
    \begin{align*}
    O^{\nature, M,\suff}(\tup{s}) &= \tup{O^\nature_\priv(s),O_\publ(s)}.\\
    O^{\nature, M,\suff}(\tup{s,a,u}) &= \tup{O^\nature_\priv(s),O_\publ(s),a,u}.\\
    O^{\nature, M,\suff}(\tup{s,a,u}\concat{\mypath^M}') &= O^{\nature, M,\suff}(\tup{s,a,u})\concat O^{M,\suff}({\mypath^M}').
    \intertext{Let $O^{\nature, M}: \Paths^{M} \to H^{\nature, M}$ defined by:}
    O^{\nature, M}(\tup{s}) &= \tup{O^\nature_\priv(s),O_\publ(s)}.\\
    O^{\nature, M}(\tup{s,a,u}) &= \tup{O^\nature_\priv(s),O_\publ(s),a,u}.\\
    O^{\nature, M}(\tup{s,a,u}\concat{\mypath^M}') &= O^{\nature, M}(\tup{s,a,u})\concat O^{\nature, M,\suff}({\mypath^M}').
    \end{align*}
\end{definition}

\begin{definition}[Paths to joint histories in POSGs]\label{app:def:paths_to_joint_histories_POSG}\strut\\
    Similarly, let $O^{G,\suff}: \Paths^{G,\suff} \to H^{G,\suff}$ defined by:
    \begin{align*}
    O^{G,\suff}(\tup{\tup{s,u^\pset}}) &= \tup{\tup{O^\agent_\priv(s),O_\publ(s)},\tup{O^\nature_\priv(s),O_\publ(s),\bot}}.\\
    O^{G,\suff}(\tup{\tup{s,u^\pset},a,\tup{s,u^\pset,a},u}) &= \tup{\tup{O^\agent_\priv(s),O_\publ(s)},\tup{O^\nature_\priv(s),O_\publ(s),\bot},a,\tup{O^\agent_\priv(s),O_\publ(s)},\tup{O^\nature_\priv(s),O_\publ(s),a},u}.\\
    O^{G,\suff}(\tup{\tup{s,u^\pset},a,\tup{s,u^\pset,a},u}\concat{\mypath^G}') &= O^{G,\suff}(\tup{\tup{s,u^\pset},a,\tup{s,u^\pset,a},u})\concat O^{G,\suff}({\mypath^G}').
    \intertext{Let $O^{G}: \Paths^{G} \to H^{G}$ defined by:}
    O^{G}(\tup{\tup{s,u^\pset}}) &= \tup{\tup{O^\agent_\priv(s),O_\publ(s)},\tup{O^\nature_\priv(s),O_\publ(s),\bot}}.\\
    O^{G}(\tup{\tup{s,u^\pset},a,\tup{s,u^\pset,a},u}) &= \tup{\tup{O^\agent_\priv(s),O_\publ(s)},\tup{O^\nature_\priv(s),O_\publ(s),\bot},a,\tup{O^\agent_\priv(s),O_\publ(s)},\tup{O^\nature_\priv(s),O_\publ(s),a},u}.\\
    O^{G}(\tup{\tup{s,u^\pset},a,\tup{s,u^\pset,a},u}\concat{\mypath^G}') &= O^{G}(\tup{\tup{s,u^\pset},a,\tup{s,u^\pset,a},u})\concat O^{G,\suff}({\mypath^G}').
    \end{align*}
\end{definition}

\begin{definition}[Paths to agent histories in POSGs]\label{app:def:paths_to_agent_histories_POSG}\strut\\
    Similarly, let $O^{\agent, G,\suff}: \Paths^{G,\suff} \to H^{\agent, G,\suff}$ defined by:
    \begin{align*}
    O^{\agent, G,\suff}(\tup{\tup{s,u^\pset}}) &= \tup{\tup{O^\agent_\priv(s),O_\publ(s)}}.\\
    O^{\agent, G,\suff}(\tup{\tup{s,u^\pset},a,\tup{s,u^\pset,a},u}) &= \tup{\tup{O^\agent_\priv(s),O_\publ(s)},a,\tup{O^\agent_\priv(s),O_\publ(s)}}.\\
    O^{\agent, G,\suff}(\tup{\tup{s,u^\pset},a,\tup{s,u^\pset,a},u}\concat{\mypath^G}') &= O^{\agent, G,\suff}(\tup{\tup{s,u^\pset},a,\tup{s,u^\pset,a},u})\concat O^{\agent, G,\suff}({\mypath^G}').
    \intertext{Let $O^{\agent, G}: \Paths^{G} \to H^{\agent, G}$ defined by:}
    O^{\agent, G}(\tup{\tup{s,u^\pset}}) &= \tup{\tup{O^\agent_\priv(s),O_\publ(s)}}.\\
    O^{\agent, G}(\tup{\tup{s,u^\pset},a,\tup{s,u^\pset,a},u}) &= \tup{\tup{O^\agent_\priv(s),O_\publ(s)},a,\tup{O^\agent_\priv(s),O_\publ(s)}}.\\
    O^{\agent, G}(\tup{\tup{s,u^\pset},a,\tup{s,u^\pset,a},u}\concat{\mypath^G}') &= O^{\agent, G}(\tup{\tup{s,u^\pset},a,\tup{s,u^\pset,a},u})\concat O^{\agent, G,\suff}({\mypath^G}').
    \end{align*}
\end{definition}

\begin{definition}[Paths to nature histories in POSGs]\label{app:def:paths_to_nature_histories_POSG}\strut\\
    Similarly, let $O^{\nature, G,\suff}: \Paths^{G,\suff} \to H^{\nature, G,\suff}$ defined by:
    \begin{align*}
    O^{\nature, G,\suff}(\tup{\tup{s,u^\pset}}) &= \tup{\tup{O^\nature_\priv(s),O_\publ(s),\bot}}.\\
    O^{\nature, G,\suff}(\tup{\tup{s,u^\pset},a,\tup{s,u^\pset,a},u}) &= \tup{\tup{O^\nature_\priv(s),O_\publ(s),\bot},\tup{O^\nature_\priv(s),O_\publ(s),a},u}.\\
    O^{\nature, G,\suff}(\tup{\tup{s,u^\pset},a,\tup{s,u^\pset,a},u}\concat{\mypath^G}') &= O^{\nature, G,\suff}(\tup{\tup{s,u^\pset},a,\tup{s,u^\pset,a},u})\concat O^{\nature, G,\suff}({\mypath^G}').
    \intertext{Let $O^{\nature, G}: \Paths^{G} \to H^{\nature, G}$ defined by:}
    O^{\nature, G}(\tup{\tup{s,u^\pset}}) &= \tup{\tup{O^\nature_\priv(s),O_\publ(s),\bot}}.\\
    O^{\nature, G}(\tup{\tup{s,u^\pset},a,\tup{s,u^\pset,a},u}) &= \tup{\tup{O^\nature_\priv(s),O_\publ(s),\bot},\tup{O^\nature_\priv(s),O_\publ(s),a},u}.\\
    O^{\nature, G}(\tup{\tup{s,u^\pset},a,\tup{s,u^\pset,a},u}\concat{\mypath^G}') &= O^{\nature, G}(\tup{\tup{s,u^\pset},a,\tup{s,u^\pset,a},u})\concat O^{\nature, G,\suff}({\mypath^G}').
    \end{align*}
\end{definition}

\paragraph{Relevant histories.}\label{app:relevant_histories}
Since our nature policies are restricted to finite probability distributions, we can generate a finite subset of all joint histories that possibly have a non-zero probability at time $t$ given a nature policy $\theta$ up to time $t$.
For simplicity, we use the RPOMDP histories.
In \Cref{app:value:function:proofs} we show that reasoning via POSG histories is equivalent.
\begin{definition}[Relevant joint history]
    Given a deterministic policy $\theta^\detr \in \Theta^\detr$, $\rel\colon \Theta^\detr \to \cP(H^M)$ gives the set of joint histories which the deterministic policy can reach.
    \[
        \rel(\theta^\detr) = \{O^M(\tup{s_I})\}\cup\{h \concat \tup{a,u,z^\agent_\priv,z^\nature_\priv,z_\publ} \in H^M \mid \theta^\detr(h^\nature, a) = u \land h \in \rel(\theta^\detr)\}.
    \]
    Where $h^\nature$ is the nature observable part of the joint history $h$.
    
    Given a mixed policy $\theta^\mix \in \Theta^\mix$, $\rel\colon \Theta^\mix \to \cP(H^M)$ gives the set of joint histories which the mixed policy can reach.
    This comes down to the histories that are relevant to one of the deterministic policies the mixed policy randomizes over.
    \[
        \rel(\theta^\mix) = \{h \in H^M \mid \exists \theta^\detr \in \Theta^\detr.\, \theta^\mix(\theta^\detr) > 0 \land h \in \rel(\theta^\detr)\}.
    \]

    Given a stochastic policy $\theta \in \Theta$, $\rel\colon \Theta \to \cP(H^M)$ gives the set of joint histories that the stochastic policy can reach.
    \begin{align*}
        \rel(\theta) &= \{O^M(\tup{s_I})\}\cup \{h\concat\tup{a,u,z^\agent_\priv,z^\nature_\priv,z_\publ} \in H^M \mid \theta(h^\nature,a)(u) > 0 \land h^\nature \in \rel(\theta)\}.
    \end{align*}
    Where $h^\nature$ is the nature observable part of the joint history $h$.
\end{definition}

This construction generalizes to relevant nature histories, indicated by $\rel^\nature$.
We use the sets of relevant histories in our proof of the existence of a Nash equilibrium in \Cref{app:nashEquilibrium}.

\paragraph{Policy types}
As introduced in \Cref{sec:preliminaries}, stochastic policies map histories to (finite) distributions over actions.
Mixed policies, on the other hand, are (finite) distributions over deterministic policies, which in turn map histories to actions deterministically.
For convenience, we repeat all types of policies we consider below.

For RPOMDPs:
\begin{equation*}
\begin{aligned}
    \text{Stochastic:}\\
    \text{Deterministic:}\\
    \text{Mixed:}
\end{aligned}
\qquad
\begin{aligned}[c]
    \pi \colon&H^{\agent,M} \to \dist{A},\\
    \pi^\detr \colon&H^{\agent,M} \to A,\\
    \pi^\mix \in&\;\dist{H^{\agent,M} \to A},
\end{aligned}
\qquad\qquad
\begin{aligned}[c]
    \theta \colon&H^{\nature,M} \times A \to \dist{\bm{U}},\\
    \theta^\detr \colon&H^{\nature,M} \times A \to \bm{U},\\
    \theta^\mix \in&\;\dist{H^{\nature,M} \times A \to \bm{U}}.
\end{aligned}
\end{equation*}
And for POSGs:
\begin{equation*}
\begin{aligned}
    \text{Stochastic:}\\
    \text{Deterministic:}\\
    \text{Mixed:}
\end{aligned}
\qquad
\begin{aligned}[c]
    \pi \colon&H^{\agent,G} \to \dist{\mathcal{A}^\agent},\\
    \pi^\detr \colon&H^{\agent,G} \to \mathcal{A}^\agent,\\
    \pi^\mix \in&\;\dist{H^{\agent,G} \to \mathcal{A}^\agent},
\end{aligned}
\qquad\qquad
\begin{aligned}[c]
    \theta \colon&H^{\nature,G} \times \mathcal{Z^\nature} \to \dist{\mathcal{A}^\nature},\\
    \theta^\detr \colon&H^{\nature,G} \times \mathcal{Z^\nature} \to \mathcal{A}^\nature,\\
    \theta^\mix \in&\;\dist{H^{\nature,G} \times \mathcal{Z^\nature} \to \mathcal{A}^\nature}.
\end{aligned}
\end{equation*}

We write $\Pi^M, \Pi^G, \Pi^{\detr,M}, \Pi^{\detr,G}, \Pi^{\mix,M},$ and $\Pi^{\mix,G}$ for the stochastic, deterministic, and mixed agent policies in RPOMDP and POSG models, respectively.
Similarly, we write $\Theta^M, \Theta^G, \Theta^{\detr,M}, \Theta^{\detr,G}, \Theta^{\mix,M},$ and $\Theta^{\mix,G}$ for the stochastic, deterministic, and mixed nature policies in RPOMDP and POSG models, respectively.
Note that we often omit the model indication superscript when it is clear from context in which type of model we operate or the results are equivalent.

The set of deterministic policies can be viewed as a subset of both the set of stochastic and the set of mixed policies using only Dirac distributions.
The value function of a mixed policy is computed as follows:
\begin{align*}
    V^{\pi^\mix,\theta^\mix}_\phi &= \sum_{\pi^\detr\in \Pi^\detr} \Bigl\{ \pi^\mix(\pi^\detr) \cdot \sum_{\theta^\detr\in \Theta^\detr} \theta^\mix(\theta^\detr) \cdot V^{\pi^\detr,\theta^\detr}_\phi\Bigr\}\\
    &= \sum_{\pi^\detr\in \Pi^\detr} \sum_{\theta^\detr\in \Theta^\detr} \Bigl\{ \pi^\mix(\pi^\detr) \cdot  \theta^\mix(\theta^\detr) \cdot V^{\pi^\detr,\theta^\detr}_\phi\Bigr\}.
\end{align*}

\subsection{Nature first}
Below, we define the same paths to histories mapping for the POSGs of nature first RPOMDPs.
The changes to the paths and histories are discussed in \Cref{app:nature_first}.
\begin{definition}[Paths to joint histories in POSGs]\strut\\
    Similarly, let $O^{G,\suff}: \Paths^{G,\suff} \to H^{G,\suff}$ defined by:
    \begin{align*}
    O^{G,\suff}(\tup{\tup{s,u^\pset},a'}) &= \tup{\tup{O^\nature_\priv(s),O_\publ(s),a'},\tup{O^\agent_\priv(s),O_\publ(s)}}.\\
    O^{G,\suff}(\tup{\tup{s,u^\pset,a'},u,\tup{s,u^\pset,u},a}) &= \tup{\tup{O^\nature_\priv(s),O_\publ(s),a'},\tup{O^\agent_\priv(s),O_\publ(s)},u,\tup{O^\nature_\priv(s),O_\publ(s),\bot},\tup{O^\agent_\priv(s),O_\publ(s)},a}.\\
    O^{G,\suff}(\tup{\tup{s,u^\pset,a'},u,\tup{s,u^\pset,u},a}\concat{\mypath^G}') &= O^{G,\suff}(\tup{\tup{s,u^\pset,a'},u,\tup{s,u^\pset,u},a})\concat O^{G,\suff}({\mypath^G}').
    \intertext{Let $O^{G}: \Paths^{G} \to H^{G}$ defined by:}
    O^{G}(\tup{\tup{s,u^\pset},\bot}) &= \tup{\tup{O^\nature_\priv(s),O_\publ(s),\bot},\tup{O^\agent_\priv(s),O_\publ(s)}}.\\
    O^{G}(\tup{\tup{s,u^\pset,\bot},u,\tup{s,u^\pset,u},a}) &= \tup{\tup{O^\nature_\priv(s),O_\publ(s),\bot},\tup{O^\agent_\priv(s),O_\publ(s)},u,\tup{O^\nature_\priv(s),O_\publ(s),\bot},\tup{O^\agent_\priv(s),O_\publ(s)},a}.\\
    O^{G}(\tup{\tup{s,u^\pset,a'},u,\tup{s,u^\pset,u},a}\concat{\mypath^G}') &= O^{G}(\tup{\tup{s,u^\pset,a'},u,\tup{s,u^\pset,u},a})\concat O^{G,\suff}({\mypath^G}').
    \end{align*}
\end{definition}

\begin{definition}[Paths to agent histories in POSGs]\strut\\
    Similarly, let $O^{\agent, G,\suff}: \Paths^{G,\suff} \to H^{\agent, G,\suff}$ defined by:
    \begin{align*}
    O^{G,\suff}(\tup{\tup{s,u^\pset},a'}) &= \tup{\tup{O^\agent_\priv(s),O_\publ(s)}}.\\
    O^{G,\suff}(\tup{\tup{s,u^\pset,a'},u,\tup{s,u^\pset,u},a}) &= \tup{\tup{O^\agent_\priv(s),O_\publ(s)},\tup{O^\agent_\priv(s),O_\publ(s)},a}.\\
    O^{G,\suff}(\tup{\tup{s,u^\pset,a'},u,\tup{s,u^\pset,u},a}\concat{\mypath^G}') &= O^{G,\suff}(\tup{\tup{s,u^\pset,a'},u,\tup{s,u^\pset,u},a})\concat O^{G,\suff}({\mypath^G}').
    \intertext{Let $O^{\agent, G}: \Paths^{G} \to H^{\agent, G}$ defined by:}
    O^{G}(\tup{\tup{s,u^\pset},\bot}) &= \tup{\tup{O^\agent_\priv(s),O_\publ(s)}}.\\
    O^{G}(\tup{\tup{s,u^\pset,\bot},u,\tup{s,u^\pset,u},a}) &= \tup{\tup{O^\agent_\priv(s),O_\publ(s)},\tup{O^\agent_\priv(s),O_\publ(s)},a}.\\
    O^{G}(\tup{\tup{s,u^\pset,a'},u,\tup{s,u^\pset,u},a}\concat{\mypath^G}') &= O^{G}(\tup{\tup{s,u^\pset,a'},u,\tup{s,u^\pset,u},a})\concat O^{G,\suff}({\mypath^G}').
    \end{align*}
\end{definition}

\begin{definition}[Paths to nature histories in POSGs]\strut\\
    Similarly, let $O^{\nature, G,\suff}: \Paths^{G,\suff} \to H^{\nature, G,\suff}$ defined by:
    \begin{align*}
    O^{G,\suff}(\tup{\tup{s,u^\pset},a'}) &= \tup{\tup{O^\nature_\priv(s),O_\publ(s),a'}}.\\
    O^{G,\suff}(\tup{\tup{s,u^\pset,a'},u,\tup{s,u^\pset,u},a}) &= \tup{\tup{O^\nature_\priv(s),O_\publ(s),a'},u,\tup{O^\nature_\priv(s),O_\publ(s),\bot}}.\\
    O^{G,\suff}(\tup{\tup{s,u^\pset,a'},u,\tup{s,u^\pset,u},a}\concat{\mypath^G}') &= O^{G,\suff}(\tup{\tup{s,u^\pset,a'},u,\tup{s,u^\pset,u},a})\concat O^{G,\suff}({\mypath^G}').
    \intertext{Let $O^{\nature, G}: \Paths^{G} \to H^{\nature, G}$ defined by:}
    O^{G}(\tup{\tup{s,u^\pset},\bot}) &= \tup{\tup{O^\nature_\priv(s),O_\publ(s),\bot}}.\\
    O^{G}(\tup{\tup{s,u^\pset,\bot},u,\tup{s,u^\pset,u},a}) &= \tup{\tup{O^\nature_\priv(s),O_\publ(s),\bot},u,\tup{O^\nature_\priv(s),O_\publ(s),\bot}}.\\
    O^{G}(\tup{\tup{s,u^\pset,a'},u,\tup{s,u^\pset,u},a}\concat{\mypath^G}') &= O^{G}(\tup{\tup{s,u^\pset,a'},u,\tup{s,u^\pset,u},a})\concat O^{G,\suff}({\mypath^G}').
    \end{align*}
\end{definition}

Since nature policies no longer depend on the last played agent action in a nature first model, the set of relevant joint histories changes as follows:
\begin{definition}[Relevant joint history]
    Given a deterministic policy $\theta^\detr \in \Theta^\detr$, $\rel\colon \Theta^\detr \to \cP(H^M)$ gives the set of joint histories which the deterministic policy can reach.
    \[
        \rel(\theta^\detr) = \{O^M(\tup{s_I})\}\cup\{h \concat \tup{a,u,z^\agent_\priv,z^\nature_\priv,z_\publ} \in H^M \mid \theta^\detr(h^\nature) = u \land h \in \rel(\theta^\detr)\}.
    \]
    Where $h^\nature$ is the nature observable part of the joint history $h$.
    
    Given a mixed policy $\theta^\mix \in \Theta^\mix$, $\rel\colon \Theta^\mix \to \cP(H^M)$ gives the set of joint histories which the mixed policy can reach.
    This comes down to the histories that are relevant to one of the deterministic policies the mixed policy randomizes over.
    \[
        \rel(\theta^\mix) = \{h \in H^M \mid \exists \theta^\detr \in \Theta^\detr.\, \theta^\mix(\theta^\detr) > 0 \land h \in \rel(\theta^\detr)\}.
    \]

    Given a stochastic policy $\theta \in \Theta$, $\rel\colon \Theta \to \cP(H^M)$ gives the set of joint histories that the stochastic policy can reach.
    \begin{align*}
        \rel(\theta) &= \{O^M(\tup{s_I})\}\cup \{h\concat\tup{a,u,z^\agent_\priv,z^\nature_\priv,z_\publ} \in H^M \mid \theta(h^\nature)(u) > 0 \land h^\nature \in \rel(\theta)\}.
    \end{align*}
    Where $h^\nature$ is the nature observable part of the joint history $h$.
\end{definition}

This construction, again, generalizes to relevant nature histories, indicated by $\rel^\nature$.
The sets of relevant histories for nature first RPOMDPs are needed to adjust the proof of the existence of a Nash equilibrium in \Cref{app:nashEquilibrium} to the nature first setting.

\clearpage
\newpage
\section{From General RPOMDP to RPOMDP With Deterministic Observations}\label{app:det.obs}
This appendix shows that our definition of RPOMDPs with deterministic state-based observations is non-restrictive.

Similar to in~\cite{DBLP:journals/ai/ChatterjeeCGK16}, we can transform an RPOMDP with a stochastic or uncertain observation function into an equivalent one with a deterministic observation function.
Let $M = (S,A,\bm{TO},R,\Zagent, \Znature, \Zpub)$ be an RPOMDP with an \emph{uncertain transition observation function} defined as $\bm{TO} \colon \bm{U} \to (S \times A \to \dist{S \times \Zagent \times \Znature \times \Zpub})$.
Note that this definition combines the transition and observation functions into one to allow for more intricate dependencies.
Multiple independent functions can replace the $\bm{TO}$ function.
This does not change the transformation below.

From the RPOMDP $M$, we construct a larger, equivalent, RPOMDP $M' = (S',A,\bm{T},R',\Zagent, \Znature, \Zpub, {O^\agent_\priv},{O^\nature_\priv}, {O_\publ})$ with $S' = S \times \Zagent \times \Znature \times \Zpub$, adjusting the reward function according to the new state space $R' \colon S' \times A \to \RR$.
We split the original transition observation function $\bm{TO}$ in a transition function $\bm{T} \colon \bm{U} \to (S' \times A \to \dist{S'})$ and three separate deterministic observation functions ${O^\agent_\priv}'\colon S' \to \Zagent$,${O^\nature_\priv}' \colon S' \to \Znature$, and ${O_\publ}' \colon S' \to \Zpub$.
The functions are then defined as follows:
\begin{itemize}
    \item $\bm{T}(u)(\tup{s,z^\agent_\priv,z^\nature_\priv,z_\publ}, a, \tup{s',{z^\agent_\priv}',{z^\nature_\priv}',{z_\publ}'}) = \bm{TO}(u)(s,a,s',{z^\agent_\priv}',{z^\nature_\priv}',{z_\publ}')$,
    \item $R'(\tup{s,z^\agent_\priv,z^\nature_\priv,z_\publ},a) = R(s,a)$,
    \item ${O^\agent_\priv}(\tup{s,z^\agent_\priv,z^\nature_\priv,z_\publ}) = z^\agent_\priv$,
    \item ${O^\nature_\priv}(\tup{s,z^\agent_\priv,z^\nature_\priv,z_\publ}) = z^\nature_\priv$,
    \item ${O_\publ}(\tup{s,z^\agent_\priv,z^\nature_\priv,z_\publ}) = z_\publ$.
\end{itemize}

The arbitrary RPOMDP $M$ has now been transformed into an RPOMDP $M'$ that satisfies our \Cref{def:rpomdp}, showing that our assumption of deterministic state-based observations is indeed non-restrictive.

\clearpage
\newpage
\section{Stickiness Examples}\label{app:observation:stickiness}
This appendix contains three examples of a stickiness function: zero, full, and observation-based stickiness.

\subsection{Zero and Full Stickiness}
As mentioned in \Cref{subsec:stickiness}, zero and full stickiness are the extremes of the spectrum of stickiness types where nature's choices never or always stick, respectively. We revisit the RMDP in \Cref{fig:small_example} in the main text and discuss the zero and full stickiness interpretations of the model.
\begin{reusefigure}[ht]{fig:small_example}
    \centering
    \resizebox{0.62\textwidth}{!}{
        \begin{tikzpicture}[state/.append style={shape = ellipse}, >=stealth,
    bobbel/.style={minimum size=1mm,inner sep=0pt,fill=black,circle},
    mynode/.style={rectangle,fill=white,anchor=center}]]
    \node[state] (s1) at (1,0) {$s_1$};
    \node[state] (s2) at (5,0) {$s_2$};
    \node[bobbel] (s1b) at (3,0.5) {};
    \node[bobbel] (s2b) at (3,-0.5) {};
    \draw[<-] (s1.west) -- +(-0.6,0);
    \draw (s1) -- (s1b);
    \draw (s2) -- (s2b);
    \draw (s1b) edge[->, bend left = 10] node[above right]{$p$} (s2);
    \draw (s1b) edge[->, bend right = 30] node[above left]{$1-p$} (s1.north);
    \draw (s2b) edge[->, bend left = 10] node[below right]{$q$} (s1);
    \draw (s2b) edge[->, bend right = 30] node[below right]{$1-q$} (s2.south);
    \node[anchor = west, text width=5cm] at (5.8,0.25) {\small $\bm{U}^1 = \left\{{p} \in [0.1,0.9], {q} \in [0.1,0.9]\right\}$};
    \node[anchor = west, text width=4cm] at (5.8,-0.25) {\small $\bm{U}^2 = \left\{{p} \in [0.1,0.4], {q = 2p}\right\}$};
\end{tikzpicture}
    }
    \caption{An example RMDP with two uncertainty sets.}
    \label{app:fig:small_example}
\end{reusefigure}

\begin{example}[Stickiness]\label{app:exam:RMDP:stickiness}
    Consider the RMDP in \Cref{app:fig:small_example} and uncertainty set $\bm{U}^1$.
    We interpret this RMDP as an RPOMDP with full observability for both players.
    Regardless of the stickiness, at the start of the game, nature has to choose a variable assignment $u \in \bm{U}^1$.
    Under full stickiness, the rest of the game is now determined, as nature can only choose the same values for $p$ and $q$ as in the initial variable assignment.
    If we assume zero stickiness, then at the next state and all future states, whether $s_1$ or $s_2$, nature can choose any new variable assignment $u' \in \bm{U}^1$.
\end{example}

\subsection{Observation-Based Stickiness}
We allow a stickiness function to depend on what nature observes, \ie, its private observations $\Znature$, public observations $\Zpub$, and the agent's actions $A$.
To define such stickiness functions, we denote for each variable $v\in U$ the state-action pairs it influences by $v^{s,a}\colon S \times A \to \{0,1\}$.

An example of an observation-based stickiness function is:
\begin{align*}
&\forall v\in U, z^\nature_\priv \in \Znature, z_\publ \in \Zpub, a\in A.\\
&\quad \sticky(v,z^\nature_\priv,z_\publ,a) = 1 \iff \exists s\in S. O^\nature_\priv(s) = z^\nature_\priv \land O_\publ(s) = z_\publ \land v^{s,a}(s,a) = 1.    
\end{align*}

Under observation-based stickiness, a variable only sticks if there is a possibility that it influenced the actual transition based on the last observations and actions.
This means that all variables that influence both a state with observations $z^\nature_\priv,z_\publ$ and the chosen action $a$ stick.
The intuition behind observation-based stickiness is that nature only needs to optimize for the transitions it might influence at that given point in time.
Note that this stickiness does not take the entire \aoh{} into account, so there can still be restrictions on variables that nature knows cannot influence the actual transition, as can be seen in the discussion of the right example below.

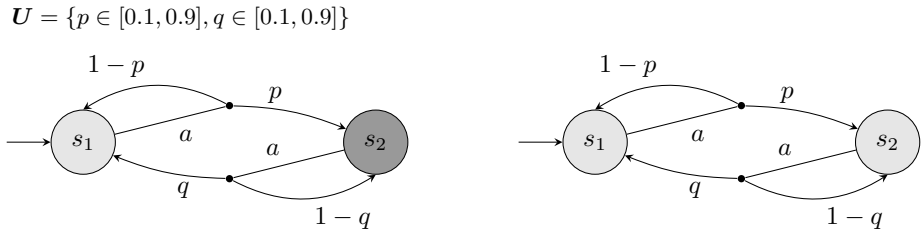
\begin{figure}[ht]
    \centering
    \resizebox{0.7\columnwidth}{!}{
    \begin{tikzpicture}[state/.append style={shape = ellipse}, >=stealth,
    bobbel/.style={minimum size=1mm,inner sep=0pt,fill=black,circle},
    mynode/.style={rectangle,fill=white,anchor=center}]]
    \node[state, fill=black!10] (l_s1) at (1,0) {$s_1$};
    \node[state, fill=black!40] (l_s2) at ($(l_s1) + (4,0)$) {$s_2$};
    \node[bobbel] (l_s1b) at ($(l_s1) + (2,0.5)$) {};
    \node[bobbel] (l_s2b) at ($(l_s1) + (2,-0.5)$) {};
    \draw[<-] (l_s1.west) -- +(-0.6,0);
    \draw (l_s1) -- node[below right]{$a$} (l_s1b);
    \draw (l_s2) -- node[above left]{$a$} (l_s2b);
    \draw (l_s1b) edge[->, bend left = 10] node[above left]{$p$} (l_s2);
    \draw (l_s1b) edge[->, bend right = 30] node[above left]{$1-p$} (l_s1.north);
    \draw (l_s2b) edge[->, bend left = 10] node[below right]{$q$} (l_s1);
    \draw (l_s2b) edge[->, bend right = 30] node[below right]{$1-q$} (l_s2.south);
    \node[anchor = west] at ($(l_s1) + (-1.1,1.7)$) {\small $\bm{U} = \left\{{p} \in [0.1,0.9], {q} \in [0.1,0.9]\right\}$};
    \node[state, fill=black!10] (r_s1) at ($(l_s1) + (7,0)$) {$s_1$};
    \node[state, fill=black!10] (r_s2) at ($(r_s1) + (4,0)$) {$s_2$};
    \node[bobbel] (r_s1b) at ($(r_s1) + (2,0.5)$) {};
    \node[bobbel] (r_s2b) at ($(r_s1) + (2,-0.5)$) {};
    \draw[<-] (r_s1.west) -- +(-0.6,0);
    \draw (r_s1) -- node[below right]{$a$} (r_s1b);
    \draw (r_s2) -- node[above left]{$a$} (r_s2b);
    \draw (r_s1b) edge[->, bend left = 10] node[above left]{$p$} (r_s2);
    \draw (r_s1b) edge[->, bend right = 30] node[above left]{$1-p$} (r_s1.north);
    \draw (r_s2b) edge[->, bend left = 10] node[below right]{$q$} (r_s1);
    \draw (r_s2b) edge[->, bend right = 30] node[below right]{$1-q$} (r_s2.south);
\end{tikzpicture}
    }
    \caption{Two example RPOMDPs with the same uncertainty set.}
    \label{fig:small_example_observation_sticky}
\end{figure}

\begin{example}\label{exam:RMDP:observation_sticky}
    \Cref{fig:small_example_observation_sticky} depicts two RPOMDPs.
    For simplicity, these RPOMDPs have no private observations.
    Furthermore, we interpret these RPOMDPs with the agent first semantics.
    Note that the left RPOMDP corresponds to the fully observable RPOMDP interpretation of the RMDP in \Cref{app:fig:small_example}.
    For both RPOMDPs, we have that:
    \[
        p^{s,a}(s_1,a) = 1, p^{s,a}(s_2,a) = 0, q^{s,a}(s_1,a) = 0, q^{s,a}(s_2,a) = 1.
    \]
    So variable $p$ only influences transitions from state $s_1$, and variable $q$ only influences transitions from state $s_2$.
    Like in \Cref{app:exam:RMDP:stickiness}, in both RPOMDPs, nature has to choose a variable assignment $u \in \bm{U}$ at the start of the game.
    First, looking at the left RPOMDP and assuming observation-based stickiness, the variable $p$ becomes restricted after this initial choice, since $O_\publ(s_1) = \tsLG \land p^{s,a}(s_1,a) = 1$, so $\sticky(p,\bot,\tsLG,a) = 1$.
    Variable $q$ remains unrestricted, as state $s_2$, the only state that $q$ influences, has a different observation.
    As long as the agent remains in state $s_1$, nature may choose different assignments for $q$.
    Once the agent reaches state $s_2$, whichever value it assigns to $q$ next will stick, and from then on, the game is fully determined as also with full stickiness.

    Looking at the right RPOMDP, again assuming observation-based stickiness, both variables immediately become restricted after the initial choice.
    $p$ again becomes restricted because $O_\publ(s_1) = \tsLG \land p^{s,a}(s_1,a) = 1$, so we have $\sticky(p,\bot,\tsLG,a) = 1$.
    Now, since $s2$ has the same observation as $s_1$, we have $O_\publ(s2) = \tsLG \land q^{s,a}(s_2,a) = 1$, which gives us $\sticky(q,\bot,\tsLG,a) = 1$.
\end{example}

\clearpage
\newpage
\section{Uncertainty Assumptions Matter}\label{app:assumptions_matter}
In this appendix, we elaborate on the results established in \Cref{thm:uncertainty:matters} and its proof.
{\label{app:thm_assumptions_matter}\theoremI*}
In the following four subsections, we compare tuples of R(PO)MDPs which only differ in either the stickiness or the order of play.
The first two subsections focus on differences in the stickiness, and the latter two focus on differences in the order of play:
\begin{enumerate}[noitemsep]
    \item Full stickiness versus zero stickiness in an $(s,a)$-rectangular model (\Cref{app:stickiness_matters_I}).
    \item Full stickiness versus observation-based stickiness versus zero stickiness in an $a$-rectangular model (\Cref{app:stickiness_matters_II}).
    \item Agent first versus nature first in a simple model (\Cref{app:order_of_play_matters_I}).
    \item Agent first versus nature first in an $a$-rectangular full sticky model (\Cref{app:order_of_play_matters_II}).
\end{enumerate}

For each of the tuples of RPOMDPs, we first state the value functions given agent and nature policies. 
In principle, we use the sets of mixed nature policies $\Theta^{\mix}$ for computing the optimal value and policy, which are equivalent to the original sets of stochastic policies, as shown in \Cref{app:mixed_policies}.
However, in three of the four RPOMDP tuples, we can consider deterministic nature policies instead of mixed nature policies, as discussed in more detail in \Cref{app:stickiness_matters_I}.
For legibility reasons, we switch back to the stochastic policy notation when writing the optimal policy.

Given the value functions and the types of optimal policies, we used the 3D calculator of \cite{geogebra} to search for the optimal value.
We plotted the value functions for finding the optimal agent and nature policies separately.
In both cases, we look for the policy values that optimize the worst-case scenario from the player's perspective.
Once we found the value that both players can achieve regardless of the other player's policy, we found the Nash equilibrium value and policies.
The models used for finding the optimal values and policies can be found at \url{https://github.com/LAVA-LAB/RPOMDP_game_semantics_value_functions}.

The computed optimal values show that the optimal values differ between the RPOMDPs in the tuples.
These tuples of RPOMDPs therefore prove Theorem \hyperref[app:thm_assumptions_matter]{1}.

After the optimal value computation, we discuss the structural differences of the equivalent POSGs.
Although the structural differences between two POSGs do not ensure that they have a different optimal value, the differences do provide an intuition in the cases where we know that the optimal values differ.

\subsection{Stickiness Matters}\label{app:stickiness_matters_I}
We first revisit the RPOMDP in \Cref{fig:full_vs_zero_sticky_rPOMDP} and show how we computed the optimal values to show that stickiness matters in RPOMDPs.

\begin{reusefigure}[ht]{fig:full_vs_zero_sticky_rPOMDP}
    \centering
    \resizebox{0.6\textwidth}{!}{
        \begin{tikzpicture}[state/.append style={shape = circle}, >=stealth,
    bobbel/.style={minimum size=1mm,inner sep=0pt,fill=black,circle}]
    \node[state] (s1) at (1,0) {$s_1$};
    \node[state, fill=black!10] (s2) at ($(s1) + (2,0.8)$) {$s_2$};
    \node[state, fill=black!40] (s3) at ($(s1) + (2,-0.8)$) {$s_3$};
    \node[state, fill=black!5, thin, dashed] (s4) at ($(s2) + (2,0.8)$) {$s_4$};
    \node[state, fill=black!5, thin, dashed] (s5) at ($(s3) + (2,-0.8)$) {$s_5$};
    \node[state, fill=black!20, thick, dotted] (s6) at ($(s4) + (2,0.7)$) {$s_6$};
    \node[state, fill=black!20, thick, dotted] (s7) at ($(s4) + (2,-0.7)$) {$s_7$};
    \node[state, fill=black!20, thick, dotted] (s8) at ($(s5) + (2,0.8)$) {$s_8$};
    \node[state, fill=black!20, thick, dotted] (s9) at ($(s5) + (2,-0.8)$) {$s_9$};
    \node[bobbel] (s6ba) at ($(s6) + (2,0.2)$) {};
    \node[bobbel] (s6bb) at ($(s6) + (2,-0.4)$) {};
    \node[bobbel] (s7ba) at ($(s7) + (2,0.4)$) {};
    \node[bobbel] (s7bb) at ($(s7) + (2,-0.2)$) {};
    \node[bobbel] (s8ba) at ($(s8) + (2,0.4)$) {};
    \node[bobbel] (s8bb) at ($(s8) + (2,-0.4)$) {};
    \node[bobbel] (s9ba) at ($(s9) + (2,0.4)$) {};
    \node[bobbel] (s9bb) at ($(s9) + (2,-0.4)$) {};
    \node[rectangle, draw, minimum width=14mm, rounded corners] (rPos1) at ($(s6ba) + (2,0)$) {R = $200$};
    \node[rectangle, draw, minimum width=14mm, rounded corners] (rNeg1) at ($(s6bb) + (2,-0.325)$)  {R = $0$};
    \node[rectangle, draw, minimum width=14mm, rounded corners] (rNeu1) at ($(s7bb) + (2,0)$) {R = $100$};
    \node[rectangle, draw, minimum width=14mm, rounded corners] (rNeg2a) at ($(s8ba) + (2,0)$) {R = $0$};
    \node[rectangle, draw, minimum width=14mm, rounded corners] (rPos2) at ($(s8bb) + (2,0)$) {R = $200$};
    \node[rectangle, draw, minimum width=14mm, rounded corners] (rNeu2) at ($(s9ba) + (2,0)$) {R = $100$};
    \node[rectangle, draw, minimum width=14mm, rounded corners] (rNeg2b) at ($(s9bb) + (2,0)$) {R = $0$};
    \draw[<-] (s1.west) -- +(-0.6,0);
    \draw (s1) edge[->] node[above left]{$0.5$} (s2);
    \draw (s1) edge[->] node[below left]{$0.5$} (s3);
    \draw (s2) edge[->] node[above left]{$0.9$} (s4);
    \draw (s2) edge[->] node[above right]{$0.1$} (s5);
    \draw (s3) edge[->] node[below left]{$1$} (s5);
    \draw (s4) edge[->] node[above left]{$p$} (s6);
    \draw (s4) edge[->] node[below, xshift=-2mm, yshift=-1mm]{$1-p$} (s7);
    \draw (s5) edge[->] node[above left]{$q$} (s8);
    \draw (s5) edge[->] node[below, xshift=-2mm, yshift=-1mm]{$1-q$} (s9);
    \draw (s6) edge node[above left]{$a$} (s6ba);
    \draw (s6) edge node[below left]{$b$} (s6bb);
    \draw (s7) edge node[above left]{$a$} (s7ba);
    \draw (s7) edge node[below left]{$b$} (s7bb);
    \draw (s8) edge node[above left]{$a$} (s8ba);
    \draw (s8) edge node[below left]{$b$} (s8bb);
    \draw (s9) edge node[above left]{$a$} (s9ba);
    \draw (s9) edge node[below left]{$b$} (s9bb);
    \draw [dashed] (s6ba) edge[->] (rPos1);
    \draw [dashed] (s6bb) edge[->] (rNeg1);
    \draw [dashed] (s7ba) edge[->] (rNeg1);
    \draw [dashed] (s7bb) edge[->] (rNeu1);
    \draw [dashed] (s8ba) edge[->] (rNeg2a);
    \draw [dashed] (s8bb) edge[->] (rPos2);
    \draw [dashed] (s9ba) edge[->] (rNeu2);
    \draw [dashed] (s9bb) edge[->] (rNeg2b);
    \node[text width=4cm] at (2,-2) {$p \in [0.1,0.9]$};
    \node[text width=4cm] at (2,-2.5) {$q \in [0.1,0.9]$};
\end{tikzpicture}
    }
    \caption{An RPOMDP where full and zero stickiness do not coincide in their optimal value.}
    \label{app:fig:full_vs_zero_sticky_rPOMDP}
\end{reusefigure}

For notation purposes, we use the distinguishing second observation to identify the longer history inputs for policies where needed.
For $\pi \in \Pi$, we write $\pi^{\msLG} = \pi(\tsW\tsLG\tsDash\tsDot)$ and $\pi^{\msDG} = \pi(\tsW\tsDG\tsDash\tsDot)$.
For $\theta \in \Theta$ of the full stickiness RPOMDP, we write $\theta^{\msW} = \theta(\tsW)$, and for $\theta \in \Theta$ of the zero stickiness RPOMDP, we write $\theta^{\msLG} = \theta(\tsW\tsLG\tsDash)$ and $\theta^{\msDG} = \theta(\tsW\tsDG\tsDash)$.
Using this notation, we can write the value function for the full stickiness RPOMDP $M_1$ in \Cref{app:fig:full_vs_zero_sticky_rPOMDP} as:
\begin{align*}
    V_\text{fh}^{M_1}(\pi,\theta^\mix) &= \sum_{\theta^\detr\in \Theta^\detr} \theta^\mix(\theta^\detr) \cdot \Bigl( 0.5 \cdot 0.9 \cdot \bigl(\theta^{\detr,\msW}(p) \cdot \pi^{\msLG}(a) \cdot 200 + (1-\theta^{\detr,\msW}(p)) \cdot \pi^{\msLG}(b) \cdot 100\bigr) \\
    & \hspace{85pt} + 0.5 \cdot 0.1 \cdot \bigl(\theta^{\detr,\msW}(q) \cdot \pi^{\msLG}(b) \cdot 200 + (1-\theta^{\detr,\msW}(q)) \cdot \pi^{\msLG}(a) \cdot 100\bigr)\\
    & \hspace{85pt} + 0.5 \cdot \bigl(\theta^{\detr,\msW}(q) \cdot \pi^{\msDG}(b) \cdot 200 + (1-\theta^{\detr,\msW}(q)) \cdot \pi^{\msDG}(a) \cdot 100\bigr) \Bigr)
\intertext{And for the zero stickiness RPOMDP $M_2$:}
    V_\text{fh}^{M_2}(\pi,\theta^\mix) &= \sum_{\theta^\detr\in \Theta^\detr} \theta^\mix(\theta^\detr) \cdot \Bigl( 0.5 \cdot 0.9 \cdot \bigl(\theta^{\detr,\msLG}(p) \cdot \pi^{\msLG}(a) \cdot 200 + (1-\theta^{\detr,\msLG}(p)) \cdot \pi^{\msLG}(b) \cdot 100\bigr) \\
    & \hspace{110pt} + 0.5 \cdot 0.1 \cdot \bigl(\theta^{\detr,\msLG}(q) \cdot \pi^{\msLG}(b) \cdot 200 + (1-\theta^{\detr,\msLG}(q)) \cdot \pi^{\msLG}(a) \cdot 100\bigr)\\
    & \hspace{110pt} + 0.5 \cdot \bigl(\theta^{\detr,\msDG}(q) \cdot \pi^{\msDG}(b) \cdot 200 + (1-\theta^{\detr,\msDG}(q)) \cdot \pi^{\msDG}(a) \cdot 100\bigr) \Bigr).
\end{align*}

We construct these value functions by following all possible paths leading to non-zero rewards.
We combined some of the paths to keep the formula manageable.
The three bigger terms correspond to the three subparts of the RPOMDP in \Cref{app:fig:full_vs_zero_sticky_rPOMDP} shown in \Cref{app:fig:branches_full_vs_zero_sticky_rPOMDP}.
A multiplication corresponds to successive branches, whereas addition corresponds to parallel branches.
Note that we removed the paths leading to rewards of zero.
Also note that in this case, we used the full stickiness theta.
For the zero stickiness theta, the history on which choices are based is different, but the multiplications, additions, and how we wrote down the formula still correspond to the subparts below.

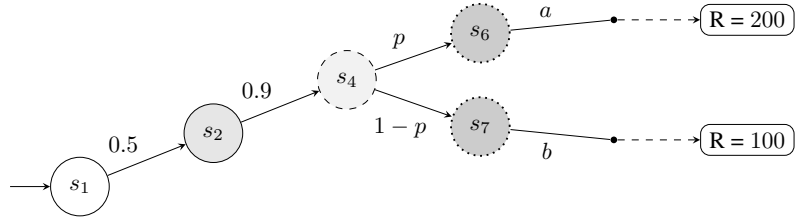
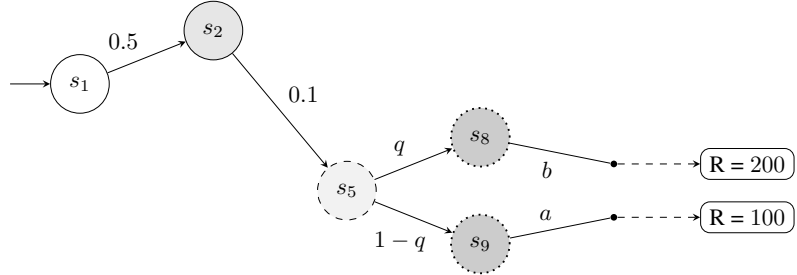
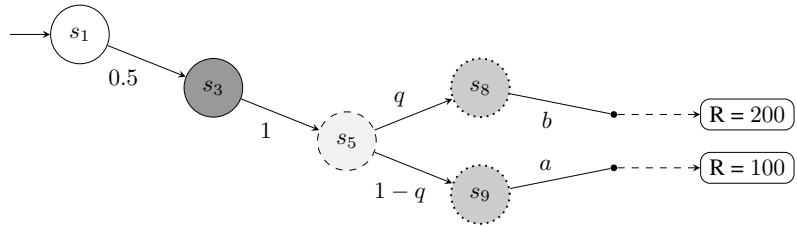
\begin{figure}[ht]
    \centering
    \begin{subfigure}[b]{\textwidth}
        \centering
        \resizebox{0.6\textwidth}{!}{
            \begin{tikzpicture}[state/.append style={shape = circle}, >=stealth,
    bobbel/.style={minimum size=1mm,inner sep=0pt,fill=black,circle}]
    \node[state] (s1) at (1,0) {$s_1$};
    \node[state, fill=black!10] (s2) at ($(s1) + (2,0.8)$) {$s_2$};
    \node[state, fill=black!5, thin, dashed] (s4) at ($(s2) + (2,0.8)$) {$s_4$};
    \node[state, fill=black!20, thick, dotted] (s6) at ($(s4) + (2,0.7)$) {$s_6$};
    \node[state, fill=black!20, thick, dotted] (s7) at ($(s4) + (2,-0.7)$) {$s_7$};
    \node[bobbel] (s6ba) at ($(s6) + (2,0.2)$) {};
    \node[bobbel] (s7bb) at ($(s7) + (2,-0.2)$) {};
    \node[rectangle, draw, minimum width=14mm, rounded corners] (rPos1) at ($(s6ba) + (2,0)$) {R = $200$};
    \node[rectangle, draw, minimum width=14mm, rounded corners] (rNeu1) at ($(s7bb) + (2,0)$) {R = $100$};
    \draw[<-] (s1.west) -- +(-0.6,0);
    \draw (s1) edge[->] node[above left]{$0.5$} (s2);
    \draw (s2) edge[->] node[above left]{$0.9$} (s4);
    \draw (s4) edge[->] node[above left]{$p$} (s6);
    \draw (s4) edge[->] node[below, xshift=-2mm, yshift=-1mm]{$1-p$} (s7);
    \draw (s6) edge node[above left]{$a$} (s6ba);
    \draw (s7) edge node[below left]{$b$} (s7bb);
    \draw [dashed] (s6ba) edge[->] (rPos1);
    \draw [dashed] (s7bb) edge[->] (rNeu1);
\end{tikzpicture}
        }
        \caption{Subpart of the RPOMDP responsible for $0.5 \cdot 0.9 \cdot \bigl(\theta^{\detr,\msW}(p) \cdot \pi^{\msLG}(a) \cdot 200 + (1-\theta^{\detr,\msW}(p)) \cdot \pi^{\msLG}(b) \cdot 100\bigr)$}
    \end{subfigure}\\[3mm]
    \begin{subfigure}[b]{\textwidth}
        \centering
        \resizebox{0.6\textwidth}{!}{
            \begin{tikzpicture}[state/.append style={shape = circle}, >=stealth,
    bobbel/.style={minimum size=1mm,inner sep=0pt,fill=black,circle}]
    \node[state] (s1) at (1,0) {$s_1$};
    \node[state, fill=black!10] (s2) at ($(s1) + (2,0.8)$) {$s_2$};
    \node[state, fill=black!5, thin, dashed] (s5) at ($(s2) + (2,-2.4)$) {$s_5$};
    \node[state, fill=black!20, thick, dotted] (s8) at ($(s5) + (2,0.8)$) {$s_8$};
    \node[state, fill=black!20, thick, dotted] (s9) at ($(s5) + (2,-0.8)$) {$s_9$};
    \node[bobbel] (s8bb) at ($(s8) + (2,-0.4)$) {};
    \node[bobbel] (s9ba) at ($(s9) + (2,0.4)$) {};
    \node[rectangle, draw, minimum width=14mm, rounded corners] (rPos2) at ($(s8bb) + (2,0)$) {R = $200$};
    \node[rectangle, draw, minimum width=14mm, rounded corners] (rNeu2) at ($(s9ba) + (2,0)$) {R = $100$};
    \draw[<-] (s1.west) -- +(-0.6,0);
    \draw (s1) edge[->] node[above left]{$0.5$} (s2);
    \draw (s2) edge[->] node[above right]{$0.1$} (s5);
    \draw (s5) edge[->] node[above left]{$q$} (s8);
    \draw (s5) edge[->] node[below, xshift=-2mm, yshift=-1mm]{$1-q$} (s9);
    \draw (s8) edge node[below left]{$b$} (s8bb);
    \draw (s9) edge node[above left]{$a$} (s9ba);
    \draw [dashed] (s8bb) edge[->] (rPos2);
    \draw [dashed] (s9ba) edge[->] (rNeu2);
\end{tikzpicture}
        }
        \caption{Subpart of the RPOMDP responsible for $0.5 \cdot 0.1 \cdot \bigl(\theta^{\detr,\msW}(q) \cdot \pi^{\msLG}(b) \cdot 200 + (1-\theta^{\detr,\msW}(q)) \cdot \pi^{\msLG}(a) \cdot 100\bigr)$}
    \end{subfigure}\\[3mm]
    \begin{subfigure}[b]{\textwidth}
        \centering
        \resizebox{0.6\textwidth}{!}{
            \begin{tikzpicture}[state/.append style={shape = circle}, >=stealth,
    bobbel/.style={minimum size=1mm,inner sep=0pt,fill=black,circle}]
    \node[state] (s1) at (1,0) {$s_1$};
    \node[state, fill=black!40] (s3) at ($(s1) + (2,-0.8)$) {$s_3$};
    \node[state, fill=black!5, thin, dashed] (s5) at ($(s3) + (2,-0.8)$) {$s_5$};
    \node[state, fill=black!20, thick, dotted] (s8) at ($(s5) + (2,0.8)$) {$s_8$};
    \node[state, fill=black!20, thick, dotted] (s9) at ($(s5) + (2,-0.8)$) {$s_9$};
    \node[bobbel] (s8bb) at ($(s8) + (2,-0.4)$) {};
    \node[bobbel] (s9ba) at ($(s9) + (2,0.4)$) {};
    \node[rectangle, draw, minimum width=14mm, rounded corners] (rPos2) at ($(s8bb) + (2,0)$) {R = $200$};
    \node[rectangle, draw, minimum width=14mm, rounded corners] (rNeu2) at ($(s9ba) + (2,0)$) {R = $100$};
    \draw[<-] (s1.west) -- +(-0.6,0);
    \draw (s1) edge[->] node[below left]{$0.5$} (s3);
    \draw (s3) edge[->] node[below left]{$1$} (s5);
    \draw (s5) edge[->] node[above left]{$q$} (s8);
    \draw (s5) edge[->] node[below, xshift=-2mm, yshift=-1mm]{$1-q$} (s9);
    \draw (s8) edge node[below left]{$b$} (s8bb);
    \draw (s9) edge node[above left]{$a$} (s9ba);
    \draw [dashed] (s8bb) edge[->] (rPos2);
    \draw [dashed] (s9ba) edge[->] (rNeu2);
\end{tikzpicture}
        }
        \caption{Subpart of the RPOMDP responsible for $0.5 \cdot \bigl(\theta^{\detr,\msW}(q) \cdot \pi^{\msDG}(b) \cdot 200 + (1-\theta^{\detr,\msW}(q)) \cdot \pi^{\msDG}(a) \cdot 100\bigr)$}
    \end{subfigure}    
    \caption{Representation of the split used in the construction of value function of the RPOMDP in \Cref{app:fig:full_vs_zero_sticky_rPOMDP}.}
    \label{app:fig:branches_full_vs_zero_sticky_rPOMDP}
\end{figure}

We show that we can reason with deterministic nature policies by showing that the value functions are linear in the deterministic nature policies. 
Combining this with the convexity of the uncertainty sets, we show that any finite probability distribution over the deterministic nature policies, \ie, any mixed nature policy, can be rewritten as another deterministic nature policy that is contained in the policy set. 
We can, therefore, limit ourselves to searching for the optimal policy in the deterministic nature policy set. 
We first prove the following lemma:

\begin{lemma}\label{app:lem_exists_deter_pol}
    Given the full stickiness and zero stickiness RPOMDPs $M_1$ and $M_2$ of \Cref{app:fig:full_vs_zero_sticky_rPOMDP}:
    \[
        \forall \theta^\mix \in \Theta^\mix, \exists \theta^\detr \in \Theta^\detr, \forall \pi\in \Pi.\; V_\text{fh}^{M_1}(\pi,\theta^\mix) = V_\text{fh}^{M_1}(\pi,\theta^\detr),
    \]
    \[
        \forall \theta^\mix \in \Theta^\mix, \exists \theta^\detr \in \Theta^\detr, \forall \pi\in \Pi.\; V_\text{fh}^{M_2}(\pi,\theta^\mix) = V_\text{fh}^{M_2}(\pi,\theta^\detr).
    \]
\end{lemma}
\begin{proof}
    We focus on the full stickiness RPOMDP $M_1$.
    The result for the zero stickiness RPOMDP $M_2$ follows from the same steps.
    Let $\theta^\mix \in \Theta^\mix$ be an arbitrary mixed nature policy and $\pi \in \Pi$ an arbitrary stochastic agent policy.
    Then, we can compute the value as follows:
    \begin{align*}
        V_\text{fh}^{M_1}(\pi,\theta^\mix) &= \sum_{\theta^\detr\in \Theta^\detr} \theta^\mix(\theta^\detr) \cdot \Bigl( 0.5 \cdot 0.9 \cdot \bigl(\theta^{\detr,\msW}(p) \cdot \pi^{\msLG}(a) \cdot 200 + (1-\theta^{\detr,\msW}(p)) \cdot \pi^{\msLG}(b) \cdot 100\bigr) \\
        & \hspace{85pt} + 0.5 \cdot 0.1 \cdot \bigl(\theta^{\detr,\msW}(q) \cdot \pi^{\msLG}(b) \cdot 200 + (1-\theta^{\detr,\msW}(q)) \cdot \pi^{\msLG}(a) \cdot 100\bigr)\\
        & \hspace{85pt} + 0.5 \cdot \bigl(\theta^{\detr,\msW}(q) \cdot \pi^{\msDG}(b) \cdot 200 + (1-\theta^{\detr,\msW}(q)) \cdot \pi^{\msDG}(a) \cdot 100\bigr) \Bigr).
    \intertext{Simplify:}
        &= \sum_{\theta^\detr\in \Theta^\detr} \theta^\mix(\theta^\detr) \cdot \Bigl( 90 \cdot \pi^{\msLG}(a) \cdot \theta^{\detr,\msW}(p) + 45 \cdot \pi^{\msLG}(b) - 45 \cdot \pi^{\msLG}(b) \cdot \theta^{\detr,\msW}(p)\\
        & \hspace{85pt} + 10 \cdot \pi^{\msLG}(b) \cdot \theta^{\detr,\msW}(q) + 5 \cdot \pi^{\msLG}(a) - 5 \cdot \pi^{\msLG}(a) \cdot \theta^{\detr,\msW}(q) \\
        & \hspace{85pt} + 100 \cdot \pi^{\msDG}(b) \cdot \theta^{\detr,\msW}(q) + 50 \cdot \pi^{\msDG}(a) - 50 \cdot \pi^{\msDG}(a) \cdot \theta^{\detr,\msW}(q) \Bigr).
    \intertext{By definition of probability distributions: $\sum_{\theta^\detr\in \Theta^\detr} \theta^\mix(\theta^\detr) = 1$, so we can move the terms depending only on $\pi$ out of the summation:}
        &= 45 \cdot \pi^{\msLG}(b) + 5 \cdot \pi^{\msLG}(a) + 50 \cdot \pi^{\msDG}(a) + \sum_{\theta^\detr\in \Theta^\detr} \theta^\mix(\theta^\detr) \cdot \Bigl( (90 \cdot \pi^{\msLG}(a) - 45 \cdot \pi^{\msLG}(b)) \cdot \theta^{\detr,\msW}(p)\\
        & \quad + (10 \cdot \pi^{\msLG}(b) - 5 \cdot \pi^{\msLG}(a) + 100 \cdot \pi^{\msDG}(b) - 50 \cdot \pi^{\msDG}(a)) \cdot \theta^{\detr,\msW}(q) \Bigr).\\
    \intertext{Split the summation:}
        &= 45 \cdot \pi^{\msLG}(b) + 5 \cdot \pi^{\msLG}(a) + 50 \cdot \pi^{\msDG}(a) + \sum_{\theta^\detr\in \Theta^\detr} \theta^\mix(\theta^\detr) \cdot \Bigl( (90 \cdot \pi^{\msLG}(a) - 45 \cdot \pi^{\msLG}(b)) \cdot \theta^{\detr,\msW}(p) \Bigr)\\
        & \quad + \sum_{\theta^\detr\in \Theta^\detr} \theta^\mix(\theta^\detr) \cdot \Bigl((10 \cdot \pi^{\msLG}(b) - 5 \cdot \pi^{\msLG}(a) + 100 \cdot \pi^{\msDG}(b) - 50 \cdot \pi^{\msDG}(a)) \cdot \theta^{\detr,\msW}(q) \Bigr).\\
    \intertext{Move the multiplication terms only depending on $\pi$ out of the summations:}
        &= 45 \cdot \pi^{\msLG}(b) + 5 \cdot \pi^{\msLG}(a) + 50 \cdot \pi^{\msDG}(a)\\
        & \quad + (90 \cdot \pi^{\msLG}(a) - 45 \cdot \pi^{\msLG}(b)) \cdot \sum_{\theta^\detr\in \Theta^\detr} \theta^\mix(\theta^\detr) \cdot \theta^{\detr,\msW}(p)\\
        & \quad + (10 \cdot \pi^{\msLG}(b) - 5 \cdot \pi^{\msLG}(a) + 100 \cdot \pi^{\msDG}(b) - 50 \cdot \pi^{\msDG}(a)) \cdot \sum_{\theta^\detr\in \Theta^\detr} \theta^\mix(\theta^\detr) \cdot \theta^{\detr,\msW}(q).
    \end{align*}
    The uncertainty set $\bm{U}$ of $M_1$ is convex.
    We, therefore, know:
    \[
        \forall u,u'\in \bm{U}, \forall \alpha \in [0,1].\; \alpha u + (1-\alpha)u' \in \bm{U}.
    \]
    By definition of valid policies (\Cref{subsec:stickiness}), we know:
    \[
        \forall \theta^\detr \in \Theta^\detr, \forall h\in H^\nature, \forall a\in A.\; \theta^\detr(h,a) \in \bm{U}^\agrees(\fixed(h)).
    \]
    At the non-singleton point of choice for the full stickiness nature policies, \ie, history $\tsW$, $\fixed(\tsW) = \emptyset = u^\bot$.
    We, therefore, know:
    \[
        \forall \theta^\detr \in \Theta^\detr.\; \theta^{\detr,\msW} \in \bm{U}^\agrees(u^\bot) = \bm{U}.
    \]
    So we can create nature policy $\theta^{\prime \detr}$ for which:
    \begin{align*}
        \theta^{\prime \detr,\msW}(p) &= \sum_{\theta^\detr\in \Theta^\detr} \theta^\mix(\theta^\detr) \cdot \theta^{\detr,\msW}(p),\\
        \theta^{\prime \detr,\msW}(q) &= \sum_{\theta^\detr\in \Theta^\detr} \theta^\mix(\theta^\detr) \cdot \theta^{\detr,\msW}(q).
    \end{align*}
    Since at the only non-singleton point of choice, $\theta^{\prime \detr}$ is a convex combination of elements of a convex set, we know that:
    \[
        \theta^{\prime \detr, \msW} \in \bm{U} = \bm{U}^\agrees(u^\bot).
    \]
    As all other choices are singletons, we have that:
    \[
        \forall h\in H^\nature, \forall a\in A.\; \theta^{\prime \detr}(h,a) \in \bm{U}^\agrees(\fixed(h)).
    \]
    From which we can conclude that $\theta^{\prime \detr}$ is a valid deterministic nature policy, \ie, $\theta^{\prime \detr} \in \Theta^\detr$.
    We continue by showing that $V_\text{fh}^{M_1}(\pi,\theta^{\prime \detr}) = V_\text{fh}^{M_1}(\pi,\theta^\mix)$ for an arbitrary agent policy $\pi \in \Pi$.
    \begin{align*}
        V_\text{fh}^{M_1}(\pi,\theta^{\prime \detr}) &= 0.5 \cdot 0.9 \cdot \bigl(\theta^{\prime \detr,\msW}(p) \cdot \pi^{\msLG}(a) \cdot 200 + (1-\theta^{\prime \detr,\msW}(p)) \cdot \pi^{\msLG}(b) \cdot 100\bigr) \\
        & \quad + 0.5 \cdot 0.1 \cdot \bigl(\theta^{\prime \detr,\msW}(q) \cdot \pi^{\msLG}(b) \cdot 200 + (1-\theta^{\prime \detr,\msW}(q)) \cdot \pi^{\msLG}(a) \cdot 100\bigr)\\
        & \quad + 0.5 \cdot \bigl(\theta^{\prime \detr,\msW}(q) \cdot \pi^{\msDG}(b) \cdot 200 + (1-\theta^{\prime \detr,\msW}(q)) \cdot \pi^{\msDG}(a) \cdot 100\bigr).
    \intertext{Simplify:}
        &= 90 \cdot \pi^{\msLG}(a) \cdot \theta^{\prime \detr,\msW}(p) + 45 \cdot \pi^{\msLG}(b) - 45 \cdot \pi^{\msLG}(b) \cdot \theta^{\prime \detr,\msW}(p)\\
        & \quad + 10 \cdot \pi^{\msLG}(b) \cdot \theta^{\prime \detr,\msW}(q) + 5 \cdot \pi^{\msLG}(a) - 5 \cdot \pi^{\msLG}(a) \cdot \theta^{\prime \detr,\msW}(q) \\
        & \quad + 100 \cdot \pi^{\msDG}(b) \cdot \theta^{\prime \detr,\msW}(q) + 50 \cdot \pi^{\msDG}(a) - 50 \cdot \pi^{\msDG}(a) \cdot \theta^{\prime \detr,\msW}(q).
    \intertext{Reorder:}
        &= 45 \cdot \pi^{\msLG}(b) + 5 \cdot \pi^{\msLG}(a) + 50 \cdot \pi^{\msDG}(a)\\
        & \quad + (90 \cdot \pi^{\msLG}(a) - 45 \cdot \pi^{\msLG}(b)) \cdot \theta^{\prime \detr,\msW}(p)\\
        & \quad + (10 \cdot \pi^{\msLG}(b) - 5 \cdot \pi^{\msLG}(a) + 100 \cdot \pi^{\msDG}(b) - 50 \cdot \pi^{\msDG}(a)) \cdot \theta^{\prime \detr,\msW}(q).
    \intertext{Using the definition of $\theta^{\prime \detr,\msW}$:}
        &= 45 \cdot \pi^{\msLG}(b) + 5 \cdot \pi^{\msLG}(a) + 50 \cdot \pi^{\msDG}(a)\\
        & \quad + (90 \cdot \pi^{\msLG}(a) - 45 \cdot \pi^{\msLG}(b)) \cdot \sum_{\theta^\detr\in \Theta^\detr} \theta^\mix(\theta^\detr) \cdot \theta^{\detr,\msW}(p)\\
        & \quad + (10 \cdot \pi^{\msLG}(b) - 5 \cdot \pi^{\msLG}(a) + 100 \cdot \pi^{\msDG}(b) - 50 \cdot \pi^{\msDG}(a)) \cdot \sum_{\theta^\detr\in \Theta^\detr} \theta^\mix(\theta^\detr) \cdot \theta^{\detr,\msW}(q)\\
        &= V_\text{fh}^{M_1}(\pi,\theta^\mix).
    \end{align*}
    Now we have that:
    \[
        \forall \pi\in \Pi.\; V_\text{fh}^{M_1}(\pi,\theta^\mix) = V_\text{fh}^{M_1}(\pi,\theta^{\prime \detr}),
    \]
    So we can conclude that:
    \[
        \forall \theta^\mix \in \Theta^\mix, \exists \theta^\detr \in \Theta^\detr, \forall \pi\in \Pi.\; V_\text{fh}^{M_1}(\pi,\theta^\mix) = V_\text{fh}^{M_1}(\pi,\theta^\detr).
    \]
\end{proof}

We can now prove the following theorem:
\begin{proposition}\label{app:prop_deter_pol}
    Given the full stickiness and zero stickiness RPOMDPs $M_1$ and $M_2$ of \Cref{app:fig:full_vs_zero_sticky_rPOMDP}:
    \[
        \sup_{\pi\in\Pi}\inf_{\theta^\mix\in\Theta^\mix} V_\text{fh}^{M_1}(\pi,\theta^\mix) = \sup_{\pi\in\Pi}\inf_{\theta^\detr\in\Theta^\detr} V_\text{fh}^{M_1}(\pi,\theta^\detr),
    \]
    \[
        \sup_{\pi\in\Pi}\inf_{\theta^\mix\in\Theta^\mix} V_\text{fh}^{M_2}(\pi,\theta^\mix) = \sup_{\pi\in\Pi}\inf_{\theta^\detr\in\Theta^\detr} V_\text{fh}^{M_2}(\pi,\theta^\detr).
    \]
\end{proposition}
\begin{proof}
    For both $M_1$ and $M_2$, the $\leq$ direction directly follows from \Cref{app:lem_exists_deter_pol} and the $\geq$ direction from $\Theta^\detr \subseteq \Theta^\mix$.
\end{proof}
Using \Cref{app:prop_deter_pol}, we can focus on deterministic nature policies in our computation of the optimal value function. 
We can compute the optimal value for the full stickiness model as follows:
\begin{align*}
    V_\text{fh}^{*, M_1} &=  \sup_{\pi\in \Pi}\inf_{\theta^\detr\in \Theta^\detr} \bigl\{45 \cdot \pi^{\msLG}(b) + 5 \cdot \pi^{\msLG}(a) + 50 \cdot \pi^{\msDG}(a)\\
    & \hspace{72pt} + (90 \cdot \pi^{\msLG}(a) - 45 \cdot \pi^{\msLG}(b)) \cdot \theta^{\detr,\msW}(p)\\
    & \hspace{72pt} + (10 \cdot \pi^{\msLG}(b) - 5 \cdot \pi^{\msLG}(a) + 100 \cdot \pi^{\msDG}(b) - 50 \cdot \pi^{\msDG}(a)) \cdot \theta^{\detr,\msW}(q)\bigr\}.\\
\intertext{And for the zero stickiness model, simplified in the same manner as the full stickiness value function:}
    V_\text{fh}^{*,M_2} &= \sup_{\pi\in \Pi}\inf_{\theta^\detr\in \Theta^\detr} \bigl\{ 45 \cdot \pi^{\msLG}(b) + 5 \cdot \pi^{\msLG}(a) + 50 \cdot \pi^{\msDG}(a)\\
    & \hspace{72pt} + (90 \cdot \pi^{\msLG}(a) - 45 \cdot \pi^{\msLG}(b)) \cdot \theta^{\detr,\msLG}(p)\\
    & \hspace{72pt} + (10 \cdot \pi^{\msLG}(b) - 5 \cdot \pi^{\msLG}(a)) \cdot \theta^{\detr,\msLG}(q)\\
    & \hspace{72pt} + (100 \cdot \pi^{\msDG}(b) - 50 \cdot \pi^{\msDG}(a)) \cdot \theta^{\detr,\msDG}(q) \bigr\}.
\intertext{Since $\theta^{\detr,\msLG}$ is independent from $\theta^{\detr,\msDG}$ and $\pi^{\msLG}$ is independent from $\pi^{\msDG}$, we can rewrite the zero stickiness optimal value function as:}
    V_\text{fh}^{*,M_2} &= \sup_{\pi\in \Pi}\inf_{\theta^\detr\in \Theta^\detr} \bigl\{ 45 \cdot \pi^{\msLG}(b) + 5 \cdot \pi^{\msLG}(a)\\
    & \hspace{72pt} + (90 \cdot \pi^{\msLG}(a) - 45 \cdot \pi^{\msLG}(b)) \cdot \theta^{\detr,\msLG}(p)\\
    & \hspace{72pt} + (10 \cdot \pi^{\msLG}(b) - 5 \cdot \pi^{\msLG}(a)) \cdot \theta^{\detr,\msLG}(q) \bigr\}\\
    & \hspace{12pt} + \sup_{\pi\in \Pi}\inf_{\theta^\detr\in \Theta^\detr} \bigl\{ 50 \cdot \pi^{\msDG}(a) + (100 \cdot \pi^{\msDG}(b) - 50 \cdot \pi^{\msDG}(a)) \cdot \theta^{\detr,\msDG}(q) \bigr\}.
\end{align*}

\Cref{app:tab:full_vs_zero_sticky_rPOMDP}  displays the computed optimal values and policies, showing the differences between the full and zero stickiness assumptions.
An underscore indicates that the value assigned to this variable does not influence the optimal value of the RPOMDP.

{\def\arraystretch{1.3}
\begin{table}[h]
    \centering
    \begin{tabular}{l|p{42mm} p{42mm}}
    \toprule
         & Full stickiness & Zero stickiness\\\midrule
        Optimal value & $66\frac{2}{3}$ & $65\frac{1}{2}$\\\hline
        Optimal agent policy & $\tsW\tsLG\tsDash\tsDot \mapsto \{a \mapsto \frac{1}{3}, b\mapsto \frac{2}{3}\},$\newline$\tsW\tsDG\tsDash\tsDot \mapsto \{a \mapsto \frac{7}{10}, b\mapsto \frac{3}{10}\}$ & $\tsW\tsLG\tsDash\tsDot \mapsto \{a \mapsto \frac{1}{3}, b\mapsto \frac{2}{3}\},$\newline$\tsW\tsDG\tsDash\tsDot \mapsto \{a \mapsto \frac{2}{3}, b\mapsto \frac{1}{3}\}$ \\\hline
        Optimal nature policy & $\tsW \mapsto \{p \mapsto \frac{1}{3}, q\mapsto \frac{1}{3}\}$ & $\tsW\tsLG\tsDash \mapsto \{p \mapsto \frac{83}{270}, q\mapsto \frac{1}{10}\},$\newline$\tsW\tsDG\tsDash \mapsto \{p \mapsto \_, q\mapsto \frac{1}{3}\}$\\
        \bottomrule
    \end{tabular}
    \caption{Optimal values and policies for the full stickiness and zero stickiness interpretations of the RPOMDP in \Cref{app:fig:full_vs_zero_sticky_rPOMDP}.}
    \label{app:tab:full_vs_zero_sticky_rPOMDP}
\end{table}}

\subsubsection{Underlying POSGs}
\Cref{app:fig:stickiness:POSGs} (restated below) depicts the first couple of states of the full and zero stickiness POSGs of the RPOMDP in \Cref{app:fig:full_vs_zero_sticky_rPOMDP}.
We briefly discuss the structural differences.
In the zero stickiness case, we have an infinite action choice for nature at every nature state, but every choice will reach the same unrestricted agent states, as variable assignments never stick in the zero stickiness case.
Due to this, the state space of the zero stickiness POSG is finite, consisting of $S + S\times A$ states, where $S$ is the set of states and $A$ is the set of actions of the original RPOMDP.
This can be seen in \Cref{app:fig:stickiness:POSGs} at nature state $\tup{s_1, \{\}, \bot}$, where each choice reaches the same agent states $\tup{s_2,\{\}}$ and $\tup{s_3,\{\}}$.
At this point, the choice does not influence the transition probability, so all choices go to the agent states with exactly the same probability and essentially have no influence.
This is no longer the case at nature states $\tup{s_4,\{\},\bot}$ and $\tup{s_5,\{\},\bot}$, where the values chosen for $p$ and $q$ directly influence the probability of reaching agent state $\tup{s_6,\{\}}, \tup{s_7,\{\}}, \tup{s_8,\{\}}$, and $\tup{s_9,\{\}}$.

In the full stickiness case, on the other hand, each of the infinite action choices at the first nature state $\tup{s_1,\{\},\bot}$ leads to a \emph{unique} continuation of the POSG, as each game continues to agent states with different restrictions on nature's choice.
The full stickiness case, hence, has an infinite state space but only one infinite action choice, namely the first.
\begin{reusefigure}[ht]{fig:stickiness:POSGs}
    \centering
    \begin{subfigure}[c]{0.48\columnwidth}
        \resizebox{\columnwidth}{!}{
        \begin{tikzpicture}[state/.append style={shape = ellipse}, >=stealth,
    bobbel/.style={minimum size=1mm,inner sep=0pt,fill=black,circle}]]
    \node[state, inner sep=-0.2pt] (s1) at (1,0) {\scriptsize $s_1,\{\}$};
    \node[state, inner sep=-2pt] (s1_bot) at (2.5,0) {\scriptsize $s_1,\{\},\bot$};
    \node[bobbel] (s1b_u1) at (5.3,0.9) {};
    \node[bobbel] (s1b_u2) at (5.3,-0.9) {};
    \node[state, fill=black!10, inner sep=-0.2pt] (s2) at (7.5,0.5) {\scriptsize $s_2,\{\}$};
    \node[state, fill=black!40, inner sep=-0.2pt] (s3) at (7.5,-0.5) {\scriptsize $s_3,\{\}$};
    \draw[<-] (s1.west) -- +(-0.4,0);
    \draw[->] (s2.east) -- +(0.4,0);
    \draw[->] (s3.east) -- +(0.4,0);
    \node[] at ($(s2.east) + (0.8,0)$) {\scriptsize $\dots$};
    \node[] at ($(s3.east) + (0.8,0)$) {\scriptsize $\dots$};
    \draw (s1) edge[->] (s1_bot);
    \draw (s1b_u1) edge[->] node[above right]{\scriptsize $0.5$} (s2);
    \draw (s1b_u1) edge[->] node[below left, pos = 0.4, yshift=1.5mm]{\scriptsize $0.5$} (s3.west);
    \draw (s1b_u2) edge[->] node[above left, pos = 0.4, yshift=-1.5mm]{\scriptsize $0.5$} (s2.west);
    \draw (s1b_u2) edge[->] node[below right]{\scriptsize $0.5$} (s3);
    \draw (s1_bot) edge[->] node[sloped, anchor=center, above]{\scriptsize $\dots$} (5.3,0.3);
    \draw (s1_bot) edge[->] node[sloped, anchor=center, below]{\scriptsize $\dots$} (5.3,-0.3);
    \node[] at (5,0.1) {\tiny $\vdots$};
    \node[] at (5,0.64) {\tiny $\vdots$};
    \node[] at (5,-0.44) {\tiny $\vdots$};
    \draw (s1_bot) edge node[sloped, anchor=center, above, xshift=-0.2mm, yshift=-0.6mm]{\scriptsize $\{p\mapsto 0.1,q\mapsto 0.1\}$} (s1b_u1);
    \draw (s1_bot) edge node[sloped, anchor=center, below, xshift=-0.2mm, yshift=0.6mm]{\scriptsize $\{p\mapsto 0.9,q\mapsto 0.9\}$} (s1b_u2);
\end{tikzpicture}%
        }
    \end{subfigure}%
    \quad
    \begin{subfigure}[c]{0.48\columnwidth}
        \resizebox{\columnwidth}{!}{
        \begin{tikzpicture}[state/.append style={shape = ellipse}, >=stealth,
    bobbel/.style={minimum size=1mm,inner sep=0pt,fill=black,circle}]]
    \node[state, inner sep=-0.2pt] (s1) at (1,0) {\scriptsize $s_1,\{\}$};
    \node[state, inner sep=-2pt] (s1_bot) at (2.5,0) {\scriptsize $s_1,\{\},\bot$};
    \node[bobbel] (s1b_u1) at (4.1,0.9) {};
    \node[bobbel] (s1b_u2) at (4.1,-0.9) {};
    \coordinate (c_s2_u1) at ($(s1b_u1) + (2.4,0.34)$);
    \coordinate (c_s3_u1) at ($(s1b_u1) + (2.4,-0.34)$);
    \coordinate (c_s2_u2) at ($(s1b_u2) + (2.4,0.34)$);
    \coordinate (c_s3_u2) at ($(s1b_u2) + (2.4,-0.34)$);
    \draw[fill=black!10] ($(c_s2_u1) + (-1.47,0)$) 
                        to [out = 85, in = 190] ($(c_s2_u1) + (-1.2,0.19)$) 
                        to [out = 10, in = 180] ($(c_s2_u1) + (0,0.29)$)
                        to [out = 0, in = 170] ($(c_s2_u1) + (1.2,0.19)$)
                        to [out = -10, in = 95] ($(c_s2_u1) + (1.47,0)$)
                        to [out = -95, in = 10] ($(c_s2_u1) + (1.2,-0.19)$)
                        to [out = 190, in = 0] ($(c_s2_u1) + (0,-0.29)$)
                        to [out = 180, in = -10] ($(c_s2_u1) + (-1.2,-0.19)$)
                        to [out = 170, in = -85] ($(c_s2_u1) + (-1.47,0)$);
    \draw[fill=black!40] ($(c_s3_u1) + (-1.47,0)$) 
                        to [out = 85, in = 190] ($(c_s3_u1) + (-1.2,0.19)$) 
                        to [out = 10, in = 180] ($(c_s3_u1) + (0,0.29)$)
                        to [out = 0, in = 170] ($(c_s3_u1) + (1.2,0.19)$)
                        to [out = -10, in = 95] ($(c_s3_u1) + (1.47,0)$)
                        to [out = -95, in = 10] ($(c_s3_u1) + (1.2,-0.19)$)
                        to [out = 190, in = 0] ($(c_s3_u1) + (0,-0.29)$)
                        to [out = 180, in = -10] ($(c_s3_u1) + (-1.2,-0.19)$)
                        to [out = 170, in = -85] ($(c_s3_u1) + (-1.47,0)$);
    \draw[fill=black!10] ($(c_s2_u2) + (-1.47,0)$) 
                        to [out = 85, in = 190] ($(c_s2_u2) + (-1.2,0.19)$) 
                        to [out = 10, in = 180] ($(c_s2_u2) + (0,0.29)$)
                        to [out = 0, in = 170] ($(c_s2_u2) + (1.2,0.19)$)
                        to [out = -10, in = 95] ($(c_s2_u2) + (1.47,0)$)
                        to [out = -95, in = 10] ($(c_s2_u2) + (1.2,-0.19)$)
                        to [out = 190, in = 0] ($(c_s2_u2) + (0,-0.29)$)
                        to [out = 180, in = -10] ($(c_s2_u2) + (-1.2,-0.19)$)
                        to [out = 170, in = -85] ($(c_s2_u2) + (-1.47,0)$);
    \draw[fill=black!40] ($(c_s3_u2) + (-1.47,0)$) 
                        to [out = 85, in = 190] ($(c_s3_u2) + (-1.2,0.19)$) 
                        to [out = 10, in = 180] ($(c_s3_u2) + (0,0.29)$)
                        to [out = 0, in = 170] ($(c_s3_u2) + (1.2,0.19)$)
                        to [out = -10, in = 95] ($(c_s3_u2) + (1.47,0)$)
                        to [out = -95, in = 10] ($(c_s3_u2) + (1.2,-0.19)$)
                        to [out = 190, in = 0] ($(c_s3_u2) + (0,-0.29)$)
                        to [out = 180, in = -10] ($(c_s3_u2) + (-1.2,-0.19)$)
                        to [out = 170, in = -85] ($(c_s3_u2) + (-1.47,0)$);
    \node[state, draw=none, inner sep=-9pt, minimum width = 29mm] (s2_u1) at (c_s2_u1) {\scriptsize $s_2,\{p\mapsto 0.1,q\mapsto 0.1\}$};
    \node[state, draw=none, inner sep=-9pt, minimum width = 29mm] (s3_u1) at (c_s3_u1) {\scriptsize $s_3,\{p\mapsto 0.1,q\mapsto 0.1\}$};
    \node[] at (6.5,0.1) {\tiny $\vdots$};
    \node[state, draw=none, inner sep=-9pt, minimum width = 29mm] (s2_u2) at (c_s2_u2) {\scriptsize $s_2,\{p\mapsto 0.9,q\mapsto 0.9\}$};
    \node[state, draw=none, inner sep=-9pt, minimum width = 29mm] (s3_u2) at (c_s3_u2) {\scriptsize $s_3,\{p\mapsto 0.9,q\mapsto 0.9\}$};  
    \draw[<-] (s1.west) -- +(-0.4,0);
    \draw[->] (s2_u1.east) -- +(0.4,0);
    \draw[->] (s3_u1.east) -- +(0.4,0);
    \draw[->] (s2_u2.east) -- +(0.4,0);
    \draw[->] (s3_u2.east) -- +(0.4,0);
    \node[] at ($(s2_u1.east) + (0.8,0)$) {\scriptsize $\dots$};
    \node[] at ($(s3_u1.east) + (0.8,0)$) {\scriptsize $\dots$};
    \node[] at ($(s2_u2.east) + (0.8,0)$) {\scriptsize $\dots$};
    \node[] at ($(s3_u2.east) + (0.8,0)$) {\scriptsize $\dots$};
    \draw (s1) edge[->] (s1_bot);
    \draw (s1b_u1) edge[->] node[above left, xshift=2mm]{\scriptsize $0.5$} (s2_u1.west);
    \draw (s1b_u1) edge[->] node[below left, xshift=2mm]{\scriptsize $0.5$} (s3_u1.west);
    \draw (s1b_u2) edge[->] node[above left, xshift=2mm]{\scriptsize $0.5$} (s2_u2.west);
    \draw (s1b_u2) edge[->] node[below left, xshift=2mm]{\scriptsize $0.5$} (s3_u2.west);
    \draw (s1_bot) edge[->] node[sloped, anchor=center, above]{\scriptsize $\dots$} (4.1,0.3);
    \draw (s1_bot) edge[->] node[sloped, anchor=center, below]{\scriptsize $\dots$} (4.1,-0.3);
    \node[] at (3.9,0.1) {\tiny $\vdots$};
    \node[] at (3.9,0.62) {\tiny $\vdots$};
    \node[] at (3.9,-0.42) {\tiny $\vdots$};
    \draw (s1_bot) edge node[sloped, anchor=center, above, xshift=-0.2mm, yshift=-0.6mm] {\scriptsize$\begin{aligned}
        \{ &p\mapsto 0.1,\\[-1mm]
        &q\mapsto 0.1\}
        \end{aligned}$} (s1b_u1);
    \draw (s1_bot) edge node[sloped, anchor=center, below, xshift=-0.2mm, yshift=0.6mm]{\scriptsize $\begin{aligned}
        \{ &p\mapsto 0.9,\\[-1mm]
        &q\mapsto 0.9\}
        \end{aligned}$} (s1b_u2);
\end{tikzpicture}%
        }
    \end{subfigure}
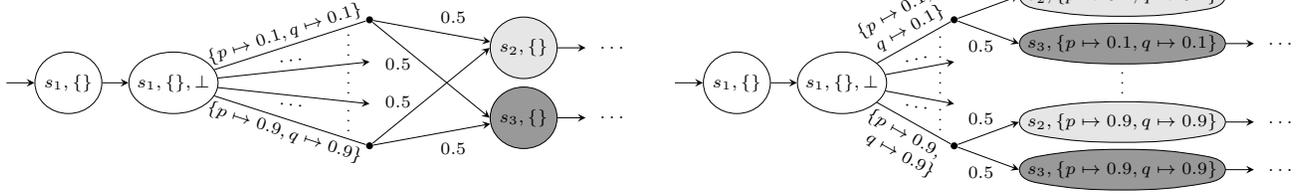
    \caption{First states of zero stickiness (left) and full stickiness (right) POSGs of the RPOMDP in  \Cref{app:fig:full_vs_zero_sticky_rPOMDP}.}
    \label{app:fig:stickiness:POSGs}
\end{reusefigure}

\subsection{Observation-based stickiness}\label{app:stickiness_matters_II}
Next, we look at the RPOMDP in \Cref{app:fig:observation_sticky_full_vs_zero_sticky_RPOMDP} to show that observation-based stickiness also differs in value from full and zero stickiness.
Note that this model extends the RPOMDP in \Cref{fig:full_vs_zero_sticky_rPOMDP}.
We interpret this model with nature first semantics.
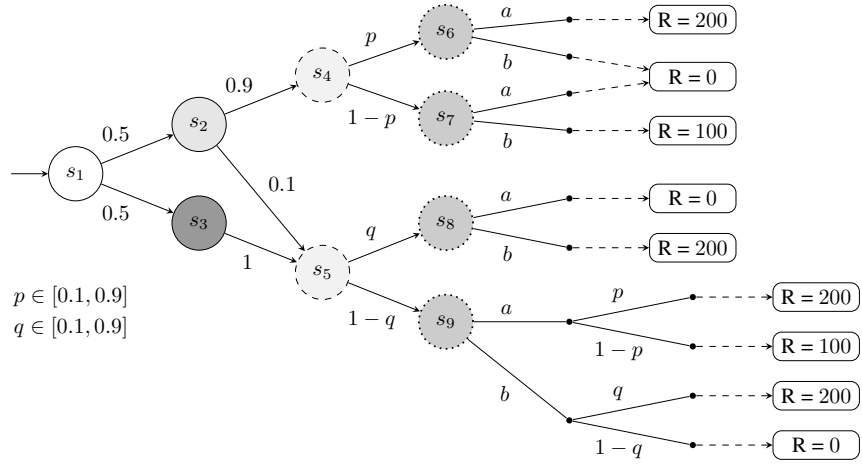
\begin{figure}[H]
    \centering
    \resizebox{0.65\columnwidth}{!}{
    \begin{tikzpicture}[state/.append style={shape = circle}, >=stealth,
    bobbel/.style={minimum size=1mm,inner sep=0pt,fill=black,circle}]
    \node[state] (s1) at (1,0) {$s_1$};
    \node[state, fill=black!10] (s2) at ($(s1) + (2,0.8)$) {$s_2$};
    \node[state, fill=black!40] (s3) at ($(s1) + (2,-0.8)$) {$s_3$};
    \node[state, fill=black!5, thin, dashed] (s4) at ($(s2) + (2,0.8)$) {$s_4$};
    \node[state, fill=black!5, thin, dashed] (s5) at ($(s3) + (2,-0.8)$) {$s_5$};
    \node[state, fill=black!20, thick, dotted] (s6) at ($(s4) + (2,0.7)$) {$s_6$};
    \node[state, fill=black!20, thick, dotted] (s7) at ($(s4) + (2,-0.7)$) {$s_7$};
    \node[state, fill=black!20, thick, dotted] (s8) at ($(s5) + (2,0.8)$) {$s_8$};
    \node[state, fill=black!20, thick, dotted] (s9) at ($(s5) + (2,-0.8)$) {$s_9$};
    \node[bobbel] (s6ba) at ($(s6) + (2,0.2)$) {};
    \node[bobbel] (s6bb) at ($(s6) + (2,-0.4)$) {};
    \node[bobbel] (s7ba) at ($(s7) + (2,0.4)$) {};
    \node[bobbel] (s7bb) at ($(s7) + (2,-0.2)$) {};
    \node[bobbel] (s8ba) at ($(s8) + (2,0.4)$) {};
    \node[bobbel] (s8bb) at ($(s8) + (2,-0.4)$) {};
    \node[bobbel] (s9ba) at ($(s9) + (2,0)$) {};
    \node[bobbel] (s9bb) at ($(s9) + (2,-1.6)$) {};
    \node[bobbel] (s9bap) at ($(s9ba) + (2,0.4)$) {};
    \node[bobbel] (s9ba-p) at ($(s9ba) + (2,-0.4)$) {};
    \node[bobbel] (s9bbq) at ($(s9bb) + (2,0.4)$) {};
    \node[bobbel] (s9bb-q) at ($(s9bb) + (2,-0.4)$) {};
    \node[rectangle, draw, minimum width=14mm, rounded corners] (rPos1) at ($(s6ba) + (2,0)$) {R = $200$};
    \node[rectangle, draw, minimum width=14mm, rounded corners] (rNeg1) at ($(s6bb) + (2,-0.325)$)  {R = $0$};
    \node[rectangle, draw, minimum width=14mm, rounded corners] (rNeu1) at ($(s7bb) + (2,0)$) {R = $100$};
    \node[rectangle, draw, minimum width=14mm, rounded corners] (rNeg2a) at ($(s8ba) + (2,0)$) {R = $0$};
    \node[rectangle, draw, minimum width=14mm, rounded corners] (rPos2) at ($(s8bb) + (2,0)$) {R = $200$};
    \node[rectangle, draw, minimum width=14mm, rounded corners] (rPos3a) at ($(s9bap) + (2,0)$) {R = $200$};
    \node[rectangle, draw, minimum width=14mm, rounded corners] (rNeu3) at ($(s9ba-p) + (2,0)$) {R = $100$};
    \node[rectangle, draw, minimum width=14mm, rounded corners] (rPos3b) at ($(s9bbq) + (2,0)$) {R = $200$};
    \node[rectangle, draw, minimum width=14mm, rounded corners] (rNeg3) at ($(s9bb-q) + (2,0)$) {R = $0$};
    \draw[<-] (s1.west) -- +(-0.6,0);
    \draw (s1) edge[->] node[above left]{$0.5$} (s2);
    \draw (s1) edge[->] node[below left]{$0.5$} (s3);
    \draw (s2) edge[->] node[above left]{$0.9$} (s4);
    \draw (s2) edge[->] node[above right]{$0.1$} (s5);
    \draw (s3) edge[->] node[below left]{$1$} (s5);
    \draw (s4) edge[->] node[above left]{$p$} (s6);
    \draw (s4) edge[->] node[below, xshift=-2mm, yshift=-1mm]{$1-p$} (s7);
    \draw (s5) edge[->] node[above left]{$q$} (s8);
    \draw (s5) edge[->] node[below, xshift=-2mm, yshift=-1mm]{$1-q$} (s9);
    \draw (s6) edge node[above left]{$a$} (s6ba);
    \draw (s6) edge node[below left]{$b$} (s6bb);
    \draw (s7) edge node[above left]{$a$} (s7ba);
    \draw (s7) edge node[below left]{$b$} (s7bb);
    \draw (s8) edge node[above left]{$a$} (s8ba);
    \draw (s8) edge node[below left]{$b$} (s8bb);
    \draw (s9) edge node[above left]{$a$} (s9ba);
    \draw (s9ba) edge node[above left]{$p$} (s9bap);
    \draw (s9ba) edge node[below, xshift=-2mm]{$1-p$} (s9ba-p);
    \draw (s9) edge node[below left]{$b$} (s9bb);
    \draw (s9bb) edge node[above left]{$q$} (s9bbq);
    \draw (s9bb) edge node[below, xshift=-2mm]{$1-q$} (s9bb-q);
    \draw [dashed] (s6ba) edge[->] (rPos1);
    \draw [dashed] (s6bb) edge[->] (rNeg1);
    \draw [dashed] (s7ba) edge[->] (rNeg1);
    \draw [dashed] (s7bb) edge[->] (rNeu1);
    \draw [dashed] (s8ba) edge[->] (rNeg2a);
    \draw [dashed] (s8bb) edge[->] (rPos2);
    \draw [dashed] (s9bap) edge[->] (rPos3a);
    \draw [dashed] (s9ba-p) edge[->] (rNeu3);
    \draw [dashed] (s9bbq) edge[->] (rPos3b);
    \draw [dashed] (s9bb-q) edge[->] (rNeg3);
    \node[text width=4cm] at (2,-2) {$p \in [0.1,0.9]$};
    \node[text width=4cm] at (2,-2.5) {$q \in [0.1,0.9]$};
\end{tikzpicture}
    }
    \caption{An RPOMDP where observation-based stickiness also leads to different values.}
    \label{app:fig:observation_sticky_full_vs_zero_sticky_RPOMDP}
\end{figure}
For $\pi \in \Pi$, we write $\pi^{\msLG} = \pi(\tsW\tsLG\tsDash\tsDot)$ and $\pi^{\msDG} = \pi(\tsW\tsDG\tsDash\tsDot)$.
Similarly, for $\theta \in \Theta$ of the full stickiness RPOMDP, we write $\theta^{\msW} = \theta(\tsW)$, for $\theta \in \Theta$ of the observation-based stickiness RPOMDP, we write $\theta^{\msLG} = \theta(\tsW\tsLG\tsDash)$ and $\theta^{\msDG} = \theta(\tsW\tsDG\tsDash)$, and for $\theta \in \Theta$ of the zero stickiness RPOMDP, we write $\theta^{\msLG} = \theta(\tsW\tsLG\tsDash)$, $\theta^{\msDG} = \theta(\tsW\tsDG\tsDash)$, $\theta^{\msLG\msDot} = \theta(\tsW\tsDG\tsDash\tsDot)$, and $\theta^{\msDG\msDot} = \theta(\tsW\tsDG\tsDash\tsDot)$.
Using this notation, we can construct the value functions for the various stickiness interpretations of the RPOMDP in \Cref{app:fig:full_vs_zero_sticky_rPOMDP}.
We construct these value functions following the same approach as for the value functions of the RPOMDP in \Cref{app:fig:full_vs_zero_sticky_rPOMDP}, see \Cref{app:stickiness_matters_I} and \Cref{app:fig:branches_full_vs_zero_sticky_rPOMDP}.
The value function for the full stickiness RPOMDP $M_1$ is:
\begin{align*}
    V_\text{fh}^{M_1}(\pi,\theta^\mix) &=  \sum_{\theta^\detr\in \Theta^\detr} \theta^\mix(\theta^\detr) \cdot \Bigl(0.5 \cdot 0.9 \cdot \bigl(\theta^{\detr, \msW}(p) \cdot \pi^{\msLG}(a) \cdot 200 + (1-\theta^{\detr, \msW}(p)) \cdot \pi^{\msLG}(b) \cdot 100\bigr) \\
    & \hspace{108pt} + 0.5 \cdot 0.1 \cdot \bigl(\theta^{\detr, \msW}(q) \cdot \pi^{\msLG}(b) \cdot 200 + (1-\theta^{\detr, \msW}(q))\\
    & \hspace{120pt} \cdot (\pi^{\msLG}(a) \cdot (\theta^{\detr, \msW}(p) \cdot 200 + (1-\theta^{\detr, \msW}(p)) \cdot 100)\\
    & \hspace{120pt} + \pi^{\msLG}(b) \cdot \theta^{\detr, \msW}(q) \cdot 200) \bigr)\\
    & \hspace{108pt} + 0.5 \cdot \bigl(\theta^{\detr, \msW}(q) \cdot \pi^{\msDG}(b) \cdot 200 + (1-\theta^{\detr, \msW}(q)) \\
    & \hspace{120pt} \cdot (\pi^{\msDG}(a) \cdot (\theta^{\detr, \msW}(p) \cdot 200 + (1-\theta^{\detr, \msW}(p)) \cdot 100)\\
    & \hspace{120pt} + \pi^{\msDG}(b) \cdot \theta^{\detr, \msW}(q) \cdot 200) \bigr)\Bigr).
\intertext{Which simplifies to:}
    &= \sum_{\theta^\detr\in \Theta^\detr} \theta^\mix(\theta^\detr) \cdot \Bigl( 45 \cdot \pi^{\msLG}(b) + 5 \cdot \pi^{\msLG}(a) + 50 \cdot \pi^{\msDG}(a)\\
    & \hspace{108pt} + (95 \cdot \pi^{\msLG}(a) - 45 \cdot \pi^{\msLG}(b) + 50 \cdot \pi^{\msDG}(a)) \cdot \theta^{\detr, \msW}(p)\\
    & \hspace{108pt} + (20 \cdot \pi^{\msLG}(b) - 5 \cdot \pi^{\msLG}(a) + 200 \cdot \pi^{\msDG}(b) - 50 \cdot \pi^{\msDG}(a)) \cdot \theta^{\detr, \msW}(q)\\
    & \hspace{108pt} - (5 \cdot \pi^{\msLG}(a) + 50 \cdot \pi^{\msDG}(a)) \cdot \theta^{\detr, \msW}(p) \cdot \theta^{\detr, \msW}(q)\\
    & \hspace{108pt} - (10 \cdot \pi^{\msLG}(b) + 100 \cdot \pi^{\msDG}(b)) \cdot \theta^{\detr, \msW}(q)^2 \Bigr).
\end{align*}
This function is not linear in the deterministic nature policies, as quickly follows from the multiplication of $\theta$ terms. 
We, therefore, need to search for the optimal nature policy in the set of mixed nature policies.

We can similarly write the value function for the observation-based stickiness RPOMDP $M_2$:
\begin{align*}
    V_\text{fh}^{M_2}(\pi,\theta^\mix) &=  \sum_{\theta^\detr\in \Theta^\detr} \theta^\mix(\theta^\detr) \cdot \Bigl( 0.5 \cdot 0.9 \cdot \bigl(\theta^{\detr, \msLG}(p) \cdot \pi^{\msLG}(a) \cdot 200 + (1-\theta^{\detr, \msLG}(p)) \cdot \pi^{\msLG}(b) \cdot 100\bigr) \\
    & \hspace{108pt} + 0.5 \cdot 0.1 \cdot \bigl(\theta^{\detr, \msLG}(q) \cdot \pi^{\msLG}(b) \cdot 200 + (1-\theta^{\detr, \msLG}(q))\\
    & \hspace{120pt} \cdot (\pi^{\msLG}(a) \cdot (\theta^{\detr, \msLG}(p) \cdot 200 + (1-\theta^{\detr, \msLG}(p)) \cdot 100)\\
    & \hspace{120pt} + \pi^{\msLG}(b) \cdot \theta^{\detr, \msLG}(q) \cdot 200) \bigr)\\
    & \hspace{108pt} + 0.5 \cdot \bigl(\theta^{\detr, \msDG}(q) \cdot \pi^{\msDG}(b) \cdot 200 + (1-\theta^{\detr, \msDG}(q)) \\
    & \hspace{120pt} \cdot (\pi^{\msDG}(a) \cdot (\theta^{\detr, \msDG}(p) \cdot 200 + (1-\theta^{\detr, \msDG}(p)) \cdot 100)\\
    & \hspace{120pt} + \pi^{\msDG}(b) \cdot \theta^{\detr, \msDG}(q) \cdot 200) \bigr)\Bigr).
\intertext{Which simplifies to:}
    &=  \sum_{\theta^\detr\in \Theta^\detr} \theta^\mix(\theta^\detr) \cdot \Bigl(45 \cdot \pi^{\msLG}(b) + 5 \cdot \pi^{\msLG}(a) + 50 \cdot \pi^{\msDG}(a)\\
    & \hspace{108pt} + (95 \cdot \pi^{\msLG}(a) - 45 \cdot \pi^{\msLG}(b)) \cdot \theta^{\detr, \msLG}(p)\\
    & \hspace{108pt} + (20 \cdot \pi^{\msLG}(b) - 5 \cdot \pi^{\msLG}(a)) \cdot \theta^{\detr, \msLG}(q)\\
    & \hspace{108pt} - 5 \cdot \pi^{\msLG}(a) \cdot \theta^{\detr, \msLG}(p) \cdot \theta^{\detr, \msLG}(q)\\
    & \hspace{108pt} - 10 \cdot \pi^{\msLG}(b) \cdot \theta^{\detr, \msLG}(q)^2\\
    & \hspace{108pt} + 50 \cdot \pi^{\msDG}(a) \cdot \theta^{\detr, \msDG}(p)\\
    & \hspace{108pt} + (200 \cdot \pi^{\msDG}(b) - 50 \cdot \pi^{\msDG}(a)) \cdot \theta^{\detr, \msDG}(q) \\
    & \hspace{108pt} - 50 \cdot \pi^{\msDG}(a) \cdot \theta^{\detr, \msDG}(p) \cdot \theta^{\detr, \msDG}(q)\\
    & \hspace{108pt} - 100 \cdot \pi^{\msDG}(b) \cdot \theta^{\detr, \msDG}(q)^2\Bigr).
\end{align*}
And the value function for the zero stickiness RPOMDP $M_3$:
\begin{align*}
    V_\text{fh}^{M_3}(\pi,\theta^\mix) &=  \sum_{\theta^\detr\in \Theta^\detr} \theta^\mix(\theta^\detr) \cdot \Bigl(0.5 \cdot 0.9 \cdot \bigl(\theta^{\detr, \msLG}(p) \cdot \pi^{\msLG}(a) \cdot 200 + (1-\theta^{\detr, \msLG}(p)) \cdot \pi^{\msLG}(b) \cdot 100\bigr) \\
    & \hspace{108pt} + 0.5 \cdot 0.1 \cdot \bigl(\theta^{\detr, \msLG}(q) \cdot \pi^{\msLG}(b) \cdot 200 + (1-\theta^{\detr, \msLG}(q))\\
    & \hspace{120pt} \cdot (\pi^{\msLG}(a) \cdot (\theta^{\detr, \msLG\msDot}(p) \cdot 200 + (1-\theta^{\detr, \msLG\msDot}(p)) \cdot 100)\\
    & \hspace{120pt} + \pi^{\msLG}(b) \cdot \theta^{\detr, \msLG\msDot}(q) \cdot 200) \bigr)\\
    & \hspace{108pt} + 0.5 \cdot \bigl(\theta^{\detr, \msDG}(q) \cdot \pi^{\msDG}(b) \cdot 200 + (1-\theta^{\detr, \msDG}(q)) \\
    & \hspace{120pt} \cdot (\pi^{\msDG}(a) \cdot (\theta^{\detr, \msDG\msDot}(p) \cdot 200 + (1-\theta^{\detr, \msDG\msDot}(p)) \cdot 100)\\
    & \hspace{120pt} + \pi^{\msDG}(b) \cdot \theta^{\detr, \msDG\msDot}(q) \cdot 200) \bigr)\Bigr).
\intertext{Which simplifies to:}
    &=  \sum_{\theta^\detr\in \Theta^\detr} \theta^\mix(\theta^\detr) \cdot \Bigl(45 \cdot \pi^{\msLG}(b) + 5 \cdot \pi^{\msLG}(a) + 50 \cdot \pi^{\msDG}(a)\\
    & \hspace{108pt} + (90  \cdot \pi^{\msLG}(a) - 45 \cdot \pi^{\msLG}(b)) \cdot \theta^{\detr, \msLG}(p)\\
    & \hspace{108pt} + (10 \cdot \pi^{\msLG}(b) - 5 \cdot \pi^{\msLG}(a)) \cdot \theta^{\detr, \msLG}(q)\\
    & \hspace{108pt} + 5 \cdot \pi^{\msLG}(a) \cdot \theta^{\detr, \msLG\msDot}(p)\\
    & \hspace{108pt} + 10 \cdot \pi^{\msLG}(b) \cdot \theta^{\detr, \msLG\msDot}(q)\\
    & \hspace{108pt} - 5 \cdot \pi^{\msLG}(a) \cdot \theta^{\detr, \msLG}(q) \cdot \theta^{\detr, \msLG\msDot}(p) \\
    & \hspace{108pt} - 10 \cdot \pi^{\msLG}(b) \cdot \theta^{\detr, \msLG}(q) \cdot \theta^{\detr, \msLG\msDot}(q)\\
    & \hspace{108pt} + (100 \cdot \pi^{\msDG}(b) - 50 \cdot \pi^{\msDG}(a)) \cdot \theta^{\detr, \msDG}(q)\\
    & \hspace{108pt} + 50 \cdot \pi^{\msDG}(a) \cdot \theta^{\detr, \msDG\msDot}(p)\\
    & \hspace{108pt} + 100 \cdot \pi^{\msDG}(b) \cdot \theta^{\detr, \msDG\msDot}(q)\\
    & \hspace{108pt} - 50 \cdot \pi^{\msDG}(a) \cdot \theta^{\detr, \msDG}(q) \cdot \theta^{\detr, \msDG\msDot}(p)\\
    & \hspace{108pt} - 100 \cdot \pi^{\msDG}(b) \cdot \theta^{\detr, \msDG}(q) \cdot \theta^{\detr, \msDG\msDot}(q)\Bigr).
\end{align*}

Using the above value functions, we can compute the optimal value for the full stickiness model as follows:
\begin{align*}
    V_\text{fh}^{*, M_1} &= \sup_{\pi\in \Pi}\inf_{\theta^\mix\in \Theta^\mix} \Bigl\{\sum_{\theta^\detr\in \Theta^\detr} \theta^\mix(\theta^\detr) \cdot \Bigl( 45 \cdot \pi^{\msLG}(b) + 5 \cdot \pi^{\msLG}(a) + 50 \cdot \pi^{\msDG}(a)\\
    & \hspace{178pt} + (95 \cdot \pi^{\msLG}(a) - 45 \cdot \pi^{\msLG}(b) + 50 \cdot \pi^{\msDG}(a)) \cdot \theta^{\detr, \msW}(p)\\
    & \hspace{178pt} + (20 \cdot \pi^{\msLG}(b) - 5 \cdot \pi^{\msLG}(a) + 200 \cdot \pi^{\msDG}(b) - 50 \cdot \pi^{\msDG}(a)) \cdot \theta^{\detr, \msW}(q)\\
    & \hspace{178pt} - (5 \cdot \pi^{\msLG}(a) + 50 \cdot \pi^{\msDG}(a)) \cdot \theta^{\detr, \msW}(p) \cdot \theta^{\detr, \msW}(q)\\
    & \hspace{178pt} - (10 \cdot \pi^{\msLG}(b) + 100 \cdot \pi^{\msDG}(b)) \cdot \theta^{\detr, \msW}(q)^2 \Bigr)\Bigr\}.
\intertext{And the optimal value for the observation-based stickiness model:}
    V_\text{fh}^{*, M_2} &= \sup_{\pi\in \Pi}\inf_{\theta^\mix\in \Theta^\mix} \Bigl\{\sum_{\theta^\detr\in \Theta^\detr} \theta^\mix(\theta^\detr) \cdot \Bigl(45 \cdot \pi^{\msLG}(b) + 5 \cdot \pi^{\msLG}(a) + 50 \cdot \pi^{\msDG}(a)\\
    & \hspace{178pt} + (95 \cdot \pi^{\msLG}(a) - 45 \cdot \pi^{\msLG}(b)) \cdot \theta^{\detr, \msLG}(p)\\
    & \hspace{178pt} + (20 \cdot \pi^{\msLG}(b) - 5 \cdot \pi^{\msLG}(a)) \cdot \theta^{\detr, \msLG}(q)\\
    & \hspace{178pt} - 5 \cdot \pi^{\msLG}(a) \cdot \theta^{\detr, \msLG}(p) \cdot \theta^{\detr, \msLG}(q)\\
    & \hspace{178pt} - 10 \cdot \pi^{\msLG}(b) \cdot \theta^{\detr, \msLG}(q)^2\\
    & \hspace{178pt} + 50 \cdot \pi^{\msDG}(a) \cdot \theta^{\detr, \msDG}(p)\\
    & \hspace{178pt} + (200 \cdot \pi^{\msDG}(b) - 50 \cdot \pi^{\msDG}(a)) \cdot \theta^{\detr, \msDG}(q) \\
    & \hspace{178pt} - 50 \cdot \pi^{\msDG}(a) \cdot \theta^{\detr, \msDG}(p) \cdot \theta^{\detr, \msDG}(q)\\
    & \hspace{178pt} - 100 \cdot \pi^{\msDG}(b) \cdot \theta^{\detr, \msDG}(q)^2\Bigr) \Bigr\}.
\intertext{As histories $\tsW\tsLG\tsDash\tsDot$ and $\tsW\tsDG\tsDash\tsDot$, and $\tsW\tsLG\tsDash$ and $\tsW\tsDG\tsDash$ are mutually exclusive, and the related non-singleton choices are independent for both the agent ($\pi^{\msLG}$ and $\pi^{\msDG}$) and nature ($\theta^{\msLG}$ and $\theta^{\msDG}$), we can rewrite observation-based stickiness as:}
    &= \sup_{\pi\in \Pi}\inf_{\theta^\mix\in \Theta^\mix} \Bigl\{\sum_{\theta^\detr\in \Theta^\detr} \theta^\mix(\theta^\detr) \cdot \Bigl(45 \cdot \pi^{\msLG}(b) + 5 \cdot \pi^{\msLG}(a)\\
    & \hspace{178pt} + (95 \cdot \pi^{\msLG}(a) - 45 \cdot \pi^{\msLG}(b)) \cdot \theta^{\detr, \msLG}(p)\\
    & \hspace{178pt} + (20 \cdot \pi^{\msLG}(b) - 5 \cdot \pi^{\msLG}(a)) \cdot \theta^{\detr, \msLG}(q)\\
    & \hspace{178pt} - 5 \cdot \pi^{\msLG}(a) \cdot \theta^{\detr, \msLG}(p) \cdot \theta^{\detr, \msLG}(q)\\
    & \hspace{178pt} - 10 \cdot \pi^{\msLG}(b) \cdot \theta^{\detr, \msLG}(q)^2\Bigr) \Bigr\}\\
    & \hspace{12pt} + \sup_{\pi\in \Pi}\inf_{\theta^\mix\in \Theta^\mix} \Bigl\{\sum_{\theta^\detr\in \Theta^\detr} \theta^\mix(\theta^\detr) \cdot \Bigl(50 \cdot \pi^{\msDG}(a)\\
    & \hspace{190pt} + 50 \cdot \pi^{\msDG}(a) \cdot \theta^{\detr, \msDG}(p)\\
    & \hspace{190pt} + (200 \cdot \pi^{\msDG}(b) - 50 \cdot \pi^{\msDG}(a)) \cdot \theta^{\detr, \msDG}(q) \\
    & \hspace{190pt} - 50 \cdot \pi^{\msDG}(a) \cdot \theta^{\detr, \msDG}(p) \cdot \theta^{\detr, \msDG}(q)\\
    & \hspace{190pt} - 100 \cdot \pi^{\msDG}(b) \cdot \theta^{\detr, \msDG}(q)^2\Bigr) \Bigr\}.
\intertext{And the optimal value for the zero stickiness model:}
    V_\text{fh}^{*, M_3} &= \sup_{\pi\in \Pi}\inf_{\theta^\mix\in \Theta^\mix} \Bigl\{\sum_{\theta^\detr\in \Theta^\detr} \theta^\mix(\theta^\detr) \cdot \Bigl(45 \cdot \pi^{\msLG}(b) + 5 \cdot \pi^{\msLG}(a) + 50 \cdot \pi^{\msDG}(a)\\
    & \hspace{178pt} + (90  \cdot \pi^{\msLG}(a) - 45 \cdot \pi^{\msLG}(b)) \cdot \theta^{\detr, \msLG}(p)\\
    & \hspace{178pt} + (10 \cdot \pi^{\msLG}(b) - 5 \cdot \pi^{\msLG}(a)) \cdot \theta^{\detr, \msLG}(q)\\
    & \hspace{178pt} + 5 \cdot \pi^{\msLG}(a) \cdot \theta^{\detr, \msLG\msDot}(p)\\
    & \hspace{178pt} + 10 \cdot \pi^{\msLG}(b) \cdot \theta^{\detr, \msLG\msDot}(q)\\
    & \hspace{178pt} - 5 \cdot \pi^{\msLG}(a) \cdot \theta^{\detr, \msLG}(q) \cdot \theta^{\detr, \msLG\msDot}(p) \\
    & \hspace{178pt} - 10 \cdot \pi^{\msLG}(b) \cdot \theta^{\detr, \msLG}(q) \cdot \theta^{\detr, \msLG\msDot}(q)\\
    & \hspace{178pt} + (100 \cdot \pi^{\msDG}(b) - 50 \cdot \pi^{\msDG}(a)) \cdot \theta^{\detr, \msDG}(q)\\
    & \hspace{178pt} + 50 \cdot \pi^{\msDG}(a) \cdot \theta^{\detr, \msDG\msDot}(p)\\
    & \hspace{178pt} + 100 \cdot \pi^{\msDG}(b) \cdot \theta^{\detr, \msDG\msDot}(q)\\
    & \hspace{178pt} - 50 \cdot \pi^{\msDG}(a) \cdot \theta^{\detr, \msDG}(q) \cdot \theta^{\detr, \msDG\msDot}(p)\\
    & \hspace{178pt} - 100 \cdot \pi^{\msDG}(b) \cdot \theta^{\detr, \msDG}(q) \cdot \theta^{\detr, \msDG\msDot}(q)\Bigr)\Bigr\}.
\intertext{As histories $\tsW\tsLG\tsDash\tsDot$ and $\tsW\tsDG\tsDash\tsDot$, and $\tsW\tsLG\tsDash$ and $\tsW\tsDG\tsDash$ are mutually exclusive and the related non-singleton choices are independent for both the agent ($\pi^{\msLG}$ and $\pi^{\msDG}$) and nature ($\theta^{\msLG}$, $\theta^{\msDG}$, $\theta^{\msLG\msDot}$, and $\theta^{\msDG\msDot}$), we can rewrite zero stickiness as:}
    &= \sup_{\pi\in \Pi}\inf_{\theta^\mix\in \Theta^\mix} \Bigl\{\sum_{\theta^\detr\in \Theta^\detr} \theta^\mix(\theta^\detr) \cdot \Bigl(45 \cdot \pi^{\msLG}(b) + 5 \cdot \pi^{\msLG}(a)\\
    & \hspace{178pt} + (90  \cdot \pi^{\msLG}(a) - 45 \cdot \pi^{\msLG}(b)) \cdot \theta^{\detr, \msLG}(p)\\
    & \hspace{178pt} + (10 \cdot \pi^{\msLG}(b) - 5 \cdot \pi^{\msLG}(a)) \cdot \theta^{\detr, \msLG}(q)\\
    & \hspace{178pt} + 5 \cdot \pi^{\msLG}(a) \cdot \theta^{\detr, \msLG\msDot}(p)\\
    & \hspace{178pt} + 10 \cdot \pi^{\msLG}(b) \cdot \theta^{\detr, \msLG\msDot}(q)\\
    & \hspace{178pt} - 5 \cdot \pi^{\msLG}(a) \cdot \theta^{\detr, \msLG}(q) \cdot \theta^{\detr, \msLG\msDot}(p) \\
    & \hspace{178pt} - 10 \cdot \pi^{\msLG}(b) \cdot \theta^{\detr, \msLG}(q) \cdot \theta^{\detr, \msLG\msDot}(q)\Bigr)\Bigr\}\\
    & \hspace{12pt} + \sup_{\pi\in \Pi}\inf_{\theta^\mix\in \Theta^\mix} \Bigl\{\sum_{\theta^\detr\in \Theta^\detr} \theta^\mix(\theta^\detr) \cdot \Bigl(50 \cdot \pi^{\msDG}(a)\\
    & \hspace{190pt} + (100 \cdot \pi^{\msDG}(b) - 50 \cdot \pi^{\msDG}(a)) \cdot \theta^{\detr, \msDG}(q)\\
    & \hspace{190pt} + 50 \cdot \pi^{\msDG}(a) \cdot \theta^{\detr, \msDG\msDot}(p)\\
    & \hspace{190pt} + 100 \cdot \pi^{\msDG}(b) \cdot \theta^{\detr, \msDG\msDot}(q)\\
    & \hspace{190pt} - 50 \cdot \pi^{\msDG}(a) \cdot \theta^{\detr, \msDG}(q) \cdot \theta^{\detr, \msDG\msDot}(p)\\
    & \hspace{190pt} - 100 \cdot \pi^{\msDG}(b) \cdot \theta^{\detr, \msDG}(q) \cdot \theta^{\detr, \msDG\msDot}(q)\Bigr)\Bigr\}.
\intertext{We can further simplify it because $p$ and $q$ are independent:}
    &= \sup_{\pi\in \Pi}\inf_{\theta^\mix\in \Theta^\mix} \Bigl\{\sum_{\theta^\detr\in \Theta^\detr} \theta^\mix(\theta^\detr) \cdot \Bigl(45 \cdot \pi^{\msLG}(b) + 5 \cdot \pi^{\msLG}(a)\\
    & \hspace{108pt} + (90  \cdot \pi^{\msLG}(a) - 45 \cdot \pi^{\msLG}(b)) \cdot \theta^{\detr, \msLG}(p)\\
    & \hspace{108pt} + (10 \cdot \pi^{\msLG}(b) - 5 \cdot \pi^{\msLG}(a)) \cdot \theta^{\detr, \msLG}(q)\\
    & \hspace{108pt} + (5 \cdot \pi^{\msLG}(a) - 5 \cdot \pi^{\msLG}(a) \cdot \theta^{\detr, \msLG}(q)) \cdot \inf_{\theta^\mix\in \Theta^\mix} \bigl\{\sum_{\theta^\detr\in \Theta^\detr} \theta^\mix(\theta^\detr) \cdot \theta^{\detr, \msLG\msDot}(p)\bigr\}\\
    & \hspace{108pt} + (10 \cdot \pi^{\msLG}(b) - 10 \cdot \pi^{\msLG}(b) \cdot \theta^{\detr, \msLG}(q)) \cdot \inf_{\theta^\mix\in \Theta^\mix} \bigl\{\sum_{\theta^\detr\in \Theta^\detr} \theta^\mix(\theta^\detr) \cdot \theta^{\detr, \msLG\msDot}(q)\bigr\}\Bigr)\Bigr\}\\
    & \hspace{12pt} + \sup_{\pi\in \Pi}\inf_{\theta^\mix\in \Theta^\mix} \Bigl\{\sum_{\theta^\detr\in \Theta^\detr} \theta^\mix(\theta^\detr) \cdot \Bigl(50 \cdot \pi^{\msDG}(a)\\
    & \hspace{120pt} + (100 \cdot \pi^{\msDG}(b) - 50 \cdot \pi^{\msDG}(a)) \cdot \theta^{\detr, \msDG}(q)\\
    & \hspace{120pt} + (50 \cdot \pi^{\msDG}(a) - 50 \cdot \pi^{\msDG}(a) \cdot \theta^{\detr, \msDG}(q)) \cdot \inf_{\theta^\mix\in \Theta^\mix} \bigl\{\sum_{\theta^\detr\in \Theta^\detr} \theta^\mix(\theta^\detr) \cdot \theta^{\detr, \msDG\msDot}(p)\bigr\}\\
    & \hspace{120pt} + (100 \cdot \pi^{\msDG}(b) - 100 \cdot \pi^{\msDG}(b) \cdot \theta^{\detr, \msDG}(q)) \cdot \inf_{\theta^\mix\in \Theta^\mix} \bigl\{\sum_{\theta^\detr\in \Theta^\detr} \theta^\mix(\theta^\detr) \cdot \theta^{\detr, \msDG\msDot}(q)\bigr\}\Bigr)\Bigr\}.
\end{align*}

\Cref{app:tab:observation_sticky_full_vs_zero_sticky_RPOMDP} displays the computed optimal values and policies, showing the differences between the full, observation-based, and zero stickiness assumptions.
{\def\arraystretch{1.4}
\begin{table}[h]
    \centering
    \resizebox{\textwidth}{!}{
    \begin{tabular}{l|p{50mm} p{43mm} p{43mm}}
    \toprule
         & Full stickiness & Observation-based stickiness & Zero stickiness\\\midrule
        Optimal value & $74\frac{11}{390}$ & $71\frac{9}{10}$& $70\frac{295}{348}$\\\hline
        Optimal agent policy & $\tsW\tsLG\tsDash\tsDot \mapsto \{a \mapsto \frac{17}{117}, b\mapsto \frac{100}{117}\},$\newline$\tsW\tsDG\tsDash\tsDot \mapsto \{a \mapsto \frac{643}{1170}, b\mapsto \frac{527}{1170}\}$ 
        & $\tsW\tsLG\tsDash\tsDot \mapsto \{a \mapsto \frac{10}{31}, b\mapsto \frac{21}{31}\}$\newline$\tsW\tsDG\tsDash\tsDot \mapsto \{a \mapsto \frac{20}{31}, b\mapsto \frac{11}{31}\}$
        & $\tsW\tsLG\tsDash\tsDot \mapsto \{a \mapsto \frac{1}{3}, b\mapsto \frac{2}{3}\}$\newline$\tsW\tsDG\tsDash\tsDot \mapsto \{a \mapsto \frac{18}{29}, b\mapsto \frac{11}{29}\}$ \\\hline
        Optimal nature policy & $\begin{aligned}[t]
            \tsW \mapsto \{&\{p \mapsto 0.1, q \mapsto 0.1\} \mapsto \tfrac{17}{24},\\ &\{p \mapsto 0.9, q \mapsto 0.1\} \mapsto \tfrac{3}{104},\\ &\{p \mapsto 0.9, q \mapsto 0.9\} \mapsto \tfrac{41}{156}\}
        \end{aligned}$ &
        $\begin{aligned}[t]
            &\tsW\tsLG\tsDash \mapsto \{\\
            &\{p \mapsto 0.1, q \mapsto 0.1\} \mapsto \tfrac{1663}{2232},\\
            &\{p \mapsto 0.9, q \mapsto 0.1\} \mapsto \tfrac{569}{2232}\}\\
            &\tsW\tsDG\tsDash \mapsto \{\\
            &\{p \mapsto 0.1, q \mapsto 0.1\} \mapsto \tfrac{187}{248},\\
            &\{p \mapsto 0.1, q \mapsto 0.9\} \mapsto \tfrac{61}{248}\}
        \end{aligned}$ &
        $\begin{aligned}[t]
            &\tsW\tsLG\tsDash \mapsto \{\\
            &\{p \mapsto 0.1, q \mapsto 0.1\} \mapsto \tfrac{1591}{2160},\\
            &\{p \mapsto 0.9, q \mapsto 0.1\} \mapsto \tfrac{569}{2160}\}\\
            &\tsW\tsDG\tsDash \mapsto \{\\
            &\{p \mapsto \_, q \mapsto 0.1\} \mapsto \tfrac{171}{232},\\
            &\{p \mapsto \_, q \mapsto 0.9\} \mapsto \tfrac{61}{232}\}\\
            &\tsW\tsLG\tsDash\tsDot \mapsto \{\\
            &\{p \mapsto 0.1, q \mapsto 0.1\} \mapsto 1\}\\
            &\tsW\tsDG\tsDash\tsDot \mapsto \{\\
            &\{p \mapsto 0.1, q \mapsto 0.1\} \mapsto 1\}
        \end{aligned}$
        \\
        \bottomrule
    \end{tabular}
    }
    \caption{Optimal values and policies for the full stickiness, observation-based stickiness, and zero stickiness interpretations of the RPOMDP in \Cref{app:fig:observation_sticky_full_vs_zero_sticky_RPOMDP}.}
    \label{app:tab:observation_sticky_full_vs_zero_sticky_RPOMDP}
\end{table}}

\subsubsection{Underlying POSGs}
The POSG of the full stickiness interpretation of the RPOMDP in \Cref{app:fig:observation_sticky_full_vs_zero_sticky_RPOMDP} displays the same structural difference with the observation-based stickiness and zero stickiness interpretations as in \Cref{app:fig:stickiness:POSGs}.
The variable assignment chosen at the first nature state $\tup{s_1,\{\},\bot}$ sticks in the full stickiness RPOMDP but not in the observation-based one zero stickiness ones.

The difference between the observation-based stickiness and zero stickiness POSGs occurs at a later stage, namely at nature states $\tup{s4,\{\},\bot}$ and $\tup{s5,\{\},\bot}$.
After one of these states, the observation-based stickiness model follows the same structure as the full stickiness model, leading to infinitely many agent states with different variable restrictions.
The POSG of the zero stickiness interpretation continues with infinitely many transitions going to the same agent states $\tup{s_6,\{\}}, \tup{s_7,\{\}}, \tup{s_8,\{\}}$, and $\tup{s_9,\{\}}$ with the totally undefined variable restriction.

\subsection{Order of Play Matters}\label{app:order_of_play_matters_I}
We first revisit the RPOMDP in \Cref{fig:agent_vs_nature_first_RPOMDP_small} and show how we computed the optimal values to show that order of play matters in RPOMDPs.
\begin{reusefigure}[H]{fig:agent_vs_nature_first_RPOMDP_small}
    \centering
    \resizebox{0.55\textwidth}{!}{
        \begin{tikzpicture}[state/.append style={shape = circle}, >=stealth,
    bobbel/.style={minimum size=1mm,inner sep=0pt,fill=black,circle}]
    \node[state] (s1) at (1,0) {$s_1$};
    \draw[<-] (s1.west) -- +(-0.6,0);
    \node[bobbel] (s1ab) at ($(s1) + (2,0.7)$) {};
    \node[bobbel] (s1bb) at ($(s1) + (2,-0.7)$) {};
    \node[bobbel] (s1apr) at ($(s1ab) + (2,0.4)$) {};
    \node[bobbel] (s1a-pr) at ($(s1ab) + (2,-0.4)$) {};
    \node[bobbel] (s1bpr) at ($(s1bb) + (2,0.4)$) {};
    \node[bobbel] (s1b-pr) at ($(s1bb) + (2,-0.4)$) {};
    \node[rectangle, draw, minimum width=14mm, rounded corners] (rPos1) at ($(s1apr) + (2,0)$) {R = $300$};
    \node[rectangle, draw, minimum width=14mm, rounded corners] (rNeg) at ($(s1a-pr) + (2,-0.3)$) {R = $0$};
    \node[rectangle, draw, minimum width=14mm, rounded corners] (rPos2) at ($(s1b-pr) + (2,0)$) {R = $300$};
    \draw (s1) edge[-] node[above left]{$a$} (s1ab);
    \draw (s1) edge[-] node[below left]{$b$} (s1bb);
    \draw (s1ab) edge[->] node[above left]{$p$} (s1apr);
    \draw (s1ab) edge[->] node[below, xshift=-2mm]{$1-p$} (s1a-pr);
    \draw (s1bb) edge[->] node[above left]{$p$} (s1bpr);
    \draw (s1bb) edge[->] node[below, xshift=-2mm]{$1-p$} (s1b-pr);
    \draw [dashed] (s1apr) edge[->] (rPos1);
    \draw [dashed] (s1a-pr) edge[->] (rNeg);
    \draw [dashed] (s1bpr) edge[->] (rNeg);
    \draw [dashed] (s1b-pr) edge[->] (rPos2);
    \node[] at ($(rPos1.east) + (1.5,0)$) {p $\in [0.1,0.9]$};
\end{tikzpicture}
    }
    \caption{An RPOMDP where agent first and nature first semantics do not coincide in their optimal value.}
    \label{app:fig:agent_vs_nature_first_RPOMDP_small}
\end{reusefigure}

For $\pi \in \Pi$, we write $\pi^{\msW} = \pi(\tsW)$.
Similarly, for $\theta \in \Theta$ of the agent first RPOMDP, we write $\theta^{a} = \theta(\tup{\tsW,a})$ and $\theta^{b} = \theta(\tup{\tsW,b})$, and for $\theta \in \Theta$ of the nature first RPOMDP, we write $\theta^{\msW} = \theta(\tsW)$.
Using this notation, we can construct the value functions for the agent first and nature interpretations of the RPOMDP in \Cref{app:fig:agent_vs_nature_first_RPOMDP_small}.
We construct these value functions following the same approach as for the value functions of the RPOMDP in \Cref{app:fig:full_vs_zero_sticky_rPOMDP}, see \Cref{app:stickiness_matters_I} and \Cref{app:fig:branches_full_vs_zero_sticky_rPOMDP}.
The value function for the agent first RPOMDP $M_1$ is:
\begin{align*}
    V_\text{fh}^{M_1}(\pi,\theta^\mix) &=  \sum_{\theta^\detr\in \Theta^\detr} \theta^\mix(\theta^\detr) \Bigl(\pi^{\msW}(a) \cdot \theta^{\detr, a}(p) \cdot 300 + \pi^{\msW}(b) \cdot (1-\theta^{\detr, b}(p)) \cdot 300\Bigr).
\intertext{And the value function for the nature first RPOMDP $M_2$:}
    V_\text{fh}^{M_2}(\pi,\theta^\mix) &=  \sum_{\theta^\detr\in \Theta^\detr} \theta^\mix(\theta^\detr) \Bigl(\pi^{\msW}(a) \cdot \theta^{\detr, \msW}(p) \cdot 300 + \pi^{\msW}(b) \cdot (1-\theta^{\detr, \msW}(p)) \cdot 300\Bigr).
\end{align*}
Both these functions are linear in the deterministic nature policies.
As we again have a convex uncertainty set, we can follow the same steps as for \Cref{app:prop_deter_pol} and restrict the search for the optimal nature policy to the set of deterministic nature policies.

Using the above functions, we can compute the optimal value for the agent first model as follows:
\begin{align*}
    V_\text{fh}^{*,M_1} &= \sup_{\pi\in \Pi}\inf_{\theta^\detr\in \Theta^\detr} \Bigl\{ \pi^{\msW}(a) \cdot \theta^{\detr, a}(p) \cdot 300 + \pi^{\msW}(b) \cdot (1-\theta^{\detr, b}(p)) \cdot 300 \Bigr\}.
\intertext{And the optimal value for the nature first model:}
    V_\text{fh}^{*,M_2} &= \sup_{\pi\in \Pi}\inf_{\theta^\detr\in \Theta^\detr} \Bigl\{ \pi^{\msW}(a) \cdot \theta^{\detr, \msW}(p) \cdot 300 + \pi^{\msW}(b) \cdot (1-\theta^{\detr, \msW}(p)) \cdot 300 \Bigr\}.
\end{align*}

\Cref{app:tab:agent_vs_nature_first_RPOMDP_small} displays the computed optimal values and policies, showing the differences between the agent and nature first assumptions.
An underscore indicates that the choice at this history does not influence the optimal value of the RPOMDP.
{\def\arraystretch{1.3}
\begin{table}[h]
    \centering
    \begin{tabular}{l|p{32mm} p{37mm}}
    \toprule
         & Agent first & Nature first\\\midrule
        Optimal value & 30 & 150 \\\hline
        Optimal agent policy & $\tsW \mapsto \_$ & $\tsW \mapsto \{a \mapsto 0.5, b\mapsto 0.5\}$\\\hline
        Optimal nature policy & $\tup{\tsW,a} \mapsto \{p \mapsto 0.1\},$\newline$\tup{\tsW,b} \mapsto \{p \mapsto 0.9\}$ & $\tsW \mapsto \{p \mapsto 0.5\}$\\
        \bottomrule
    \end{tabular}
    \caption{Optimal values and policies for the agent first and nature first interpretations of the RPOMDP in \Cref{fig:agent_vs_nature_first_RPOMDP_small}.}
    \label{app:tab:agent_vs_nature_first_RPOMDP_small}
\end{table}}

\subsubsection{Underlying POSGs}
\Cref{app:fig:order:POSGs} (restated below) depicts the agent first and nature first POSGs of the RPOMDP in \Cref{app:fig:agent_vs_nature_first_RPOMDP_small}.
We briefly discuss the structural differences.
As the agent has a finite choice of actions, this will always lead to a finite split in the POSG.
Nature's number of choices depends on the variable restrictions in the nature state.
In this simple model, we only have unrestricted, and hence infinite, nature choices.
The structural difference between the two POSGs is caused entirely by the order of play.

When the agent chooses first, the POSG has a finite number of states.
The infinite choice in the nature states $\tup{s_1,\{\},a}$ and $\tup{s_1,\{\},b}$ all lead to the same reward states, just with different probabilities determined by the chosen variable assignment.

When nature chooses first, we get an infinite number of agent states after nature's infinite choice in nature state $\tup{s_1,\{\}}$, as the chosen variable assignment needs to be recorded for the transition after the agent's choice.
The number of states in the nature first model is hence infinite.
Each resulting agent state only has a finite choice leading to the same reward states.
\begin{reusefigure}[ht]{fig:order:POSGs}
    \centering
    \begin{subfigure}[c]{0.43\columnwidth}
        \resizebox{\columnwidth}{!}{
        \begin{tikzpicture}[state/.append style={shape = ellipse}, >=stealth,
    bobbel/.style={minimum size=1mm,inner sep=0pt,fill=black,circle},
    mynode/.style={rectangle,fill=white,anchor=center}]]
    \node[state, inner sep=-0.2pt] (s1) at (1,0) {\scriptsize $s_1,\{\}$};
    \node[state, inner sep=-2pt] (s1_a) at (3.2,1.15) {\scriptsize $s_1,\{\},a$};
    \node[state, inner sep=-2pt] (s1_b) at (3.2,-1.15) {\scriptsize $s_1,\{\},b$};
    \node[bobbel] (s1ba_u1) at ($(s1_a) + (2.5,0.85)$) {};
    \node[bobbel] (s1ba_u2) at ($(s1_a) + (2.5,-0.85)$) {};
    \node[bobbel] (s1bb_u1) at ($(s1_b) + (2.5,0.85)$) {};
    \node[bobbel] (s1bb_u2) at ($(s1_b) + (2.5,-0.85)$) {};
    \node[rectangle, draw, minimum width=11mm, rounded corners] (rPosa) at ($(s1ba_u1) + (1.5,-0.25)$) {\scriptsize R = $300$};
    \node[rectangle, draw, minimum width=11mm, rounded corners] (rNega) at ($(s1ba_u2) + (1.5,0.25)$) {\scriptsize R = $0$};
    \node[rectangle, draw, minimum width=11mm, rounded corners] (rNegb) at ($(s1bb_u1) + (1.5,-0.25)$) {\scriptsize R = $0$};
    \node[rectangle, draw, minimum width=11mm, rounded corners] (rPosb) at ($(s1bb_u2) + (1.5,0.25)$) {\scriptsize R = $300$};
    \draw[<-] (s1.west) -- +(-0.4,0);
    \draw (s1) edge[->] node[above left, yshift=-0.5mm, xshift=1mm]{\scriptsize $a$} (s1_a);
    \draw (s1) edge[->] node[below left, yshift=0.5mm, xshift=1mm]{\scriptsize $b$} (s1_b);
    \draw (s1ba_u1) edge[->] node[above, yshift=-0.3mm, xshift=1mm]{\scriptsize $0.1$} (rPosa);
    \draw (s1ba_u1) edge[->] node[below left, pos=0.4, yshift=1mm, xshift=0.8mm]{\scriptsize $0.9$} (rNega);
    \draw (s1ba_u2) edge[->] node[above left, pos=0.4, yshift=-1mm, xshift=0.8mm]{\scriptsize $0.9$} (rPosa);
    \draw (s1ba_u2) edge[->] node[below, yshift=0.3mm, xshift=1mm]{\scriptsize $0.1$} (rNega);
    \draw (s1_a) edge[->] node[sloped, anchor=center, above]{\scriptsize $\dots$} ($(s1ba_u1) + (0,-0.55)$);
    \draw (s1_a) edge[->] node[sloped, anchor=center, below]{\scriptsize $\dots$} ($(s1ba_u2) + (0,0.55)$);
    \node[] at ($(s1_a -| s1ba_u1) + (-0.25,0.6)$) {\tiny $\vdots$};
    \node[] at ($(s1_a -| s1ba_u1) + (-0.25,0.1)$) {\tiny $\vdots$};
    \node[] at ($(s1_a -| s1ba_u1) + (-0.25,-0.4)$) {\tiny $\vdots$};
    \draw (s1_a) edge node[sloped, anchor=center, above, xshift=-0.2mm, yshift=-0.6mm]{\scriptsize $\{p\mapsto 0.1\}$} (s1ba_u1);
    \draw (s1_a) edge node[sloped, anchor=center, below, xshift=-0.2mm, yshift=0.6mm]{\scriptsize $\{p\mapsto 0.9\}$} (s1ba_u2);
    \draw (s1bb_u1) edge[->] node[above, yshift=-0.3mm, xshift=1mm]{\scriptsize $0.1$} (rNegb);
    \draw (s1bb_u1) edge[->] node[below left, pos=0.4, yshift=1mm, xshift=0.8mm]{\scriptsize $0.9$} (rPosb);
    \draw (s1bb_u2) edge[->] node[above left, pos=0.4, yshift=-1mm, xshift=0.8mm]{\scriptsize $0.9$} (rNegb);
    \draw (s1bb_u2) edge[->] node[below, yshift=0.3mm, xshift=1mm]{\scriptsize $0.1$} (rPosb);
    \draw (s1_b) edge[->] node[sloped, anchor=center, above]{\scriptsize $\dots$} ($(s1bb_u1) + (0,-0.55)$);
    \draw (s1_b) edge[->] node[sloped, anchor=center, below]{\scriptsize $\dots$} ($(s1bb_u2) + (0,0.55)$);
    \node[] at ($(s1_b -| s1bb_u1) + (-0.25,0.6)$) {\tiny $\vdots$};
    \node[] at ($(s1_b -| s1bb_u1) + (-0.25,0.1)$) {\tiny $\vdots$};
    \node[] at ($(s1_b -| s1bb_u1) + (-0.25,-0.4)$) {\tiny $\vdots$};
    \draw (s1_b) edge node[sloped, anchor=center, above, xshift=-0.2mm, yshift=-0.6mm]{\scriptsize $\{p\mapsto 0.1\}$} (s1bb_u1);
    \draw (s1_b) edge node[sloped, anchor=center, below, xshift=-0.2mm, yshift=0.6mm]{\scriptsize $\{p\mapsto 0.9\}$} (s1bb_u2);
\end{tikzpicture}%
        }
    \end{subfigure}%
    \quad
    \begin{subfigure}[c]{0.43\columnwidth}
        \resizebox{\columnwidth}{!}{
        \begin{tikzpicture}[state/.append style={shape = ellipse}, >=stealth,
    bobbel/.style={minimum size=1mm,inner sep=0pt,fill=black,circle}]
    \node[state, inner sep=-0.2pt] (s1) at (1,0) {\scriptsize $s_1,\{\}$};
    \node[state, inner sep=-4pt] (s1_u1) at (4.2,0.85) {\scriptsize $s_1,\{\},\{p\mapsto 0.1\}$};
    \node[state, inner sep=-4pt] (s1_u2) at (4.2,-0.85) {\scriptsize $s_1,\{\},\{p\mapsto 0.9\}$};
    \node[bobbel] (s1_u1_ab) at ($(s1_u1) + (1.5,0.6)$) {};
    \node[bobbel] (s1_u1_bb) at ($(s1_u1) + (1.5,-0.6)$) {};
    \node[bobbel] (s1_u2_ab) at ($(s1_u2) + (1.5,0.6)$) {};
    \node[bobbel] (s1_u2_bb) at ($(s1_u2) + (1.5,-0.6)$) {};
    \node[rectangle, draw, minimum width=11mm, rounded corners] (rPosa) at ($(s1_u1_ab) + (1.5,-0.25)$) {\scriptsize R = $300$};
    \node[rectangle, draw, minimum width=11mm, rounded corners] (rNega) at ($(s1_u1_bb) + (1.5,0.25)$) {\scriptsize R = $0$};
    \node[rectangle, draw, minimum width=11mm, rounded corners] (rPosb) at ($(s1_u2_ab) + (1.5,-0.25)$) {\scriptsize R = $300$};
    \node[rectangle, draw, minimum width=11mm, rounded corners] (rNegb) at ($(s1_u2_bb) + (1.5,0.25)$) {\scriptsize R = $0$};
    \draw[<-] (s1.west) -- +(-0.4,0);
    \draw (s1) edge[->] node[sloped, anchor=center, above, xshift=-0.2mm, yshift=-0.6mm]{\scriptsize $\{p\mapsto 0.1\}$} (s1_u1.west);
    \draw (s1) edge[->] node[sloped, anchor=center, below, xshift=-0.2mm, yshift=0.6mm]{\scriptsize $\{p\mapsto 0.9\}$} (s1_u2.west);
    \draw (s1) edge[->] node[sloped, anchor=center, above]{\scriptsize $\dots$} (s1_u1.west |- 0,0.3);
    \draw (s1) edge[->] node[sloped, anchor=center, below]{\scriptsize $\dots$} (s1_u1.west |- 0,-0.3);
    \node[] at ($(s1_u1.west |- 1,0.6) + (-0.25,0)$) {\tiny $\vdots$};
    \node[] at ($(s1_u1.west |- 1,0.1) + (-0.25,0)$) {\tiny $\vdots$};
    \node[] at ($(s1_u1.west |- 1,-0.38) + (-0.25,0)$) {\tiny $\vdots$};
    \node[] at (s1_u1 |- 1,0.1) {\tiny $\vdots$};
    \draw (s1_u1) edge node[above left, yshift=-0.5mm, xshift=1mm]{\scriptsize $a$} (s1_u1_ab);
    \draw (s1_u1) edge node[below left, yshift=0.5mm, xshift=1mm]{\scriptsize $b$} (s1_u1_bb);
    \draw (s1_u1_ab) edge[->] node[above, yshift=-0.3mm, xshift=1mm]{\scriptsize $0.1$} (rPosa);
    \draw (s1_u1_ab) edge[->] node[below left, pos=0.4, yshift=1.5mm]{\scriptsize $0.9$} (rNega);
    \draw (s1_u1_bb) edge[->] node[above left, pos=0.4, yshift=-1.5mm]{\scriptsize $0.9$} (rPosa);
    \draw (s1_u1_bb) edge[->] node[below, yshift=0.3mm, xshift=1mm]{\scriptsize $0.1$} (rNega);
    \draw (s1_u2) edge node[above left, yshift=-0.5mm, xshift=1mm]{\scriptsize $a$} (s1_u2_ab);
    \draw (s1_u2) edge node[below left, yshift=0.5mm, xshift=1mm]{\scriptsize $b$} (s1_u2_bb);
    \draw (s1_u2_ab) edge[->] node[above, yshift=-0.3mm, xshift=1mm]{\scriptsize $0.9$} (rPosb);
    \draw (s1_u2_ab) edge[->] node[below left, pos=0.4, yshift=1.5mm]{\scriptsize $0.1$} (rNegb);
    \draw (s1_u2_bb) edge[->] node[above left, pos=0.4, yshift=-1.5mm]{\scriptsize $0.1$} (rPosb);
    \draw (s1_u2_bb) edge[->] node[below, yshift=0.3mm, xshift=1mm]{\scriptsize $0.9$} (rNegb);
\end{tikzpicture}%
        }
    \end{subfigure}
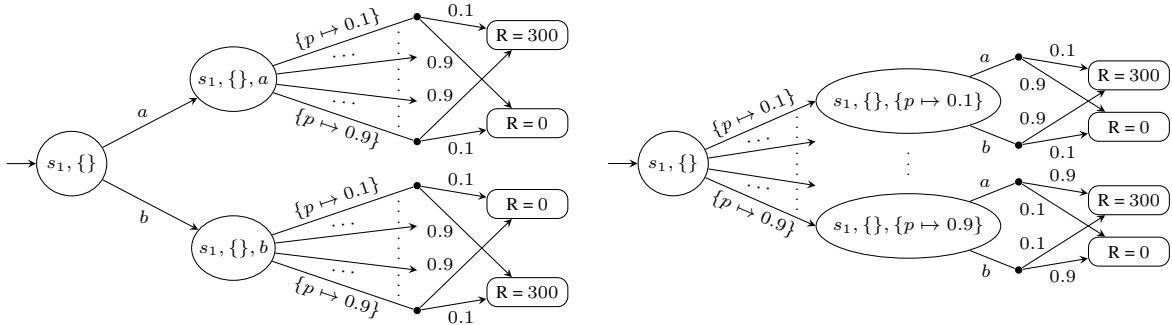
    \caption{Agent first (left) and nature first (right) POSGs of the RPOMDP in \Cref{app:fig:agent_vs_nature_first_RPOMDP_small}.}
    \label{app:fig:order:POSGs}
\end{reusefigure}

\subsection{$a$-Rectangularity}\label{app:order_of_play_matters_II}
Next, we look at an $a$-rectangular RPOMDP to show that order of play still matters under a form of rectangularity and is not only a concern in non-rectangular RPOMDPs.
Consider the RPOMDP in \Cref{app:fig:agent_vs_nature_first_RPOMDP_bigger}.
We interpret this RPOMDP with full stickiness semantics.

\begin{figure}[H]
    \centering
    \resizebox{0.6\columnwidth}{!}{
    \begin{tikzpicture}[state/.append style={shape = circle}, >=stealth,
    bobbel/.style={minimum size=1mm,inner sep=0pt,fill=black,circle}]
    \node[state] (s1) at (1,0) {$s_1$};
    \node[bobbel] (s1ba) at ($(s1) + (2,2)$) {};
    \node[bobbel] (s1bb) at ($(s1) + (2,0)$) {};
    \node[bobbel] (s1bap) at ($(s1ba) + (2,0.4)$) {};
    \node[bobbel] (s1ba-p) at ($(s1ba) + (2,-0.4)$) {};
    \node[bobbel] (s1bbq) at ($(s1bb) + (2,0.8)$) {};
    \node[bobbel] (s1bb-q) at ($(s1bb) + (2,0)$) {};
    \node[state] (s2) at ($(s1bb) + (2,-1.6)$) {$s_2$};
    \node[bobbel] (s2ba) at ($(s2) + (2,0)$) {};
    \node[bobbel] (s2bb) at ($(s2) + (2,-1.6)$) {};
    \node[bobbel] (s2bap) at ($(s2ba) + (2,0.4)$) {};
    \node[bobbel] (s2ba-p) at ($(s2ba) + (2,-0.4)$) {};
    \node[bobbel] (s2bbq) at ($(s2bb) + (2,0.4)$) {};
    \node[bobbel] (s2bb-q) at ($(s2bb) + (2,-0.4)$) {};
    \node[rectangle, draw, minimum width=14mm, rounded corners] (rPos) at ($(s1bap) + (2,0)$) {R = $300$};
    \node[rectangle, draw, minimum width=14mm, rounded corners] (rNeg1) at ($(s1bap) + (2,-1.2)$) {R = $0$};
    \node[rectangle, draw, minimum width=14mm, rounded corners] (rNeu1) at ($(s1bb-q) + (2,0)$) {R = $100$};
    \node[rectangle, draw, minimum width=14mm, rounded corners] (rNeg2) at ($(s2bap) + (2,0)$) {R = $0$};
    \node[rectangle, draw, minimum width=14mm, rounded corners] (rNeu2) at ($(s2bap) + (2,-1.2)$) {R = $100$};
    \node[rectangle, draw, minimum width=14mm, rounded corners] (rNeg3) at ($(s2bb-q) + (2,0)$) {R = $0$};
    \draw[<-] (s1.west) -- +(-0.6,0);
    \draw (s1) edge[-] node[above left]{$a$} (s1ba);
    \draw (s1) edge[-] node[below left]{$b$} (s1bb);
    \draw (s1ba) edge[-] node[above left]{$p$} (s1bap);
    \draw (s1ba) edge[-] node[below, xshift=-2mm]{$1-p$} (s1ba-p);
    \draw (s1bb) edge[-] node[above left]{$q$} (s1bbq);
    \draw (s1bb) edge[-] node[below, pos = 0.6]{$0.5-q$} (s1bb-q);
    \draw (s1bb) edge[->] node[below left]{$0.5$} (s2);
    \draw (s2) edge[-] node[above left]{$a$} (s2ba);
    \draw (s2) edge[-] node[below left]{$b$} (s2bb);
    \draw (s2ba) edge[-] node[above left]{$p$} (s2bap);
    \draw (s2ba) edge[-] node[below, xshift=-2mm]{$1-p$} (s2ba-p);
    \draw (s2bb) edge[-] node[above left]{$q$} (s2bbq);
    \draw (s2bb) edge[-] node[below, xshift=-2mm]{$1-q$} (s2bb-q);
    \draw [dashed] (s1bap) edge[->] (rPos);
    \draw [dashed] (s1ba-p) edge[->] (rNeg1);
    \draw [dashed] (s1bbq) edge[->] (rNeg1);
    \draw [dashed] (s1bb-q) edge[->] (rNeu1);
    \draw [dashed] (s2bap) edge[->] (rNeg2);
    \draw [dashed] (s2ba-p) edge[->] (rNeu2);
    \draw [dashed] (s2bbq) edge[->] (rNeu2);
    \draw [dashed] (s2bb-q) edge[->] (rNeg3);
    \node[anchor = west] at ($(s1) + (-1.1,-3)$) {$p \in [0.1,0.4]$};
    \node[anchor = west] at ($(s1) + (-1.1,-3.5)$) {$q \in [0.1,0.4]$};
\end{tikzpicture}
    }
    \caption{An $a$-rectangular RPOMDP where agent first and nature first semantics do not coincide in their optimal value.}
    \label{app:fig:agent_vs_nature_first_RPOMDP_bigger}
\end{figure}
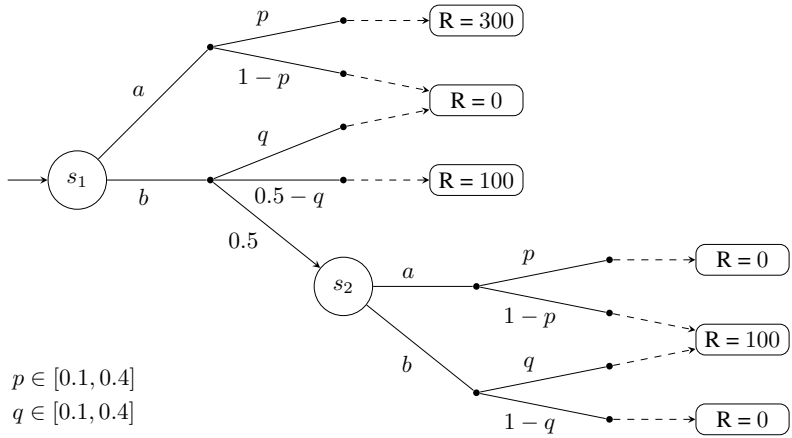

For $\pi \in \Pi$, we write $\pi^{\msW} = \pi(\tsW)$ and $\pi^{\msW\msW} = \pi(\tsW\tsW)$.
Similarly, for $\theta \in \Theta$ of the agent first RPOMDP, we write $\theta^a= \theta(\tup{\tsW,a})$ and $\theta^b= \theta(\tup{\tsW,b})$, and for $\theta \in \Theta$ of the nature first RPOMDP, we write $\theta^{\msW} = \theta(\tsW)$.
Using this notation, we can construct the value functions for the agent first and nature interpretations of the RPOMDP in \Cref{app:fig:agent_vs_nature_first_RPOMDP_bigger}.
We construct these value functions following the same approach as for the value functions of the RPOMDP in \Cref{app:fig:full_vs_zero_sticky_rPOMDP}, see \Cref{app:stickiness_matters_I} and \Cref{app:fig:branches_full_vs_zero_sticky_rPOMDP}.
The value function for the agent first RPOMDP $M_1$ is:
\begin{align*}
    V_\text{fh}^{M_1}(\pi,\theta^\mix) &= \sum_{\theta^\detr\in \Theta^\detr} \theta^\mix(\theta^\detr) \Bigl( \pi^{\msW}(a) \cdot \theta^{\detr,a}(p)\cdot 300 + \pi^{\msW}(b) \cdot \bigl((0.5 - \theta^{\detr,b}(q))\cdot 100 \\
    & \hspace{100pt} + 0.5 \cdot (\pi^{\msW\msW}(a) \cdot (1-\theta^{\detr,b}(p)) \cdot 100 + \pi^{\msW\msW}(b) \cdot \theta^{\detr,b}(q) \cdot 100) \bigr)\Bigr).
\intertext{And the value function for the nature first RPOMDP $M_2$:}
    V_\text{fh}^{M_1}(\pi,\theta^\mix) &= \sum_{\theta^\detr\in \Theta^\detr} \theta^\mix(\theta^\detr) \Bigl( \pi^{\msW}(a) \cdot \theta^{\detr\msW}(p)\cdot 300 + \pi^{\msW}(b) \cdot \bigl((0.5 - \theta^{\detr\msW}(q))\cdot 100 \\
    & \hspace{100pt} + 0.5 \cdot (\pi^{\msW\msW}(a) \cdot (1-\theta^{\detr\msW}(p)) \cdot 100 + \pi^{\msW\msW}(b) \cdot \theta^{\detr\msW}(q) \cdot 100) \bigr)\Bigr).
\end{align*}
Both these functions are linear in the deterministic nature policies.
As we again have a convex uncertainty set, we can follow the same steps as for \Cref{app:prop_deter_pol} and restrict the search for the optimal nature policy to the set of deterministic nature policies.

Using the above functions, we can compute the optimal value for the agent first model as follows:
\begin{align*}
    V_\text{fh}^{*,M_1} &= \sup_{\pi\in \Pi}\inf_{\theta^\detr\in \Theta^\detr} \Bigl\{ \pi^{\msW}(a) \cdot \theta^{\detr,a}(p)\cdot 300 + \pi^{\msW}(b) \cdot \bigl((0.5 - \theta^{\detr,b}(q))\cdot 100 \\
    & \hspace{72pt} + 0.5 \cdot (\pi^{\msW\msW}(a) \cdot (1-\theta^{\detr,b}(p)) \cdot 100 + \pi^{\msW\msW}(b) \cdot \theta^{\detr,b}(q) \cdot 100) \bigr)\Bigr\}.
\intertext{And the optimal value for the nature first RPOMDP $M_2$:}
    V_\text{fh}^{*,M_2} &= \sup_{\pi\in \Pi}\inf_{\theta^\detr\in \Theta^\detr} \Bigl\{ \pi^{\msW}(a) \cdot \theta^{\detr,\msW}(p)\cdot 300 + \pi^{\msW}(b) \cdot \bigl((0.5 - \theta^{\detr,\msW}(q))\cdot 100 \\
    & \hspace{72pt} + 0.5 \cdot (\pi^{\msW\msW}(a) \cdot (1-\theta^{\detr,\msW}(p)) \cdot 100 + \pi^{\msW\msW}(b) \cdot \theta^{\detr,\msW}(q) \cdot 100) \bigr)\Bigr\}.
\end{align*}

\Cref{app:tab:agent_vs_nature_first_RPOMDP_bigger} displays the computed optimal values and policies, showing the differences between the agent and nature first assumptions.
An underscore indicates that the value assigned to this variable does not influence the optimal value of the RPOMDP.
{\def\arraystretch{1.3}
\begin{table}[h]
    \centering
    \begin{tabular}{l|p{45mm} p{40mm}}
    \toprule
         & Agent first & Nature first\\\midrule
        Optimal value & 40 & $51\frac{3}{7}$\\\hline
        Optimal agent policy & $\tsW \mapsto \{a \mapsto 0, b\mapsto 1\}$,\newline$\tsW\tsW \mapsto \{a \mapsto 1, b\mapsto 0\}$ & $\tsW \mapsto \{a \mapsto \frac{1}{7}, b\mapsto \frac{6}{7}\}$,\newline$\tsW\tsW \mapsto \{a \mapsto 1, b\mapsto 0\}$\\\hline
        Optimal nature policy & $\tup{\tsW,a} \mapsto \{p\mapsto 0.1, q\mapsto \_\},$\newline$\tup{\tsW,b} \mapsto \{p\mapsto 0.4, q\mapsto 0.4\}$ & $\tsW \mapsto \{p\mapsto \frac{6}{35}, q\mapsto 0.4\}$\\
        \bottomrule
    \end{tabular}
    \caption{Optimal values and agent policies for the agent first and nature first interpretations of the RPOMDP in \Cref{app:fig:agent_vs_nature_first_RPOMDP_bigger}.}
    \label{app:tab:agent_vs_nature_first_RPOMDP_bigger}
\end{table}}

\subsubsection{Underlying POSGs}
\Cref{fig:order:POSGs_bigger} depicts the agent first and nature first POSGs of the RPOMDP in \Cref{app:fig:agent_vs_nature_first_RPOMDP_bigger}.
The structural difference between these POSGs, like between the POSGs in \Cref{app:fig:order:POSGs}, is caused by the different points of infinite branching.

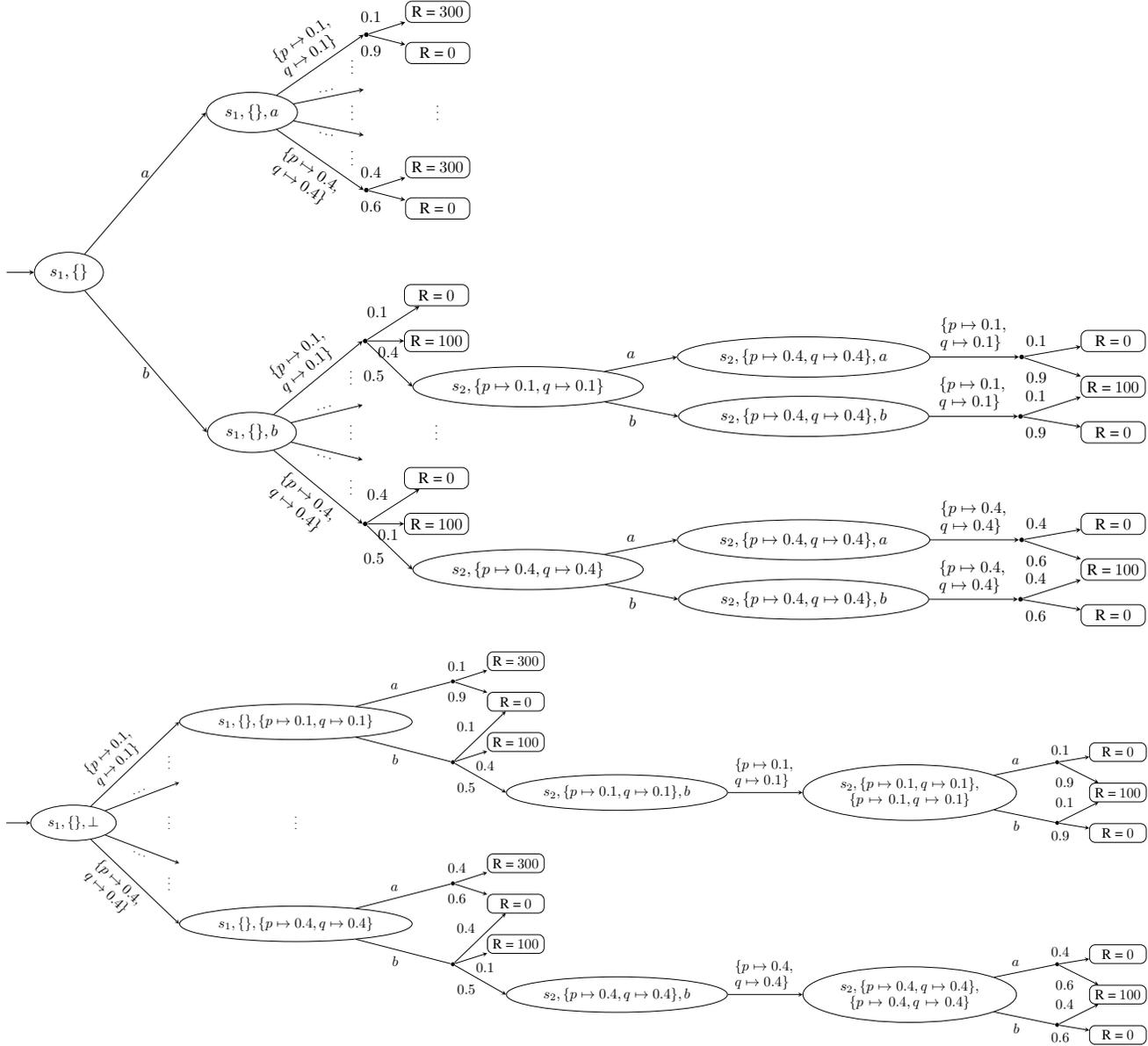
\begin{figure}[H]
    \centering
    \begin{subfigure}[b]{0.98\columnwidth}
        \resizebox{\columnwidth}{!}{
        \begin{tikzpicture}[state/.append style={shape = ellipse}, >=stealth,
    bobbel/.style={minimum size=1mm,inner sep=0pt,fill=black,circle}]
    \node[state] (s1) at (1,0) {$s_1,\{\}$};
    \node[state] (s1_a) at (5,3.5) {$s_1,\{\},a$};
    \node[state] (s1_b) at (5,-3.5) {$s_1,\{\},b$};
    \node[bobbel] (s1_a_u1) at ($(s1_a.east) + (1.5,1.7)$) {};
    \node[bobbel] (s1_a_u2) at ($(s1_a.east) + (1.5,-1.7)$) {};
    \node[bobbel] (s1_b_u1) at ($(s1_b.east) + (1.5,2)$) {};
    \node[bobbel] (s1_b_u2) at ($(s1_b.east) + (1.5,-2)$) {};
    \node[rectangle, draw, minimum width=14mm, rounded corners] (rPos_a_u1) at ($(s1_a_u1.east) + (1.5,0.5)$) {R = $300$};
    \node[rectangle, draw, minimum width=14mm, rounded corners] (rNeg_a_u1) at ($(s1_a_u1.east) + (1.5,-0.4)$) {R = $0$};
    \node[rectangle, draw, minimum width=14mm, rounded corners] (rPos_a_u2) at ($(s1_a_u2.east) + (1.5,0.5)$) {R = $300$};
    \node[rectangle, draw, minimum width=14mm, rounded corners] (rNeg_a_u2) at ($(s1_a_u2.east) + (1.5,-0.4)$) {R = $0$};
    \node[rectangle, draw, minimum width=14mm, rounded corners] (rNeg_b_u1) at ($(s1_b_u1.east) + (1.5,1)$) {R = $0$};
    \node[rectangle, draw, minimum width=14mm, rounded corners] (rNeu_b_u1) at ($(s1_b_u1.east) + (1.5,0)$) {R = $100$};
    \node[rectangle, draw, minimum width=14mm, rounded corners] (rNeg_b_u2) at ($(s1_b_u2.east) + (1.5,1)$) {R = $0$};
    \node[rectangle, draw, minimum width=14mm, rounded corners] (rNeu_b_u2) at ($(s1_b_u2.east) + (1.5,0)$) {R = $100$};
    \node[state] (s2_u1) at ($(s1_b_u1.east) + (3.5,-1)$) {$s_2,\{p\mapsto 0.1,q \mapsto 0.1\}$};
    \node[state] (s2_u2) at ($(s1_b_u2.east) + (3.5,-1)$) {$s_2,\{p\mapsto 0.4,q \mapsto 0.4\}$};
    \node[state] (s2_u1_a) at ($(s2_u1.east) + (3.5,0.65)$) {$s_2,\{p\mapsto 0.4,q \mapsto 0.4\},a$};
    \node[state] (s2_u1_b) at ($(s2_u1.east) + (3.5,-0.65)$) {$s_2,\{p\mapsto 0.4,q \mapsto 0.4\},b$};
    \node[state] (s2_u2_a) at ($(s2_u2.east) + (3.5,0.65)$) {$s_2,\{p\mapsto 0.4,q \mapsto 0.4\},a$};
    \node[state] (s2_u2_b) at ($(s2_u2.east) + (3.5,-0.65)$) {$s_2,\{p\mapsto 0.4,q \mapsto 0.4\},b$};
    \node[bobbel] (s2_u1_ab) at ($(s2_u1_a.east) + (2,0)$) {};
    \node[bobbel] (s2_u1_bb) at ($(s2_u1_b.east) + (2,0)$) {};
    \node[bobbel] (s2_u2_ab) at ($(s2_u2_a.east) + (2,0)$) {};
    \node[bobbel] (s2_u2_bb) at ($(s2_u2_b.east) + (2,0)$) {};
    \node[rectangle, draw, minimum width=14mm, rounded corners] (rNega2_1_1) at ($(s2_u1.east -| s2_u1_ab) + (2,1)$) {R = $0$};
    \node[rectangle, draw, minimum width=14mm, rounded corners] (rNeu2_1) at ($(s2_u1.east -| s2_u1_ab) + (2,0)$) {R = $100$};
    \node[rectangle, draw, minimum width=14mm, rounded corners] (rNega2_1_2) at ($(s2_u1.east -| s2_u1_ab) + (2,-1)$) {R = $0$};
    \node[rectangle, draw, minimum width=14mm, rounded corners] (rNega2_2_1) at ($(s2_u2.east -| s2_u2_ab) + (2,1)$) {R = $0$};
    \node[rectangle, draw, minimum width=14mm, rounded corners] (rNeu2_2) at ($(s2_u2.east -| s2_u2_ab) + (2,0)$) {R = $100$};
    \node[rectangle, draw, minimum width=14mm, rounded corners] (rNega2_2_2) at ($(s2_u2.east -| s2_u2_ab) + (2,-1)$) {R = $0$};
    \draw[<-] (s1.west) -- +(-0.6,0);
    \draw (s1) edge[->] node[above]{$a$} (s1_a.west);
    \draw (s1_a) edge[->] node[text width = 1.5cm, sloped, anchor=center, above]{$\{p\mapsto 0.1,$\\$q \mapsto 0.1\}$} (s1_a_u1.west);
    \draw (s1_a) edge[->] node[text width = 1.5cm, sloped, anchor=center, below]{$\{p\mapsto 0.4,$\\$q \mapsto 0.4\}$} (s1_a_u2.west);
    \draw (s1_a) edge[->] node[sloped, anchor=center, above]{\scriptsize $\dots$} ($(s1_a_u1.west) + (0,-1.2)$);
    \draw (s1_a) edge[->] node[sloped, anchor=center, below]{\scriptsize $\dots$} ($(s1_a_u2.west) + (0,1.2)$);
    \node[] at ($(s1_a_u1.west) + (-0.25,-0.6)$) {\scriptsize $\vdots$};
    \node[] at ($(s1_a_u1.west) + (-0.25,-1.6)$) {\scriptsize $\vdots$};
    \node[] at ($(s1_a_u2.west) + (-0.25,0.8)$) {\scriptsize $\vdots$};
    \node[] at ($(s1_a_u1.east) + (1.5,-1.6)$) {\scriptsize $\vdots$};
    \draw (s1_a_u1) edge[->] node[above left]{$0.1$} (rPos_a_u1);
    \draw (s1_a_u1) edge[->] node[below left]{$0.9$} (rNeg_a_u1);
    \draw (s1_a_u2) edge[->] node[above left]{$0.4$} (rPos_a_u2);
    \draw (s1_a_u2) edge[->] node[below left]{$0.6$} (rNeg_a_u2);
    \draw (s1) edge[->] node[below]{$b$} (s1_b.west);
    \draw (s1_b) edge[->] node[text width = 1.5cm, sloped, anchor=center, above]{$\{p\mapsto 0.1,$\\$q \mapsto 0.1\}$} (s1_b_u1.west);
    \draw (s1_b) edge[->] node[text width = 1.5cm, sloped, anchor=center, below]{$\{p\mapsto 0.4,$\\$q \mapsto 0.4\}$} (s1_b_u2.west);
    \draw (s1_b) edge[->] node[sloped, anchor=center, above]{\scriptsize $\dots$} ($(s1_b_u1.west) + (0,-1.4)$);
    \draw (s1_b) edge[->] node[sloped, anchor=center, below]{\scriptsize $\dots$} ($(s1_b_u2.west) + (0,1.4)$);
    \node[] at ($(s1_b_u1.west) + (-0.25,-0.7)$) {\scriptsize $\vdots$};
    \node[] at ($(s1_b_u1.west) + (-0.25,-1.9)$) {\scriptsize $\vdots$};
    \node[] at ($(s1_b_u2.west) + (-0.25,0.9)$) {\scriptsize $\vdots$};
    \node[] at ($(s1_b_u1.east) + (1.5,-1.9)$) {\scriptsize $\vdots$};
    \draw (s1_b_u1) edge[->] node[above left]{$0.1$} (rNeg_b_u1);
    \draw (s1_b_u1) edge[->] node[below left,pos=1]{$0.4$} (rNeu_b_u1);
    \draw (s1_b_u1) edge[->] node[below left]{$0.5$} (s2_u1.west);
    \draw (s1_b_u2) edge[->] node[above left]{$0.4$} (rNeg_b_u2);
    \draw (s1_b_u2) edge[->] node[below left,pos=1]{$0.1$} (rNeu_b_u2);
    \draw (s1_b_u2) edge[->] node[below left]{$0.5$} (s2_u2.west);
    \draw (s2_u1) edge[->] node[above left]{$a$} (s2_u1_a.west);
    \draw (s2_u1) edge[->] node[below left]{$b$} (s2_u1_b.west);
    \draw (s2_u2) edge[->] node[above left]{$a$} (s2_u2_a.west);
    \draw (s2_u2) edge[->] node[below left]{$b$} (s2_u2_b.west);
    \draw (s2_u1_a.east) edge[->] node[text width = 1.5cm, sloped, anchor=center, above]{$\{p\mapsto 0.1,$\\$q \mapsto 0.1\}$} (s2_u1_ab.west);
    \draw (s2_u1_b.east) edge[->] node[text width = 1.5cm, sloped, anchor=center, above]{$\{p\mapsto 0.1,$\\$q \mapsto 0.1\}$} (s2_u1_bb.west);
    \draw (s2_u2_a.east) edge[->] node[text width = 1.5cm, sloped, anchor=center, above]{$\{p\mapsto 0.4,$\\$q \mapsto 0.4\}$} (s2_u2_ab.west);
    \draw (s2_u2_b.east) edge[->] node[text width = 1.5cm, sloped, anchor=center, above]{$\{p\mapsto 0.4,$\\$q \mapsto 0.4\}$} (s2_u2_bb.west);
    \draw (s2_u1_ab) edge[->] node[above left]{$0.1$} (rNega2_1_1);
    \draw (s2_u1_ab) edge[->] node[below left]{$0.9$} (rNeu2_1);
    \draw (s2_u1_bb) edge[->] node[above left]{$0.1$} (rNeu2_1);
    \draw (s2_u1_bb) edge[->] node[below left]{$0.9$} (rNega2_1_2);
    \draw (s2_u2_ab) edge[->] node[above left]{$0.4$} (rNega2_2_1);
    \draw (s2_u2_ab) edge[->] node[below left]{$0.6$} (rNeu2_2);
    \draw (s2_u2_bb) edge[->] node[above left]{$0.4$} (rNeu2_2);
    \draw (s2_u2_bb) edge[->] node[below left]{$0.6$} (rNega2_2_2);
\end{tikzpicture}
        }
    \end{subfigure}\\[3.5mm]
    \begin{subfigure}[b]{0.98\columnwidth}
        \resizebox{\columnwidth}{!}{
        \begin{tikzpicture}[state/.append style={shape = ellipse}, >=stealth,
    bobbel/.style={minimum size=1mm,inner sep=0pt,fill=black,circle}]
    \node[state] (s1) at (1,0) {$s_1,\{\},\bot$};
    \node[state] (s1_u1) at (6.5,2.5) {$s_1,\{\},\{p\mapsto 0.1, q\mapsto 0.1\}$};
    \node[state] (s1_u2) at (6.5,-2.5) {$s_1,\{\},\{p\mapsto 0.4, q\mapsto 0.4\}$};
    \node[bobbel] (s1_u1_ab) at ($(s1_u1.east) + (1,1)$) {};
    \node[bobbel] (s1_u1_bb) at ($(s1_u1.east) + (1,-1)$) {};
    \node[bobbel] (s1_u2_ab) at ($(s1_u2.east) + (1,1)$) {};
    \node[bobbel] (s1_u2_bb) at ($(s1_u2.east) + (1,-1)$) {};
    \node[rectangle, draw, minimum width=14mm, rounded corners] (rPosa1) at ($(s1_u1_ab.east) + (1.5,0.5)$) {R = $300$};
    \node[rectangle, draw, minimum width=14mm, rounded corners] (rNega1) at ($(s1_u1_ab.east) + (1.5,-0.5)$) {R = $0$};
    \node[rectangle, draw, minimum width=14mm, rounded corners] (rNeu1) at ($(s1_u1_ab.east) + (1.5,-1.5)$) {R = $100$};
    \node[rectangle, draw, minimum width=14mm, rounded corners] (rPosa2) at ($(s1_u2_ab.east) + (1.5,0.5)$) {R = $300$};
    \node[rectangle, draw, minimum width=14mm, rounded corners] (rNega2) at ($(s1_u2_ab.east) + (1.5,-0.5)$) {R = $0$};
    \node[rectangle, draw, minimum width=14mm, rounded corners] (rNeu2) at ($(s1_u2_ab.east) + (1.5,-1.5)$) {R = $100$};
    \node[state] (s2_u1) at ($(s1_u1_bb.east) + (4,-0.75)$) {$s_2,\{p\mapsto 0.1, q\mapsto 0.1\},b$};
    \node[state] (s2_u2) at ($(s1_u2_bb.east) + (4,-0.75)$) {$s_2,\{p\mapsto 0.4, q\mapsto 0.4\},b$};
    \node[state, text width=3.5cm, align=center] (s22_u1) at ($(s2_u1.east) + (4.5,0)$) {$s_2,\{p\mapsto 0.1, q\mapsto 0.1\},$\\$\{p\mapsto 0.1, q\mapsto 0.1\}$};
    \node[state, text width=3.5cm, align=center] (s22_u2) at ($(s2_u2.east) + (4.5,0)$) {$s_2,\{p\mapsto 0.4, q\mapsto 0.4\},$\\$\{p\mapsto 0.4, q\mapsto 0.4\}$};
    \node[bobbel] (s2_u1_ab) at ($(s22_u1.east) + (1,0.75)$) {};
    \node[bobbel] (s2_u1_bb) at ($(s22_u1.east) + (1,-0.75)$) {};
    \node[bobbel] (s2_u2_ab) at ($(s22_u2.east) + (1,0.75)$) {};
    \node[bobbel] (s2_u2_bb) at ($(s22_u2.east) + (1,-0.75)$) {};
    \node[rectangle, draw, minimum width=14mm, rounded corners] (rNega2_1_1) at ($(s22_u1.east) + (2.5,1)$) {R = $0$};
    \node[rectangle, draw, minimum width=14mm, rounded corners] (rNeu2_1) at ($(s22_u1.east) + (2.5,0)$) {R = $100$};
    \node[rectangle, draw, minimum width=14mm, rounded corners] (rNega2_1_2) at ($(s22_u1.east) + (2.5,-1)$) {R = $0$};
    \node[rectangle, draw, minimum width=14mm, rounded corners] (rNega2_2_1) at ($(s22_u2.east) + (2.5,1)$) {R = $0$};
    \node[rectangle, draw, minimum width=14mm, rounded corners] (rNeu2_2) at ($(s22_u2.east) + (2.5,0)$) {R = $100$};
    \node[rectangle, draw, minimum width=14mm, rounded corners] (rNega2_2_2) at ($(s22_u2.east) + (2.5,-1)$) {R = $0$};
    \draw[<-] (s1.west) -- +(-0.6,0);
    \draw (s1) edge[->] node[text width = 2cm,sloped,anchor=center, above]{$\{p\mapsto 0.1,$\\$q\mapsto 0.1\}$} (s1_u1.west);
    \draw (s1) edge[->] node[text width = 2cm,sloped, anchor=center, below]{$\{p\mapsto 0.4,$\\$q\mapsto 0.4\}$} (s1_u2.west);
    \draw (s1) edge[->] node[sloped, anchor=center, above]{\scriptsize $\dots$} (s1_u1.west |- 1,1);
    \draw (s1) edge[->] node[sloped, anchor=center, below]{\scriptsize $\dots$} (s1_u1.west |- 0,-1);
    \node[] at ($(s1_u1.west |- 1,1.6) + (-0.25,0)$) {\scriptsize $\vdots$};
    \node[] at ($(s1_u1.west |- 1,0.1) + (-0.25,0)$) {\scriptsize $\vdots$};
    \node[] at ($(s1_u1.west |- 1,-1.4) + (-0.25,0)$) {\scriptsize $\vdots$};
    \node[] at (s1_u1 |- 1,0.1) {\scriptsize $\vdots$};
    \draw (s1_u1) edge node[above left]{$a$} (s1_u1_ab);
    \draw (s1_u1) edge node[below left]{$b$} (s1_u1_bb);
    \draw (s1_u1_ab) edge[->] node[above left]{$0.1$} (rPosa1);
    \draw (s1_u1_ab) edge[->] node[below left]{$0.9$} (rNega1);
    \draw (s1_u1_bb) edge[->] node[above left]{$0.1$} (rNega1);
    \draw (s1_u1_bb) edge[->] node[below right]{$0.4$} (rNeu1);
    \draw (s1_u1_bb) edge[->] node[below left]{$0.5$} (s2_u1.west);
    \draw (s1_u2) edge node[above left]{$a$} (s1_u2_ab);
    \draw (s1_u2) edge node[below left]{$b$} (s1_u2_bb);
    \draw (s1_u2_ab) edge[->] node[above left]{$0.4$} (rPosa2);
    \draw (s1_u2_ab) edge[->] node[below left]{$0.6$} (rNega2);
    \draw (s1_u2_bb) edge[->] node[above left]{$0.4$} (rNega2);
    \draw (s1_u2_bb) edge[->] node[below right]{$0.1$} (rNeu2);
    \draw (s1_u2_bb) edge[->] node[below left]{$0.5$} (s2_u2.west);
    \draw (s2_u1.east) edge[->] node[text width = 1.5cm, sloped, anchor=center, above]{$\{p\mapsto 0.1,$\\$q \mapsto 0.1\}$} (s22_u1.west);
    \draw (s22_u1) edge node[above left]{$a$} (s2_u1_ab);
    \draw (s22_u1) edge node[below left]{$b$} (s2_u1_bb);
    \draw (s2_u1_ab) edge[->] node[above left]{$0.1$} (rNega2_1_1);
    \draw (s2_u1_ab) edge[->] node[below left]{$0.9$} (rNeu2_1);
    \draw (s2_u1_bb) edge[->] node[above left]{$0.1$} (rNeu2_1);
    \draw (s2_u1_bb) edge[->] node[below left]{$0.9$} (rNega2_1_2);
    \draw (s2_u2.east) edge[->] node[text width = 1.5cm,sloped, anchor=center, above]{$\{p\mapsto 0.4,$\\$q\mapsto 0.4\}$} (s22_u2.west);
    \draw (s22_u2) edge node[above left]{$a$} (s2_u2_ab);
    \draw (s22_u2) edge node[below left]{$b$} (s2_u2_bb);
    \draw (s2_u2_ab) edge[->] node[above left]{$0.4$} (rNega2_2_1);
    \draw (s2_u2_ab) edge[->] node[below left]{$0.6$} (rNeu2_2);
    \draw (s2_u2_bb) edge[->] node[above left]{$0.4$} (rNeu2_2);
    \draw (s2_u2_bb) edge[->] node[below left]{$0.6$} (rNega2_2_2);
\end{tikzpicture}
        }
    \end{subfigure}
    \caption{Agent first (top) and nature first (bottom) POSGs of the RPOMDP in \Cref{app:fig:agent_vs_nature_first_RPOMDP_bigger}.}
    \label{fig:order:POSGs_bigger}
\end{figure}

\clearpage
\newpage
\section{Nature First Semantics}\label{app:nature_first}
Throughout the main paper and the appendix, the definitions and proofs are all written with the agent first order of play.
This appendix discusses the changes required to achieve the same results with the nature first semantics.

\paragraph{Policies in the RPOMDP.}
When nature moves first, nature receives the agent's action after choosing its own action.
Therefore, nature policies in nature first RPOMDPs are of the following types:%
\begin{equation*}
\begin{aligned}
    \text{Stochastic:}\\
    \text{Deterministic:}\\
    \text{Mixed:}
\end{aligned}
\qquad
\begin{aligned}[c]
    \theta \colon&H^{\nature,M} \to \dist{\bm{U}},\\
    \theta^\detr \colon&H^{\nature,M} \to \bm{U},\\
    \theta^\mix \in&\;\dist{H^{\nature,M} \to \bm{U}}.
\end{aligned}
\end{equation*}
The agent policies do not change, as we assume the agent still cannot observe the variable assignments nature chooses.

\paragraph{Nature first POSG.}
Given a nature first RPOMDP, we define its POSG as follows.
\begin{definition}[Equivalent nature first POSG]\label{def:equivalent:nature:zsposg}
Given a robust POMDP $\tup{S, A, \bm{T}, R, \Zagent, \Znature, \Zpub, O^\agent_\priv, O^\nature_\priv, O_\publ}$, we define the POSG where nature chooses first as a tuple $\tup{\mathcal{S^\agent, S^\nature, A^\agent, A^\nature, T, R, Z^\agent, Z^\nature, O^\agent, O^\nature}}$, where 
$\mathcal{S}^\nature, \mathcal{A}^\agent, \mathcal{A}^\nature, \mathcal{Z}^\agent$, and $\mathcal{Z}^\nature$ remain the same as in \Cref{def:equivalent:agent:zsposg}.
The agent's state-space is given by $\mathcal{S}^\agent = S \times \bm{U}^\pset \times \bm{U}$, and the transition, reward, and observation functions are defined as follows:
\begin{itemize}
    \item $\mathcal{T}^\agent \colon \mathcal{S}^\agent \times \mathcal{A}^\agent \to \dist{\mathcal{S}^\nature}$, by $\mathcal{T^\agent}(\tup{s,u^\pset,u},a,\tup{s',\upd(u^\pset,u,O^\nature_\priv(s),O_\publ(s),a),a}) = \bm{T}(u)(s,a,s')$. 
    \item $\mathcal{T}^\nature \colon \mathcal{S}^\nature \times \mathcal{A}^\nature \to \mathcal{S}^\agent$, by $\mathcal{T^\nature}(\tup{s,u^\pset,a},u,\tup{s,u^\pset,u}) = \begin{cases}
       1 & \quad \text{if $u \in \bm{U}^\agrees(u^\pset)$,}\\
       0 & \quad \text{otherwise.}
    \end{cases}$
    \item $\mathcal{R \colon S^\agent \times A^\agent} \to \RR$ by $\mathcal{R}(\tup{s,u^\pset,u},a) = R(s,a)$.
    \item $\mathcal{O}^\agent\colon (\mathcal{S}^\agent\cup \mathcal{S}^\nature) \to \mathcal{Z}^\agent$ by $\mathcal{O^\agent}(s) = \begin{cases}
       \tup{O^\agent_\priv(s'), O_\publ(s')} & \text{if $s = \tup{s',u^\pset,u}\in \mathcal{S^\agent}$,}\\
       \tup{O^\agent_\priv(s'), O_\publ(s')} & \text{if $s = \tup{s',u^\pset,a} \in \mathcal{S^\nature}$.}
    \end{cases}$
    \item $\mathcal{O}^\nature\colon (\mathcal{S}^\agent\cup \mathcal{S}^\nature) \to \mathcal{Z}^\nature$ by\\
    $\mathcal{O^\nature}(s) = \begin{cases}
        \tup{O^\nature_\priv(s'), O_\publ(s'),\bot} & \text{if $s = \tup{s',u^\pset,u} \in \mathcal{S^\agent}$,}\\
        \tup{O^\nature_\priv(s'), O_\publ(s'),a} &  \text{if $s = \tup{s',u^\pset,a} \in \mathcal{S^\nature}$.}
    \end{cases}$
\end{itemize}
\end{definition}
\noindent The $a$ observed in a $\mathcal{S}^\nature$ state corresponds to the previously chosen $a$.
So, the action that nature observes in a nature state $\tup{s',u^\pset,a} \in \mathcal{S^\nature}$ is the action $a \in A$ that was taken to reach the current state $s' \in S$ of the RPOMDP, not the action the agent will take from the current state.
This game starts in a $\mathcal{S}^\nature$ state consisting of the initial state $s_I \in S$ in the RPOMDP, the totally undefined variable assignment $u^\bot \in \bm{U}^\pset$, and a placeholder for the action $\bot$.

\paragraph{Paths and histories.}
When reasoning with the nature first semantics, the order of the paths in the POSG changes:
\[
\Paths^G:(\mathcal{S}^\agent \times \mathcal{A}^\agent \times \mathcal{S}^\nature \times \mathcal{A}^\nature)^* \times \mathcal{S}^\agent \Longrightarrow (\mathcal{S}^\nature \times \mathcal{A}^\nature \times \mathcal{S}^\agent \times \mathcal{A}^\agent)^* \times \mathcal{S}^\nature.
\]
As a result, the histories similarly change:
\[
H^{G}:(\mathcal{Z}^\agent \times \mathcal{Z}^\nature \times \mathcal{A}^\agent \times \mathcal{Z}^\agent \times \mathcal{Z}^\nature \times \mathcal{A}^\nature)^* \times \mathcal{Z}^\agent \times \mathcal{Z}^\nature \Longrightarrow (\mathcal{Z}^\nature \times \mathcal{Z}^\agent \times \mathcal{A}^\nature \times \mathcal{Z}^\nature \times \mathcal{Z}^\agent \times \mathcal{A}^\agent)^* \times \mathcal{Z}^\nature \times \mathcal{Z}^\agent.
\]
\[
H^{\agent,G}:(\mathcal{Z}^\agent \times \mathcal{A}^\agent \times \mathcal{Z}^\agent)^* \times \mathcal{Z}^\agent \Longrightarrow (\mathcal{Z}^\agent \times \mathcal{Z}^\agent \times \mathcal{A}^\agent)^* \times \mathcal{Z}^\agent.
\]
\[
H^{\nature,G}:(\mathcal{Z}^\nature \times \mathcal{Z}^\nature \times \mathcal{A}^\nature)^* \times \mathcal{Z}^\nature \Longrightarrow (\mathcal{Z}^\nature \times \mathcal{A}^\nature \times \mathcal{Z}^\nature)^* \times \mathcal{Z}^\nature.
\]

\paragraph{Policies in the POSG}
As the order of the paths and historis changed, the policy types also change.
The agent now observes an extra state, instead of nature.
Note that this extra observation contains no extra information for the agent, while it did for nature in the agent-first semantics.
\begin{equation*}
\begin{aligned}[c]
    \text{Stochastic:}\quad &&\pi \colon&H^{\agent,G} \to \dist{\mathcal{A}^\agent} &\implies && \pi \colon&H^{\agent,G} \times \mathcal{Z^\agent} \to \dist{\mathcal{A}^\agent},\\
    \text{Deterministic:}\quad &&\pi^\detr \colon&H^{\agent,G} \to \mathcal{A}^\agent &\implies && \pi^\detr \colon&H^{\agent,G} \times \mathcal{Z^\agent}\to \mathcal{A}^\agent,\\
    \text{Mixed:}\quad &&\pi^\mix \in&\;\dist{H^{\agent,G} \to \mathcal{A}^\agent} &\implies && \pi^\mix \in&\;\dist{H^{\agent,G} \times \mathcal{Z^\agent} \to \mathcal{A}^\agent},\\
    \\
    \text{Stochastic:}\quad &&\theta \colon&H^{\nature,G} \times \mathcal{Z^\nature} \to \dist{\mathcal{A}^\nature}, &\implies && \theta \colon&H^{\nature,G} \to \dist{\mathcal{A}^\nature}\\
    \text{Deterministic:}\quad &&\theta^\detr \colon&H^{\nature,G} \times \mathcal{Z^\nature} \to \mathcal{A}^\nature, &\implies && \theta^\detr \colon&H^{\nature,G} \to \mathcal{A}^\nature\\
    \text{Mixed:}\quad &&\theta^\mix \in&\;\dist{H^{\nature,G} \times \mathcal{Z^\nature} \to \mathcal{A}^\nature} &\implies && \theta^\mix \in&\;\dist{H^{\nature,G} \to \mathcal{A}^\nature}.
\end{aligned}
\end{equation*}

\paragraph{Additional adaptations.}
At the end of \Cref{app:notation_and_prelim,app:value:function:proofs,app:nashEquilibrium}, the adjustments required for the definitions or proofs to work with the nature first semantics are briefly discussed.

\clearpage
\newpage
\section{Equivalent Values}\label{app:value:function:proofs}
Given the stickiness and order of play, we show that the value of an RPOMDP and its POSG are equivalent (\Cref{thm:equivalent:values}).
To do so, we construct a bijection between the sets of paths of the two models.
We use this bijection to subsequently construct new bijections between the sets of histories and policies and finally conclude that the values are equivalent.
For convenience, we repeat the proposition from the main text and then split it into several lemmas.

\noindent
\textbf{Proposition 1} (Bijection between paths and histories)\textbf{.} 
\textit{
Let $M$ be an RPOMDP, and $G$ the POSG of $M$.
There exists a bijection $f \colon \Paths^M \to \Paths^G$ and bijections between individual players' histories:
    \begin{itemize}
        \item Let $H^{\agent,M}$ and $H^{\agent,G}$ be the set of all agent histories in $M$ and $G$, respectively.
        There exists a bijection $f^{\agent,h}\colon H^{\agent,M} \to H^{\agent,G}$.
        \item Let $H^{\nature,M}$ and $H^{\nature,G}$ be the set of all nature histories in $M$ and $G$, respectively.
        There exists a bijection $f^{\nature,h}\colon H^{\nature,M} \to H^{\nature,G}$.
    \end{itemize}
}

\begin{lemma}[Bijection between paths]
Let $M$ be an RPOMDP, and $G$ the POSG of $M$.
There exists a bijection $f \colon \Paths^M \to \Paths^G$.
\end{lemma}

\begin{proof}
    Let $\Paths^{M,\suff} \subseteq (S\times A \times \bm{U})^* \times S^?$ with $? \in \{0,1\}$ be the set of all path segments in the RPOMDP, and let $\Paths^{G,\suff} \subseteq (\mathcal{S}^\agent \times \mathcal{A}^\agent \times \mathcal{S}^\nature\times \mathcal{A}^\nature)^* \times (\mathcal{S}^\agent)^?$ with $? \in \{0,1\}$ be the set of all path segments in the POSG.
    With path segment we mean that the path can starts at any time steps $t \in \NN$ and can end at any time step $t' \in \NN, t<=t'$.
    The optional last state is only used for path segments until the horizon.
    Note that $\Paths^M \subseteq \Paths^{M,\suff}$ and $\Paths^G \subseteq \Paths^{G,\suff}$.
    Let $\mypath^M = \tup{s_0,a_0,u_0,s_1, \dots, s_n}\in\Paths^{M}$ and $t \leq n$, then $\mypath^M(t)$ indicates the $t$-th segment $\tup{s_t,a_t,u_t}$ of $\mypath^M$.
    Note that if $t = n$, the segment will only consist of the final state $\tup{s_n}$.
    Similarly, let $\mypath^G = \tup{s^\agent_0, a^\agent_0,s^\nature_0,a^\nature_0,s^\agent_1,\dots,s^\agent_n} = \tup{\tup{s_0,u^\bot},a_0,\tup{s_0,u^\bot,a_0},u_0, \tup{s_1,u^\pset_1}, \dots, \tup{s_n,u^\pset_n}} \in\Paths^G$ and $t \leq n$, then $\mypath^G(t)$ indicates the $t$-th segment $\tup{s^\agent_t,a^\agent_t,s^\nature_t,a^\nature_t} = \tup{\tup{s_t,u^\pset_t},a_t,\tup{s_t,u^\pset_t,a_t},u_t}$ of $\mypath^G$.
    Note that if $t = n$, the segment will only consist of the final agent state $\tup{s^\agent_n} = \tup{\tup{s_n,u^\pset_n}}$.\\

    \noindent Let $g\colon \Paths^{M,\suff} \times \bm{U}^\pset \pto \Paths^{G,\suff}$ defined by:
    \begin{align*}
        g(\tup{s},u^\pset) &= \tup{\tup{s,u^\pset}}.\\
        g(\tup{s,a,u}, u^\pset) &= \begin{cases}
            \tup{\tup{s,u^\pset},a,\tup{s,u^\pset,a},u} & \text{ if } u \in \bm{U}^\agrees(u^\pset),\\
            \bot & \text{ otherwise.}
        \end{cases}\\
        g(\tup{s,a,u}\concat{\mypath^M}', u^\pset) &= \begin{cases}
            g(\tup{s,a,u}, u^\pset)\concat g({\mypath^M}', \upd(u^\pset,u,O^\nature_\priv(s),O_\publ(s),a))& \text{ if } u \in \bm{U}^\agrees(u^\pset),\\
            \bot & \text{ otherwise.}
        \end{cases}
    \end{align*}
    
    Let $f\colon \Paths^M \to \Paths^G$ defined by:
    \begin{align*}
        f(\tup{s}) &= \tup{\tup{s,u^\bot}}.\\
        f(\tup{s,a,u}) &= \tup{\tup{s,u^\bot},a,\tup{s,u^\bot,a},u}.\\
        f(\tup{s,a,u}\concat{\mypath^M}') &= f(\tup{s,a,u})\concat g({\mypath^M}', \upd(u^\bot,u,O^\nature_\priv(s),O_\publ(s),a)).
    \end{align*}
    Where $u^\bot \in \bm{U}^\pset$ is the totally undefined function.
    Note that the results of $f$ and $g$ are in $\Paths^G$ and $\Paths^{G,\suff}$ by construction.
    Also, note that any call to $g$ that originated from a call in $f$ will have a result by construction.\\

    \noindent We show that $f$ is a bijection, meaning $f$ is injective and surjective.
    We first show $f$ is injective, so we show that:
    \[
        \forall \mypath^{1,M}, \mypath^{2,M} \in \Paths^M. \mypath^{1,M} \neq \mypath^{2,M} \implies f(\mypath^{1,M}) \neq f(\mypath^{2,M}).
    \]
    Given arbitrary $\mypath^{1,M}, \mypath^{2,M} \in \Paths^M$, we distinguish between the paths with superscripts $1$ and $2$, respectively.
    Assume $\mypath^{1,M} \neq \mypath^{2,M}$.
    If $\mypath^{1,M}$ and $\mypath^{2,M}$ do not have the same horizon length, then neither do $f(\mypath^{1,M})$ and $f(\mypath^{2,M})$.
    Then trivially, $f(\mypath^{1,M}) \neq f(\mypath^{2,M})$.\\
    
    \noindent Assume $\mypath^{1,M}$ and $\mypath^{2,M}$ have the same horizon length $n$.
    Then $\exists t \leq n$ where $\mypath^{1,M}$ and $\mypath^{2,M}$ deviate, so $\mypath^{1,M}(t) \neq \mypath^{2,M}(t)$.
    Let $q$ be the smallest number where the paths deviate.
    So $\forall t < q. \mypath^{1,M}(t) = \mypath^{2,M}(t)$ and $\mypath^{1,M}(q) \neq \mypath^{2,M}(q)$.
    Assume $q < n$.
    Then we know $\tup{s^1_q,a^1_q,u^1_q} \neq \tup{s^2_q,a^2_q,u^2_q}$, which comes down to: $s^1_q \neq s^2_q \lor a^1_q \neq a^2_q \lor u^1_q \neq u^2_q$.
    \begin{align*}
        f(\mypath^{1,M}) &= f(\mypath^{1,M}(1)\concat {\mypath^{1,M}}')\\
        &= \mypath^{1,G}(1)\concat g({\mypath^{1,M}}',\upd(u^\bot,u^1_0,O^\nature_\priv(s^1_0),O_\publ(s^1_0),a^1_0)).
    \intertext{Unfold $g$ until $q$:}
        &= \bigconcat_{t=0}^{q}(\mypath^{1,G}(t))\concat\tup{\tup{s^1_q,\fixed(\mypath^{1,M}_{0:q})},a^1_q,\tup{s^1_q,\fixed(\mypath^{1,M}_{0:q}),a^1_q},u^1_q}\concat g({\mypath^{1,M}}'',\fixed(\mypath^{1,M}_{0:q+1})).
    \intertext{Since $s^1_q \neq s^2_q \lor a^1_q \neq a^2_q \lor u^1_q \neq u^2_q$:}
        &\neq \bigconcat_{t=0}^{q}(\mypath^{1,G}(t))\concat \tup{\tup{s^2_q,\fixed(\mypath^{1,M}_{0:q})},a^2_q,\tup{s^2_q,\fixed(\mypath^{1,M}_{0:q}),a^2_q},u^2_q}\concat g({\mypath^{2,M}}'',\fixed(\mypath^{1,M}_{0:q+1}))\\
        &= \bigconcat_{t=0}^{q}(\mypath^{2,G}(t))\concat \tup{\tup{s^2_q,\fixed(\mypath^{2.M}_{0:q})},a^2_q,\tup{s^2_q,\fixed(\mypath^{2,M}_{0:q}),a^2_q},u^2_q}\concat g({\mypath^{2,M}}'',\fixed(\mypath^{2,M}_{0:q+1})).
    \intertext{Fold $g$ until $1$:}
        &= \mypath^{2,G}(1)\concat g({\mypath^{2,M}}',\upd(u^\bot,u^2_0,O^\nature_\priv(s^2_0),O_\publ(s^2_0),a^2_0))\\
        &= f(\mypath^{2,M}(1)\concat {\mypath^{2,M}}')\\
        &= f(\mypath^{2,M}).
    \end{align*}
    If $q = n$, then the same result follows by removing everything after $\tup{s^1_q,\fixed(\mypath^{1,M}_{0:q})}, \tup{s^2_q,\fixed(\mypath^{1,M}_{0:q})}$, and $\tup{s^2_q,\fixed(\mypath^{2,M}_{0:q})}$.
    We thus have that $f(\mypath^{1,M}) \neq f(\mypath^{2,M})$, so $f$ is injective.\\
        
    \noindent Next, we show that $f$ is surjective, so we show that:
    \[
        \forall \mypath^G \in \Paths^G, \exists \mypath^M \in \Paths^M. f(\mypath^M) = \mypath^G.
    \]
    We show this holds by induction on the horizon length of the $\mypath^G \in \Paths^G$. 
    We write the length of $\mypath^G$ as $|\mypath^G|$.

    \noindent Assume $|\mypath^G| = 0$.
    Then $\mypath^G = \tup{s_I, u^\bot}$.
    We have that for $\tup{s_I} \in \Paths^M, f(\tup{s_I}) = \tup{\tup{s_I, u^\bot}} = \mypath^G$.
    So for paths of horizon length $0$, $f$ is surjective.\\

    \noindent Now assume we know, given $q\in \NN, q \geq 1$, that:
    \[
        \forall \mypath^G \in \Paths^G. |\mypath^G| = q-1 \implies \exists  \mypath^M \in \Paths^M: f(\mypath^M) = \mypath^G.
    \]
    Take arbitrary $\mypath^G \in \Paths^G$ with horizon length $|\mypath^G| = q$.
    Then we have $\mypath^G = \mypath^G_{0:q-1}\concat \tup{a_{q-1},\tup{s_{q-1}, u^\pset_{q-1},a_{q-1}}, u_{q-1},\tup{s_q, u^\pset_{q}}}$.
    Then $\mypath^G_{0:q-1} \in \Paths^G$ and $|\mypath^G_{0:q-1}| = q-1$.
    By assumption, we get that:
    \[
        \exists \mypath^M_{0:q-1} \in \Paths^M, f(\mypath^M_{0:q-1}) = \mypath^G_{0:q-1}.
    \]
    Let $\mypath^M_{0:q-1} \in \Paths^M$ such that $f(\mypath^M_{0:q-1}) = \mypath^G_{0:q-1}$.
    We then know that in $\mypath^G$:
    \[
        \forall t < q. u^\pset_{t} = \fixed(\mypath^M_{0:t}).
    \]
    And, by \Cref{def:equivalent:agent:zsposg} and the definition of $\Paths^G$, that:
    \[
        u^\pset_{q} = \upd(\fixed(\mypath^M_{0:q-1}),u_{q-1}, O^\nature_\priv(s_{q-1}), O_\publ(s_{q-1}), a_{q-1}).
    \]
    Let $\tup{\tup{s_{q-2}, \fixed(\mypath^M_{0:q-2})}, a_{q-2}, \tup{s_{q-2}, \fixed(\mypath^M_{0:q-2}),a_{q-2}}, u_{q-2},\tup{s_{q-1}, \fixed(\mypath^M_{0:q-1})}}$ be the last two segments of $\mypath^G_{0:q-1}$.
    Then by definition and injectivity of $f$, we know that the last two segments of $\mypath^M_{0:q-1}$ are $\tup{s_{q-2},a_{q-2},u_{q-2},s_{q-1}}$.\\
    
    \noindent Now, by definition \Cref{def:equivalent:agent:zsposg} and the definition of $\Paths^G$, we know that:
    \[\mypath^G = \mypath^G_{0:q-1}\concat\tup{a_{q-1},\tup{s_{q-1}, u^\pset_{q-1},a_{q-1}}, u_{q-1},\tup{s_q, u^\pset_q}} \in \Paths^G \]
    \[\Longleftrightarrow \]
    \[\mypath^G_{0:q-1} \in \Paths^G 
    \land \mathcal{T}^\agent(\tup{s_{q-1},u^\pset_{q-1}},a_{q-1},\tup{s_{q-1},u^\pset_{q-1},a_{q-1}}) > 0\]
    \[\land \mathcal{T}^\nature(\tup{s_{q-1},u^\pset_{q-1}),a_{q-1}}, u_{q-1}, \tup{s_{q},u^\pset_{q}}) > 0\] 
    \[\Longleftrightarrow \]
    \[\mypath^G_{0:q-1} \in \Paths^G 
    \land u_{q-1} \in \bm{U}^\agrees(u^\pset_{q-1})
    \land \bm{T}(u_{q-1})(s_{q-1},a_{q-1},s_{q}) > 0
    \]
    \[\Longleftrightarrow \]
    \[\mypath^G_{0:q-1} \in \Paths^G 
    \land u_{q-1} \in \bm{U}^\agrees(\fixed(\mypath^M_{0:q-1}))
    \land \bm{T}(u_{q-1})(s_{q-1},a_{q-1},s_{q}) > 0.
    \]
    So, since $\mypath^G \in \Paths^G$, we know $u_{q-1} \in \bm{U}^\agrees(\fixed(\mypath^M_{0:q-1}))$ and $\bm{T}(u_{q-1})(s_{q-1},a_{q-1},s_{q}) > 0$, which are the restrictions for $\mypath^M = \mypath^M_{0:q-1}\concat \tup{a_{q-1}, u_{q-1}, s_q} \in \Paths^M$ to hold.
    \begin{align*}
        f(\mypath^M ) &= f(\mypath^M_{0:q-1}\concat \tup{a_{q-1}, u_{q-1}, s_q})\\
        &= \bigconcat_{t=0}^{q-2}\mypath^G_{0:q-1}(t)\concat g(\tup{s_{q-1}, a_{q-1}, u_{q-1},s_q}, \fixed(\mypath^M_{0:q-1}))\\
        &= \bigconcat_{t=0}^{q-2}\mypath^G_{0:q-1}(t)\concat g(\tup{s_{q-1}, a_{q-1}, u_{q-1}}, \fixed(\mypath^M_{0:q-1}))\concat g(\tup{s_q}, \fixed(\mypath^M_{0:q}))\\
        &= \bigconcat_{t=0}^{q-2}\mypath^G_{0:q-1}(t)\concat g(\tup{s_{q-1}, a_{q-1}, u_{q-1}}, \fixed(\mypath^M_{0:q-1}))\concat g(\tup{s_q}, \upd(\fixed(\mypath^M_{0:q-1}),u_{q-1}, O^\nature_\priv(s_{q-1}), O_\publ(s_{q-1}), a_{q-1}))\\
        &= \bigconcat_{t=0}^{q-2}\mypath^G_{0:q-1}(t)\concat g(\tup{s_{q-1}, a_{q-1}, u_{q-1}}, \fixed(\mypath^M_{0:q-1}))\concat g(\tup{s_q}, u^\pset_{q})\\
        &= \bigconcat_{t=0}^{q-2}\mypath^G_{0:q-1}(t)\concat \tup{\tup{s_{q-1}, \fixed(\mypath^M_{0:q-1})},a_{q-1},\tup{s_{q-1}, \fixed(\mypath^M_{0:q-1}),a_{q-1}}, u_{q-1}}\concat g(\tup{s_q}, u^\pset_{q})\\
        &= \bigconcat_{t=0}^{q-2}\mypath^G_{0:q-1}(t)\concat \tup{\tup{s_{q-1}, u^{\pset}_{0:q-1}},a_{q-1},\tup{s_{q-1}, u^{\pset}_{q-1},a_{q-1}}, u_{q-1}}\concat g(\tup{s_q}, u^\pset_{q})\\
        &= \mypath^G_{0:q-1}\concat \tup{a_{q-1},\tup{s_{q-1}, u^{\pset}_{q-1},a_{q-1}}, u_{q-1}}\concat g(\tup{s_q}, u^\pset_{q})\\
        &= \mypath^G_{0:q-1}\concat \tup{a_{q-1},\tup{s_{q-1}, u^{\pset}_{q-1},a_{q-1}}, u_{q-1}, \tup{s_q, u^{\pset}_{q}}}\\
        &= \mypath^G.
    \end{align*}
    So if $f$ is surjective for paths of arbitrary length $q-1 \in \NN$, $f$ is surjective for paths of length $q$.
    Hence, by induction, $f$ is surjective.\\

    \vspace{-1mm}
    \noindent $f$ is injective and surjective, hence $f$ is a bijection.
    \end{proof}

We write $\simeq$ to indicate equivalence between objects in the RPOMDP and the POSG.
\begin{corollary}[Corresponding paths] $f$ is a bijection between $\Paths^M$ and $\Paths^G$, so the set of paths in the RPOMDP is equivalent to the set of paths in the POSG:
    \[\Paths^M \simeq \Paths^G,\vspace{-1mm}\]
where $\forall \mypath^M \in \Paths^M, \forall \mypath^G \in \Paths^G.$
    \[\mypath^M \simeq \mypath^G \Longleftrightarrow f(\mypath^M) = \mypath^G.\]
\end{corollary}

We show and prove a bijection between joint histories, as introduced in \Cref{app:appendix_prelim}.
The individual agent and nature histories' bijection proofs follow the same line of reasoning, omitting elements private to the other player.
\begin{proposition}[Bijection between joint histories]\label{app:theorem_bijection_joint_histories}
Let $M$ be an RPOMDP, and $G$ the POSG of $M$.
There exists a bijection $f^{h}\colon H^{M} \to H^{G}$.
\end{proposition}

\begin{proof}
\vspace{-1mm}
    Let $H^{M,\suff}$ be the joint \aohs{} segments for the RPOMDP and let $H^{G,\suff}$ be the joint \aohs{} segments for the parameterized POSG.
    Again, we have that $H^M \subseteq H^{M,\suff}$ and $H^G \subseteq H^{G,\suff}$.
    Let $h^M = \tup{z^\agent_{\priv,0},z^\nature_{\priv,0},z_{\publ,0},a_0,u_0,z^\agent_{\priv,1},z^\nature_{\priv,1},z_{\publ,1},\dots, z^\agent_{\priv,n},z^\nature_{\priv,n},z_{\publ,n}}\in H^{M}$ and $t \leq n$, then $h^M(t)$ indicates the $t$-th segment $\tup{z^\agent_{\priv,t},z^\nature_{\priv,t},z_{\publ,t},a_t,u_t}$ of $h^M$.
    Note that if $t = n$, the segment will only consist of the final observations $\tup{z^\agent_{\priv,n},z^\nature_{\priv,n},z_{\publ,n}}$.
    Similarly, let $h^G = \tup{\tup{z^\agent_{\priv,0},z_{\publ,0}},\tup{z^\nature_{\priv,0},z_{\publ,0},\bot},a_0,\tup{z^\agent_{\priv,0},z_{\publ,0}},\tup{z^\nature_{\priv,0},z_{\publ,0},a_0},u_0, \tup{z^\agent_{\priv,1},z_{\publ,1}},\tup{z^\nature_{\priv,1},z_{\publ,1},\bot}, \dots, \tup{z^\agent_{\priv,n},z_{\publ,n}},\tup{z^\nature_{\priv,n},z_{\publ,n},\bot}} \in H^G$ and $t \leq n$, then $h^G(t)$ indicates the $t$-th segment $\tup{\tup{z^\agent_{\priv,t},z_{\publ,t}},\tup{z^\nature_{\priv,t},z_{\publ,t},\bot},a_t,\tup{z^\agent_{\priv,t},z_{\publ,t}},\tup{z^\nature_{\priv,t},z_{\publ,t},a_t},u_t}$ of $h^G$.
    Note that if $t = n$, the segment will only consist of the final observations $\tup{\tup{z^\agent_{\priv,n},z_{\publ,n}},\tup{z^\nature_{\priv,n},z_{\publ,n},\bot}}$.\\

    \noindent Let $g^h\colon H^{M,\suff} \to H^{G,\suff}$ defined by:
    \begin{align*}
    g^h(\tup{z^\agent_\priv,z^\nature_\priv,z_\publ}) &= \tup{\tup{z^\agent_\priv,z_\publ},\tup{z^\nature_\priv,z_\publ,\bot}}.\\
    g^h(\tup{z^\agent_\priv,z^\nature_\priv,z_\publ,a,u}) &= \tup{\tup{z^\agent_\priv,z_\publ},\tup{z^\nature_\priv,z_\publ,\bot},a,\tup{z^\agent_\priv,z_\publ},\tup{z^\nature_\priv,z_\publ,a},u}.\\
    g^h(\tup{z^\agent_\priv,z^\nature_\priv,z_\publ,a,u}\concat h') &= g^h(\tup{z^\agent_\priv,z^\nature_\priv,z_\publ,a,u})\concat g^h(h').
    \end{align*}
    
    \noindent Let $f^h\colon H^M \to H^G$ defined by:
    \begin{align*}    
    f^h(\tup{z^\agent_\priv,z^\nature_\priv,z_\publ}) &= \tup{\tup{z^\agent_\priv,z_\publ},\tup{z^\nature_\priv,z_\publ,\bot}}.\\
    f^h(\tup{z^\agent_\priv,z^\nature_\priv,z_\publ,a,u}) &= \tup{\tup{z^\agent_\priv,z_\publ},\tup{z^\nature_\priv,z_\publ,\bot},a,\tup{z^\agent_\priv,z_\publ},\tup{z^\nature_\priv,z_\publ,a},u}.\\
    f^h(\tup{z^\agent_\priv,z^\nature_\priv,z_\publ,a,u}\concat h') &= f^h(\tup{z^\agent_\priv,z^\nature_\priv,z_\publ,a,u})\concat g^h(h').
    \end{align*}
    
    \noindent Note that the results of $f^h$ and $g^h$ are in $H^G$ and $H^{G,\suff}$ by construction. Also, note that these function definitions are similar to those for paths.\\
    
    \noindent We show that $f^h$ is a bijection.
    We first show $f^h$ is injective, so we show that:
    \[
        \forall h^{1,M}, h^{2,M} \in H^M. h^{1,M} \neq h^{2,M} \implies f^h(h^{1,M}) \neq f^h(h^{2,M}).
    \]
    Given arbitrary $h^{1,M}, h^{2,M} \in H^M$, we distinguish between the histories with a superscript $1,2$ respectively.
    Assume $h^{1,M} \neq h^{2,M}$.
    If $h^{1,M}$ and $h^{2,M}$ do not have the same horizon length, then neither do $f^h(h^{1,M})$ and $f^h(h^{2,M})$.
    Then trivially, $f^h(h^{1,M}) \neq f^h(h^{2,M})$.\\
    
    \noindent Assume $f^h(h^{1,M})$ and $f^h(h^{2,M})$ have the same horizon length $n$.
    Then $\exists t \leq n$ where $f^h(h^{1,M})$ and $f^h(h^{2,M})$ deviate, so $h^{1,M}(t) \neq h^{2,M}(t)$.
    Let $q$ be the smallest number where the histories deviate.
    So $\forall t < q. h^{1,M}(t) = h^{2,M}(t)$ and $h^{1,M}(q) \neq h^{2,M}(q)$.
    Assume $q < n$.
    Then we know $\tup{z^{\agent,1}_{\priv,q},z^{\nature,1}_{\priv,q},z^1_{\publ,q},a^1_q,u^1_q} \neq \tup{z^{\agent,2}_{\priv,q},z^{\nature,2}_{\priv,q},z^2_{\publ,q},a^2_q,u^2_q}$, which comes down to: $z^{\agent,1}_{\priv,q} \neq z^{\agent,2}_{\priv,q} \lor z^{\nature,1}_{\priv,q} \neq z^{\nature,2}_{\priv,q} \lor z^1_{\publ,q} \neq z^2_{\publ,q} \lor a^1_q \neq a^2_q \lor u^1_q \neq u^2_q$.
    \begin{align*}
        f^h(h^{1,M}) &= f^h(h^{1,M}(1)\concat {h_1^{M}}')\\
        &= h_1^G(1) \concat g^h({h^{1,M}}').
        \intertext{Unfold $g^h$ until $q$:}
        &= \bigconcat_{t=0}^{q-1}(h_1^G(t))\concat\tup{\tup{z^{\agent,1}_{\priv,q},z^1_{\publ,q}},\tup{z^{\nature,1}_{\priv,q},z^1_{\publ,q},\bot},a^1_q,\tup{z^{\agent,1}_{\priv,q},z^1_{\publ,q}},\tup{z^{\nature,1}_{\priv,q},z_{\publ,q},a^1_q},u^1_q}\concat g^h({h^{1,M}}'').
        \intertext{Since $z^{\agent,1}_{\priv,q} \neq z^{\agent,2}_{\priv,q} \lor z^{\nature,1}_{\priv,q} \neq z^{\nature,2}_{\priv,q} \lor z^1_{\publ,q} \neq z^2_{\publ,q} \lor a^1_q \neq a^2_q \lor u^1_q \neq u^2_q$:}
        &\neq \bigconcat_{t=0}^{q-1}(h_1^G(t))\concat \tup{\tup{z^{\agent,2}_{\priv,q},z^2_{\publ,q}},\tup{z^{\nature,2}_{\priv,q},z^2_{\publ,q},\bot},a^2_q,\tup{z^{\agent,2}_{\priv,q},z^2_{\publ,q}},\tup{z^{\nature,2}_{\priv,q},z_{\publ,q},a^2_q},u^2_q}\concat g^h({h^{2,M}}'')\\
        &= \bigconcat_{t=0}^{q-1}(h^{2,G}(t))\concat\tup{\tup{z^{\agent,2}_{\priv,q},z^2_{\publ,q}},\tup{z^{\nature,2}_{\priv,q},z^2_{\publ,q},\bot},a^2_q,\tup{z^{\agent,2}_{\priv,q},z^2_{\publ,q}},\tup{z^{\nature,2}_{\priv,q},z_{\publ,q},a^2_q},u^2_q}\concat g^h({h^{2,M}}'').
        \intertext{Fold $g^h$ until $1$:}
        &= h^{2,G}(1)\concat g^h({h^{2,M}}')\\
        &= f^h(h^{2,M}(1)\concat {h_2^{M}}')\\
        &= f^h(h^{2,M}).
    \end{align*}
    If $q=n$, then the same result follows by removing everything after $\tup{z^{\nature,1}_{\priv,q},z^1_{\publ,q},\bot}$, and $\tup{z^{\nature,2}_{\priv,q},z^2_{\publ,q},\bot}$.
    We thus have that $f^h(h^{1,M}) \neq f^h(h^{2,M})$, so $f^h$ is injective.\\
        
    \noindent Next, we show that $f^h$ is surjective, so we show that:
    \[
        \forall h^G \in H^G, \exists h^M \in H^M. f^h(h^M) = h^G.
    \]
    Take arbitrary $h^G \in H^G$.
    By construction of $H^G$, $O^G$ (see \Cref{app:appendix_prelim}) is surjective, so we know $\exists \mypath^G \in \Paths^G, O^G(\mypath^G) = h^G$.
    Take $\mypath^G \in \Paths^G$ such that $O^G(\mypath^G) = h^G$.
    Let $\mypath^M \in \Paths^M$ be the corresponding path in the RPOMDP.
    So $f(\mypath^M) = \mypath^G$.
    Then $h^M = O^M(\mypath^M) \in H^M$.\\

    \noindent We proof $f^h(h^M) = h^G$ by contradiction.
    Assume $f^h(h^M) = h_2^G \neq h^G$.
    By construction of $f$, we know $|\mypath^M| = |\mypath^G|$.
    Then, by construction of $O^M$, $O^G$, and $f^h$, which each map a segment to a segment, we know $|h_2^G| = |h^G|$.\\
    
    \noindent Let $n$ be the horizon length of $h_2^G$ and $h^G$.
    Then $\exists t\in \NN$ where $h_2^G$ and $h^G$ deviate, so $h^G(t) \neq h_2^G(t)$.
    Let $q$ be such a number where the horizons deviate.
    So $h^G(q) \neq h_2^G(q)$.
    Note that $q \geq 1$, since there is only one initial state, therefore $h^G(0) = h_2^G(0)$.
    Assume $q < n$ and let $\mypath^G(q) = \tup{\tup{s_{q}, u^{\pset}_{q}}, a_{q}, \tup{s_{q}, u^{\pset}_{q},a_{q}}, u_{q}}$.
    Then by the definition and bijectivity of $f$, we know $\mypath^M(q) = \tup{s_q,a_q,u_q}$.
    Furthermore, by construction, we know that $f^h,g^h, O^G, O^{G,\suff}, O^M$, and $O^{M,\suff}$ all apply on the segments of the paths or histories separately.
    \begin{align*}
        h^G(q) &= O^G(\mypath^G(q))\\
        &= O^{G}(\tup{\tup{s_{q}, u^{\pset}_{q}}, a_{q}, \tup{s_{q}, u^{\pset}_{q},a_{q}}, u_{q}})\\
        &= \tup{\tup{O^\agent_\priv(s_q),O_\publ(s_q)},\tup{O^\nature_\priv(s_q),O_\publ(s_q),\bot},a_q,\tup{O^\agent_\priv(s_q),O_\publ(s_q)},\tup{O^\nature_\priv(s_q),O_\publ(s_q),a_q},u_q}.\\
        h_2^G(q) &= f^h(O^M(\mypath^M(q)))\\
        &= f^h(O^M(\tup{s_q,a_q,u_q}))\\
        &= f^h(\tup{O^\agent_\priv(s_q),O^\nature_\priv(s_q),O_\publ(s_q),a_q,u_q})\\
        &= \tup{\tup{O^\agent_\priv(s_q),O_\publ(s_q)},\tup{O^\nature_\priv(s_q),O_\publ(s_q),\bot},a_q,\tup{O^\agent_\priv(s_q),O_\publ(s_q)},\tup{O^\nature_\priv(s_q),O_\publ(s_q),a_q},u_q}\\
        &= h^G(q).
    \end{align*}
    Hence $\exists t\in \NN: h^G(t) \neq h_2^G(t)$ is false.
    If $q = n$, the same result follows by removing everything from $a_q$.
    We hence get that $f^h(h^M) = h^G$, so $\exists h^M \in H^M, f^h(h^M) = h^G$, therefore $f^h$ is surjective.\\

    \noindent $f^h$ is injective and surjective, hence $f^h$ is a bijection.
    \end{proof}

\begin{lemma}[Bijection between agent and nature histories] 
    Following \Cref{app:theorem_bijection_joint_histories}, we get bijections for the agent and nature histories by omitting the private objects of the other player.
    \noindent Let $g^{\agent,h}\colon H^{\agent,M,\suff} \to H^{\agent,G,\suff}$ defined by:
    \begin{align*}
    g^{\agent,h}(\tup{z^\agent_\priv,z_\publ}) &= \tup{\tup{z^\agent_\priv,z_\publ}}.\\
    g^{\agent,h}(\tup{z^\agent_\priv,z_\publ,a}) &= \tup{\tup{z^\agent_\priv,z_\publ},a,\tup{z^\agent_\priv,z_\publ}}.\\
    g^{\agent,h}(\tup{z^\agent_\priv,z_\publ,a,h'}) &= g^{\agent,h}(\tup{z^\agent_\priv,z_\publ,a})\concat g^{\agent,h}(h').
    \end{align*}

    \noindent Let $f^{\agent,h}\colon H^{\agent,M} \to H^{\agent,G}$ defined by:
    \begin{align*}
    f^{\agent,h}(\tup{z^\agent_\priv,z_\publ}) &= \tup{\tup{z^\agent_\priv,z_\publ}}.\\
    f^{\agent,h}(\tup{z^\agent_\priv,z_\publ,a}) &= \tup{\tup{z^\agent_\priv,z_\publ},a,\tup{z^\agent_\priv,z_\publ}}.\\
    f^{\agent,h}(\tup{z^\agent_\priv,z_\publ,a,h'}) &= f^{\agent,h}(\tup{z^\agent_\priv,z_\publ,a})\concat g^{\agent,h}(h').
    \end{align*}
    $f^{\agent,h}$ is a bijection.\\

    \noindent Let $g^{\nature,h}\colon H^{\nature,M,\suff} \to H^{\nature,G,\suff}$ defined by:
    \begin{align*}
    g^{\nature,h}(\tup{z^\nature_\priv,z_\publ}) &= \tup{\tup{z^\nature_\priv,z_\publ,\bot}}.\\
    g^{\nature,h}(\tup{z^\nature_\priv,z_\publ,a,u}) &= \tup{\tup{z^\nature_\priv,z_\publ,\bot},\tup{z^\nature_\priv,z_\publ,a},u}.\\
    g^{\nature,h}(\tup{z^\nature_\priv,z_\publ,a,u,h'}) &= g^{\nature,h}(\tup{z^\nature_\priv,z_\publ,a,u})\concat g^{\nature,h}(h').
    \end{align*}

    \noindent Let $f^{\nature,h}\colon H^{\nature,M} \to H^{\nature,G}$ defined by:
    \begin{align*}
    f^{\nature,h}(\tup{z^\nature_\priv,z_\publ}) &= \tup{\tup{z^\nature_\priv,z_\publ,\bot}}.\\
    f^{\nature,h}(\tup{z^\nature_\priv,z_\publ,a,u}) &= \tup{\tup{z^\nature_\priv,z_\publ,\bot},\tup{z^\nature_\priv,z_\publ,a},u}.\\
    f^{\nature,h}(\tup{z^\nature_\priv,z_\publ,a,u,h'}) &= f^{\nature,h}(\tup{z^\nature_\priv,z_\publ,a,u})\concat g^{\nature,h}(h').
    \end{align*}
    $f^{\nature,h}$ is a bijection.
\end{lemma}

\begin{corollary}[Corresponding \aohs{}] $f^h$ is a bijection between $H^M$ and $H^G$, so the set of histories in the RPOMDP is equivalent to the set of histories in the parameterized POSG:
    \[H^M \simeq H^G,\]
where $\forall h^M \in H^M, \forall h^G \in H^G.$
    \[h^M \simeq h^G \Longleftrightarrow f^h(h^M) = h^G.\]
Similarly:
    \[H^{\agent,M} \simeq H^{\agent,G},\]
where $\forall h^{\agent,M} \in H^{\agent,M}, \forall h^{\agent,G} \in H^{\agent,G}.$
    \[h^{\agent,M} \simeq h^{\agent,G} \Longleftrightarrow f^{\agent,h}(h^{\agent,M}) = h^{\agent,G}.\]
And:
    \[H^{\nature,M} \simeq H^{\nature,G},\]
where $\forall h^{\agent,M} \in H^{\agent,M}, \forall h^{\agent,G} \in H^{\agent,G}.$
    \[h^{\agent,M} \simeq h^{\agent,G} \Longleftrightarrow f^{\agent,h}(h^{\agent,M}) = h^{\agent,G}.\]
\end{corollary}

Proposition 2 from the main text is a direct corollary of the bijections between histories established above. 
For completeness, we repeat the proposition here.\\

\noindent
\textbf{Proposition 2} (Bijection between policies)\textbf{.} \hspace{1em}
\textit{
Let $f^{\pi}\colon \Pi^M \to \Pi^G$ defined by:
\[
    f^{\pi}(\pi^M)(h^{\agent,G}) = \pi^M((f^{\agent,h})^{-1}(h^{\agent,G})),
\]
then $f^{\pi}$ is a bijection.\\\\
Let $f^{\theta}\colon \Theta^M \to \Theta^G$ defined by:
\[
    f^{\theta}(\theta^M)(h^{\nature,G},\tup{z^\nature_\priv,z_\publ,a}) = \theta^M((f^{\nature,h})^{-1}(h^{\nature,G}),a),
\]
then $f^{\theta}$ is a bijection.
}

\begin{corollary}[Corresponding policies]
    $f^{\pi}$ is a bijection between agent policies, so the set of agent policies in the RPOMDP is equivalent to the set of agent policies in the parameterized POSG:
    \[\Pi^{M} \simeq \Pi^{G},\]
where $\forall \pi^M \in \Pi^M, \forall \pi^G \in \Pi^G.$
    \[\pi^M \simeq \pi^G \Longleftrightarrow f^{\pi}(\pi^M) = \pi^G.\]
Similarly:
    \[\Theta^{M} \simeq \Theta^{G},\]
where $\forall \theta^M \in \Theta^M, \forall \theta^G \in \Theta^G.$
    \[\theta^M \simeq \theta^G \Longleftrightarrow f^{\theta}(\theta^M) = \theta^G.\]
\end{corollary}

\noindent
{\theoremII*\label{app:thm_equiv_values}}

{\allowdisplaybreaks
\begin{proof}
    We prove $R(\mypath^M) = \mathcal{R}(\mypath^G)$ for corresponding paths $\mypath^M, \mypath^G$.
    By definition of corresponding paths, we know that $\forall t\in \NN. \mypath^G(t) = g(\mypath^M(t), u^\pset_{t-1})$.
    \begin{align*}
        R(\mypath^M)
        &= \sum_{t\in \NN} R(\mypath^M(t))\\
        &= \sum_{t\in \NN} R(s_t, a_t, u_t)\\
        &= \sum_{t\in \NN} R(s_t, a_t)\\
        &= \sum_{t\in \NN} \mathcal{R}(\tup{s_t, u^{\pset}_{t-1}}, a_t)\\
        &= \sum_{t\in \NN} \mathcal{R}(\tup{s_t, u^{\pset}_{t-1}}, a_t, \tup{s_t, u^{\pset}_{t-1},a_t}, u_t)\\
        &= \sum_{t\in \NN} \mathcal{R}(\mypath^G(t))\\
        &= \mathcal{R}(\mypath^G).
    \end{align*}
    Now, since joint corresponding policies $\pi^M,\theta^M$ and $\pi^G,\theta^G$ lead to the same distribution over corresponding paths, we know that $V_\phi^{\pi^M,\theta^M} = V_\phi^{\pi^G,\theta^G}$.
\end{proof}}

\subsection{Mixed policies}
Given the bijection between histories and stochastic policies, we can define a bijection between mixed policies similar to the one between stochastic policies.
Note that we apply the bijection between stochastic policies to deterministic policies.

\begin{proposition}[Bijection between mixed policies]
    Let $f^{\pi,\mix}\colon \Pi^{M,\mix} \to \Pi^{G,\mix}$ defined by:
    	\[
         f^{\pi,\mix}(\pi^{M,\mix})(\pi^{G,\detr}) = \pi^{M,\mix}((f^{\pi})^{-1}(\pi^{G,\detr})),
         \]
    then $f^{\pi,\mix}$ is a bijection.\\
    Let $f^{\theta,\mix}\colon \Theta^{M,\mix} \to \Theta^{G,\mix}$ defined by:
    	\[
         f^{\theta,\mix}(\theta^{M,\mix})(\theta^{G,\detr}) = \theta^{M,\mix}((f^{\theta})^{-1}(\theta^{G,\detr})),
         \]
    then $f^{\theta,\mix}$ is a bijection.
\end{proposition}

\begin{corollary}[Corresponding mixed policies]
    $f^{\pi,\mix}$ is a bijection between agent policies, so the set of mixed agent policies in the RPOMDP is equivalent to the set of mixed agent policies in the parameterized POSG:
        \[\Pi^{M,\mix} \simeq \Pi^{G,\mix},\]
    where $\forall \pi^{M,\mix} \in \Pi^{M,\mix}, \forall \pi^{G,\mix} \in \Pi^{G,\mix}.$
        \[\pi^{M,\mix} \simeq \pi^{G,\mix} \Longleftrightarrow f^{\pi,\mix}(\pi^{M,\mix}) = \pi^{G,\mix}.\]
    Similarly:
        \[\Theta^{M,\mix} \simeq \Theta^{G,\mix},\]
    where $\forall \theta^{M,\mix} \in \Theta^{M,\mix}, \forall \theta^{G,\mix} \in \Theta^{G,\mix}.$
        \[\theta^{M,\mix} \simeq \theta^{G,\mix} \Longleftrightarrow f^{\theta,\mix}(\theta^{M,\mix}) = \theta^{G,\mix}.\]
\end{corollary}

\begin{theorem}[Equivalent values mixed policies]
    Let $M$ be an RPOMDP, and $G$ the POSG of $M$.
    Let $\pi^{M,\mix} \in \Pi^{M,\mix}, \pi^{G,\mix} = f^{\pi,\mix}(\pi^{M,\mix}) \in \Pi^{G,\mix}$ be corresponding agent policies, and $\theta^{M,\mix} \in \Theta^{M,\mix}, \theta^{G,\mix} = f^{\theta,\mix}(\theta^{M,\mix}) \in \Theta^{G,\mix}$ be corresponding nature policies.
    Then, their values for the RPOMDP and POSG coincide:
    \[
    V_{\phi}^{\pi^{M,\mix},\theta^{M,\mix}} = V_{\phi}^{\pi^{G,\mix},\theta^{G,\mix}}.
    \]
\end{theorem}
The proof follows the same steps as for Theorem \hyperref[app:thm_equiv_values]{2} for stochastic policies, where the distribution over corresponding paths now follows from the same distribution over corresponding deterministic policies, which in turn leads to the same distributions over corresponding paths.

\subsection{Nature first}
As explained in \Cref{app:nature_first}, when reasoning with nature first semantics, the paths of the POSGs change.
The bijections in this appendix start from the bijection between paths.
Below, we give the bijections for the nature first semantics that differ from the bijections for the agent first semantics.
The bijection proofs follow the same steps in the nature first case as in the agent first case.
We show the adjusted proof for the path bijection to illustrate how to deal with the delayed observation of the last agent action.

\begin{lemma}[Nature first bijection between paths]
	Let M be an RPOMDP, and G the POSG of M.
	There exists a bijection $f\colon \Paths^M \to \Paths^G$
\end{lemma}

{\allowdisplaybreaks
\begin{proof}
     Let $\Paths^{M,\suff} \subseteq (S\times A \times \bm{U})^* \times S^?$ with $? \in \{0,1\}$ be the set of all path segments in the RPOMDP, and let $\Paths^{G,\suff} \subseteq \mathcal{S}^\nature \times \mathcal{A}^\nature \times \mathcal{S}^\agent\times \mathcal{A}^\agent)^* \times (\mathcal{S}^\nature)^?$ with $? \in \{0,1\}$ be the set of all path segments in the POSG.
    With path segment we mean that the path can starts at any time steps $t \in \NN$ and can end at any time step $t' \in \NN, t<=t'$.
    The optional last state is only used for path segments until the horizon.
Note that $\Paths^M \subseteq \Paths^{M,\suff}$ and $\Paths^G \subseteq \Paths^{G,\suff}$.
Let $\mypath^M = \tup{s_0,a_0,u_0,s_1, \dots, s_n}\in\Paths^{M}$ and $t \leq n$, then $\mypath^M(t)$ indicates the $t$-th segment $\tup{s_t,a_t,u_t}$ of $\mypath^M$.
Note that if $t = n$, the segment will only consist of the final state $\tup{s_n}$.
Similarly, let $\mypath^G = \tup{s^\nature_0, a^\nature_0,s^\agent_0,a^\agent_0,s^\nature_1,\dots,s^\nature_n} = \tup{\tup{s_0,u^\bot,\bot},u_0,\tup{s_0,u^\bot,u_0},a_0, \tup{s_1,u^\pset_1,a_0}, \dots, \tup{s_n,u^\pset_n,a_{n-1}}} \in\Paths^G$ and $t \leq n$, then $\mypath^G(t)$ indicates the $t$-th segment $\tup{s^\nature_t,a^\nature_t,s^\agent_t,a^\agent_t} = \tup{\tup{s_t,u^\pset_t,a_{t-1}},u_t,\tup{s_t,u^\pset_t,u_t},a_t}$ of $\mypath^G$.
Note that if $t = n$, the segment will only consist of the final nature state $\tup{s^\agent_n} = \tup{\tup{s_n,u^\pset_n,a_{n-1}}}$.\\
    
\noindent Let $g\colon \Paths^{M,\suff} \times \bm{U}^\pset \times A \pto \Paths^{G,\suff}$ defined by:\\
    \begin{align*}
    g(\tup{s},u^\pset,a) &= \tup{\tup{s,u^\pset,a'}}.\\
    g(\tup{s,a,u}, u^\pset,a') &= \begin{cases}
        \tup{\tup{s,u^\pset,a'},u,\tup{s,u^\pset,u},a} & \text{ if } u \in \bm{U}^\agrees(u^\pset),\\
        \bot & \text{ otherwise.}
    \end{cases}\\
    g(\tup{s,a,u}\concat{\mypath^M}', u^\pset,a') &= \begin{cases}
        g(\tup{s,a,u}, u^\pset,a')\concat g({\mypath^M}', \upd(u^\pset,u,O^\nature_\priv(s),O_\publ(s),a),a)& \text{ if } u \in \bm{U}^\agrees(u^\pset),\\
        \bot & \text{ otherwise.}
    \end{cases}
    \end{align*}

    Let $f\colon \Paths^M \to \Paths^G$ defined by:
    \begin{align*}
    f(\tup{s})& = \tup{\tup{s,u^\bot,\bot}}.\\
    f(\tup{s,a,u}) &= \tup{\tup{s,u^\bot,\bot},u,\tup{s,u^\bot,u},a}.\\
    f(\tup{s,a,u}\concat{\mypath^M}') &= f(\tup{s,a,u})\concat g({\mypath^M}', \upd(u^\bot,u,O^\nature_\priv(s),O_\publ(s),a),a).
    \end{align*}
    Where $u^\bot \in \bm{U}^\pset$ is the totally undefined function.
    Note that the results of $f$ and $g$ are in $\Paths^G$ and $\Paths^{G,\suff}$ by construction.
    Also, note that any call to $g$ that originated from a call in $f$ will have a result by construction.\\

    \noindent We show that $f$ is a bijection, meaning $f$ is injective and surjective.
    We first show $f$ is injective, so we show that:
    \[\forall \mypath^{1,M}, \mypath^{2,M} \in \Paths^M. \mypath^{1,M} \neq \mypath^{2,M} \implies f(\mypath^{1,M}) \neq f(\mypath^{2,M}).\]
    Given arbitrary $\mypath^{1,M}, \mypath^{2,M} \in \Paths^M$, we distinguish between the paths with superscripts $1$ and $2$, respectively.
    Assume $\mypath^{1,M} \neq \mypath^{2,M}$.
    If $\mypath^{1,M}$ and $\mypath^{2,M}$ do not have the same horizon length, then neither do $f(\mypath^{1,M})$ and $f(\mypath^{2,M})$.
    Then trivially, $f(\mypath^{1,M}) \neq f(\mypath^{2,M})$.\\
    
    \noindent Assume $\mypath^{1,M}$ and $\mypath^{2,M}$ have the same horizon length $n$.
    Then $\exists t \leq n$ where $\mypath^{1,M}$ and $\mypath^{2,M}$ deviate, so $\mypath^{1,M}(t) \neq \mypath^{2,M}(t)$.
    Let $q$ be the smallest number where the paths deviate.
    So $\forall t < q. \mypath^{1,M}(t) = \mypath^{2,M}(t)$ and $\mypath^{1,M}(q) \neq \mypath^{2,M}(q)$.
    Assume $q < n$.
    Then we know $\tup{s^1_q,a^1_q,u^1_q} \neq \tup{s^2_q,a^2_q,u^2_q}$, which comes down to: $s^1_q \neq s^2_q \lor a^1_q \neq a^2_q \lor u^1_q \neq u^2_q$.
    \begin{align*}
        f(\mypath^{1,M}) &= f(\mypath^{1,M}(1)\concat {\mypath^{1,M}}')\\
        &= \mypath^{1,G}(1)\concat g({\mypath^{1,M}}',\upd(u^\bot,u^1_0,O^\nature_\priv(s^1_0),O_\publ(s^1_0),a^1_0),a^1_0).
        \intertext{Unfold $g$ until $q$:}
        &= \bigconcat_{t=0}^{q}(\mypath^{1,G}(t))\concat\tup{\tup{s^1_q,\fixed(\mypath^{1,M}_{0:q}),a^1_{q-1}},u^1_q,\tup{s^1_q,\fixed(\mypath^{1,M}_{0:q}),u^1_q},a^1_q}\concat g({\mypath^{1,M}}'',\fixed(\mypath^{1,M}_{0:q+1}),a^1_q)
        \intertext{Since $s^1_q \neq s^2_q \lor a^1_q \neq a^2_q \lor u^1_q \neq u^2_q$:}
        &\neq \bigconcat_{t=0}^{q}(\mypath^{1,G}(t))\concat \tup{\tup{s^2_q,\fixed(\mypath^{1,M}_{0:q}),a^1_{q-1}},u^2_q,\tup{s^2_q,\fixed(\mypath^{1,M}_{0:q}),u^2_q},a^2_q}\concat g({\mypath^{2,M}}'',\fixed(\mypath^{1,M}_{0:q+1}),a^2_q)\\
        &= \bigconcat_{t=0}^{q}(\mypath^{2,G}(t))\concat \tup{\tup{s^2_q,\fixed(\mypath^{2.M}_{0:q}),a^2_{q-1}},u^2_q,\tup{s^2_q,\fixed(\mypath^{2,M}_{0:q}),u^2_q},a^2_q}\concat g({\mypath^{2,M}}'',\fixed(\mypath^{2,M}_{0:q+1}),a^2_q).
        \intertext{Fold $g$ until $1$:}
        &= \mypath^{2,G}(1)\concat g({\mypath^{2,M}}',\upd(u^\bot,u^2_0,O^\nature_\priv(s^2_0),O_\publ(s^2_0),a^2_0),a^2_0)\\
        &= f(\mypath^{2,M}(1)\concat {\mypath^{2,M}}')\\
        &= f(\mypath^{2,M}).
    \end{align*}
    If $q = n$, then the same result follows by removing everything after $\tup{s^1_q,\fixed(\mypath^{1,M}_{0:q}),a^1_{q-1}}, \tup{s^2_q,\fixed(\mypath^{1,M}_{0:q}),a^1_{q-1}}$, and $\tup{s^2_q,\fixed(\mypath^{2,M}_{0:q}),a^2_{q-1}}$.
    We thus have that $f(\mypath^{1,M}) \neq f(\mypath^{2,M})$, so $f$ is injective.\\
        
    \noindent Next, we show that $f$ is surjective, so we show that:
    \[\forall \mypath^G \in \Paths^G, \exists \mypath^M \in \Paths^M. f(\mypath^M) = \mypath^G.\]
    We show this holds by induction on the horizon length of the $\mypath^G \in \Paths^G$. 
    We write the length of $\mypath^G$ as $|\mypath^G|$.

    \noindent Assume $|\mypath^G| = 0$.
    Then $\mypath^G = \tup{s_I, u^\bot}$.
    We have that for $\tup{s_I} \in \Paths^M, f(\tup{s_I}) = \tup{\tup{s_I, u^\bot}} = \mypath^G$.
    So for paths of horizon length $0$, $f$ is surjective.\\

    \noindent Now assume we know, given $q\in \NN, q \geq 1$, that:
    \[\forall \mypath^G \in \Paths^G. |\mypath^G| = q-1 \implies \exists  \mypath^M \in \Paths^M: f(\mypath^M) = \mypath^G.\]
    Take arbitrary $\mypath^G \in \Paths^G$ with horizon length $|\mypath^G| = q$.
    Then we have $\mypath^G = \mypath^G_{0:q-1}\concat \tup{u_{q-1},\tup{s_{q-1}, u^\pset_{q-1},u_{q-1}}, a_{q-1},\tup{s_q, u^\pset_{q},a_{q-1}}}$.
    Then $\mypath^G_{0:q-1} \in \Paths^G$ and $|\mypath^G_{0:q-1}| = q-1$.
    By assumption, we get that:
    \[\exists \mypath^M_{0:q-1} \in \Paths^M, f(\mypath^M_{0:q-1}) = \mypath^G_{0:q-1}.\]
    Let $\mypath^M_{0:q-1} \in \Paths^M$ such that $f(\mypath^M_{0:q-1}) = \mypath^G_{0:q-1}$.
    We then know that in $\mypath^G$:
    \[\forall t < q. u^\pset_{t} = \fixed(\mypath^M_{0:t}).\]
    And, by \Cref{def:equivalent:nature:zsposg} and the definition of $\Paths^G$, that:
    \[u^\pset_{q} = \upd(\fixed(\mypath^M_{0:q-1}),u_{q-1}, O^\nature_\priv(s_{q-1}), O_\publ(s_{q-1}), a_{q-1}).\]
    Let $\tup{\tup{s_{q-2}, \fixed(\mypath^M_{0:q-2}),a_{q-3}}, u_{q-2}, \tup{s_{q-2}, \fixed(\mypath^M_{0:q-2}),u_{q-2}}, a_{q-2},\tup{s_{q-1}, \fixed(\mypath^M_{0:q-1}),a_{q-2}}}$ be the last two segments of $\mypath^G_{0:q-1}$.
    Then by definition and injectivity of $f$, we know that the last two segments of $\mypath^M_{0:q-1}$ are $\tup{s_{q-2},a_{q-2},u_{q-2},s_{q-1}}$.\\
    
    \noindent Now, by definition \Cref{def:equivalent:nature:zsposg} and the definition of $\Paths^G$, we know that:
    \[\mypath^G = \mypath^G_{0:q-1}\concat \tup{u_{q-1},\tup{s_{q-1}, u^\pset_{q-1},u_{q-1}}, a_{q-1},\tup{s_q, u^\pset_{q},a_{q-1}}} \in \Paths^G \]
    \[\Longleftrightarrow \]
    \[\mypath^G_{0:q-1} \in \Paths^G 
    \land \mathcal{T}^\nature(\tup{s_{q-1},u^\pset_{q-1},a_{q-2}},u_{q-1},\tup{s_{q-1},u^\pset_{q-1},u_{q-1}}) > 0\]
    \[\land \mathcal{T}^\agent(\tup{s_{q-1},u^\pset_{q-1}),u_{q-1}}, a_{q-1}, \tup{s_{q},u^\pset_{q},a_{q-1}}) > 0\] 
    \[\Longleftrightarrow \]
    \[\mypath^G_{0:q-1} \in \Paths^G 
    \land u_{q-1} \in \bm{U}^\agrees(u^\pset_{q-1})
    \land \bm{T}(u_{q-1})(s_{q-1},a_{q-1},s_{q}) > 0
    \]
    \[\Longleftrightarrow \]
    \[\mypath^G_{0:q-1} \in \Paths^G 
    \land u_{q-1} \in \bm{U}^\agrees(\fixed(\mypath^M_{0:q-1}))
    \land \bm{T}(u_{q-1})(s_{q-1},a_{q-1},s_{q}) > 0.
    \]
    So, since $\mypath^G \in \Paths^G$, we know $u_{q-1} \in \bm{U}^\agrees(\fixed(\mypath^M_{0:q-1}))$ and $\bm{T}(u_{q-1})(s_{q-1},a_{q-1},s_{q}) > 0$, which are the restrictions for $\mypath^M = \mypath^M_{0:q-1}\concat \tup{a_{q-1}, u_{q-1}, s_q} \in \Paths^M$ to hold.
    \begin{align*}
        f(\mypath^M ) &= f(\mypath^M_{0:q-1}\concat \tup{a_{q-1}, u_{q-1}, s_q})\\
        &= \bigconcat_{t=0}^{q-2}\mypath^G_{0:q-1}(t)\concat g(\tup{s_{q-1}, a_{q-1}, u_{q-1},s_q}, \fixed(\mypath^M_{0:q-1}),a_{q-2})\\
        &= \bigconcat_{t=0}^{q-2}\mypath^G_{0:q-1}(t)\concat g(\tup{s_{q-1}, a_{q-1}, u_{q-1}}, \fixed(\mypath^M_{0:q-1}),a_{q-2})\concat g(\tup{s_q}, \fixed(\mypath^M_{0:q}),a_{q-1})\\
        &= \bigconcat_{t=0}^{q-2}\mypath^G_{0:q-1}(t)\concat g(\tup{s_{q-1}, a_{q-1}, u_{q-1}}, \fixed(\mypath^M_{0:q-1}),a_{q-2})\concat g(\tup{s_q}, \upd(\fixed(\mypath^M_{0:q-1}),u_{q-1}, O^\nature_\priv(s_{q-1}), O_\publ(s_{q-1}), a_{q-1}),a_{q-1})\\
        &= \bigconcat_{t=0}^{q-2}\mypath^G_{0:q-1}(t)\concat g(\tup{s_{q-1}, a_{q-1}, u_{q-1}}, \fixed(\mypath^M_{0:q-1}),a_{q-2})\concat g(\tup{s_q}, u^\pset_{q},a_{q-1})\\
        &= \bigconcat_{t=0}^{q-2}\mypath^G_{0:q-1}(t)\concat \tup{\tup{s_{q-1}, \fixed(\mypath^M_{0:q-1}),a_{q-2}},u_{q-1},\tup{s_{q-1}, \fixed(\mypath^M_{0:q-1}),u_{q-1}}, a_{q-1}}\concat g(\tup{s_q}, u^\pset_{q},a_{q-1})\\
        &= \bigconcat_{t=0}^{q-2}\mypath^G_{0:q-1}(t)\concat \tup{\tup{s_{q-1}, u^{\pset}_{q-1},a_{q-2}},u_{q-1},\tup{s_{q-1}, u^{\pset}_{q-1},u_{q-1}}, a_{q-1}}\concat g(\tup{s_q}, u^\pset_{q},a_{q-1})\\
        &= \mypath^G_{0:q-1}\concat \tup{u_{q-1},\tup{s_{q-1}, u^{\pset}_{q-1},u_{q-1}}, a_{q-1}}\concat g(\tup{s_q}, u^\pset_{q},a_{q-1})\\
        &= \mypath^G_{0:q-1}\concat \tup{u_{q-1},\tup{s_{q-1}, u^{\pset}_{q-1},u_{q-1}}, a_{q-1}, \tup{s_q, u^{\pset}_{q},a_{q-1}}}\\
        &= \mypath^G.
    \end{align*}
    So if $f$ is surjective for paths of arbitrary length $q-1 \in \NN$, $f$ is surjective for paths of length $q$.
    Hence, by induction, $f$ is surjective.\\
    
    \noindent $f$ is injective and surjective, hence $f$ is a bijection.
\end{proof}}

\begin{corollary}[Nature first bijections between histories]
    \noindent Let $g^h\colon H^{M,\suff} \times A \to H^{G,\suff}$ defined by:
    \begin{align*}
    g^h(\tup{z^\agent_\priv,z^\nature_\priv,z_\publ},a') &= \tup{\tup{z^\nature_\priv,z_\publ,a'},\tup{z^\agent_\priv,z_\publ}}.\\
    g^h(\tup{z^\agent_\priv,z^\nature_\priv,z_\publ,a,u},a') &= \tup{\tup{z^\nature_\priv,z_\publ,a'},\tup{z^\agent_\priv,z_\publ},u,\tup{z^\nature_\priv,z_\publ,\bot},\tup{z^\agent_\priv,z_\publ},a}.\\
    g^h(\tup{z^\agent_\priv,z^\nature_\priv,z_\publ,a,u}\concat h',a') &= g^h(\tup{z^\agent_\priv,z^\nature_\priv,z_\publ,a,u},\bot)\concat g^h(h',a).
    \end{align*}

    \noindent Let $f^h\colon H^M \to H^G$ defined by:
    \begin{align*}
    f^h(\tup{z^\agent_\priv,z^\nature_\priv,z_\publ}) &= \tup{\tup{z^\nature_\priv,z_\publ,\bot},\tup{z^\agent_\priv,z_\publ}}.\\
    f^h(\tup{z^\agent_\priv,z^\nature_\priv,z_\publ,a,u}) &= \tup{\tup{z^\nature_\priv,z_\publ,\bot},\tup{z^\agent_\priv,z_\publ},u,\tup{z^\nature_\priv,z_\publ,\bot},\tup{z^\agent_\priv,z_\publ},a}.\\
    f^h(\tup{z^\agent_\priv,z^\nature_\priv,z_\publ,a,u}\concat h') &= f^h(\tup{z^\agent_\priv,z^\nature_\priv,z_\publ,a,u})\concat g^h(h',a).
    \end{align*}

    \noindent Let $g^{\agent,h}\colon H^{\agent,M,\suff} \to H^{\agent,G,\suff}$ defined by:
    \begin{align*}
    g^{\agent,h}(\tup{z^\agent_\priv,z_\publ}) &= \tup{\tup{z^\agent_\priv,z_\publ}}.\\
    g^{\agent,h}(\tup{z^\agent_\priv,z_\publ,a}) &= \tup{\tup{z^\agent_\priv,z_\publ},\tup{z^\agent_\priv,z_\publ},a}.\\
    g^{\agent,h}(\tup{z^\agent_\priv,z_\publ,a,h'}) &= g^{\agent,h}(\tup{z^\agent_\priv,z_\publ,a})\concat g^{\agent,h}(h').
    \end{align*}

    \noindent Let $f^{\agent,h}\colon H^{\agent,M} \to H^{\agent,G}$ defined by:
    \begin{align*}
    f^{\agent,h}(\tup{z^\agent_\priv,z_\publ}) &= \tup{\tup{z^\agent_\priv,z_\publ}}.\\
    f^{\agent,h}(\tup{z^\agent_\priv,z_\publ,a}) &= \tup{\tup{z^\agent_\priv,z_\publ},\tup{z^\agent_\priv,z_\publ},a}.\\
    f^{\agent,h}(\tup{z^\agent_\priv,z_\publ,a,h'}) &= f^{\agent,h}(\tup{z^\agent_\priv,z_\publ,a})\concat g^{\agent,h}(h').
    \end{align*}
    $f^{\agent,h}$ is a bijection.\\

    \noindent Let $g^{\nature,h}\colon H^{\nature,M,\suff} \times A \to H^{\nature,G,\suff}$ defined by:
    \begin{align*}
    g^{\nature,h}(\tup{z^\nature_\priv,z_\publ},a') &= \tup{\tup{z^\nature_\priv,z_\publ,a'}}.\\
    g^{\nature,h}(\tup{z^\nature_\priv,z_\publ,a,u},a') &= \tup{\tup{z^\nature_\priv,z_\publ,a'},u,\tup{z^\nature_\priv,z_\publ,\bot}}.\\
    g^{\nature,h}(\tup{z^\nature_\priv,z_\publ,a,u,h'},a') &= g^{\nature,h}(\tup{z^\nature_\priv,z_\publ,a,u},a')\concat g^{\nature,h}(h',a).
    \end{align*}

    \noindent Let $f^{\nature,h}\colon H^{\nature,M} \to H^{\nature,G}$ defined by:
    \begin{align*}
    f^{\nature,h}(\tup{z^\nature_\priv,z_\publ}) &= \tup{\tup{z^\nature_\priv,z_\publ,\bot}}.\\
    f^{\nature,h}(\tup{z^\nature_\priv,z_\publ,a,u}) &= \tup{\tup{z^\nature_\priv,z_\publ,\bot},u,\tup{z^\nature_\priv,z_\publ,\bot}}.\\
    f^{\nature,h}(\tup{z^\nature_\priv,z_\publ,a,u,h'}) &= f^{\nature,h}(\tup{z^\nature_\priv,z_\publ,a,u})\concat g^{\nature,h}(h',a).
    \end{align*}
    $f^{\nature,h}$ is a bijection.
\end{corollary}
    
The bijection between stochastic nature policies no longer involves an extra state observation, whereas the bijection between stochastic agent policy now does.
Note that the extra agent state observation for the inverse function for the agent policy bijection can be derived from the history input, as this observation is the same as the last nature state observation contained in that history.
\begin{corollary}[Bijection between policies]
Let $f^{\pi}\colon \Pi^M \to \Pi^G$ defined by:
\[
    f^{\pi}(\pi^M)(h^{\agent,G},\tup{z^\agent_\priv,z_\publ}) = \pi^M((f^{\agent,h})^{-1}(h^{\agent,G})),
\]
then $f^{\pi}$ is a bijection.\\

Let $f^{\theta}\colon \Theta^M \to \Theta^G$ defined by:
\[
    f^{\theta}(\theta^M)(h^{\nature,G}) = \theta^M((f^{\nature,h})^{-1}(h^{\nature,G})),
\]
then $f^{\theta}$ is a bijection.

\end{corollary}

\clearpage
\newpage
\section{Nash Equilibrium}\label{app:nashEquilibrium}
This appendix contains all the proofs required to show the existence of a Nash equilibrium in our POSGs, \ie, \Cref{thm:exist:nash}, restated below.
{\label{app:thm_nash_exists}\theoremIII*}
Throughout this appendix, we use the RPOMDP histories, paths, and policies, as these require simpler notation.
We refer to \Cref{app:value:function:proofs} for the bijections between the RPOMDP and POSG paths, histories, and policies.

\subsection{Sufficient Statistic}\label{app:suff.stat}
Where in POMDPs the \aoh{} of the agent is enough to reason optimally, this is not the case for RPODMPs and their equivalent POSGs.
Apart from their own \aoh{}, the players must also consider all possible \aohs{} of the other player.

We adjust the notion of occupancy state used in \cite{Springer:HSVI} to work with the infinite nature action space of our RPOMDP and equivalent POSGs.
As mentioned in the introduction of this appendix, we use the RPOMDP notation.
Given $\pi_{0:t-1} \in \Pi_{0:t-1}, \theta_{0:t-1} \in \Theta_{0:t-1}$, \cite{Springer:HSVI} defines the occupancy state $\ocs{0:t-1}$ as the probability distribution over all joint \aohs{} given agent and nature policies $\pi_{0:t-1}$ and $\theta_{0:t-1}$.
\[\forall h_t \in H_t: \ocs{0:t-1}(h_t) = \Pr(h_t \given \pi_{0:t-1},\theta_{0:t-1}).\]
\[\sum_{h_t\in H_t} \ocs{0:t-1}(h_t) = 1.\]
In the original definition of the occupancy state, nature's action space is finite.
This occupancy state is a sufficient statistic for computing the next occupancy state and the expected reward at time $t$ given the next $\pi_t$ and $\theta_t$ in their POSG models.

However, as we deal with an infinite action space, we must make adjustments to ensure the subset of joint \aohs{} for each occupancy state is finite.
Therefore, we keep track of nature's policy to be able to generate the finite subset of joint \aohs{} that the corresponding occupancy states have a distribution over.

Our version of the occupancy state extends the original occupancy state with the corresponding nature policy.
Given  $\pi_{0:t-1} \in \Pi_{0:t-1}, \theta_{0:t-1} \in \Theta_{0:t-1}$:
\[\Ocs{0:t-1} \stackrel{\text{def}}{=} \tup{\ocs{0:t-1}, \theta_{0:t-1}}.\]
\[\forall h_t \in H_t: \ocs{0:t-1}{\theta_{0:t-1}}(h_t) = \Pr(h_t \given \pi_{0:t-1},\theta_{0:t-1}).\]
\[\sum_{h_t\in H_t} \ocs{0:t-1}(h_t) = 1.\]

We show that the occupancy state $\Ocs{0:t-1}$ together with agent and nature policies $\pi_t, \theta_t$ at time $t$, is a sufficient statistic for computing the next occupancy state $\Ocs{0:t}$ and the expected reward $\mathsf{R}(\Ocs{0:t-1}, \pi_t, \theta_t) = \EE[r_t \given \pi_{0:t-1}, \pi_t, \theta_{0:t-1}, \theta_t]$.
Note that these proofs are based on the proofs in Appendix B of \cite{Springer:HSVI}.
{\allowdisplaybreaks
\begin{align*}
    &\Ocs{0:t}(h_t \concat \tup{a_t, u_t, z_\priv^\agent, z_\priv^\nature, z_\publ}) \stackrel{\text{def}}{=} \tup{\ocs{0:t}(h_t \concat \tup{a_t, u_t, z_\priv^\agent, z_\priv^\nature, z_\publ}), \theta_{0:t}}.\\
\intertext{Where:}
    &\theta_{0:t} \stackrel{\text{def}}{=} \theta_{0:t-1} \concat \theta_t.\\
    &\ocs{0:t}(h_t \concat \tup{a_t, u_t, z_\priv^\agent, z_\priv^\nature, z_\publ}) \stackrel{\text{def}}{=} \Pr(h_t, a_t, u_t, z_\priv^\agent, z_\priv^\nature, z_\publ \given \pi_{0:t}, \theta_{0:t})\\
    &= \sum_{s, s'\in S} \Pr(h_t, a_t, u_t, z_\priv^\agent, z_\priv^\nature, z_\publ, s, s' \given \pi_{0:t}, \theta_{0:t})\\
    &= \sum_{s, s'\in S} \Pr(z_\priv^\agent, z_\priv^\nature, z_\publ \given h_t, a_t, u_t, s, s', \pi_{0:t}, \theta_{0:t})\Pr(h_t, a_t, u_t, s, s' \given \pi_{0:t}, \theta_{0:t}).
\intertext{The chance of an observation only depends on the state:}
    &= \sum_{s, s'\in S} \Pr(z_\priv^\agent, z_\priv^\nature, z_\publ \given s')\Pr(h_t, a_t, u_t, s, s' \given \pi_{0:t}, \theta_{0:t})\\
    &= \sum_{s, s'\in S} \Pr(z_\priv^\agent, z_\priv^\nature, z_\publ \given s')\Pr(s' \given h_t, a_t, u_t, s,\pi_{0:t}, \theta_{0:t})\Pr(h_t, a_t, u_t, s \given \pi_{0:t}, \theta_{0:t}).
\intertext{The chance of reaching a state only depends on the previous state and the agent and nature actions:}
    &= \sum_{s, s'\in S} \Pr(z_\priv^\agent, z_\priv^\nature, z_\publ \given s')\Pr(s' \given a_t, u_t, s)\Pr(h_t, a_t, u_t, s \given \pi_{0:t}, \theta_{0:t})\\
    &= \sum_{s, s'\in S} \Pr(z_\priv^\agent, z_\priv^\nature, z_\publ \given s')\Pr(s' \given a_t, u_t, s)\Pr(u_t \given h_t, a_t, s, \pi_{0:t}, \theta_{0:t})\Pr(h_t, a_t, s \given \pi_{0:t}, \theta_{0:t}).
\intertext{The chance of a nature action only depends on nature's policy at time $t$, the history, and the agent action at time $t$:}
    &= \sum_{s, s'\in S} \Pr(z_\priv^\agent, z_\priv^\nature, z_\publ \given s')\Pr(s' \given a_t, u_t, s)\Pr(u_t \given h_t, a_t, \theta_{t})\Pr(h_t, a_t, s \given \pi_{0:t}, \theta_{0:t})\\
    &= \sum_{s, s'\in S} \Pr(z_\priv^\agent, z_\priv^\nature, z_\publ \given s')\Pr(s' \given a_t, u_t, s)\Pr(u_t \given h_t, a_t, \theta_{t})\Pr(a_t \given h_t, s, \pi_{0:t}, \theta_{0:t})\Pr(h_t, s \given \pi_{0:t}, \theta_{0:t}).
\intertext{The chance of an agent action only depends on the agent's policy at time $t$, and the history:}
    &= \sum_{s, s'\in S} \Pr(z_\priv^\agent, z_\priv^\nature, z_\publ \given s')\Pr(s' \given a_t, u_t, s)\Pr(u_t \given h_t, a_t, \theta_{t})\Pr(a_t \given h_t, \pi_{t})\Pr(h_t, s \given \pi_{0:t}, \theta_{0:t})\\
    &= \sum_{s, s'\in S} \Pr(z_\priv^\agent, z_\priv^\nature, z_\publ \given s')\Pr(s' \given a_t, u_t, s)\Pr(u_t \given h_t, a_t, \theta_{t})\Pr(a_t \given h_t, \pi_{t})\Pr(s \given h_t, \pi_{0:t}, \theta_{0:t})\Pr(h_t \given \pi_{0:t}, \theta_{0:t}).
\intertext{The chance of being in a state can be computed via the belief generated by the joint history (see \Cref{app:appendix_prelim}):}
    &= \sum_{s, s'\in S} \Pr(z_\priv^\agent, z_\priv^\nature, z_\publ \given s')\Pr(s' \given a_t, u_t, s)\Pr(u_t \given h_t, a_t, \theta_{t})\Pr(a_t \given h_t, \pi_{t})\Pr(s \given h_t)\Pr(h_t \given \pi_{0:t}, \theta_{0:t}).
\intertext{The chance of a history at time $t$ does not depend on actions of time $t$:}
    &= \sum_{s, s'\in S} \Pr(z_\priv^\agent, z_\priv^\nature, z_\publ \given s')\Pr(s' \given a_t, u_t, s)\Pr(u_t \given h_t, a_t, \theta_{t})\Pr(a_t \given h_t, \pi_{t})\Pr(s \given h_t)\Pr(h_t \given \pi_{0:t-1}, \theta_{0:t-1})\\
    &= \sum_{s, s'\in S} O^\agent_\priv(s',z_\priv^\agent) O^\nature_\priv(s',z_\priv^\nature) O_\publ(s',z_\publ)\Pr(s' \given a_t, u_t, s)\Pr(u_t \given h_t, a_t, \theta_{t})\Pr(a_t \given h_t, \pi_{t})\Pr(s \given h_t)\Pr(h_t \given \pi_{0:t}, \theta_{0:t})\\
    &= \sum_{s, s'\in S} O^\agent_\priv(s',z_\priv^\agent) O^\nature_\priv(s',z_\priv^\nature) O_\publ(s',z_\publ) \bm{T}(u_t)(s,a_t)(s')\Pr(u_t \given h_t, a_t, \theta_{t})\Pr(a_t \given h_t, \pi_{t})\Pr(s \given h_t)\Pr(h_t \given \pi_{0:t}, \theta_{0:t})\\
    &= \sum_{s, s'\in S} O^\agent_\priv(s',z_\priv^\agent) O^\nature_\priv(s',z_\priv^\nature) O_\publ(s',z_\publ) \bm{T}(u_t)(s,a_t)(s') \theta_t(h^\nature_t,a_t)(u_t)\Pr(a_t \given h_t, \pi_{t})\Pr(s \given h_t)\Pr(h_t \given \pi_{0:t}, \theta_{0:t})\\
    &= \sum_{s, s'\in S} O^\agent_\priv(s',z_\priv^\agent) O^\nature_\priv(s',z_\priv^\nature) O_\publ(s',z_\publ) \bm{T}(u_t)(s,a_t)(s') \theta_t(h^\nature_t,a_t)(u_t)\pi_t(h^\agent_t)(a_t)\Pr(s \given h_t)\Pr(h_t \given \pi_{0:t}, \theta_{0:t})\\
    &= \sum_{s, s'\in S} O^\agent_\priv(s',z_\priv^\agent) O^\nature_\priv(s',z_\priv^\nature) O_\publ(s',z_\publ) \bm{T}(u_t)(s,a_t)(s') \theta_t(h^\nature_t,a_t)(u_t)\pi_t(h^\agent_t)(a_t)b(s, h_t)\Pr(h_t \given \pi_{0:t}, \theta_{0:t}).
\intertext{Where $b(s, h_t)$ is the belief computed by $t$ belief updates given the joint \aoh{} $h_t$ (see \Cref{app:appendix_prelim}).}
    &= \sum_{s, s'\in S} O^\agent_\priv(s',z_\priv^\agent) O^\nature_\priv(s',z_\priv^\nature) O_\publ(s',z_\publ) \bm{T}(u_t)(s,a_t)(s') \theta_t(h^\nature_t,a_t)(u_t)\pi_t(h^\agent_t)(a_t)b(s, h_t)\ocs{0:t-1}(h_t).
\end{align*}
}

This shows that we can compute the successor occupancy state using only the previous occupancy state $\Ocs{0:t-1} = \tup{\ocs{0:t-1}, \theta_{0:t-1}}$ and policies $\pi_t, \theta_t$ at time $t$.
Note that we can use the nature policy $\theta_{0:t}$ to generate the finite subset of relevant histories $\rel(\theta_{0:t}) \subset H$ that possibly have a non-zero probability.
We can then select the relevant histories of $t$ time steps, denoted as $\rel(\theta_{0:t})_t \subset H_t$, to compute the occupancy states.
See \Cref{app:relevant_histories} for details on the set of relevant histories.

Next, we look at the expected reward.
{\allowdisplaybreaks
\begin{align*}
    \EE[r_t \given \pi_{0:t}, \theta_{0:t}] &= \sum_{s\in S}\sum_{a\in A} R(s,a) \Pr(s,a \given \pi_{0:t}, \theta_{0:t})\\
    &= \sum_{s\in S}\sum_{a\in A} R(s,a) \sum_{h_t\in \rel(\theta_{0:t-1})_t}\Pr(s, a, h_t \given \pi_{0:t}, \theta_{0:t})\\
    &= \sum_{s\in S}\sum_{a\in A} R(s,a) \sum_{h_t\in \rel(\theta_{0:t-1})_t}\Pr(a \given s, h_t, \pi_{0:t}, \theta_{0:t})\Pr(s, h_t \given \pi_{0:t}, \theta_{0:t}).
\intertext{The chance of an agent action only depends on the agents's policy at time $t$, and the history:}
    &= \sum_{s\in S}\sum_{a\in A} R(s,a) \sum_{h_t\in \rel(\theta_{0:t-1})_t}\Pr(a \given h_t, \pi_{t})\Pr(s, h_t \given \pi_{0:t}, \theta_{0:t})\\
    &= \sum_{s\in S}\sum_{a\in A} R(s,a) \sum_{h_t\in \rel(\theta_{0:t-1})_t}\Pr(a \given h_t, \pi_{t})\Pr(s \given h_t, \pi_{0:t}, \theta_{0:t})\Pr(h_t \given \pi_{0:t}, \theta_{0:t}).
\intertext{The chance of being in a state can be computed via the belief generated by the joint history (see \Cref{app:appendix_prelim}):}
    &= \sum_{s\in S}\sum_{a\in A} R(s,a) \sum_{h_t\in \rel(\theta_{0:t-1})_t}\Pr(a \given h_t, \pi_{t})\Pr(s \given h_t)\Pr(h_t \given \pi_{0:t}, \theta_{0:t}).
\intertext{The chance of a history at time $t$ does not depend on actions of time $t$:}
    &= \sum_{s\in S}\sum_{a\in A} R(s,a) \sum_{h_t\in \rel(\theta_{0:t-1})_t}\Pr(a \given h_t, \pi_{t})\Pr(s \given h_t)\Pr(h_t \given \pi_{0:t-1}, \theta_{0:t-1})\\
    &= \sum_{s\in S}\sum_{a\in A} R(s,a) \sum_{h_t\in \rel(\theta_{0:t-1})_t}\pi_t(h^\agent_t)(a)\Pr(s \given h_t)\Pr(h_t \given \pi_{0:t-1}, \theta_{0:t-1})\\
    &= \sum_{s\in S}\sum_{a\in A} R(s,a) \sum_{h_t\in \rel(\theta_{0:t-1})_t}\pi_t(h^\agent_t)(a)b(s,h_t)\Pr(h_t \given \pi_{0:t-1}, \theta_{0:t-1})\\
    &= \sum_{s\in S}\sum_{a\in A} R(s,a) \sum_{h_t\in \rel(\theta_{0:t-1})_t}\pi_t(h^\agent_t)(a)b(s,h_t)\ocs{0:t-1}(h_t).
\end{align*}
This shows that we can compute the expected reward at time $t$ using only the previous occupancy state $\Ocs{0:t-1} = (\ocs{0:t-1}, \theta_{0:t-1})$ and policy $\pi_t$ at time $t$.\\

\begin{remark}
    The equivalent formulation of the occupancy state using POSG notation is as follows:
    \begin{align*}
        &\Ocs{0:t}(h_t \concat \tup{\tup{z_\priv^\agent, z_\publ}, \tup{z_\priv^\nature, z_\publ, \bot}, a_t, \tup{z_\priv^\agent, z_\publ}, \tup{z_\priv^\nature, z_\publ, a_t}, u_t})\\
        & \qquad \stackrel{\text{def}}{=} \tup{\ocs{0:t}(h_t \concat \tup{\tup{z_\priv^\agent, z_\publ}, \tup{z_\priv^\nature, z_\publ, \bot}, a_t, \tup{z_\priv^\agent, z_\publ}, \tup{z_\priv^\nature, z_\publ, a_t}, u_t}), \theta_{0:t}}.\\
    \intertext{Where:}
        &\theta_{0:t} \stackrel{\text{def}}{=} \theta_{0:t-1} \concat \theta_t.\\
        &\ocs{0:t}(h_t \concat \tup{\tup{z_\priv^\agent, z_\publ}, \tup{z_\priv^\nature, z_\publ, \bot}, a_t, \tup{z_\priv^\agent, z_\publ}, \tup{z_\priv^\nature, z_\publ, a_t}, u_t}) \stackrel{\text{def}}{=} \Pr(h_t, a_t, u_t, z_\priv^\agent, z_\priv^\nature, z_\publ \given \pi_{0:t}, \theta_{0:t})\\
        &= \sum_{s\in \mathcal{S}^\agent} \sum_{s'\in \mathcal{S}^\agent} \mathcal{O}^\agent(s',\tup{z_\priv^\agent, z_\publ})\mathcal{O}^\nature(s',\tup{z_\priv^\nature, z_\publ})\mathcal{T}^\nature(\mathcal{T}^\agent(s,a_t), u_t, s')\theta_t(h^\nature_t,a_t)(u_t)\pi_t(h^\agent_t)(a_t)b(s, h_t)\ocs{0:t-1}(h_t).
    \end{align*}
    We can compute the expected reward as follows:
    \begin{align*}
        \EE[r_t \given \pi_{0:t}, \theta_{0:t}] &= \sum_{s\in \mathcal{S}^\agent}\sum_{a\in \mathcal{A}^\agent} \mathcal{R}(s,a) \sum_{h_t\in \rel(\theta_{0:t-1})_t}\pi_t(h^\agent_t)(a)b(s,h_t)\ocs{0:t-1}(h_t).
    \end{align*}
    This formulation again shows that the occupancy state $\Ocs{0:t-1}$ together with agent and nature policies $\pi_t, \theta_t$ at time $t$, is a sufficient statistic for computing the next occupancy state $\Ocs{0:t}$ and the expected reward $\mathsf{R}(\Ocs{0:t-1}, \pi_t, \theta_t)$.
\end{remark}

\subsection{Occupancy Game}\label{app:occupacy_game}
Given the occupancy state, we can define a non-observable, non-stochastic game: a zero-sum occupancy game (OG) \cite{Springer:HSVI}.
As the occupancy state is a sufficient statistic, computing a Nash equilibrium in this OG is equivalent to a Nash equilibrium in our original RPOMDP.

\begin{definition}[Occupancy game]\label{app:def:occupancy_game}
Given an RPOMDP $\tup{S, A, \bm{T}, R, \Zagent, \Znature, \Zpub, O^\agent_\priv, O^\nature_\priv, O_\publ,\sticky,\agent}$, and a horizon $K\in \NN$, we define the occupancy game as a tuple $(\mathsf{S}^\agent, \mathsf{S}^\nature, \mathsf{A}^\agent, \mathsf{A}^\nature, \mathsf{T}, \mathsf{R})$ where the sets of states and actions are defined as follows:
$\mathsf{S}^\agent = \bigcup_{t=0}^{K-1} \bigcup_{\pi_{0:t} \in \Pi_{0:t}} \bigcup_{\theta_{0:t}\in\Theta_{0:t}} \Ocs{0:t}$ is the infinite set of agent states, and $\mathsf{S}^\nature = \bigcup_{t=0}^{K-1} (\bigcup_{\pi_{0:t} \in \Pi_{0:t}} \bigcup_{\theta_{0:t}\in\Theta_{0:t}} \Ocs{0:t} \times \Pi_{t+1})$ the infinite set of nature states;
$\mathsf{A}^\agent = \bigcup_{t=0}^{K-1} \Pi_{t}$ is the infinite set of agent actions, and $\mathsf{A}^\nature = \bigcup_{t=0}^{K-1} \Theta_{t}$ the infinite set of nature actions;
The transition and reward functions are then defined as:
\begin{itemize}
    \item $\mathsf{T} = \mathsf{T}^\agent \cup \mathsf{T}^\nature$, the transition function, where:
    \begin{itemize}
        \item $\mathsf{T}^\agent \colon \mathsf{S}^\agent \times \mathsf{A}^\agent \pto \mathsf{S}^\nature$ the agent's transition function.
        \item $\mathsf{T}^\nature \colon \mathsf{S}^\nature \times \mathsf{A}^\nature \pto \mathsf{S}^\agent$ nature's transition function.
    \end{itemize}
    \item $\mathsf{R}\colon \mathsf{S}^\agent \times \mathsf{A}^\agent \to \RR$ the reward function.
\end{itemize}
Where:
\begin{itemize}
    \item $\mathsf{R}(\tup{\ocs{0:t}, \theta_{0:t}},\pi_{t+1}) = \sum_{s\in S}\sum_{a\in A} R(s,a) \sum_{h_t\in \rel(\theta_{0:t-1})_t}\pi_t(h^\agent_t, a)b(s,h_t)\ocs{0:t-1}(h_t)$.
    \item $\mathsf{T^\agent}(\tup{\ocs{0:t}, \theta_{0:t}},\pi_{t+1}) = \tup{\tup{\ocs{0:t}, \theta_{0:t}},\pi_{t+1}}$.
    \item $\mathsf{T^\nature}(\tup{\tup{\ocs{0:t}, \theta_{0:t}},\pi_{t+1}},\theta_{t+1}) = \tup{\ocs{0:t+1}, \theta_{0:t+1}}$, where:\\
    \begin{itemize}
        \item $\theta_{0:t+1} = \theta_{0:t} \concat \theta_{t+1}$.
        \item $\forall h_{t+1}\in \rel(\theta_{0:t})_{t+1}, \forall a \in \mathcal{A}^\agent, \forall u \in \mathcal{A}^\nature, \forall z^\agent_\priv, z^\nature_\priv, z_\publ \in \Zagent \times \Znature \times \Zpub, \ocs{0:t+1}(\tup{h_{t+1}, a, u, z^\agent_\priv, z^\nature_\priv, z_\publ}) =$\\
    \[ \sum_{s, s'\in S} O^\agent_\priv(s',z_\priv^\agent) O^\nature_\priv(s',z_\priv^\nature) O_\publ(s',z_\publ) \bm{T}(u_t)(s,a_t)(s') \theta_t(h^\nature_t,a_t)(u_t)\pi_t(h^\agent_t)(a_t)b(s, h_t)\ocs{0:t-1}(h_t).\]
    \end{itemize}
\end{itemize}
\end{definition}
Where $b(s, h_t)$ is the belief computed by $t$ belief updates given the joint \aoh{} $h_t$ (see \Cref{app:appendix_prelim}).
For deriving the reward and transition functions, see \Cref{app:suff.stat}.

\subsection{Mixed policies}\label{app:mixed_policies}
As shown in \Cref{app:suff.stat}, the occupancy state is a sufficient statistic for the POSG and, hence, for the RPOMDP.
In \cref{app:convex_semi-infinite_game}, we show that the occupancy game has a Nash equilibrium for finite horizon reward maximization, and an optimal policy for the agent exists.
To prove that this Nash equilibrium exists, we need to reason with mixed policies instead of stochastic policies.
In this section, we prove that reasoning with the set of mixed agent and nature policies results in the same set of distributions over paths, and hence the same possible values, as reasoning with the set of stochastic policies.

Concretely, we need to show that for every stochastic policy, there exists a mixed policy that behaves equivalently, meaning it results in the same distribution over paths.
We focus on the RPOMDP policies.
The same results follow for the POSG policies using the bijections from \Cref{app:value:function:proofs}.
\begin{restatable}[Existence of equivalent mixed policy]{theorem}{existsMixedPolicy}\label{app:thm_exists_mixed_policy}
    Let $\mu^{\pi,\theta} \in \dist{\Paths^{M}}$ be the probability distribution over paths in the RPOMDP resulting from executing agent policy $\pi$ and nature policy $\theta$. 
    Then:
    \[
        \forall \pi \in \Pi, \exists \pi^\mix \in \Pi^\mix, \forall \theta\in \Theta \cup \Theta^\mix.\, \mu^{\pi,\theta} = \mu^{\pi^\mix,\theta},
    \]\[
        \forall \theta \in \Theta, \exists \theta^\mix \in \Theta^\mix, \forall \pi\in \Pi \cup \Pi^\mix.\, \mu^{\pi,\theta} = \mu^{\pi,\theta^\mix}.
    \]
\end{restatable}

Before proving the theorem above, we consider the other key results that follow.
We immediately get the following corollary from \Cref{app:thm_exists_mixed_policy}.
\begin{corollary}[stochastic policies $\subseteq$ mixed policies]\label{app:cor_subset_mixed_policies}
    Let $\mu^{\pi,\theta} \in \dist{\Paths^{M}}$ be the probability distribution over paths in the RPOMDP resulting from executing agent policy $\pi$ and nature policy $\theta$.
    Then we have the following:
    \[
        \{\mu^{\pi,\theta}\mid \pi\in \Pi,\theta \in \Theta\} \subseteq \{\mu^{\pi^\mix,\theta^\mix}\mid \pi^\mix\in \Pi^\mix,\theta^\mix \in \Theta^\mix\}.
    \]
\end{corollary}

We also need to show that for every mixed policy, there exists a stochastic policy that behaves equivalently.
This means that there are no new behaviors, and consequently no new values, introduced by looking at the set of mixed policies.
\begin{restatable}[Existence of equivalent stochastic policy]{theorem}{existsstochasticPolicy}\label{app:thm_exists_stochastic_policy}
    Let $\mu^{\pi,\theta} \in \dist{\Paths^{M}}$ be the probability distribution over paths in the RPOMDP resulting from executing agent policy $\pi$ and nature policy $\theta$. Then:
    \[
        \forall \pi^\mix \in \Pi^\mix, \exists \pi \in \Pi, \forall \theta\in \Theta \cup \Theta^\mix. \mu^{\pi^\mix,\theta} = \mu^{\pi,\theta},
    \]\[
        \forall \theta^\mix \in \Theta^\mix, \exists \theta \in \Theta, \forall \pi\in \Pi \cup \Pi^\mix.\, \mu^{\pi,\theta^\mix} = \mu^{\pi,\theta}.
    \]
\end{restatable}

\Cref{app:thm_exists_stochastic_policy} comes with the following corollary.
\begin{corollary}[stochastic policies $\supseteq$ mixed policies]\label{app:cor_subset_stochastic_policies}
    Let $\mu^{\pi,\theta} \in \dist{\Paths^{M}}$ be the probability distribution over paths in the RPOMDP resulting from executing agent policy $\pi$ and nature policy $\theta$.
    Then we have the following:
    \[
        \{\mu^{\pi,\theta}\mid \pi\in \Pi,\theta \in \Theta\} \supseteq \{\mu^{\pi^\mix,\theta^\mix}\mid \pi^\mix\in \Pi^\mix,\theta^\mix \in \Theta^\mix\}.
    \]
\end{corollary}

By combining \Cref{app:cor_subset_mixed_policies,app:cor_subset_stochastic_policies}, it follows that the sets of mixed policies give exactly the same sets of distributions over paths in our original RPOMDP as the sets of stochastic policies do.
\begin{corollary}[Equivalent set of mixed policies]\label{app:cor_equiv_policies}
    Let $\mu^{\pi,\theta} \in \dist{\Paths^{M}}$ be the probability distribution over paths in the RPOMDP resulting from executing agent policy $\pi$ and nature policy $\theta$.
    Then we have the following:
    \[
        \{\mu^{\pi,\theta}\mid \pi\in \Pi,\theta \in \Theta\} = \{\mu^{\pi^\mix,\theta^\mix}\mid \pi^\mix\in \Pi^\mix,\theta^\mix \in \Theta^\mix\}.
    \]
    It follows that the sets of possible values, \ie, the subset of $\mathbb{R}$ the value function can attain under all stochastic policies and all mixed policies, is the same:
    \[
        \{V^{\pi,\theta}\mid \pi \in \Pi, \theta \in \Theta\} = \{V^{\pi^\mix,\theta^\mix}\mid \pi^\mix \in \Pi^\mix, \theta^\mix \in \Theta^\mix\}.
    \]
\end{corollary}

Below, we first define some definitions and lemmas and then prove \Cref{app:thm_exists_mixed_policy,app:thm_exists_stochastic_policy}.
We focus on the nature policies, as these require dealing with an infinite action space and, therefore, with an infinite set of deterministic policies.
The proofs for the agent policies follow the same steps.

\subsection*{Additional definitions and lemmas}
We begin by defining the set of relevant histories given a policy.
The actions chosen for histories outside the relevant history set do not influence the results of a game since the policy never reaches them.

Using the relevant histories, we can define the set of relevant deterministic policies given a history.
\begin{definition}[Relevant deterministic policies]
    Given a history $h^\nature\in H^\nature$, we define the set of relevant deterministic policies $\Theta^{\detr,h}$, containing all deterministic policies that could have generated the current history.
    \[\Theta^{\detr, h^\nature} = \{\theta^\detr \in \Theta^\detr\mid h^\nature \in \rel^\nature(\theta^\detr)\}.\]
\end{definition}

We define a helper function $\eta^{\pi,\theta}\colon \dist{\Paths^M}$ to compute the probability over paths for all stochastic and deterministic policies $\pi \in \Pi$ and $\theta \in \Theta$:
\begin{align*}
    \eta^{\pi,\theta}(\tup{s_I}) &= 1,\\
    \eta^{\pi,\theta}(\mypath'\concat \tup{s,a,u,s'}) &= \eta^{\pi,\theta}(\mypath'\concat\tup{s}) \cdot \pi(O^{\agent,M}(\mypath'\concat\tup{s}))(a) \cdot \theta(O^{\nature,M}(\mypath'\concat\tup{s}),a)(u) \cdot \bm{T}(u)(s,a,s').
\end{align*}

We now prove the following lemmas about the probability distribution over paths for deterministic policies.

Given a path, if a deterministic policy is not relevant for generating the history of that path, then the probability of reaching that path with the deterministic policy is zero.
\begin{lemma}[Zero probability of non-relevant paths]\label{app:lem_eta_zero}
    Given a path $\mypath$ and policies $\pi\in \Pi$, $\theta^\detr\in \Theta^\detr$, we have that:
    \[
    \theta^\detr \notin \Theta^{\detr, O^{\nature, M}(\mypath)} \implies \eta^{\pi,\theta^\detr}(\mypath) = 0.
    \]
\end{lemma}
\begin{proof}
Take arbitrary path $\mypath$ and policies $\pi\in \Pi$, $\theta^\detr\in \Theta^\detr$.
\[
\theta^\detr \notin \Theta^{\detr, O^{\nature,M}(\mypath)} \iff O^{\nature,M}(\mypath) \notin \rel^\nature(\theta^\detr).
\]
By definition, we know that $O^{\nature,M}(\tup{s_I}) \in \rel^\nature(\theta^\detr)$.
Furthermore, we know that:
\[
O^{\nature,M}(\mypath'\concat\tup{a',u',s}) \notin \rel^\nature(\theta^\detr) \iff \theta^\detr(O^{\nature,M}(\mypath'),a') \neq u' \lor O^{\nature,M}(\mypath') \notin \rel^\nature(\theta^\detr).
\]
Since $O^{\nature,M}(\tup{s_I}) \in \rel^\nature(\theta^\detr)$, we will eventually reach a prefix of $\mypath$, for which the condition is violated. So then we get:
\begin{align*}
    O^{\nature,M}(\mypath) \notin \rel^\nature(\theta^\detr) &\iff \exists \mypath''.\, \mypath'' \concat \tup{a'',u'',s',\dots} = \mypath \land \theta^\detr(O^{\nature,M}(\mypath''),a'') \neq u''\\
    &\iff \exists \mypath''.\, \mypath'' \concat \tup{a'',u'',s',\dots} = \mypath \land \theta^\detr(O^{\nature,M}(\mypath''),a'')(u'') = 0\\
    &\:\implies \exists \mypath''.\, \mypath'' \concat \tup{a'',u'',s',\dots} = \mypath \land \eta^{\pi,\theta^\detr}(\mypath''\concat \tup{a'',u'',s'}) = 0\\
    &\:\implies \eta^{\pi,\theta^\detr}(\mypath) = 0.
\end{align*}
Hence:
\[
    \theta^\detr \notin \Theta^{\detr, O^{\nature,M}(\mypath)} \implies \eta^{\pi,\theta^\detr}(\mypath) = 0.
\]
\end{proof}

The next lemma states that, given a path, if two deterministic policies are both relevant for generating the history of that path, then the probability of reaching that path is the same for both policies.
\begin{lemma}[Constant probability of paths for relevant deterministic policies]\label{app:lem_eta_const}
    Given a path $\mypath$ and policy $\pi\in \Pi$, we have that:
    \[
    \forall \theta^\detr, \theta^{\detr'}\in \Theta^{\detr, O^{\nature,M}(\mypath)}.\, \eta^{\pi,\theta^\detr}(\mypath) = \eta^{\pi,\theta^{\detr'}}(\mypath).
    \]
\end{lemma}
\begin{proof}
    Take arbitrary path $\mypath$, policy $\pi\in \Pi$ and $\theta^\detr, \theta^{\detr'}\in \Theta^{\detr, O^{\nature,M}(\mypath)}$.
    We show $\eta^{\pi,\theta^\detr}(\mypath) = \eta^{\pi,\theta^{\detr'}}(\mypath)$ by induction on the length of $\mypath$.

    Assume $|\mypath| = 0$.
    Then $\mypath = \tup{s_I}$.
    Then:
    \[
    \eta^{\pi,\theta^\detr}(\tup{s_I}) = 1 = \eta^{\pi,\theta^{\detr'}}(\tup{s_I}).
    \]
    So for paths $\mypath$ of length $0$, we know that $\eta^{\pi,\theta^\detr}(\mypath) = \eta^{\pi,\theta^{\detr'}}(\mypath)$.

    Now assume we know, given $q\in \NN, q \geq 1$, that:
    \[
    \forall \mypath \in \Paths^M. |\mypath| = q-1 \implies \eta^{\pi,\theta^\detr}(\mypath) = \eta^{\pi,\theta^{\detr'}}(\mypath).
    \]
    Take arbitrary $\mypath \in \Paths^M$ with horizon length $|\mypath| = q$.
    Then we have:
    \[
    \mypath = \mypath_{0:q-1}\concat \tup{a_{q-1},u_{q-1},s_q} = \mypath_{0:q-2}\concat \tup{a_{q-2},u_{q-2},s_{q-1},a_{q-1},u_{q-1},s_q}.
    \]
    Then $\mypath_{0:q-1} \in \Paths^M$ and $|\mypath_{0:q-1}| = q-1$.
    By assumption, we get that:
    \[
    \eta^{\pi,\theta^\detr}(\mypath_{0:q-1}) = \eta^{\pi,\theta^{\detr'}}(\mypath_{0:q-1}).
    \]
    Additionally, we know that:
    \begin{align*}
        \theta^\detr \in \Theta^{\detr, O^{\nature,M}(\mypath)} &\iff O^{\nature,M}(\mypath) \in \rel^\nature(\theta^\detr)\\
        &\iff \theta^\detr(O^{\nature,M}(\mypath_{0:q-1}), a_{q-1}) = u_{q-1} \land O^{\nature,M}(\mypath_{0:q-1}) \in \rel^\nature(\theta^\detr)\\
        &\iff \theta^\detr(O^{\nature,M}(\mypath_{0:q-1}), a_{q-1})(u_{q-1})= 1 \land O^{\nature,M}(\mypath_{0:q-1}) \in \rel^\nature(\theta^\detr).
    \end{align*}
    So we have that:
    \[
    \theta^\detr(O^{\nature,M}(\mypath_{0:q-1}), a_{q-1})(u_{q-1}) = 1 = \theta^{\detr'}(O^{\nature,M}(\mypath_{0:q-1}), a_{q-1})(u_{q-1}).
    \]
    Finally, we get:
    \begin{align*}
        \eta^{\pi,\theta^\detr}(\mypath) &= \eta^{\pi,\theta^\detr}(\mypath_{0:q-1}) \cdot \pi(O^{\agent,M}(\mypath_{0:q-1}))(a_{q-1}) \cdot \theta^\detr(O^{\nature,M}(\mypath_{0:q-1}),a_{q-1})(u_{q-1}) \cdot \bm{T}(u_{q-1})(s_{q-1},a_{q-1},s_q)\\
        &= \eta^{\pi,\theta^{\detr'}}(\mypath_{0:q-1}) \cdot \pi(O^{\agent,M}(\mypath_{0:q-1}))(a_{q-1}) \cdot \theta^{\detr'}(O^{\nature,M}(\mypath_{0:q-1}),a_{q-1})(u_{q-1}) \cdot \bm{T}(u_{q-1})(s_{q-1},a_{q-1},s_q)\\
        &= \eta^{\pi,\theta^{\detr'}}(\mypath).
    \end{align*}
    So, if $\eta^{\pi,\theta^\detr}(\mypath) = \eta^{\pi,\theta^{\detr'}}(\mypath)$ holds for paths of arbitrary length $q-1 \in \NN$, then $\eta^{\pi,\theta^\detr}(\mypath) = \eta^{\pi,\theta^{\detr'}}(\mypath)$ holds for paths of length $q$.
    Hence, by induction, $\eta^{\pi,\theta^\detr}(\mypath) = \eta^{\pi,\theta^{\detr'}}(\mypath)$.

    As $\theta^\detr, \theta^{\detr'}\in \Theta^{\detr, O^{\nature,M}(\mypath)}$ were arbitrarily chosen, we conclude that:
    \[
    \forall \theta^\detr, \theta^{\detr'}\in \Theta^{\detr, O^{\nature,M}(\mypath)}.\, \eta^{\pi,\theta^\detr}(\mypath) = \eta^{\pi,\theta^{\detr'}}(\mypath).
    \]
\end{proof}

Using our helper function $\eta$, we can define the four ways of computing the probability distribution over paths depending on the type of policies involved as follows:
\begin{align*}
    \intertext{(1)~$\pi \in \Pi$ and $\theta \in \Theta$:}
    \mu^{\pi,\theta}(\mypath) &= \eta^{\pi,\theta}(\mypath).
    \intertext{(2)~$\pi \in \Pi$ and $\theta^\mix \in \Theta^\mix$:}
    \mu^{\pi,\theta^\mix}(\mypath) &= \sum_{\theta^\detr\in \Theta^\detr} \theta^\mix(\theta^\detr) \cdot \eta^{\pi,\theta^\detr}(\mypath).
    \intertext{(3)~$\pi^\mix \in \Pi^\mix$ and $\theta \in \Theta$:}
    \mu^{\pi^\mix,\theta}(\mypath) &= \sum_{\pi^\detr\in \Pi^\detr} \pi^\mix(\pi^\detr) \cdot \eta^{\pi^\detr,\theta}(\mypath).
    \intertext{(4)~$\pi^\mix \in \Pi^\mix$ and $\theta^\mix \in \Theta^\mix$:}
    \mu^{\pi^\mix,\theta^\mix}(\mypath) &= \sum_{\pi^\detr\in \Pi^\detr} \pi^\mix(\pi^\detr) \cdot \sum_{\theta^\detr\in \Theta^\detr} \theta^\mix(\theta^\detr) \cdot \eta^{\pi^\detr,\theta^\detr}(\mypath),\\
    &= \sum_{\pi^\detr\in \Pi^\detr} \sum_{\theta^\detr\in \Theta^\detr} \pi^\mix(\pi^\detr) \cdot  \theta^\mix(\theta^\detr) \cdot \eta^{\pi^\detr,\theta^\detr}(\mypath).
\end{align*}
Recall that a deterministic policy can be interpreted as both a stochastic and a mixed policy using only Dirac distributions.
The probabilities over paths generated by deterministic policies can, therefore, be computed using any of the above formulas.

\subsection*{Proof of \Cref{app:thm_exists_mixed_policy}}
The standard way to define a mixed strategy given a stochastic strategy is to simply assign to each deterministic policy the product of the probabilities the stochastic policy assigns to the same choices \cite{kuhn1953extensive}.
The problem in our case, however, is that, due to the infinite number of nature policies, this leads to infinitely many deterministic policies having a non-zero probability in the resulting mixed policy.

To create a mixed policy with a finite number of deterministic policies with a non-zero probability, we define an equivalence class for deterministic policies that assign the same action for all histories that are relevant to the given stochastic policy.
\begin{definition}
    Given a stochastic policy $\theta\in \Theta$, we define the following equivalence relation $\hsim$ between deterministic policies, which we call $\rel^\nature(\theta)$-equivalent:
    \[
        \forall \theta^\detr, \theta^{\detr'} \in \Theta^\detr.\, \theta^\detr \hsim \theta^{\detr'} \iff \forall h^\nature\in \rel^\nature(\theta), \forall a\in A.\, \theta^\detr(h^\nature,a) = \theta^{\detr'}(h^\nature,a).
    \]
    The reflexivity, symmetry, and transitivity of the $\rel^\nature(\theta)$-equivalence relation follow from the reflexivity, symmetry, and transitivity of the equality relation. 
\end{definition}

The $\rel^\nature(\theta)$-equivalence relation provides us with $\rel^\nature(\theta)$-equivalence classes $[\theta^\detr]_\hsim$, which partition the set of deterministic policies.
We select one member of each $\rel^\nature(\theta)$-equivalence class to define a new set $\theta^{\detr,\hsim}$ called the $\rel^\nature(\theta)$-representation set.
Note that this set is not the same as the quotient set $\Theta^\detr/\hsim$, as the quotient set is a set of sets of deterministic policies, whereas the $\rel^\nature(\theta)$-representation set is a set of deterministic policies.
Clearly, $\Theta^\detr/\hsim \subseteq \Theta^\detr$.

Using the $\rel^\nature(\theta)$-representation set, we define a function $g\colon \Theta \to \Theta^\mix$ with which we will construct our equivalent mixed policy.
This function uses a similar construction as in \cite{kuhn1953extensive} for the deterministic policies in the $\rel^\nature(\theta)$-representation set but gives the rest of the deterministic policies a zero probability automatically.
\[
    g(\theta)(\theta^\detr) = \begin{cases}
        \prod_{h^\nature\in \rel^\nature(\theta), a\in A} \theta(h,a)(\theta^\detr(h^\nature,a)) & \text{ if } \theta^\detr \in \Theta^{\detr,\hsim},\\
        0 & \text{ otherwise.}
    \end{cases}
\]
We first show that $g$ correctly maps to a mixed policy (\Cref{app:lem_g_in_mixed}) and then that this resulting policy results in the same distribution over paths in the RPOMDP given any agent policy (\Cref{app:lem_g_equiv}).

\begin{lemma}\label{app:lem_g_in_mixed}
    $g(\theta)$ is a mixed policy:
    \[
        \forall \theta \in \Theta.\, g(\theta) \in \Theta^\mix.
    \]
\end{lemma}
\begin{proof}
    Take arbitrary $\theta \in \Theta$.
    By construction, we have that $g(\theta) \colon \Theta^\detr \to [0,1]$.
    Now to show that $g(\theta) \in \Theta^\mix$, we must show two things: $g(\theta)$ assigns a non-zero probability to a finite number of deterministic policies (finitely randomizing) and $g(\theta)$ is a probability distribution, meaning the probabilities sum up to $1$.
    \begin{align}
        g(\theta) \text{ is finitely randomizing}, \label{eq1_lem_g_in_mixed}\\
        \sum_{\theta^\detr \in \Theta^\detr} g(\theta)(\theta^\detr) = 1.\label{eq2_lem_g_in_mixed}
    \end{align}
    When proving \Cref{eq1_lem_g_in_mixed}, we can restrict ourselves to $\theta^\detr\in \Theta^{\detr,\hsim}$, as we assign a zero probability to all other deterministic policies.
    There can be infinitely many $\rel^\nature(\theta)$-equivalence classes and, therefore, infinitely many elements of $\Theta^{\detr,\hsim}$. 
    However, we know that there is a finite number of nature histories in $\rel^\nature(\theta)$ since $\theta$ is finitely randomizing, $\Znature, \Zpub$, and $A$ are finite, and we consider a finite horizon.
    Furthermore, we know that because $\theta$ is finitely randomizing, there are only finitely many $u$ that can be chosen by the deterministic policies at each history action pair $h^\nature,a$ for which $\theta(h^\nature,a)(u)$ gives a non-zero probability.
    Due to the $\rel^\nature(\theta)$-equivalence classes, we only have one deterministic policy per unique choice combination for all relevant histories.
    Combining this with the fact that there are a finite number of choices that give a non-zero probability, we can conclude that there is a finite number of deterministic policies that give a non-zero probability.

    \Cref{eq2_lem_g_in_mixed} follows from the fact that, by construction, there is exactly one deterministic policy with a non-zero probability for each choice combination of choices available in the stochastic policies over all relevant nature histories of that stochastic policy.
    Each of these deterministic policies is assigned the product of the probabilities assigned to the same choices by the stochastic policy for the relevant nature histories.
    Summing over the deterministic policies will hence equal summing over the product of the probabilities assigned to the choices by the stochastic policy for the relevant nature histories.
    Since the stochastic policy assigns a probability distribution over its choices for each relevant nature history, we also get that:
    \[
    \sum_{\theta^\detr \in \Theta^\detr} g(\theta)(\theta^\detr) = 1.
    \]
\end{proof}

The following lemma is the last but vital ingredient towards the proof of \Cref{app:thm_exists_mixed_policy}.
\begin{lemma}\label{app:lem_g_equiv}
    $g(\theta)$ equivalent to $\theta$:
    \[
        \forall \theta \in \Theta, \forall \pi\in \Pi \cup \Pi^\mix.\, \mu^{\pi,\theta} = \mu^{\pi,g(\theta)}.
    \]
\end{lemma}
\begin{proof}
    Take arbitrary $\theta \in \Theta$.
    We show $\forall \pi \in \Pi \cup \Pi^\mix, \forall \mypath \in \Paths^M.\, \mu^{\pi,\theta}(\mypath) = \mu^{\pi,g(\theta)}(\mypath)$ by induction on the length of the path $\mypath$.
    We write the length of $\mypath$ as $|\mypath|$.

    \noindent Assume $|\mypath| = 0$.
    Then $\mypath = \tup{s_I}$.
    Then we have for $\pi \in \Pi$:
    \begin{align*}
        \mu^{\pi,\theta}(\tup{s_I}) &= \eta^{\pi,\theta}(\tup{s_I})\\
        &= 1\\
        &= \sum_{\theta^\detr\in \Theta^\detr} g(\theta)(\theta^\detr) \cdot 1\\
        &= \sum_{\theta^\detr\in \Theta^\detr} g(\theta)(\theta^\detr) \cdot \eta^{\pi,\theta^\detr}(\tup{s_I})\\
        &= \mu^{\pi,g(\theta)}(\tup{s_I}).
    \intertext{And for $\pi^\mix \in \Pi^\mix$:}
        \mu^{\pi^\mix,\theta}(\tup{s_I}) &= \sum_{\pi^\detr\in \Pi^\detr} \pi^\mix(\pi^\detr) \cdot \eta^{\pi,\theta}(\tup{s_I})\\
        &= \sum_{\pi^\detr\in \Pi^\detr} \pi^\mix(\pi^\detr) \cdot 1\\
        &= \sum_{\pi^\detr\in \Pi^\detr} \pi^\mix(\pi^\detr) \cdot \sum_{\theta^\detr\in \Theta^\detr} g(\theta)(\theta^\detr) \cdot 1\\
        &= \sum_{\pi^\detr\in \Pi^\detr} \pi^\mix(\pi^\detr) \cdot \sum_{\theta^\detr\in \Theta^\detr} g(\theta)(\theta^\detr) \cdot \eta^{\pi,\theta^\detr}(\tup{s_I})\\
        &= \mu^{\pi^\mix,g(\theta)}(\tup{s_I}).
    \end{align*}
    So for paths $\mypath$ of length 0, we know that $\forall \pi \in \Pi \cup \Pi^\mix.\, \mu^{\pi,\theta}(\mypath) = \mu^{\pi,g(\theta)}(\mypath)$.

    Now assume we know, given $q\in \NN, q \geq 1$, that:
    \[\forall \mypath \in \Paths^M. |\mypath| = q-1 \implies \forall \pi \in \Pi \cup \Pi^\mix.\, \mu^{\pi,\theta}(\mypath) = \mu^{\pi,g(\theta)}(\mypath).\]
    Take arbitrary $\mypath \in \Paths^M$ with horizon length $|\mypath| = q$.
    Then we have:
    \[\mypath = \mypath_{0:q-1}\concat \tup{a_{q-1},u_{q-1},s_q} = \mypath_{0:q-2}\concat \tup{a_{q-2},u_{q-2},s_{q-1},a_{q-1},u_{q-1},s_q}.\]
    Then $\mypath_{0:q-1} \in \Paths^M$ and $|\mypath_{0:q-1}| = q-1$.
    By assumption, we get that:
    \[\forall \pi \in \Pi \cup \Pi^\mix.\, \mu^{\pi,\theta}(\mypath_{0:q-1}) = \mu^{\pi,g(\theta)}(\mypath_{0:q-1}).\]

    We need to distinguish two cases for the proof: $\pi \in \Pi$ and $\pi \in \Pi^\mix$.
    We write out the more complicated case of mixed agent policies $\pi \in \Pi^\mix$.
    The proof for stochastic agent policies follows along the same lines.
    We highlight subtle changes in the equations using either \changeB{blue} or \changeR{red} text.

    {\small
    \begin{align*}
        \mu^{\pi^\mix,\theta}(\mypath) &= \sum_{\pi^\detr\in \Pi^\detr} \pi^\mix(\pi^\detr) \cdot \eta^{\pi^\detr,\theta}(\mypath).
    \intertext{Unfolding $\eta^{\pi^\detr,\theta}(\mypath)$:}
        &= \sum_{\pi^\detr\in \Pi^\detr} \pi^\mix(\pi^\detr) \cdot \eta^{\pi^\detr,\theta}(\mypath_{0:q-1}) \cdot \pi^\detr(O^{\agent,M}(\mypath_{0:q-1}))(a_{q-1}) \cdot \theta(O^{\nature,M}(\mypath_{0:q-1}),a_{q-1})(u_{q-1}) \cdot \strut\\
        &\qquad \bm{T}(u_{q-1})(s_{q-1},a_{q-1},s_q).
    \intertext{Reordering:}
        &= \bm{T}(u_{q-1})(s_{q-1},a_{q-1},s_q) \cdot \sum_{\pi^\detr\in \Pi^\detr} \pi^\mix(\pi^\detr) \cdot \pi^\detr(O^{\agent,M}(\mypath_{0:q-1}))(a_{q-1}) \cdot \changeB{\eta}^{\pi^\detr,\theta}(\mypath_{0:q-1})\cdot \strut\\
        & \qquad \theta(O^{\nature,M}(\mypath_{0:q-1}),a_{q-1})(u_{q-1}).
    \intertext{Using the definition of $\mu$ for deterministic or stochastic agent and nature policies:}
        &= \bm{T}(u_{q-1})(s_{q-1},a_{q-1},s_q) \cdot \sum_{\pi^\detr\in \Pi^\detr} \pi^\mix(\pi^\detr) \cdot \pi^\detr(O^{\agent,M}(\mypath_{0:q-1}))(a_{q-1}) \cdot \changeB{\mu}^{\pi^\detr,\changeR{\theta}}(\mypath_{0:q-1})\cdot \strut\\
        & \qquad \theta(O^{\nature,M}(\mypath_{0:q-1}),a_{q-1})(u_{q-1}).
    \intertext{Using our assumption $\forall \pi\in \Pi\cup \Pi^\mix.\;\mu^{\pi,\theta}(\mypath_{0:q-1}) = \mu^{\pi,g(\theta)}(\mypath_{0:q-1})$, we get:}
        &= \bm{T}(u_{q-1})(s_{q-1},a_{q-1},s_q) \cdot \sum_{\pi^\detr\in \Pi^\detr} \pi^\mix(\pi^\detr) \cdot \pi^\detr(O^{\agent,M}(\mypath_{0:q-1}))(a_{q-1}) \cdot \mu^{\pi^\detr,\changeR{g(\theta)}}(\mypath_{0:q-1})\cdot \strut\\
        & \qquad \theta(O^{\nature,M}(\mypath_{0:q-1}),a_{q-1})(u_{q-1}).
    \intertext{Unfolding the definition of $\mu$ for deterministic or stochastic agent policies and mixed nature policies:}
        &= \bm{T}(u_{q-1})(s_{q-1},a_{q-1},s_q) \cdot \sum_{\pi^\detr\in \Pi^\detr} \pi^\mix(\pi^\detr) \cdot \pi^\detr(O^{\agent,M}(\mypath_{0:q-1}))(a_{q-1}) \cdot \sum_{\theta^\detr\in\Theta^\detr} g(\theta)(\theta^\detr) \cdot \eta^{\pi^\detr,\theta^\detr}(\mypath_{0:q-1})\cdot \strut\\
        & \qquad \theta(O^{\nature,M}(\mypath_{0:q-1}),a_{q-1})(u_{q-1}).
    \intertext{Multiplying by a term equal to 1:}
        &= \bm{T}(u_{q-1})(s_{q-1},a_{q-1},s_q) \cdot \sum_{\pi^\detr\in \Pi^\detr} \pi^\mix(\pi^\detr) \cdot \pi^\detr(O^{\agent,M}(\mypath_{0:q-1}))(a_{q-1}) \cdot \sum_{\theta^\detr\in\Theta^\detr} g(\theta)(\theta^\detr) \cdot \eta^{\pi^\detr,\theta^\detr}(\mypath_{0:q-1})\cdot \strut\\
        & \qquad \theta(O^{\nature,M}(\mypath_{0:q-1}),a_{q-1})(u_{q-1}) \cdot \dfrac{\theta(O^{\nature,M}(\mypath_{0:q-1}),a_{q-1})(u_{q-1})}{\theta(O^{\nature,M}(\mypath_{0:q-1}),a_{q-1})(u_{q-1})}.
    \intertext{Using that $g(\theta)$ is a probability distribution over the set of deterministic policies, we get:}
        &= \bm{T}(u_{q-1})(s_{q-1},a_{q-1},s_q) \cdot \sum_{\pi^\detr\in \Pi^\detr} \pi^\mix(\pi^\detr) \cdot \pi^\detr(O^{\agent,M}(\mypath_{0:q-1}))(a_{q-1}) \cdot \sum_{\theta^\detr\in\Theta^\detr} g(\theta)(\theta^\detr) \cdot \eta^{\pi^\detr,\theta^\detr}(\mypath_{0:q-1})\cdot \strut\\
        & \qquad \theta(O^{\nature,M}(\mypath_{0:q-1}),a_{q-1})(u_{q-1}) \cdot \dfrac{\theta(O^{\nature,M}(\mypath_{0:q-1}),a_{q-1})(u_{q-1})}{\sum_{\theta^\detr\in\Theta^\detr} g(\theta)(\theta^\detr)\cdot \theta(O^{\nature,M}(\mypath_{0:q-1}),a_{q-1})(u_{q-1})}.
    \intertext{Using that $\Theta^{\detr,\hsim}$ contains exactly one deterministic policy for each choice combination of available choices for $\theta$ for each history relevant for $\theta$, we get:}
        &= \bm{T}(u_{q-1})(s_{q-1},a_{q-1},s_q) \cdot \sum_{\pi^\detr\in \Pi^\detr} \pi^\mix(\pi^\detr) \cdot \pi^\detr(O^{\agent,M}(\mypath_{0:q-1}))(a_{q-1}) \cdot \sum_{\theta^\detr\in\Theta^\detr} g(\theta)(\theta^\detr) \cdot \eta^{\pi^\detr,\theta^\detr}(\mypath_{0:q-1})\cdot \strut\\
        & \qquad \theta(O^{\nature,M}(\mypath_{0:q-1}),a_{q-1})(u_{q-1}) \cdot \dfrac{\sum_{\theta^\detr\in\Theta^{\detr,\hsim}} \prod_{h^\nature\in\rel^\nature(\theta),a\in A} \theta(h^\nature,a)(\theta^\detr(h^\nature,a)) \cdot \theta^\detr(O^{\nature,M}(\mypath_{0:q-1}),a_{q-1})(u_{q-1})}{\sum_{\theta^\detr\in\Theta^\detr} g(\theta)(\theta^\detr)\cdot \theta(O^{\nature,M}(\mypath_{0:q-1}),a_{q-1})(u_{q-1})}.
    \intertext{Using the definition of $g$:}
        &= \bm{T}(u_{q-1})(s_{q-1},a_{q-1},s_q) \cdot \sum_{\pi^\detr\in \Pi^\detr} \pi^\mix(\pi^\detr) \cdot \pi^\detr(O^{\agent,M}(\mypath_{0:q-1}))(a_{q-1}) \cdot \sum_{\theta^\detr\in\Theta^\detr} g(\theta)(\theta^\detr) \cdot \eta^{\pi^\detr,\theta^\detr}(\mypath_{0:q-1})\cdot \strut\\
        & \qquad \theta(O^{\nature,M}(\mypath_{0:q-1}),a_{q-1})(u_{q-1}) \cdot \dfrac{\sum_{\theta^\detr\in\Theta^{\detr,\hsim}} g(\theta)(\theta^\detr) \cdot \theta^\detr(O^{\nature,M}(\mypath_{0:q-1}),a_{q-1})(u_{q-1})}{\sum_{\theta^\detr\in\Theta^\detr} g(\theta)(\theta^\detr)\cdot \theta(O^{\nature,M}(\mypath_{0:q-1}),a_{q-1})(u_{q-1})}.
    \intertext{Let $\theta^{\detr'}$ be an arbitrary deterministic policy in the set of relevant deterministic policies $\Theta^{\detr,O^{\nature,M}(\mypath_{0:q-1})}$, we get:}
        &= \bm{T}(u_{q-1})(s_{q-1},a_{q-1},s_q) \cdot \sum_{\pi^\detr\in \Pi^\detr} \pi^\mix(\pi^\detr) \cdot \pi^\detr(O^{\agent,M}(\mypath_{0:q-1}))(a_{q-1}) \cdot \sum_{\theta^\detr\in\Theta^\detr} g(\theta)(\theta^\detr) \cdot \eta^{\pi^\detr,\theta^\detr}(\mypath_{0:q-1})\cdot \strut\\
        & \qquad \theta(O^{\nature,M}(\mypath_{0:q-1}),a_{q-1})(u_{q-1}) \cdot \dfrac{\eta^{\pi^\detr,\theta^{\detr'}}(\mypath_{0:q-1}) \cdot \sum_{\theta^\detr\in\Theta^{\detr,\hsim}} g(\theta)(\theta^\detr) \cdot \theta^\detr(O^{\nature,M}(\mypath_{0:q-1}),a_{q-1})(u_{q-1})}{\eta^{\pi^\detr,\theta^{\detr'}}(\mypath_{0:q-1}) \cdot \sum_{\theta^\detr\in\Theta^\detr} g(\theta)(\theta^\detr)\cdot \theta(O^{\nature,M}(\mypath_{0:q-1}),a_{q-1})(u_{q-1})}.
    \intertext{Using \Cref{app:lem_eta_zero,app:lem_eta_const} and the fact that $g(\theta)(\theta^\detr) = 0$ when $\theta^\detr \notin \Theta^{\detr,\hsim}$, we get that:}
        &= \bm{T}(u_{q-1})(s_{q-1},a_{q-1},s_q) \cdot \sum_{\pi^\detr\in \Pi^\detr} \pi^\mix(\pi^\detr) \cdot \pi^\detr(O^{\agent,M}(\mypath_{0:q-1}))(a_{q-1}) \cdot \sum_{\theta^\detr\in\Theta^\detr} g(\theta)(\theta^\detr) \cdot \eta^{\pi^\detr,\theta^\detr}(\mypath_{0:q-1})\cdot \strut\\
        & \qquad \theta(O^{\nature,M}(\mypath_{0:q-1}),a_{q-1})(u_{q-1}) \cdot \dfrac{\sum_{\theta^\detr\in\Theta^{\detr,\hsim}} g(\theta)(\theta^\detr) \cdot \eta^{\pi^\detr,\theta^{\detr}}(\mypath_{0:q-1}) \cdot \theta^\detr(O^{\nature,M}(\mypath_{0:q-1}),a_{q-1})(u_{q-1})}{\sum_{\theta^\detr\in\Theta^\detr} g(\theta)(\theta^\detr) \cdot \eta^{\pi^\detr,\theta^{\detr}}(\mypath_{0:q-1}) \cdot \theta(O^{\nature,M}(\mypath_{0:q-1}),a_{q-1})(u_{q-1})}.
    \intertext{The denominator of the fraction is now equal to a term it is multiplied by:}
        &= \bm{T}(u_{q-1})(s_{q-1},a_{q-1},s_q) \cdot \sum_{\pi^\detr\in \Pi^\detr} \pi^\mix(\pi^\detr) \cdot \pi^\detr(O^{\agent,M}(\mypath_{0:q-1}))(a_{q-1})\cdot \strut\\
        & \qquad \sum_{\theta^\detr\in\Theta^{\detr,\hsim}} g(\theta)(\theta^\detr) \cdot \eta^{\pi^\detr,\theta^{\detr}}(\mypath_{0:q-1}) \cdot \theta^\detr(O^{\nature,M}(\mypath_{0:q-1}),a_{q-1})(u_{q-1}).
    \intertext{Using that $\forall \theta^\detr\in \Theta^\detr\backslash\Theta^{\detr,\hsim}.\, g(\theta)(\theta^\detr) = 0$ by definition, we get that:}
        &= \bm{T}(u_{q-1})(s_{q-1},a_{q-1},s_q) \cdot \sum_{\pi^\detr\in \Pi^\detr} \pi^\mix(\pi^\detr) \cdot \pi^\detr(O^{\agent,M}(\mypath_{0:q-1}))(a_{q-1}) \cdot \sum_{\theta^\detr\in \Theta^\detr} g(\theta)(\theta^\detr) \cdot \eta^{\pi^\detr,\theta^\detr}(\mypath_{0:q-1}) \cdot \strut\\
        &\qquad \theta^\detr(O^{\nature,M}(\mypath_{0:q-1}),a_{q-1})(u_{q-1}).
    \intertext{Reordering:}
        &= \sum_{\pi^\detr\in \Pi^\detr} \pi^\mix(\pi^\detr) \cdot \sum_{\theta^\detr\in \Theta^\detr} g(\theta)(\theta^\detr) \cdot \eta^{\pi^\detr,\theta^\detr}(\mypath_{0:q-1}) \cdot \pi^\detr(O^{\agent,M}(\mypath_{0:q-1}))(a_{q-1}) \cdot \strut\\
        &\qquad \theta^\detr(O^{\nature,M}(\mypath_{0:q-1}),a_{q-1})(u_{q-1}) \cdot \bm{T}(u_{q-1})(s_{q-1},a_{q-1},s_q).
    \intertext{Folding $\eta^{\pi^\detr,\theta^\detr}(\mypath)$:}
        &= \sum_{\pi^\detr\in \Pi^\detr} \pi^\mix(\pi^\detr) \cdot \sum_{\theta^\detr\in \Theta^\detr} g(\theta)(\theta^\detr) \cdot \eta^{\pi^\detr,\theta^\detr}(\mypath).
    \intertext{Using the definition of $\mu$ for mixed agent and nature policies:}
        &= \mu^{\pi^\mix,g(\theta)}(\mypath).
    \end{align*}
    }%
    So if $\forall \pi \in \Pi \cup \Pi^\mix.\, \mu^{\pi,\theta}(\mypath) = \mu^{\pi,g(\theta)}(\mypath)$ holds for paths of arbitrary length $q-1 \in \NN$, $\forall \pi \in \Pi \cup \Pi^\mix.\, \mu^{\pi,\theta}(\mypath) = \mu^{\pi,g(\theta)}(\mypath)$ holds for paths of length $q$.
    Hence, by induction, $\forall \mypath\in \Paths^M,\forall \pi \in \Pi \cup \Pi^\mix.\, \mu^{\pi,\theta}(\mypath) = \mu^{\pi,g(\theta)}(\mypath)$.

    As $\theta \in \Theta$ was arbitrarily chosen, we conclude that:
    \[
    \forall \theta \in \Theta, \forall \pi\in \Pi \cup \Pi^\mix.\, \mu^{\pi,\theta} = \mu^{\pi,g(\theta)}.
    \]
\end{proof}

We can now prove the nature case of \Cref{app:thm_exists_mixed_policy}:
\existsMixedPolicy*
\begin{proof}
    Take arbitrary $\theta \in \Theta$.
    By \Cref{app:lem_g_in_mixed}, we know that:
    \[g(\theta) \in \Theta^\mix.\]
    Furthermore, by \Cref{app:lem_g_equiv}, we know that:
    \[\forall \pi\in \Pi \cup \Pi^\mix.\, \mu^{\pi,\theta} = \mu^{\pi,g(\theta)}.\]
    Hence, we know that:
    \[\exists \theta^\mix \in \Theta^\mix, \forall \pi\in \Pi \cup \Pi^\mix.\, \mu^{\pi,\theta} = \mu^{\pi,\theta^\mix}.\]
    As $\theta \in \Theta$ was arbitrarily chosen, we conclude that:
    \[\forall \theta \in \Theta, \exists \theta^\mix \in \Theta^\mix, \forall \pi\in \Pi \cup \Pi^\mix.\, \mu^{\pi,\theta} = \mu^{\pi,\theta^\mix}.\]
\end{proof}

\subsection*{Proof of \Cref{app:thm_exists_stochastic_policy}}
We define a function $f\colon \Theta^\mix \to \Theta$ with which we will construct our equivalent stochastic policy.
This function follows the construction used in \cite{kuhn1953extensive}.
Note that we cannot directly apply Kuhn's theorem, as our game is not finite.
If a nature history is relevant for the mixed policy, the probability for each choice in the resulting stochastic policy will only take the probability of deterministic policies that can reach that nature history into account.
Nature histories that are not relevant for the mixed policy will also not be relevant for the resulting policy, so for those histories, we can just look at the probability of all deterministic policies.
\[
    f(\theta^\mix)(h^\nature,a)(u) = \begin{cases}
        \dfrac{\sum_{\theta^\detr\in \Theta^{\detr,h^\nature}} \theta^\detr(h^\nature,a)(u) \cdot \theta^\mix(\theta^\detr)}{\sum_{\theta^\detr\in \Theta^{\detr,h^\nature}}\theta^\mix(\theta^\detr)} & \text{ if } h^\nature \in \rel^\nature(\theta^\mix),\\[5mm]
        \sum_{\theta^\detr\in \Theta^\detr} \theta^\detr(h^\nature,a)(u) \cdot \theta^\mix(\theta^\detr) & \text{ if } h^\nature \not\in \rel^\nature(\theta^\mix).
    \end{cases}
\]
where
\begin{align*}
    \theta^\detr(h^\nature,a)(u) &= \begin{cases}
        1 & \text{ if } \theta^\detr(h^\nature,a) = u,\\
        0 & \text{ otherwise}.
    \end{cases}
\end{align*}

We first show that $f$ correctly maps to a stochastic policy (\Cref{app:lem_fmix_in_stochastic}) and then that this resulting policy results in the same distribution over paths in the RPOMDP given any agent policy (\Cref{app:lem_mix_equiv_fmix}).
\begin{lemma}[$f(\theta^\mix)$ is a stochastic policy]\label{app:lem_fmix_in_stochastic}
    \[
        \forall \theta^\mix \in \Theta^\mix.\, f(\theta^\mix) \in \Theta.
    \]
\end{lemma}
\begin{proof}
    Take arbitrary $\theta^\mix \in \Theta^\mix$.
    By construction, we have that $f(\theta^\mix) \in H^\nature \times A \to \dist{\bm{U}}$.
    Now to show that $f(\theta^\mix) \in \Theta$, we must show two things: $f(\theta^\mix)$ assigns a non-zero probability to a finite number of variable assignments (finitely randomizing) and $f(\theta^\mix)$ is valid, meaning it adheres to the stickiness restrictions (see \Cref{subsec:stickiness}).
    \begin{align}
        \forall h^\nature\in H^\nature, \forall a\in A. f(\theta^\mix)(h^\nature,a) \text{ is finitely randomizing}, \label{eq1_lem_fmix_in_stochastic}\\
        \forall h^\nature\in H^\nature, \forall a\in A, \forall u \in f(\theta^\mix)(h^\nature,a). u \in \bm{U}^{\cP}(\fixed(h^\nature)).\label{eq2_lem_fmix_in_stochastic}
    \end{align}
    \Cref{eq1_lem_fmix_in_stochastic} follows from $\theta^\mix$ being finitely randomizing by definition.
    Every choice in $f(\theta^\mix)$ assigns a non-zero probability to a number of variable assignments less than or equal to the number of the deterministic policies with a non-zero probability in the mixed policy.
    As this number of deterministic policies is finite, the resulting policy $f(\theta^\mix)$ is finitely randomizing.

    \Cref{eq2_lem_fmix_in_stochastic} follows from the fact that the deterministic policies used to construct $f(\theta^\mix)$ are valid policies by definition, so we have 
    \[
        \forall \theta^\detr \in \Theta^\detr, \forall h^\nature\in H^\nature, \forall a\in A, \theta^\detr(h^\nature,a) \in \bm{U}^{\cP}(\fixed(h^\nature)).
    \]
    Combining this with the fact that 
    \[
        \forall h^\nature\in H^\nature, \forall a\in A, \forall u \in f(\theta^\mix)(h^\nature,a), \exists \theta^\detr \in \Theta^\detr.\, \theta^\mix(\theta^\detr) >0 \land \theta^\detr(h^\nature,a) = u,
    \]
    gives us the desired result:
    \[
        \forall h^\nature\in H^\nature, \forall a\in A, \forall u \in f(\theta^\mix)(h^\nature,a). u \in \bm{U}^{\cP}(\fixed(h^\nature)).
    \]
\end{proof}

\begin{lemma}\label{app:lem_mix_equiv_fmix}
    $f(\theta^\mix)$ equivalent to $\theta^\mix$:
    \[
        \forall \theta^\mix \in \Theta^\mix, \forall \pi\in \Pi \cup \Pi^\mix.\, \mu^{\pi,\theta^\mix} = \mu^{\pi,f(\theta^\mix)}.
    \]
\end{lemma}
\begin{proof}
    Take arbitrary $\theta^\mix \in \Theta^\mix$.
    We show $\forall \pi \in \Pi \cup \Pi^\mix, \forall \mypath \in \Paths^M.\, \mu^{\pi,\theta^\mix}(\mypath) = \mu^{\pi,f(\theta^\mix)}(\mypath)$ by induction on the length of the path $\mypath$.
    We write the length of $\mypath$ as $|\mypath|$.

    \noindent Assume $|\mypath| = 0$.
    Then $\mypath = \tup{s_I}$.
    Then we have for $\pi \in \Pi$:
    \begin{align*}
        \mu^{\pi,\theta^\mix}(\tup{s_I}) &= \sum_{\theta^\detr\in \Theta^\detr} \theta^\mix(\theta^\detr) \cdot \eta^{\pi,\theta^\detr}(\tup{s_I})\\
        &= \sum_{\theta^\detr\in \Theta^\detr} \theta^\mix(\theta^\detr) \cdot 1\\
        &= 1\\
        &= \eta^{\pi,f(\theta^\mix)}(\tup{s_I})\\
        &= \mu^{\pi,f(\theta^\mix)}(\tup{s_I}),
    \intertext{and for $\pi^\mix \in \Pi^\mix$:}
        \mu^{\pi^\mix,\theta^\mix}(\tup{s_I}) &= \sum_{\pi^\detr\in \Pi^\detr} \pi^\mix(\pi^\detr) \cdot \sum_{\theta^\detr\in \Theta^\detr} \theta^\mix(\theta^\detr) \cdot \eta^{\pi,\theta^\detr}(\tup{s_I})\\
        &= \sum_{\pi^\detr\in \Pi^\detr} \pi^\mix(\pi^\detr) \cdot \sum_{\theta^\detr\in \Theta^\detr} \theta^\mix(\theta^\detr) \cdot 1\\
        &= \sum_{\pi^\detr\in \Pi^\detr} \pi^\mix(\pi^\detr) \cdot 1\\
        &= \sum_{\pi^\detr\in \Pi^\detr} \pi^\mix(\pi^\detr) \cdot \eta^{\pi,f(\theta^\mix)}(\tup{s_I})\\
        &= \mu^{\pi^\mix,f(\theta^\mix)}(\tup{s_I}).
    \end{align*}
    So for paths $\mypath$ of length 0, we know that $\forall \pi \in \Pi \cup \Pi^\mix.\, \mu^{\pi,\theta^\mix}(\mypath) = \mu^{\pi,f(\theta^\mix)}(\mypath)$.

    Now assume we know, given $q\in \NN, q \geq 1$, that:
    \[\forall \mypath \in \Paths^M. |\mypath| = q-1 \implies \forall \pi \in \Pi \cup \Pi^\mix.\, \mu^{\pi,\theta^\mix}(\mypath) = \mu^{\pi,f(\theta^\mix)}(\mypath).\]
    Take arbitrary $\mypath \in \Paths^M$ with horizon length $|\mypath| = q$.
    Then we have:
    \[\mypath = \mypath_{0:q-1}\concat \tup{a_{q-1},u_{q-1},s_q} = \mypath_{0:q-2}\concat \tup{a_{q-2},u_{q-2},s_{q-1},a_{q-1},u_{q-1},s_q}.\]
    Then $\mypath_{0:q-1} \in \Paths^M$ and $|\mypath_{0:q-1}| = q-1$.
    By assumption, we get that:
    \[\forall \pi \in \Pi \cup \Pi^\mix.\, \mu^{\pi,\theta^\mix}(\mypath_{0:q-1}) = \mu^{\pi,f(\theta^\mix)}(\mypath_{0:q-1}).\]
    We also assume that $O^{\nature,M}(\mypath_{0:q-1}) \in \rel^\nature(\theta^\mix)$.
    If $O^{\nature,M}(\mypath_{0:q-1}) \notin \rel^\nature(\theta^\mix)$, so if the history of the path's prefix is not relevant for the mixed policy, the path will not be generated by the mixed policy or the stochastic policies.
    Then we trivially have:
    \[\forall \pi \in \Pi \cup \Pi^\mix.\, \mu^{\pi,\theta^\mix}(\mypath) = 0 = \mu^{\pi,f(\theta^\mix)}(\mypath).\]

    We need to distinguish two cases for the proof: $\pi \in \Pi$ and $\pi \in \Pi^\mix$.
    We write out the more complicated case: $\pi \in \Pi^\mix$.
    The other proof follows along the same lines.
    We indicate hard-to-read changes in the equations using either \changeB{blue} or \changeR{red} text.

    {\small
    \begin{align*}
        \mu^{\pi^\mix,\theta^\mix}(\mypath) &= \sum_{\pi^\detr\in \Pi^\detr} \pi^\mix(\pi^\detr) \cdot \sum_{\theta^\detr\in \Theta^\detr} \theta^\mix(\theta^\detr) \cdot \eta^{\pi^\detr,\theta^\detr}(\mypath).
    \intertext{Unfolding $\eta^{\pi^\detr,\theta^\detr}(\mypath)$:}
        &= \sum_{\pi^\detr\in \Pi^\detr} \pi^\mix(\pi^\detr) \cdot \sum_{\theta^\detr\in \Theta^\detr} \theta^\mix(\theta^\detr) \cdot \eta^{\pi^\detr,\theta^\detr}(\mypath_{0:q-1}) \cdot \pi^\detr(O^{\agent,M}(\mypath_{0:q-1}))(a_{q-1}) \cdot \strut\\
        &\qquad \theta^\detr(O^{\nature,M}(\mypath_{0:q-1}),a_{q-1})(u_{q-1}) \cdot \bm{T}(u_{q-1})(s_{q-1},a_{q-1},s_q).
    \intertext{Reordering:}
        &= \bm{T}(u_{q-1})(s_{q-1},a_{q-1},s_q) \cdot \sum_{\pi^\detr\in \Pi^\detr} \pi^\mix(\pi^\detr) \cdot \pi^\detr(O^{\agent,M}(\mypath_{0:q-1}))(a_{q-1}) \cdot \strut\\
        &\qquad \sum_{\theta^\detr\in \Theta^\detr} \theta^\mix(\theta^\detr) \cdot \eta^{\pi^\detr,\theta^\detr}(\mypath_{0:q-1}) \cdot \theta^\detr(O^{\nature,M}(\mypath_{0:q-1}),a_{q-1})(u_{q-1}).
    \intertext{Multiplying by a term equal to 1:}
        &= \bm{T}(u_{q-1})(s_{q-1},a_{q-1},s_q) \cdot \sum_{\pi^\detr\in \Pi^\detr} \pi^\mix(\pi^\detr) \cdot \pi^\detr(O^{\agent,M}(\mypath_{0:q-1}))(a_{q-1}) \cdot \strut\\
        &\qquad  \sum_{\theta^\detr\in \Theta^\detr} \theta^\mix(\theta^\detr) \cdot \eta^{\pi^\detr,\theta^\detr}(\mypath_{0:q-1}) \cdot \theta^\detr(O^{\nature,M}(\mypath_{0:q-1}),a_{q-1})(u_{q-1}) \cdot \dfrac{\sum_{\theta^\detr\in \Theta^\detr} \theta^\mix(\theta^\detr) \cdot \eta^{\pi^\detr,\theta^\detr}(\mypath_{0:q-1})}{\sum_{\theta^\detr\in \Theta^\detr} \theta^\mix(\theta^\detr) \cdot \eta^{\pi^\detr,\theta^\detr}(\mypath_{0:q-1})}.
    \intertext{Reordering:}
        &= \bm{T}(u_{q-1})(s_{q-1},a_{q-1},s_q) \cdot \sum_{\pi^\detr\in \Pi^\detr} \pi^\mix(\pi^\detr) \cdot \pi^\detr(O^{\agent,M}(\mypath_{0:q-1}))(a_{q-1}) \cdot \sum_{\theta^\detr\in \Theta^\detr} \theta^\mix(\theta^\detr) \cdot \eta^{\pi^\detr,\theta^\detr}(\mypath_{0:q-1}) \cdot \strut\\
        &\qquad \dfrac{\sum_{\theta^\detr\in \Theta^\detr} \theta^\mix(\theta^\detr) \cdot \eta^{\pi^\detr,\theta^\detr}(\mypath_{0:q-1}) \cdot \theta^\detr(O^{\nature,M}(\mypath_{0:q-1}),a_{q-1})(u_{q-1})}{\sum_{\theta^\detr\in \Theta^\detr} \theta^\mix(\theta^\detr) \cdot \eta^{\pi^\detr,\theta^\detr}(\mypath_{0:q-1})}.
    \intertext{Using the definition of $\mu$ for deterministic or stochastic agent policies and mixed nature policies:}
        &= \bm{T}(u_{q-1})(s_{q-1},a_{q-1},s_q) \cdot \sum_{\pi^\detr\in \Pi^\detr} \pi^\mix(\pi^\detr) \cdot \pi^\detr(O^{\agent,M}(\mypath_{0:q-1}))(a_{q-1}) \cdot \mu^{\pi^\detr,\theta^\mix}(\mypath_{0:q-1}) \cdot \strut\\
        &\qquad \dfrac{\sum_{\theta^\detr\in \changeB{\Theta^\detr}} \theta^\mix(\theta^\detr) \cdot \eta^{\pi^\detr,\theta^\detr}(\mypath_{0:q-1}) \cdot \theta^\detr(O^{\nature,M}(\mypath_{0:q-1}),a_{q-1})(u_{q-1})}{\sum_{\theta^\detr\in \Theta^\detr} \theta^\mix(\theta^\detr) \cdot \eta^{\pi^\detr,\theta^\detr}(\mypath_{0:q-1})}.
    \intertext{Using \cref{app:lem_eta_zero}, we get:}
        &= \bm{T}(u_{q-1})(s_{q-1},a_{q-1},s_q) \cdot \sum_{\pi^\detr\in \Pi^\detr} \pi^\mix(\pi^\detr) \cdot \pi^\detr(O^{\agent,M}(\mypath_{0:q-1}))(a_{q-1}) \cdot \mu^{\pi^\detr,\theta^\mix}(\mypath_{0:q-1}) \cdot \strut\\
        &\qquad \dfrac{\sum_{\theta^\detr\in \Theta^{\detr,O^{\nature,M}(\mypath_{0:q-1})}} \theta^\mix(\theta^\detr) \cdot \eta^{\pi^\detr,\changeR{\theta^\detr}}(\mypath_{0:q-1}) \cdot \theta^\detr(O^{\nature,M}(\mypath_{0:q-1}),a_{q-1})(u_{q-1})}{\sum_{\theta^\detr\in \changeB{\Theta^{\detr,O^{\nature,M}(\mypath_{0:q-1})}}} \theta^\mix(\theta^\detr) \cdot \eta^{\pi^\detr,\changeR{\theta^\detr}}(\mypath_{0:q-1})}.
    \intertext{Let $\theta^{\detr'}$ be an arbitrary deterministic policy in the set of relevant deterministic policies $\Theta^{\detr,O^{\nature,M}(\mypath_{0:q-1})}$. Using \cref{app:lem_eta_const}, we get:}
        &= \bm{T}(u_{q-1})(s_{q-1},a_{q-1},s_q) \cdot \sum_{\pi^\detr\in \Pi^\detr} \pi^\mix(\pi^\detr) \cdot \pi^\detr(O^{\agent,M}(\mypath_{0:q-1}))(a_{q-1}) \cdot \mu^{\pi^\detr,\theta^\mix}(\mypath_{0:q-1}) \cdot \strut\\
        &\qquad \dfrac{\sum_{\theta^\detr\in \Theta^{\detr,O^{\nature,M}(\mypath_{0:q-1})}} \theta^\mix(\theta^\detr) \cdot \eta^{\pi^\detr,\changeR{\theta^{\detr'}}}(\mypath_{0:q-1}) \cdot \theta^\detr(O^{\nature,M}(\mypath_{0:q-1}),a_{q-1})(u_{q-1})}{\sum_{\theta^\detr\in \Theta^{\detr,O^{\nature,M}(\mypath_{0:q-1})}} \theta^\mix(\theta^\detr) \cdot \eta^{\pi^\detr,\changeR{\theta^{\detr'}}}(\mypath_{0:q-1})}.
    \intertext{Reordering:}
        &= \bm{T}(u_{q-1})(s_{q-1},a_{q-1},s_q) \cdot \sum_{\pi^\detr\in \Pi^\detr} \pi^\mix(\pi^\detr) \cdot \pi^\detr(O^{\agent,M}(\mypath_{0:q-1}))(a_{q-1}) \cdot \mu^{\pi^\detr,\theta^\mix}(\mypath_{0:q-1}) \cdot \strut\\
        &\qquad \dfrac{\eta^{\pi^\detr,\theta^{\detr'}}(\mypath_{0:q-1}) \cdot \sum_{\theta^\detr\in \Theta^{\detr,O^{\nature,M}(\mypath_{0:q-1})}} \theta^\mix(\theta^\detr) \cdot \theta^\detr(O^{\nature,M}(\mypath_{0:q-1}),a_{q-1})(u_{q-1})}{\eta^{\pi^\detr,\theta^{\detr'}}(\mypath_{0:q-1}) \cdot \sum_{\theta^\detr\in \Theta^{\detr,O^{\nature,M}(\mypath_{0:q-1})}} \theta^\mix(\theta^\detr)}.
    \intertext{Now we can simplify the fraction:}
        &= \bm{T}(u_{q-1})(s_{q-1},a_{q-1},s_q) \cdot \sum_{\pi^\detr\in \Pi^\detr} \pi^\mix(\pi^\detr) \cdot \pi^\detr(O^{\agent,M}(\mypath_{0:q-1}))(a_{q-1}) \cdot \mu^{\pi^\detr,\changeB{\theta^\mix}}(\mypath_{0:q-1}) \cdot \strut\\
        &\qquad \dfrac{\sum_{\theta^\detr\in \Theta^{\detr,O^{\nature,M}(\mypath_{0:q-1})}} \theta^\mix(\theta^\detr) \cdot \theta^\detr(O^{\nature,M}(\mypath_{0:q-1}),a_{q-1})(u_{q-1})}{\sum_{\theta^\detr\in \Theta^{\detr,O^{\nature,M}(\mypath_{0:q-1})}} \theta^\mix(\theta^\detr)}.
    \intertext{Using our assumption $\forall \pi\in \Pi\cup \Pi^\mix\mu^{\pi,\theta^\mix}(\mypath_{0:q-1}) = \mu^{\pi,f(\theta^\mix)}(\mypath_{0:q-1})$, we get:}
        &= \bm{T}(u_{q-1})(s_{q-1},a_{q-1},s_q) \cdot \sum_{\pi^\detr\in \Pi^\detr} \pi^\mix(\pi^\detr) \cdot \pi^\detr(O^{\agent,M}(\mypath_{0:q-1}))(a_{q-1}) \cdot \changeR{\mu}^{\pi^\detr,\changeB{f(\theta^\mix)}}(\mypath_{0:q-1}) \cdot \strut\\
        &\qquad \dfrac{\sum_{\theta^\detr\in \Theta^{\detr,O^{\nature,M}(\mypath_{0:q-1})}} \theta^\mix(\theta^\detr) \cdot \theta^\detr(O^{\nature,M}(\mypath_{0:q-1}),a_{q-1})(u_{q-1})}{\sum_{\theta^\detr\in \Theta^{\detr,O^{\nature,M}(\mypath_{0:q-1})}} \theta^\mix(\theta^\detr)}.
    \intertext{Using the definition of $\mu$ for deterministic or stochastic agent and nature policies:}
        &= \bm{T}(u_{q-1})(s_{q-1},a_{q-1},s_q) \cdot \sum_{\pi^\detr\in \Pi^\detr} \pi^\mix(\pi^\detr) \cdot \pi^\detr(O^{\agent,M}(\mypath_{0:q-1}))(a_{q-1}) \cdot \changeR{\eta}^{\pi^\detr,f(\theta^\mix)}(\mypath_{0:q-1})\cdot \strut\\
        & \qquad \dfrac{\sum_{\theta^\detr\in \Theta^{\detr,O^{\nature,M}(\mypath_{0:q-1})}} \theta^\mix(\theta^\detr) \cdot \theta^\detr(O^{\nature,M}(\mypath_{0:q-1}),a_{q-1})(u_{q-1})}{\sum_{\theta^\detr\in \Theta^{\detr,O^{\nature,M}(\mypath_{0:q-1})}} \theta^\mix(\theta^\detr)}.
    \intertext{Using the definition of $f$ and our assumption the $O^{\nature,M}(\mypath_{0:q-1}) \in \rel^\nature(\theta^\mix)$, we get:}
        &= \bm{T}(u_{q-1})(s_{q-1},a_{q-1},s_q) \cdot \sum_{\pi^\detr\in \Pi^\detr} \pi^\mix(\pi^\detr) \cdot \pi^\detr(O^{\agent,M}(\mypath_{0:q-1}))(a_{q-1}) \cdot \eta^{\pi^\detr,f(\theta^\mix)}(\mypath_{0:q-1})\cdot \strut\\
        & \qquad f(\theta^\mix)(O^{\nature,M}(\mypath_{0:q-1}),a_{q-1})(u_{q-1}).
    \intertext{Reordering:}
        &= \sum_{\pi^\detr\in \Pi^\detr} \pi^\mix(\pi^\detr) \cdot \eta^{\pi^\detr,f(\theta^\mix)}(\mypath_{0:q-1}) \cdot \pi^\detr(O^{\agent,M}(\mypath_{0:q-1}))(a_{q-1}) \cdot f(\theta^\mix)(O^{\nature,M}(\mypath_{0:q-1}),a_{q-1})(u_{q-1}) \cdot \strut\\
        &\qquad \bm{T}(u_{q-1})(s_{q-1},a_{q-1},s_q).
    \intertext{Folding $\eta^{\pi^\detr,f(\theta^\mix)}(\mypath)$:}
        &= \sum_{\pi^\detr\in \Pi^\detr} \pi^\mix(\pi^\detr) \cdot \eta^{\pi^\detr,f(\theta^\mix)}(\mypath).
    \intertext{Using the definition of $\mu$ for mixed agent policies and deterministic or stochastic nature policies:}
        &= \mu^{\pi^\mix,f(\theta^\mix)}(\mypath).
    \end{align*}
    }%
    So if $\forall \pi \in \Pi \cup \Pi^\mix.\, \mu^{\pi,\theta^\mix}(\mypath) = \mu^{\pi,f(\theta^\mix)}(\mypath)$ holds for paths of arbitrary length $q-1 \in \NN$, $\forall \pi \in \Pi \cup \Pi^\mix.\, \mu^{\pi,\theta^\mix}(\mypath) = \mu^{\pi,f(\theta^\mix)}(\mypath)$ holds for paths of length $q$.
    Hence, by induction, $\forall \mypath\in \Paths^M,\forall \pi \in \Pi \cup \Pi^\mix.\, \mu^{\pi,\theta^\mix}(\mypath) = \mu^{\pi,f(\theta^\mix)}(\mypath)$.

    As $\theta^\mix \in \Theta^\mix$ was arbitrarily chosen, we conclude that:
    \[
    \forall \theta^\mix \in \Theta^\mix, \forall \pi\in \Pi \cup \Pi^\mix.\, \mu^{\pi,\theta^\mix} = \mu^{\pi,f(\theta^\mix)}.
    \]
\end{proof}

We can now prove the nature case of \Cref{app:thm_exists_stochastic_policy}:
\existsstochasticPolicy*
\begin{proof}
    Take arbitrary $\theta^\mix \in \Theta^\mix$.
    By \Cref{app:lem_fmix_in_stochastic}, we know that:
    \[f(\theta^\mix) \in \Theta.\]
    Furthermore, by \Cref{app:lem_mix_equiv_fmix}, we know that:
    \[\forall \pi\in \Pi \cup \Pi^\mix.\, \mu^{\pi,\theta^\mix} = \mu^{\pi,f(\theta^\mix)}.\]
    Hence, we know that:
    \[\exists \theta \in \Theta, \forall \pi\in \Pi \cup \Pi^\mix.\, \mu^{\pi,\theta^\mix} = \mu^{\pi,\theta}.\]
    As $\theta^\mix \in \Theta^\mix$ was arbitrarily chosen, we conclude that:
    \[\forall \theta^\mix \in \Theta^\mix, \exists \theta \in \Theta, \forall \pi\in \Pi \cup \Pi^\mix.\, \mu^{\pi,\theta^\mix} = \mu^{\pi,\theta}.\]
\end{proof}

\subsection{Convex Semi-Infinite Game}\label{app:convex_semi-infinite_game}

Although we follow the occupancy game construction from \cite{Springer:HSVI}, their proof for Nash equilibrium existence requires both players to have convex subsets of a Euclidean space as their policy space. 
Our set of nature policies does not meet this requirement, as the number of deterministic nature policies is infinite, and therefore, the mixed policy space is infinite-dimensional.
Instead, we show that our occupancy game is a convex semi-infinite game as defined in~\cite{Convex_semi-infinite_games} and given below:

\begin{definition}[Convex semi-infinite game {[Lopez and Vercher, 1986]}]\label{def:thm:convex_semi-infinite_game}
    Given an arbitrary, possibly infinite set $T$, a convex semi-infinite game is a zero-sum game with policy sets $\Gamma$ and $C$ of the following types:
    \begin{itemize}
        \item $\Gamma = \{\vec{\lambda} = (\lambda_t)_{t\in T} \mid \text{only finitely many } \lambda_t \neq 0, \lambda_t \geq 0, \text{ and } \sum_{t\in T} \lambda_t = 1\}$.
        \item $C = $ a nonempty closed convex set in $\RR^n$ with $n \in \NN$ a finite number.
    \end{itemize}
    And a family of convex functions with finite values $\mathcal{F}_T = \{F_t \colon \RR^n \to \RR \mid t\in T\}$, such that the value function of the convex semi-infinite game $W\colon C \times \Gamma \to \RR$ is:
    \[W(\vec{x}, \vec{\lambda}) = \sum_{t\in T} \lambda_tF_t(\vec{x})\text{, with }\vec{\lambda} \in \Gamma \text{ and } \vec{x}\in C.\]
\end{definition}

\begin{theorem}[Occupancy game is convex semi-infinite]\label{app:thm:convex_semi-infinite_game}
    Given an RPOMDP $M$ and horizon $K \in \NN$, the corresponding occupancy game OG (\Cref{app:def:occupancy_game}) is an convex semi-infinite game, where:
    \begin{itemize}
        \item $T$ is the set of deterministic nature policies $\Theta^\detr_{0:K-1}$.
        \item $\Gamma$ is the set of mixed nature policies $\Theta^\mix_{0:K-1}$.
        \item $C$ is the set of mixed agent policies $\Pi^\mix_{0:K-1}$.
        \item $F_{\theta^\detr}(\vec{x}) = \sum_{\pi^\detr\in \Pi^\detr_{0:K-1}} x_{\pi^\detr}\cdot V^{\pi^\detr,\theta^\detr}$ with $\theta^\detr \in \Theta^\detr_{0:K-1} = T$ and $\vec{x} \in \RR^n$ where $n$ is the finite number of deterministic agent policies $|\Pi^\detr_{0:K-1}|$.
    \end{itemize}
\end{theorem}
Note that we omitted the history length indication for the policies for readability purposes.
The history lengths on which the policies are defined can be derived from the policy sets from which they are taken.
See \Cref{tab:notation} in \Cref{app:notation_and_prelim} for the notation glossary.

We prove \Cref{app:thm:convex_semi-infinite_game} by proving several smaller lemmas, showing that the suggested mapping of the occupancy game to the convex semi-infinite game definition is correct.
\Cref{app:lem_closed_convex} shows that the set of mixed agent policies $\Pi^\mix_{0:K-1}$ meets the conditions for policy set $C$ of the convex semi-infinite game definition.

\begin{lemma}\label{app:lem_closed_convex}
    $\Pi^\mix_{0:K-1}$ is a nonempty closed convex set in $\RR^n$ with $n = |\Pi^\detr_{0:K-1}|$ a finite number.
\end{lemma}
\begin{proof}
    By definition, $\Pi^\mix_{0:K-1} = \dist{\Pi^\detr_{0:K-1}}$.
    Since the set of actions $A$, the set of agent observation $\Zagent$, and the set of public observation $\Zpub$ are all finite and nonempty, the number of deterministic policies $|\Pi^\detr_{0:K-1}|$ is finite and nonempty and less than or equal to 
    \[
    \sum_{t = 0}^{K}(|A|\cdot|Z^\agent_\priv|\cdot|Z_\publ|)^t\cdot|A|.
    \]
    Let $n$ be the actual number of different deterministic policies in our RPOMDP.
    Then we know the set of mixed policies $\Pi^\mix_{0:K-1} \subseteq \RR^n$ with $n$ a finite number.
    Finally, since the set of mixed policies is the probability simplex $\dist{\Pi^\detr_{0:K-1}} = \dist{\RR^n}$, we know that it is a closed and convex set.
\end{proof}

As no restrictions are given on $T$, we can take the infinite set of deterministic nature policies $\Theta^\detr_{0:K-1}$.
\Cref{app:lem_finite_prob_dist} shows that the mixed nature policies then meet the conditions for policy set $\Gamma$ of the convex semi-infinite game definition.

\begin{lemma}\label{app:lem_finite_prob_dist}
    $\Theta^\mix_{0:K-1}$ is the set of finite probability distributions over the set of deterministic nature policies $\Theta^\detr_{0:K-1}$:
    \[\Theta^\mix_{0:K-1} = \{\vec{\lambda} = (\lambda_{\theta^\detr})_{\theta^\detr\in \Theta^\detr_{0:K-1}} \mid \text{only finitely many } \lambda_{\theta^\detr} \neq 0, \lambda_{\theta^\detr} \geq 0, \text{ and } \sum_{\theta^\detr\in \Theta^\detr_{0:K-1}} \lambda_{\theta^\detr} = 1\}.\]
\end{lemma}
\begin{proof}
    By definition, the set of mixed nature policies $\Theta^\mix_{0:K-1}$ is the set of probability distributions over the set of deterministic nature policies $\Theta^\detr_{0:K-1}$.
    As stated in \Cref{sec:preliminaries}, we only consider finite probability distributions over infinite sets.
    The set of mixed nature policies is hence also restricted to the finite probability distributions.
    This means $\vec{\lambda} = \theta^\mix \in \Theta^\mix_{0:K-1}$ with $\lambda_{\theta^\detr} = \theta^\mix(\theta^\detr)$.
\end{proof}

Recall the family of functions $F_{\theta^\detr}\colon \RR^n \to \RR$ defined on the set of deterministic nature policies $\Theta^\detr_{0:K-1}$ in \Cref{app:thm:convex_semi-infinite_game}:
\[F_{\theta^\detr}(x_{\pi^\detr}) = \sum_{\pi^\detr\in \Pi^\detr_{0:K-1}} x_{\pi^\detr}\cdot V^{\pi^\detr,\theta^\detr}.\]

The next two lemmas show that this family of functions $F_{\theta^\detr}$ is convex (\Cref{app:lem_fam_convex}) and all functions in the family have a finite value (\Cref{app:lem_fam_finite}).
\begin{lemma}\label{app:lem_fam_convex}
    $\forall \theta^\detr \in \Theta^\detr_{0:K-1}.\; F_{\theta^\detr}$ is a convex function.
\end{lemma}
\begin{proof}
    Take arbitrary $\theta^\detr\in \Theta^\detr_{0:K-1}$, $\vec{x},\vec{y} \in \RR^n$, and $\alpha \in [0,1]$ then:
    \begin{align*}
        F_{\theta^\detr}(\alpha \cdot \vec{x} + (1-\alpha)\cdot \vec{y}) &= \sum_{\pi^\detr\in\Pi^\detr} (\alpha \cdot \vec{x} + (1-\alpha)\cdot \vec{y})_{\pi^\detr} \cdot V^{\pi^\detr, \theta^\detr}\\
        &= \sum_{\pi^\detr\in\Pi^\detr} (\alpha \cdot x_{\pi^\detr} + (1-\alpha)\cdot y_{\pi^\detr}) \cdot V^{\pi^\detr, \theta^\detr}\\
        &= \sum_{\pi^\detr\in\Pi^\detr} \alpha \cdot x_{\pi^\detr} \cdot V^{\pi^\detr, \theta^\detr} + (1-\alpha)\cdot y_{\pi^\detr} \cdot V^{\pi^\detr, \theta^\detr}\\
        &= \sum_{\pi^\detr\in\Pi^\detr} \bigl\{\alpha \cdot x_{\pi^\detr} \cdot V^{\pi^\detr, \theta^\detr}\bigr\} + \sum_{\pi^\detr\in\Pi^\detr} \bigl\{ (1-\alpha)\cdot y_{\pi^\detr} \cdot V^{\pi^\detr, \theta^\detr}\bigr\}\\
        &= \alpha \cdot \sum_{\pi^\detr\in\Pi^\detr} \bigl\{x_{\pi^\detr} \cdot V^{\pi^\detr, \theta^\detr}\bigr\} + (1-\alpha)\cdot \sum_{\pi^\detr\in\Pi^\detr} \bigl\{ y_{\pi^\detr} \cdot V^{\pi^\detr, \theta^\detr}\bigr\}\\
        &= \alpha \cdot F_{\theta^\detr}(\vec{x}) + (1-\alpha)\cdot F_{\theta^\detr}(\vec{y}).
    \end{align*}
    So $\forall_{\theta^\detr\in \Theta^\detr}.\; F_{\theta^\detr}$ is a convex function.
\end{proof}
\begin{lemma}\label{app:lem_fam_finite}
    Given $\theta^\detr \in \Theta^\detr$, $\vec{x} \in \RR^n$:
    \[
        F_{\theta^\detr}(\vec{x}) \text{ has a finite value.}
    \]
\end{lemma}
\begin{proof}
    $\forall \theta^\detr \in \Theta^\detr, \forall \pi^\detr\in\Pi^\detr$ the value of $ V^{\pi^\detr_i, \theta^\detr}$ is finite, as it is bounded by $K \cdot (\max_{s,a \in S\times A} R(s,a))$.
    Furthermore, $\vec{x} \in \RR^n$ can only take finite values.
    Together, this shows that:
    \[
    \forall \theta^\detr \in \Theta^\detr, \forall \pi^\detr\in\Pi^\detr.\, F_{\theta^\detr}(\vec{x}) \text{ has a finite value}.
    \]
\end{proof}

Finally, \Cref{app:lem_conv_semi-inf_equiv} shows that the value function $W \colon \Pi^\mix_{0:K-1} \times \Theta^\mix_{0:K-1} \to \RR$ of the convex semi-infinite game constructed as in \Cref{app:thm:convex_semi-infinite_game} is equivalent to the value function $V \colon \Pi^\mix_{0:K-1} \times \Theta^\mix_{0:K-1} \to \RR$ of the occupancy game.

\begin{lemma}\label{app:lem_conv_semi-inf_equiv}
    The value function of the convex semi-infinite game is equivalent to the value function of the occupancy game.
    \[
        \forall \pi^\mix \in \Pi^\mix_{0:K-1}, \forall \theta^\mix \in \Theta^\mix_{0:K-1}.\, W(\pi^\mix,\theta^\mix) = V^{\pi^\mix,\theta^\mix}.
    \]
\end{lemma}
\begin{proof}
    Recall the definition of the value function of a convex semi-infinite game:
    \[
        W(\pi^\mix,\theta^\mix) = \sum_{\theta^\detr \in \Theta^\detr} \theta^\mix(\theta^\detr) \cdot F_{\theta^\detr}(\pi^\mix).
    \]
    By construction, the value function of our occupancy game is the same as that of our original RPOMDP. 
    However, as shown in \Cref{app:mixed_policies}, we can reason with mixed policies, giving us the following value function:
    \[
        V^{\pi^\mix,\theta^\mix} = \sum_{\pi^\detr\in \Pi^\detr}\pi^\mix(\pi^\detr) \cdot \sum_{\theta^\detr\in \Theta^\detr} \theta^\mix(\theta^\detr) \cdot V^{\pi^\detr,\theta^\detr}.
    \]
    Take arbitrary mixed agent and nature policies $\pi^\mix \in \Pi^\mix_{0:K-1}, \theta^\mix \in \Theta^\mix_{0:K-1}$.
    Then:
    \begin{align*}
        W(\pi^\mix,\theta^\mix) &= \sum_{\theta^\detr \in \Theta^\detr} \theta^\mix(\theta^\detr) \cdot F_{\theta^\detr}(\pi^\mix)\\
        &= \sum_{\theta^\detr \in \Theta^\detr} \theta^\mix(\theta^\detr) \cdot \sum_{\theta^\detr\in\Theta^\detr} \pi^\mix(\pi^\detr) \cdot V^{\pi^\detr, \theta^\detr}\\
        &= \sum_{\pi^\detr\in \Pi^\detr}\pi^\mix(\pi^\detr) \cdot \sum_{\theta^\detr\in \Theta^\detr} \theta^\mix(\theta^\detr) \cdot V^{\pi^\detr,\theta^\detr}\\
        &= V^{\pi^\mix,\theta^\mix}.
    \end{align*}    
\end{proof}

\Cref{app:lem_conv_semi-inf_equiv} is the final step in proving \Cref{app:thm:convex_semi-infinite_game}.
We conclude that our occupancy game is a convex semi-infinite game.

The final step in proving the existence of a finite horizon Nash equilibrium in our RPOMDPs follows from \cite[Theorem 3.2]{Convex_semi-infinite_games}, stating that in a convex semi-infinite game, if the convex functions and the convex agent policy set have no common direction of recession, then a Nash equilibrium and an optimal strategy for the agent exist.

\begin{lemma}\label{app:lem_condintion_R}
    In the convex semi-infinite game of our occupancy game, the convex functions $F_{\theta^\detr}$ with $\theta^\detr\in\Theta^\detr$ and the set of mixed agent policies $\Pi^\mix$ have no common direction of recession.
\end{lemma}
\begin{proof}
    As our set of agent policies is a convex polytope in $\RR^n$, it is a closed and bounded convex subset of $\RR^n$.
    Then, the recession cone consists only of the zero vector~\cite{Book_convex_analysis_general_vector_spaces}.
    The zero vector is also trivially contained in the recession cones of the convex functions.
    Therefore, we know that the convex functions and the convex agent policy set have no common direction of recession.    
\end{proof}

As shown in \Cref{app:lem_condintion_R}, our occupancy game meets the condition given in \cite{Convex_semi-infinite_games}.
It hence follows that a Nash equilibrium and an optimal strategy for the agent exist and that the saddle point condition holds, proving  Theorem \hyperref[app:thm_nash_exists]{3}.

\subsection{Nature First}
When reasoning with the nature first semantics, the nature policy no longer relies on the last action of the agent.
This influences the proof of the sufficient statistic as follows:

{\allowdisplaybreaks
\begin{align*}
    &\Ocs{0:t}(\tup{h_t, a_t, u_t, z_\priv^\agent, z_\priv^\nature, z_\publ}) \stackrel{\text{def}}{=} \tup{\ocs{0:t}(\tup{h_t, a_t, u_t, z_\priv^\agent, z_\priv^\nature, z_\publ}), \theta_{0:t}},\\
    \intertext{where:}
    &\theta_{0:t} \stackrel{\text{def}}{=} \tup{\theta_{0:t-1}, \theta_t}.\\
    &\ocs{0:t}(\tup{h_t, a_t, u_t, z_\priv^\agent, z_\priv^\nature, z_\publ}) \stackrel{\text{def}}{=} \Pr(h_t, a_t, u_t, z_\priv^\agent, z_\priv^\nature, z_\publ \given \pi_{0:t}, \theta_{0:t})\\
    &= \sum_{s\in \mathcal{S}^\nature} \sum_{s'\in \mathcal{S}^\nature} \Pr(z_\priv^\agent, z_\priv^\nature, z_\publ \given s')\Pr(s' \given a_t, u_t, s)\Pr(h_t, a_t, u_t, s \given \pi_{0:t}, \theta_{0:t})\\
    &= \sum_{s\in \mathcal{S}^\nature} \sum_{s'\in \mathcal{S}^\nature} \Pr(z_\priv^\agent, z_\priv^\nature, z_\publ \given s')\Pr(s' \given a_t, u_t, s)\Pr(u_t \given h_t, a_t, s, \pi_{0:t}, \theta_{0:t})\Pr(h_t, a_t, s \given \pi_{0:t}, \theta_{0:t}).
    \intertext{The chance of a nature action only depends on nature's policy at time $t$ and the history:}
    &= \sum_{s\in \mathcal{S}^\nature} \sum_{s'\in \mathcal{S}^\nature} \Pr(z_\priv^\agent, z_\priv^\nature, z_\publ \given s')\Pr(s' \given a_t, u_t, s)\Pr(u_t \given h_t, \theta_{t})\Pr(h_t, a_t, s \given \pi_{0:t}, \theta_{0:t})\\
    &= \sum_{s\in \mathcal{S}^\nature} \sum_{s'\in \mathcal{S}^\nature} \Pr(z_\priv^\agent, z_\priv^\nature, z_\publ \given s')\Pr(s' \given a_t, u_t, s)\Pr(u_t \given h_t, a_t, \theta_{t})\Pr(a_t \given h_t, \pi_{t})\Pr(s \given h_t)\Pr(h_t \given \pi_{0:t-1}, \theta_{0:t-1})\\
    &= \sum_{s\in \mathcal{S}^\nature} \sum_{s'\in \mathcal{S}^\nature} \mathcal{O}^\agent(s',z_\priv^\agent, z_\publ,)\mathcal{O}^\nature(s',z_\priv^\nature, z_\publ)\mathcal{T}^\agent(\mathcal{T}^\nature(s,u_t), a_t, s')\theta_t(h^\nature_t, u_t)\pi_t(h^\agent_t, a_t)b(s, h_t)\ocs{0:t-1}(h_t).
\end{align*}
}
We can hence still compute the successor occupancy state using only the previous occupancy state $\Ocs{0:t-1} = \tup{\ocs{0:t-1}, \theta_{0:t-1}}$ and policies $\pi_t, \theta_t$ at time $t$.
The expected reward proof requires no modifications.\\

We define the nature first OG as follows:
\begin{definition}[OG]\label{def:underlying:agent:zsog}
Given a POSG as defined in \Cref{def:equivalent:nature:zsposg} $\tup{\mathcal{S^\agent, S^\nature, A^\agent, A^\nature, T, R, Z^\agent, Z^\nature, O^\agent, O^\nature}}$, and a horizon $K\in \NN$, we define the OG as a tuple $(\mathsf{S}^\agent, \mathsf{S}^\nature, \mathsf{A}^\agent, \mathsf{A}^\nature, \mathsf{T}, \mathsf{R})$ where the sets of states and actions are defined as follows:
$\mathsf{S}^\agent = \bigcup_{t=0}^{K-1} (\bigcup_{\pi_{0:t} \in \Pi_{0:t}} \bigcup_{\theta_{0:t}\in\Theta_{0:t}} \Ocs{0:t} \times \Theta_{t+1})$ is the infinite set of agent states, and $\mathsf{S}^\nature = \bigcup_{t=0}^{K-1} \bigcup_{\pi_{0:t} \in \Pi_{0:t}} \bigcup_{\theta_{0:t}\in\Theta_{0:t}} \Ocs{0:t}$ the infinite set of nature states;
$\mathsf{A}^\agent = \bigcup_{t=0}^{K-1} \Pi_{t}$ is the infinite set of agent actions, and $\mathsf{A}^\nature = \bigcup_{t=0}^{K-1} \Theta_{t}$ the infinite set of nature actions;
The transition and reward functions are then defined as:
\begin{itemize}
    \item $\mathsf{T} = \mathsf{T}^\agent \cup \mathsf{T}^\nature$, the transition function, where:
    \begin{itemize}
        \item $\mathsf{T}^\agent \colon \mathsf{S}^\agent \times \mathsf{A}^\agent \pto \mathsf{S}^\nature$ the agent's transition function.
        \item $\mathsf{T}^\nature \colon \mathsf{S}^\nature \times \mathsf{A}^\nature \pto \mathsf{S}^\agent$ nature's transition function.
    \end{itemize}
    \item $\mathsf{R}\colon \mathsf{S}^\agent \times \mathsf{A}^\agent \to \RR$ the reward function.
\end{itemize}
Where:
\begin{itemize}
    \item $\mathsf{R}(\tup{\ocs{0:t}, \theta_{0:t}},\pi_{t+1}) = \sum_{s\in \mathcal{S}^\agent}\sum_{a\in \mathcal{A}^\agent}\Bigl\{\mathcal{R}(s,a) \cdot \sum_{h_{t+1}\in H_{t+1}(\theta_{0:t})}\bigl\{\pi_{t+1}(h_{t+1}, a)b(s,h_{t+1})\ocs{0:t}(h_{t+1})\bigr\}\Bigr\}$.
    \item $\mathsf{T^\nature}(\tup{\ocs{0:t}, \theta_{0:t}},\theta_{t+1}) = \tup{\tup{\ocs{0:t}, \theta_{0:t}},\theta_{t+1}}$.
    \item $\mathsf{T^\agent}(\tup{\tup{\ocs{0:t}, \theta_{0:t}},\theta_{t+1}},\pi_{t+1}) = \tup{\ocs{0:t+1}, \theta_{0:t+1}}$, where:\\
    \begin{itemize}
        \item $\theta_{0:t+1} = \theta_{0:t} \concat \theta_{t+1}$.
        \item $\forall h_{t+1}\in H_{t+1}(\theta_{0:t}), \forall a \in \mathcal{A}^\agent, \forall u \in \mathcal{A}^\nature, \forall z^\agent_\priv, z^\nature_\priv, z_\publ \in \Zagent \times \Znature \times \Zpub, \ocs{0:t+1}(\tup{h_{t+1}, a, u, z^\agent_\priv, z^\nature_\priv, z_\publ}) =$\\
    \[ \sum_{s\in \mathcal{S}^\nature} \sum_{s'\in \mathcal{S}^\nature} \mathcal{O}^\agent(s',z_\priv^\agent, z_\publ,)\mathcal{O}^\nature(s',z_\priv^\nature, z_\publ)\mathcal{T}^\agent(\mathcal{T}^\nature(s,u_t), a_t, s')\theta_t(h^\nature_t, u_t)\pi_t(h^\agent_t, a_t)b(s, h_t)\ocs{0:t-1}(h_t).\]
    \end{itemize}
\end{itemize}
\end{definition}
Where $b(s, h_t)$ is the belief computed by $t$ belief updates given the joint \aoh{} $h_t$.
Where $H_t(\theta_{0:t-1}) \subset H_t$ is the subset with $u_i \in \bm{U}$ determined by $\theta_i$ given the history $h^\nature_{0:i-1}$ and action $a$. This is a finite subset of the infinite set of possible joint histories.

To show that for every stochastic policy there exists a mixed policy that behaves equivalently and vice versa in the nature first setting, the proofs follow the same steps as in \Cref{app:mixed_policies} for the agent first policies.
The only required changes are to remove the agent action input for the nature policies and the corresponding $\forall a \in A$.

The nature first OG still meets all requirements for having an optimal value, which can be shown by following the same proof steps as for the agent first OG in \Cref{app:convex_semi-infinite_game}.

}

\end{document}